\documentclass[twoside,11pt]{article}
 \usepackage{jmlr2e}

 \usepackage{hyperref}
 \usepackage{url}
 \usepackage{amsmath,amssymb,subcaption}
 \usepackage{graphicx}
 \usepackage{epstopdf}
 \usepackage{mathtools}
 \usepackage{paralist}
 \usepackage{cleveref}
 \usepackage{enumitem}
 \usepackage{tikz}
 \usepackage[format=hang]{caption}
 \usepackage[toc,page]{appendix}
 
 \newtheorem{assumption}{Assumption}

\usepackage{amsmath,amsfonts,bm}

\def\eqref#1{equation~\ref{#1}}

\def\1{\bm{1}}

\def\rva{{\mathbf{a}}}
\def\rvb{{\mathbf{b}}}

\def\rvg{{\mathbf{g}}}

\def\rvu{{\mathbf{i}}}

\def\rvp{{\mathbf{p}}}
\def\rvq{{\mathbf{q}}}
\def\rvr{{\mathbf{r}}}
\def\rvs{{\mathbf{s}}}

\def\rvu{{\mathbf{u}}}
\def\rvv{{\mathbf{v}}}
\def\rvw{{\mathbf{w}}}
\def\rvx{{\mathbf{x}}}
\def\rvy{{\mathbf{y}}}
\def\rvz{{\mathbf{z}}}

\def\rmI{{\mathbf{I}}}

\def\rmW{{\mathbf{W}}}
\def\rmX{{\mathbf{X}}}

\DeclareMathAlphabet{\mathsfit}{\encodingdefault}{\sfdefault}{m}{sl}
\SetMathAlphabet{\mathsfit}{bold}{\encodingdefault}{\sfdefault}{bx}{n}

\def\gH{{\mathcal{H}}}

\def\sN{{\mathbb{N}}}

\def\sR{{\mathbb{R}}}
\def\sS{{\mathbb{S}}}

\usepackage{lastpage}
\jmlrheading{26}{2025}{1-\pageref{LastPage}}{5/25; Revised
	10/25}{11/25}{25-1089}{Akshay Kumar and Jarvis Haupt}

\ShortHeadings{Gradient Flow Dynamics Beyond the Origin}{Kumar and Haupt}
\firstpageno{1}

\begin{document}

\title{Towards Understanding Gradient Flow Dynamics of Homogeneous Neural Networks Beyond the Origin}

\author{\name Akshay Kumar \email kumar511@umn.edu \\
       \addr Department of Electrical and Computer Engineering\\
       University of Minnesota\\
       Minneapolis, MN 55455, USA
       \AND
       \name Jarvis Haupt \email jdhaupt@umn.edu \\
       \addr Department of Electrical and Computer Engineering\\
       University of Minnesota\\
       Minneapolis, MN 55455, USA}

\editor{Mahdi Soltanolkotabi}

\maketitle

\begin{abstract}
	Recent works exploring the training dynamics of homogeneous neural network weights under gradient flow with small initialization have established that in the early stages of training, the weights remain small and near the origin, but converge in direction. Building on this, the current paper studies the gradient flow dynamics of homogeneous neural networks with locally Lipschitz gradients, \emph{after} they escape the origin. Insights gained from this analysis are used to characterize the first saddle point encountered by gradient flow after escaping the origin. Also, it is shown that for homogeneous \emph{feed-forward} neural networks, under certain conditions, the sparsity structure emerging among the weights before the escape is preserved after escaping the origin and until reaching the next saddle point.
\end{abstract}


\section{Introduction}
Modern deep neural networks have a surprisingly good generalization behavior when trained via gradient-based methods, even when having sufficient capacity to overfit the training set. A widespread belief is that the training algorithm induces \emph{implicit regularization}, which leads to solutions with favorable generalization performance \citep{soudry_ib}. This has motivated several works to study the training dynamics of neural networks in variety of settings including, for example, the \emph{Neural Tangent Kernel} (NTK)/large initialization regime \citep{ntk, chizat_lazy, arora_exact}, late phase of training homogeneous networks with classification loss \citep{Lyu_ib, ji_matus_align}, and linear and non-linear neural networks \citep{cohen_mtx_fct,ji_gradient,timor_lr,chizat_inf,jacot_rank,abbe_msp}.

This paper studies the training dynamics of neural networks in the small initialization regime\textemdash an important yet not fully understood area. Understanding this regime is particularly important because, for small initialization, neural networks trained via gradient descent operate in the \textit{feature learning} regime, allowing them to learn underlying features present in the data \citep{fl_yang,geiger_feature,mei_mean}. Also, smaller initialization has been observed to lead to better generalization in various tasks \citep{chizat_lazy, geiger_feature}. However, the highly non-linear evolution of the weights in this regime poses significant challenges towards understanding the training dynamics.

Recent studies have focused on the early stages of training homogeneous neural networks via gradient flow with small initialization \citep{maennel_quant,kumar_dc, early_dc,luo_condense,atanasov_align,boursier_early}. For such networks, the origin is a critical point, causing the weights to remain near the origin for some time after training begins, provided the initialization is ``small''. However, interesting structure among the weights begins to emerge even while the weights remain near the origin. In \citet{kumar_dc, early_dc}, it is shown that for sufficiently small initialization, the weights remain small and near the origin for a sufficiently long time, during which they converge in direction towards a Karush-Kuhn-Tucker (KKT) point of a so-called \emph{Neural Correlation Function} (NCF), a phenomenon referred to as early directional convergence. Also, for feed-forward homogeneous neural networks, weights were observed to converge to a KKT point that exhibits sparsity \citep{early_dc,atanasov_align,dc_three_layer}, and these KKT points were characterized in \citet{early_dc}.

Compared to \cite{kumar_dc, early_dc}, which focus on the gradient flow dynamics of homogeneous neural networks near the origin, this work studies the dynamics of gradient flow after escaping the origin. To understand our contributions, consider the experiment shown in \Cref{fig:two_layer_nn}, where we train a two-layer neural network with square activation function (a $3$-homogeneous network) using gradient descent; more details about the experiment are in the figure caption. The initial weights are small and random, as depicted in \Cref{fig:wt_init_2l}. The evolution of the training loss in \Cref{fig:loss_evol_2l} resembles a piece-wise constant function, alternating between periods of stagnation and sharp decreases. Similar behavior has been observed in previous works for other neural networks (see \Cref{sec_related_works} for details). In \Cref{fig:wt_1}, we plot the weights at iteration $i_1$ (marked in \Cref{fig:loss_evol_2l}), just before escaping the origin. The weights remain small, but are larger than the initial weights, and more importantly, they are sparse. This aligns with observations made in \cite{early_dc}, that in the early stages of training, weights tend to converge in direction towards a KKT point of the NCF that exhibits sparsity. Next, we make two key observations about the dynamics after escaping the origin:
\begin{itemize}
	\item After gradient descent escapes the origin, the loss rapidly decreases and soon becomes stagnant again, indicating the trajectory of gradient descent  is near a saddle point.
	\item The sparsity structure of the weights is preserved even after escaping the origin. This is evident when comparing \Cref{fig:wt_1} with  \Cref{fig:wt_2}, where the latter depicts the weights at iteration $i_2$, immediately after reaching the next saddle point.
\end{itemize}
\begin{figure}[!b]
	\centering
	\begin{subfigure}{0.35\textwidth}
		\centering
		\includegraphics[width=\textwidth]{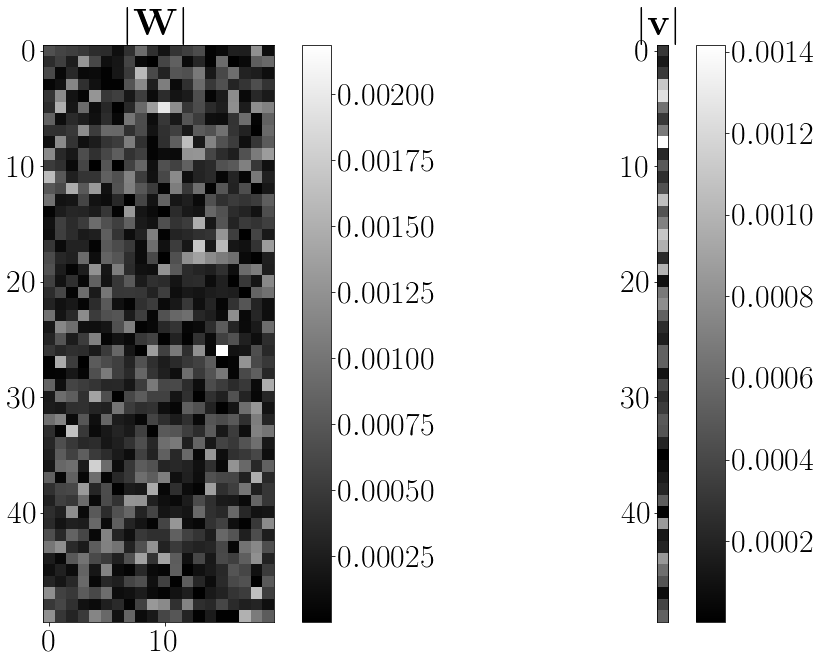}
		\caption{Weights at initialization}
		\label{fig:wt_init_2l}
	\end{subfigure}
	\hfill
	\begin{subfigure}{0.42\textwidth}
		\centering
		\includegraphics[width=\textwidth]{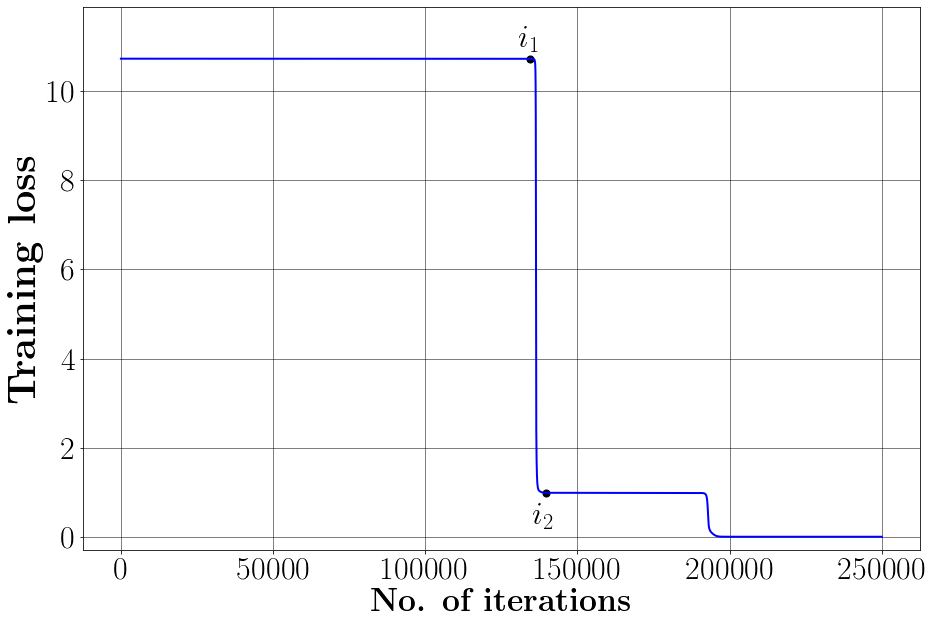}
		\caption{Evolution of training loss}
		\label{fig:loss_evol_2l}
	\end{subfigure}
	
	
	\begin{subfigure}{0.35\textwidth}
		\centering
		\includegraphics[width=\textwidth]{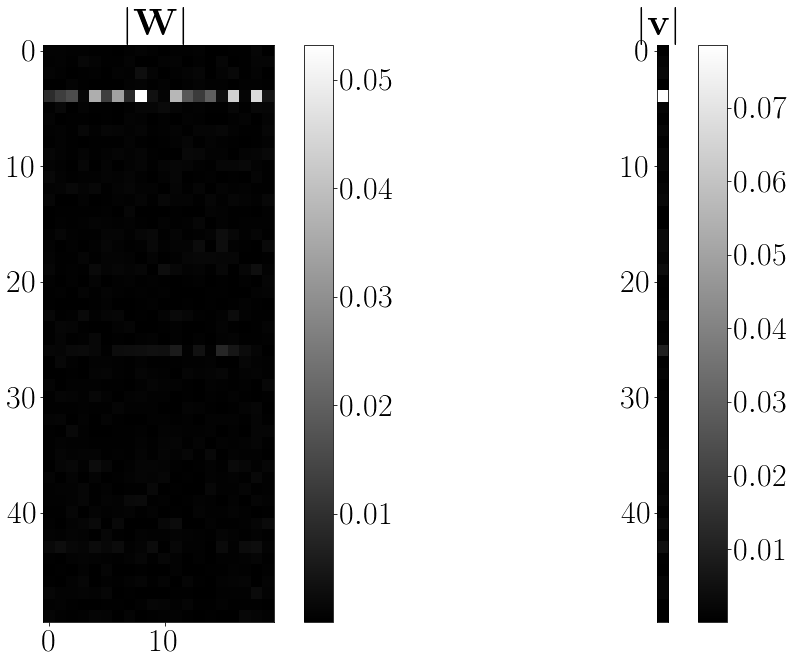}
		\caption{Weights at iteration $i_1$}
		\label{fig:wt_1}
	\end{subfigure}
	\hfill
	\begin{subfigure}{0.35\textwidth}
		\centering
		\includegraphics[width=\textwidth]{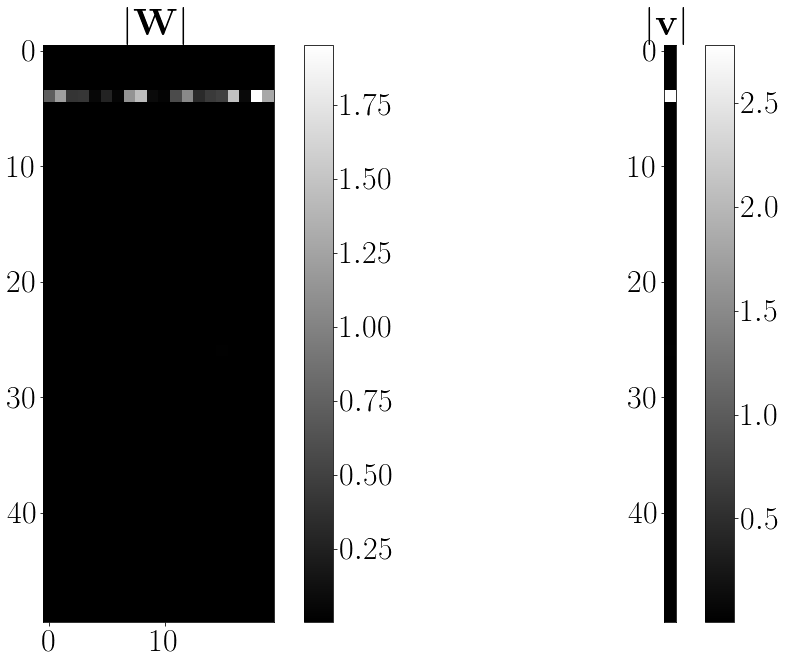}
		\caption{Weights at iteration $i_2$}
		\label{fig:wt_2}
	\end{subfigure}
	\caption{We train a two-layer neural network with output $\rvv^\top\sigma(\rmW\rvx) ,$ where $\sigma(x) = x^2$, and trainable weights $\rvv\in \mathbb{R}^{50},\rmW \in \mathbb{R}^{50 \times 20}$. The training set has $100$ points sampled uniformly from the unit sphere in $\sR^{20}$. We minimize the square loss with respect to the output of a smaller two-layer neural network with two neurons and square activation. We train using gradient descent with small initial weights, as depicted in panel (a). Panel (b) shows the evolution of loss with iterations. Panels (c) and (d) depict the absolute value of weights at iteration $i_1$ and $i_2$ (marked in panel (b)), approximately just before escaping the origin and immediately after reaching the next saddle point, respectively (the gap between them is 5000 iterations).  Panels (c) and (d) show that the sparsity structure emerging among the weights before escaping the origin is preserved until reaching the next saddle point.}
	\label{fig:two_layer_nn}
\end{figure}

 \textbf{Our contributions.} This paper attempts to explain the above observations for homogeneous neural networks. Our main contributions can be summarized as follows:  
 \begin{itemize}
 	\item  \textbf{Gradient Flow Dynamics Post-Escape:} In \Cref{main_thm_2hm} and \Cref{main_thm_Lhm}, we describe the gradient flow dynamics of homogeneous neural networks that have locally Lipschitz gradients and order of homogeneity at least two, after escaping the origin. Our results show that for sufficiently small initialization, the gradient flow escapes along the same path as the gradient flow with small initial weights and initial direction along a KKT point of the NCF. Subsequent corollaries provide a characterization of the saddle point encountered by gradient flow following its escape from the origin.
 	\item \textbf{Post-Escape Sparsity in Feed-Forward Networks:} In \Cref{sec_implications}, we show that for homogeneous \emph{feed-forward} neural networks, the sparsity structure that emerges among the weights before escaping the origin is preserved post-escape under certain specified conditions. This is achieved by showing that, for KKT points of the NCF, the weights with zero magnitude form a \emph{zero-preserving subset}\textemdash a subset of weights that remain zero throughout training under gradient flow, provided they are initialized at zero.  This insight is combined with \Cref{main_thm_2hm} and \Cref{main_thm_Lhm} to prove our result.
 \end{itemize}
Note that we assume the neural network to have locally Lipschitz gradients, which excludes ReLU networks. However, our experimental results suggest that some of our findings may extend to ReLU networks as well. Also, our results only capture a segment of the gradient flow dynamics beyond the origin, specifically up to the first saddle point encountered by gradient flow after escaping the origin. More details about this are provided in later sections.
\subsection{Related Works}\label{sec_related_works}
Several works have investigated the training dynamics of neural networks in the small initialization regime. One of the earliest works examined diagonal linear networks, showing that small initialization leads gradient flow to converge towards minimum $\ell_1$-norm solutions \citep{srebro_ib, dln_sparse}. For linear neural networks, it has been observed that gradient descent with small initialization tends to converge toward low-rank solutions \citep{guna_mtx_fct, cohen_mtx_fct}, with rigorous results in some cases \citep{mahdi_ib, lee_saddle,chou_mf,lawrence_linear}. A similar sparsity-inducing effect of small initialization has been observed for non-linear neural networks as well \citep{chizat_lazy}, but theoretical results have mostly been established for two-layer neural networks trained on simple data sets \citep{lyu_simp,gf_orth,wang_saddle}. Recent works on two-layer networks have explored more challenging scenarios \citep{damian_reps, abbe_msp, mousavi_sgd}, but they often make other assumptions about the training algorithm, such as layer-wise training, the use of explicit regularization like weight decay, etc., whereas we study the dynamics of gradient flow with respect to all the weights and do not add any explicit regularization in the training loss. It is also worth noting that most of the aforementioned works examine the \emph{entire} training process; in contrast, our work describes a segment of the gradient flow dynamics beyond the origin, however, our results hold for a wider class of neural networks. 

Another line of work has identified the so-called saddle-to-saddle dynamics in the trajectory of gradient descent when the initialization is small \citep{jacot_sd, lyu_resolving}.  These works have observed that during training, the trajectory of gradient descent passes through a sequence of saddle points. Moreover, the loss curve almost appears like a piece-wise constant function, alternating between being stagnant and decreasing sharply; see \Cref{fig:loss_evol_2l} for an example. Other works refer to this phenomenon as \emph{incremental learning} \citep{gidel_incr, gissin_incr,razin_incr,slutzky_incr}, since the function learned by the neural network gradually increases in complexity as it moves from one saddle to another. So far, this kind of saddle-to-saddle dynamics has been rigorously established for diagonal linear neural networks \citep{pesme_sd, abbe_inc} and two-layer linear and non-linear neural networks trained with various gradient-based methods under data-related assumptions \citep{lee_saddle, gf_orth, wang_saddle, montanari_saddle,abbe_sgd}, and is conjectured to be true for a wider class of neural networks. Our work, which describes the first saddle point encountered by gradient flow after escaping the origin, can be seen as a step towards establishing it.

Lastly, we highlight the work by \cite{lyu_resolving}, which studies the gradient flow dynamics of linear neural networks under small initialization after escaping the origin. While our paper studies a broader class of neural networks, their proof technique has inspired our approach.

\subsection{Organization}
The paper is organized as follows: \Cref{sec_ps} outlines the problem setup and reviews previous works on early directional convergence. \Cref{sec_gf_beyond} present our findings on the gradient flow dynamics after escaping the origin for homogeneous neural networks, with detailed proofs provided in the Appendix. In \Cref{sec_implications}, we discuss the implications of these results on the gradient flow dynamics of feed-forward homogeneous neural networks. Finally, \Cref{sec_conc} provides concluding remarks, which is followed by the Appendix.
\section{Preliminaries and Problem Setup}\label{sec_ps}
The set of natural numbers is denoted by $\sN$, and for any $L\in\sN$, we let $[L] \coloneqq \{1,2,\cdots, L\}$. We use $\|\cdot\|_2$ to denote the $\ell_2$-norm for a vector, while $\|\cdot\|_F$ and $\|\cdot\|_2$ denote the Frobenius and spectral norm for a matrix, respectively. The $d$-dimensional unit-norm sphere is denoted by $\sS^{d-1}$. For a non-zero vector $\mathbf{z}$, $\rvz^\perp \coloneqq \{\rvb: \rvb^\top\rvz = 0,\rvb\in\sS^{d-1} \}$, that is, set of unit-norm vectors orthogonal to $\rvz$. The $i$th entry of a vector $\rvz$ is denoted by $z_i$. A KKT point of an optimization problem is called a non-negative (positive, zero) KKT point if the objective value at the KKT point is non-negative (positive, zero). \\\\
\textbf{First- and second-order KKT points.} We state the first- and second-order KKT conditions for a maximization problem on unit sphere, which can be easily proved. Consider the following optimization problem
\begin{align*}
\max_{\rvw}  f(\rvw), \text{ s.t. } \|\rvw\|_2^2 = 1.
\end{align*}
Suppose  $f(\rvw)$ is twice-continuously differentiable near $\rvw_*$, then 
\begin{itemize}
	\item $\rvw_*$ is a first-order KKT point if  there exists a $\lambda_*$ such that $\nabla f(\rvw_*) - 2\lambda_*\rvw_* = \mathbf{0}. $
	\item For $\Delta> 0$, we define a first-order KKT point $\rvw_*$ to be a $\Delta-$second-order KKT point if $2\lambda_* -\rvb^\top\nabla^2 f(\rvw_*) \rvb \geq \Delta,$ for all $\rvb \in \rvw_*^\perp.$
\end{itemize}
\textbf{Homogeneous neural netwoks.} For a neural network $\mathcal{H}$, $\mathcal{H}(\rvx;\rvw)$ denotes its output, where $\rvx\in \sR^d$ is the input and $\rvw\in \sR^k$ is a vector containing all the weights. A neural network $\mathcal{H}$ is referred to as $L$-\emph{positively homogeneous} if  
\begin{align*}
\mathcal{H}(\rvx;c\rvw) =  c^L\mathcal{H}(\rvx;\rvw), \text{ for all } c\geq 0 \text{ and } \mathbf{w}\in \sR^k.
\end{align*} 

\begin{assumption}\label{homo_assumption}
	We make the following assumptions on the neural networks considered in this paper: $(i)$ For any fixed $\rvx$, $\mathcal{H}(\rvx;\rvw)$ is locally Lipschitz in $\mathbf{w}$ and is  definable in some o-minimal structure. $(ii)$ The neural network $\mathcal{H}$ is $L$-positively homogeneous, for some $L\geq 2$. $(iii)$ The gradient of  $\mathcal{H}(\rvx;\rvw)$ with respect to $\rvw$, $\nabla \mathcal{H}(\rvx;\rvw)$, is locally Lipschitz  in $\mathbf{w}$.
\end{assumption}
The first two conditions are satisfied by deep feed-forward neural networks with homogeneous activation functions like ReLU and polynomial ReLU $(\max(0, x)^p, p\geq 1).$ We note that definability in some o-minimal structure is a mild technical assumption and is satisfied by all modern deep neural networks \citep{ji_matus_align}. Assuming locally Lipschitz  gradient is crucial for proving our results rigorously, and while it rules out deep ReLU neural networks,  it does include deep linear networks and feed-forward neural networks with polynomial ReLU activation functions $(\max(0, x)^p, p\geq 2)$.\\\\
\textbf{Training setup.} Let $\{\rvx_i,y_i\}_{i=1}^n \in \sR^{d}\times \sR$ be the training data, and define $\rmX = \left[\rvx_1,\cdots,\rvx_n\right] \in \sR^{d \times n}$, $\rvy = \left[y_1,\cdots,y_n\right]^\top\in\sR^{ n}$. Let $\gH(\rmX;\rvw) = \left[\gH(\rvx_1;\rvw), \cdots, \gH(\rvx_n;\rvw)\right]\in\sR^{ n}$ be the vector containing outputs of the neural network, and $\mathcal{J}(\rmX;\rvw)$ denotes the Jacobian of $\gH(\rmX;\rvw)$ with respect to $\mathbf{w}$. For training, we minimize the following objective:
\begin{align}
\mathcal{L}(\rvw) = \sum_{i=1}^n \ell(\mathcal{H}(\rvx_i;\rvw), y_i) = \ell(\mathcal{H}(\rmX;\rvw), \rvy) ,
\label{loss_fn}
\end{align} 
where $\ell(\cdot,\cdot)$ is the loss function, and $\ell(\rvp,\rvq) = \sum_{i=1}^n\ell(p_i,q_i),$ for any two vectors $\rvp, \rvq \in \sR^n$. We use $\ell'(\cdot,\cdot)$ to denote the derivative of $\ell(\cdot,\cdot)$ with respect to the first variable, and define $\ell'(\rvp,\rvq) = \left[\ell'(p_1,q_1),\cdots,\ell'(p_n,q_n)\right]^\top\in \sR^n$. Also, $\ell''(\cdot,\cdot)$ denotes the second-order derivative of $\ell(\cdot,\cdot)$ with respect to the first variable.
\begin{assumption}\label{ass_loss}
	We make the following assumptions on the loss function:
	\begin{itemize}
		\item Smoothness: For some $K>0$, $|\ell''(p,q)| \leq K,$ for all $p, q\in \sR.$
		\item Convexity: $(\ell'(\rvp,\rvy)-\ell'(\rvq,\rvy))^\top(\rvp-\rvq)\geq 0,$ for all $\rvp, \rvq\in \sR^n$.
	\end{itemize}
\end{assumption}
The above assumption is satisfied by common loss functions such as square and logistic loss. We minimize the optimization problem in \cref{loss_fn} using gradient flow:
\begin{align}
\dot{\rvw} = -\nabla \mathcal{L}(\rvw) = -\mathcal{J}(\rmX;\rvw)^\top\ell'(\mathcal{H}(\rmX;\rvw), \rvy),
\label{gf_eq}
\end{align} 
and use $\bm{\psi}(t,\rvw(0))$ to denote the solution of above differential equation, where  $\rvw(0)$ is the initialization. We aim to study the evolution of $\bm{\psi}(t,\rvw(0))$ with time for small initialization.\\\\
\textbf{Early directional convergence.} We next briefly discuss the results of \cite{kumar_dc, early_dc} which study the phenomenon of early directional convergence.
\begin{lemma}
	The origin is a critical point of the optimization problem in \cref{loss_fn}.
	\label{0_crit_lemma}
\end{lemma}
\begin{proof}
	If $\mathcal{H}$ is $L-$homogeneous, then $\mathcal{J}(\rmX;\rvw)$ is $(L-1)$-homogeneous. Since $L\geq 2$, we get $\mathcal{J}(\rmX;\mathbf{0}) = \mathbf{0}$, which implies $\nabla \mathcal{L}(\mathbf{0}) = \mathbf{0}.$
\end{proof}
Therefore, if the gradient flow in \cref{gf_eq} is  initialized near the origin, then it is expected to remain near the origin for some time, before escaping from it. In \cite{kumar_dc, early_dc}, authors study the dynamics of gradient flow with small initialization in the early stages of training and while the gradient flow remains near the origin. 

We start with some basic concepts introduced in \cite{kumar_dc, early_dc}. Let $\widetilde{\rvy}\coloneqq -\ell'(0,\rvy)$, then  the Neural Correlation Function (NCF) is defined as 
\begin{align}
\mathcal{N}(\rvu) =  \widetilde{\rvy}^\top\gH(\rmX;\rvu),
\label{ncf_gn}
\end{align}
and the constrained NCF refers to the following optimization problem
\begin{align}
\max_{\rvu} \mathcal{N}(\rvu), \text{ s.t. } \|\rvu\|_2^2 = 1.
\label{ncf_gn_const}
\end{align}
Next, consider the (positive) gradient flow of the NCF:
\begin{align}
\dot{\rvu} = \nabla\mathcal{N}(\rvu) = \mathcal{J}(\rmX;\rvu)^\top\widetilde{\rvy}.
\label{ncf_gf}
\end{align} 
We use $\bm{\phi}(t,\rvu(0))$ to denote the solution of above differential equation, where $\rvu(0)$ is the initialization. Then, as shown in \cite{kumar_dc, early_dc}, for any unit-norm vector $\rvu_0$,\footnote{The unit-norm assumption is for simplicity, the result will hold for other vectors as well.} $\bm{\phi}(t,\rvu_0)$  satisfies the following condition: $(i)$ either $\bm{\phi}(t,\rvu_0)$  converges to the origin, or $(ii)$ $\bm{\phi}(t,\rvu_0) /\|\bm{\phi}(t,\rvu_0)\|_2$ converges to a non-negative KKT point of the constrained NCF. Also, if the order of homogeneity is strictly greater than two, then $\bm{\phi}(t,\rvu_0) /\|\bm{\phi}(t,\rvu_0)\|_2$ may converge in finite time.

We next define the notion of stable set for a non-negative KKT point. 
\begin{definition}
	The stable set $\mathcal{S}(\rvu_*)$ of a non-negative KKT point $\rvu_*$ of \cref{ncf_gn_const} is the set of all unit-norm initializations such that  gradient flow in \cref{ncf_gf} converges in direction to $\rvu_*$:
	\begin{align*}
	\mathcal{S}(\rvu_*) \coloneqq \left\{\rvu_0 \in \sS^{k-1}: \frac{\bm{\phi}(t,\rvu_0)}{\|\bm{\phi}(t,\rvu_0)\|_2} \rightarrow \rvu_*\right\}
	\end{align*} 
\end{definition}
We now present the main result of \cite{kumar_dc, early_dc}, which describes the dynamics of gradient flow with small initialization in the early stages of training. 
\begin{lemma}\label{lemma_early_dc}
	Suppose $\rvw_0\in \mathcal{S}(\rvw_*)$, where $\rvw_*$  is a non-negative KKT point of \cref{ncf_gn_const}. For any sufficiently small $\epsilon>0$, there exists $T$ and $\overline{\delta}$ such that the following holds: for any $\delta\in (0,\overline{\delta})$ we have
	\begin{align*}
	\|\bm{\psi}(t,\delta\rvw_0)\|_2 = O(\delta), \text{ for all } t\in [0,T/\delta^{L-2}], \text{ and } \frac{\bm{\psi}(T/\delta^{L-2},\delta\rvw_0)^\top\rvw_*}{\|\bm{\psi}(T/\delta^{L-2},\delta\rvw_0)\|_2} = 1- O(\epsilon).
	\end{align*}
\end{lemma}
In the above lemma, the initialization is $\delta\rvw_0$, where $\rvw_0$ is a vector and $\delta>0$ is a scalar that controls the scale of initialization. It is shown that, for all sufficiently small initialization, the weights remain small for sufficiently long time and converge in direction towards a non-negative KKT point of the constrained NCF. Also, as discussed earlier, if $\rvw_0$ does not belong to the stable set of a non-negative KKT point, then $\bm{\phi}(t,\rvw_0)$  must converge to the origin. In this case, rather than directional convergence,  the weights approximately become zero; see \cite{kumar_dc, early_dc} for more details.
\section{Gradient Flow Dynamics Beyond the Origin}\label{sec_gf_beyond}
In this section, we describe the gradient flow dynamics of homogeneous neural networks after escaping the origin.
\subsection{Two-Homogeneous Neural Networks}\label{sec_2hm}
We begin by considering two-homogeneous neural networks and the following theorem describes the gradient flow dynamics of such neural networks after escaping the origin.
\begin{theorem}\label{main_thm_2hm}
	Suppose $\mathcal{H}$ is $2$-homogeneous, and let $\rvw_0\in\mathcal{S}(\rvw_*)$, where $\rvw_*$ is a $\Delta-$second-order positive KKT point of  \cref{ncf_gn_const}, for some $\Delta>0$. 
	Then, for any fixed $\widetilde{T}\in  (-\infty,\infty)$, there exists a $\widetilde{C}>0$ such that for all sufficiently small $\delta>0$,
	\begin{align}
	\left\|\bm{\psi}\left(t+T_1 +\frac{\ln(1/b_\delta)}{2\mathcal{N}(\rvw_*)}+ \frac{\ln({1}/{\widetilde{\delta}})}{2\mathcal{N}(\rvw_*)},\delta\rvw_0\right) - \rvp(t) \right\|_2 \leq \widetilde{C}\delta^\frac{\Delta}{\Delta+8\mathcal{N}(\rvw_*)}, \text{ for all }t\in [-\widetilde{T},\widetilde{T}],
	\label{err_bd_2hm}
	\end{align} 
	where $T_1>0$ is a constant, $\widetilde{\delta}\in (A_2\delta-A_1\delta^3, A_2\delta+A_1\delta^3)$, for some positive constants $A_1, A_2$,  and $b_\delta\in [\kappa_1,\kappa_2]$, for some $\kappa_2\geq \kappa_1>0$, depends on $\delta$. Moreover,
	\begin{align*}
	\rvp(t) \coloneqq \lim_{\delta\rightarrow 0} \bm{\psi}\left(t+\frac{\ln\left({1}/{\delta}\right)}{2\mathcal{N}(\rvw_*)},\delta\rvw_*\right),
	\end{align*}
	which exists for all $t\in (-\infty,\infty)$, and $\rvp(t) = \bm{\psi}\left(t,\rvp(0)\right),$ where $\rvp(0) =\lim\limits_{\delta\to 0} \bm{\psi}\left(\frac{\ln\left({1}/{\delta}\right)}{2\mathcal{N}(\rvw_*)},\delta\rvw_*\right) $ and $\mathcal{L}(\rvp(0)) \leq \mathcal{L}(\mathbf{0}) - \eta,  $ for some $\eta>0$.
\end{theorem}
We begin by explaining the motivation behind defining $\rvp(t)$. Recall that $\bm{\psi}\left(t, \delta\rvw_*\right)$ denotes the solution of the gradient flow of the training loss with initialization $\delta\rvw_*$. For small $\delta$, $\bm{\psi}\left(t, \delta\rvw_*\right)$ will remain small and near the origin for some time after the training begins. Now, loosely speaking, it turns out that $\bm{\psi}\left(t, \delta\rvw_*\right)$ would have escaped from the origin after $({\ln\left({1}/{\delta}\right)}/{2\mathcal{N}(\rvw_*)} + O(1))$ time has elapsed, where recall that $\mathcal{N}(\cdot)$ is the NCF. Since our interest is in the dynamics of gradient flow after it escapes the origin, this time is added in $\bm{\psi}\left(t, \delta\rvw_*\right)$ while defining $\rvp(t)$.  Taking $\delta\to 0$, gives us the limiting solution of $\bm{\psi}\left(t, \delta\rvw_*\right)$ after it escapes from the origin. Therefore, $\rvp(t)$ can be viewed as the approximate path $\bm{\psi}\left(t, \delta\rvw_*\right)$ takes after it escapes from the origin, for all sufficiently small $\delta$. 

We also note two key properties of $\rvp(t)$. First,  $\rvp(t) = \bm{\psi}\left(t,\rvp(0)\right) $, which implies $\rvp(t)$ is a solution of the gradient flow of the training loss with initialization $\rvp(0)$. In fact, this follows directly from the definition of $\rvp(t)$ and continuity of $\psi(\cdot,\cdot)$ with respect to the initialization:
\begin{align*}
\rvp(t) = \lim_{\delta\rightarrow 0} \bm{\psi}\left(t+\frac{\ln\left({1}/{\delta}\right)}{2\mathcal{N}(\rvw_*)},\delta\rvw_*\right) &= \lim_{\delta\rightarrow 0}\bm{\psi}\left(t,  \bm{\psi}\left(\frac{\ln\left({1}/{\delta}\right)}{2\mathcal{N}(\rvw_*)},\delta\rvw_*\right)\right)\\ &=  \bm{\psi}\left(t,  \lim_{\delta\rightarrow 0}\bm{\psi}\left(\frac{\ln\left({1}/{\delta}\right)}{2\mathcal{N}(\rvw_*)},\delta\rvw_*\right)\right) = \bm{\psi}\left(t,\rvp(0)\right) .
\end{align*} 
Another key property is that $\mathcal{L}(\rvp(0)) \leq \mathcal{L}(\mathbf{0}) - \eta,  $ for some $\eta>0$. This implies that $\|\rvp(0)\|_2 \neq 0$, and thus, $\rvp(0)$ is away from the origin. Now, since the origin is a critical point of the training loss, $\|\rvp(t)\|_2 \neq 0$, for all finite $t\leq 0$.\footnote{If $\|\rvp(t_0)\|_2 = 0$, for some finite $t_0<0$, then since the origin is a critical point of the training loss, $\|\rvp(t)\|_2 = 0$, for all $t\geq t_0$, which contradicts $\|\rvp(0)\|_2\neq 0.$  } Also, since loss always decreases under gradient flow, $\mathcal{L}(\rvp(t)) \leq \mathcal{L}(\rvp(0)) \leq \mathcal{L}(\mathbf{0}) - \eta,$ for all $t\geq 0$, which implies $\rvp(t)$ remains away from the origin, for all $t\geq 0$, including at infinity. Therefore, $\rvp(t)$ has truly escaped from the origin. We highlight this property because, if $\rvw_*$ were a zero KKT point instead of a positive KKT point, that is, $\mathcal{N}(\rvw_*) = 0$, then $\bm{\psi}\left(t, \delta\rvw_*\right)$ may never escape from the origin (see Lemma \ref{ex_not_escape} for an example). Assuming $\rvw_*$ to be a positive KKT point ensures that $\bm{\psi}\left(t, \delta\rvw_*\right)$  eventually escapes from the origin.

We next discuss \cref{err_bd_2hm} and its consequences. For small $\delta$, $\bm{\psi}\left(t, \delta\rvw_0\right)$ will remain small and near the origin for some time after the training begins. From \cref{err_bd_2hm}, we observe that, for all sufficiently small $\delta$, $\bm{\psi}\left(t, \delta\rvw_0\right)$ would have definitely escaped from the origin after $T_1+ (\ln(1/b_\delta)+{\ln({1}/\widetilde{\delta})})/{2\mathcal{N}(\rvw_*)}$ time has elapsed, since $\rvp(0)$ is away from the origin and the RHS of \cref{err_bd_2hm} is small, for small $\delta$. More importantly, \cref{err_bd_2hm} implies that $\bm{\psi}\left(t, \delta\rvw_0\right)$ will remain close to $\rvp(t)$ after escaping the origin for arbitrarily long time, provided the initialization is sufficiently small. Using the definition of $\rvp(t)$, it can be further said that the trajectory of  $\bm{\psi}\left(t, \delta\rvw_0\right)$  and $\bm{\psi}\left(t, \delta\rvw_*\right)$ will approximately be same after they escape from the origin, for arbitrarily long time, if the initialization is sufficiently small. Therefore, the behavior of the gradient flow after escaping from the origin is primarily determined by the KKT point whose stable set contains the initial direction of the weights.

Note that $\rvp(t)$ is a solution of the gradient flow of the training loss. Suppose the trajectory of $\rvp(t)$ is bounded for all $t\geq 0$, and it converge to a saddle point, or more generally, a stationary point of the training loss. Then, in the following corollary, we show that $\bm{\psi}\left(t, \delta\rvw_0\right)$  gets close to that saddle point at some time, for all sufficiently small $\delta$.
\begin{corollary}\label{cor_sq_loss_2hm}
	Consider the setting of \Cref{main_thm_2hm}. Suppose  $\rvp(t)$ is bounded for all $t\geq 0$, and let $\rvp^*= \lim_{t\to \infty} \rvp(t)$, where $\nabla \mathcal{L}(\rvp^*) = 0$ Then, for any sufficiently small $\epsilon>0$, there exists a time $T_\epsilon$ such that for all sufficiently small $\delta>0$,
	\begin{align*}
	\left\|\bm{\psi}\left(T_\delta,\delta\rvw_0\right) - \rvp^*\right\|_2 \leq \epsilon \text{ and } \left\|\nabla \mathcal{L}\left(\bm{\psi}\left(T_\delta,\delta\rvw_0\right)\right)\right\|_2 \leq \epsilon,
	\end{align*}
	where $T_\delta\coloneqq T_\epsilon+T_1 +\frac{\ln(1/b_\delta)}{2\mathcal{N}(\rvw_*)}+ \frac{\ln({1}/{\widetilde{\delta}})}{2\mathcal{N}(\rvw_*)}$.
\end{corollary}
The above corollary implies that for all sufficiently small $\delta$, after escaping from the origin, gradient flow get close to the saddle point to which $\rvp(t)$ converges. 

The above corollary assumes $\rvp(t)$ is bounded and converges to a finite saddle point. This is a reasonable assumption when square loss is used for training. For loss functions such as logistic loss, the trajectory of gradient flow may not be bounded and saddle points of the training loss could be at infinity. The next corollary considers such cases.
\begin{corollary}\label{cor_log_loss_2hm}
	Consider the setting of \Cref{main_thm_2hm}. Suppose $\lim_{t\to \infty} \rvp(t)/\|\rvp(t)\|_2 = \rvp^*$ and $\lim_{t\to \infty} \|\rvp(t)\| = \infty$, such that $\lim_{t\to\infty} \nabla \mathcal{L}(\rvp(t)) = \mathbf{0}$ and $\lim_{\alpha\to\infty}\nabla \mathcal{L}(\alpha\rvp^*) = 0$.  Then, for any sufficiently small $\epsilon>0$, there exists a time $T_\epsilon$ such that for all sufficiently small $\delta>0$,
	\begin{align*}
	&\left\|\bm{\psi}\left(T_\delta,\delta\rvw_0\right)\right\|_2 \geq \frac{1}{2\epsilon}, \frac{\bm{\psi}\left(T_\delta,\delta\rvw_0\right)^\top\rvp^*}{\left\|\bm{\psi}\left(T_\delta,\delta\rvw_0\right)\right\|_2 }\geq 1-\epsilon, \text{ and } \left\|\nabla \mathcal{L}\left(\bm{\psi}\left(T_\delta,\delta\rvw_0\right)\right)\right\|_2 \leq \epsilon,
	\end{align*}
	where $T_\delta\coloneqq T_\epsilon+T_1 +\frac{\ln(1/b_\delta)}{2\mathcal{N}(\rvw_*)}+ \frac{\ln({1}/{\widetilde{\delta}})}{2\mathcal{N}(\rvw_*)}$.
\end{corollary}
In the above corollary, we have assumed that $\rvp(t)$ ``converges" to a saddle point at infinity, in the sense that its norm diverges to infinity, but the direction converges. Under this assumption, $\bm{\psi}\left(t, \delta\rvw_0\right)$  gets close to that saddle point at some time, for all sufficiently small initialization, since its norm becomes large, and it is approximately aligned in direction with the saddle point.

We would like to emphasize that the above corollaries \emph{do not} imply that $\bm{\psi}\left(t, \delta\rvw_0\right)$ will also converge to the same saddle point to which $\rvp(t)$ converges. For instance, consider Corollary \ref{cor_sq_loss_2hm}, which shows that $\bm{\psi}\left(t, \delta\rvw_0\right)$ gets close to the saddle point $\rvp^*$ at some time. It is possible that $\bm{\psi}\left(t, \delta\rvw_0\right)$  may eventually escape from this saddle point, in contrast to $\rvp(t)$ which converges towards $\rvp^*$. The evolution of gradient flow after escaping this saddle point is not described by the above results, and is an important direction for future research. Thus, in this sense, our work only captures a segment of  the gradient flow dynamics beyond the origin, specifically until the first saddle point encountered by gradient flow after escaping the origin. The following two-dimensional example may be helpful in understanding our results better, the complete details of which are in Appendix \ref{additonal_results}.
\begin{figure}
	\centering
	\includegraphics[width=0.43\textwidth]{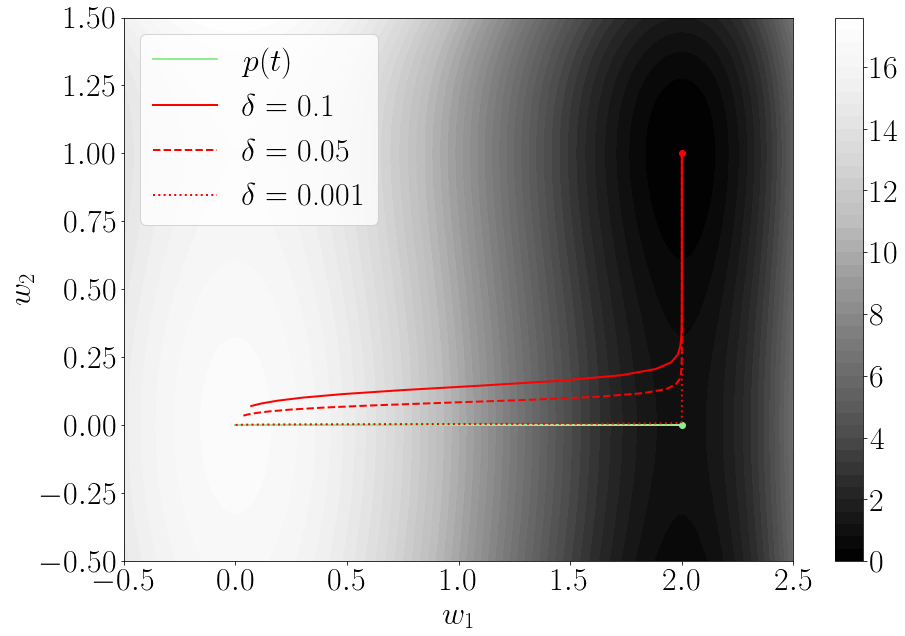}
	\caption{The contour of the loss function in \cref{loss_ex} is in the background. The foreground contains evolution of $\bm{\psi}(t,\delta\rvw_0) $, for $\delta \in\{0.1,0.05,0.001\}$ and $t\in [0,3]$ (in red), and $\rvp(t)$, for $t\in [-1,1]$ (in green). The saddle point at $(2,0)$ and the global minimum at $(2,1)$ are marked with green and red dot respectively.}
	\label{fig:psi_p_ex}
\end{figure}
\begin{example}\label{ex_2hm}
	Suppose $\mathcal{H}(\rvx;\rvw) = w_1^2x_1+w_2^2x_2$, which is two-homogeneous. The training data has two points $(1,0)$ and $(0,1)$, where the corresponding labels are $4$ and $1$ respectively. Assuming square loss, the training loss becomes
	\begin{align}
	\mathcal{L}(\rvw) = (w_1^2-4)^2 + (w_2^2-1)^2,
	\label{loss_ex}
	\end{align}
	and $\mathcal{N}(\rvw) = 8w_1^2+2w_2^2$. Let $\rvw_0 = (1/\sqrt{2},1/\sqrt{2})$, then $\rvw_0\in\mathcal{S}(\rvw_*)$, where $\rvw_* = (1,0)$, and $\mathcal{N}(\rvw_*) = 8$. For any $\delta\in(0,1)$, we have
	\begin{align*}
	\bm{\psi}(t,\delta\rvw_0) = \left(\frac{2\delta}{\sqrt{\delta^2 + (8-\delta^2)e^{-32t}}}, \frac{\delta}{\sqrt{\delta^2 + (2-\delta^2)e^{-8t}}}\right), 
	\end{align*}
	\begin{align*}
	\bm{\psi}(t,\delta\rvw_*) =  \left(\frac{2\delta}{\sqrt{\delta^2 + (4-\delta^2)e^{-32t}}}, 0\right), \text{ and }\rvp(t) = \left(\frac{2}{\sqrt{1 + 4e^{-32t}}}, 0\right).
	\end{align*}
	Note that, for any $\gamma\in (0,1)$,
	\begin{align*}
	\left\|\bm{\psi}\left(\frac{(1-\gamma)\ln\left({1}/{\delta}\right)}{2\mathcal{N}(\rvw_*)},\delta\rvw_*\right) \right\|_2 = O(\delta^\gamma), \left\|\bm{\psi}\left(\frac{(1-\gamma)\ln\left({1}/{\delta}\right)}{2\mathcal{N}(\rvw_*)},\delta\rvw_0\right) \right\|_2 = O(\delta^\gamma),
	\end{align*}
	and for any $\kappa>0$,
	\begin{align*}
	\left\|\bm{\psi}\left(\frac{\ln\left({\kappa}/{\delta}\right)}{2\mathcal{N}(\rvw_*)},\delta\rvw_*\right) \right\|_2 = O(1), \left\|\bm{\psi}\left(\frac{\ln\left({\kappa}/{\delta}\right)}{2\mathcal{N}(\rvw_*)},\delta\rvw_0\right) \right\|_2 = O(1),
	\end{align*}
	Therefore, after $(O(1)+{\ln\left({1}/{\delta}\right)}/{2\mathcal{N}(\rvw_*)})$ time has elapsed, $\bm{\psi}(t,\delta\rvw_*)$ and $\bm{\psi}(t,\delta\rvw_0)$ would have escaped from the origin. Next, $\rvp(t)$ and $\bm{\psi}(t,\delta\rvw_0)$ converge to different limits since
	\begin{align*}
	\lim_{t\to \infty}  \rvp(t) = (2,0) \text{ and } \lim_{t\to \infty} \bm{\psi}(t,\delta\rvw_0) = (2,1).
	\end{align*}
	In  \Cref{fig:psi_p_ex}, we plot the evolution of $\rvp(t)$ and $\bm{\psi}(t,\delta\rvw_0)$, for small values of $\delta$. For the smallest value of $\delta$, observe that $\bm{\psi}(t,\delta\rvw_0)$  follows the same path as $\rvp(t)$ all the way until it gets very close to $(2,0)$, which is a saddle point. Then it escapes from it and eventually converges to $(2,1)$. Thus, even though $\bm{\psi}(t,\delta\rvw_0)$ gets close to $(2,0)$, it does not converge to it, unlike $\rvp(t)$.\\\\
	It is also worth noting that for any $\kappa>0$,
	\begin{align*}
	\bm{\psi}\left(\frac{\ln\left({\kappa}/{\delta}\right)}{2\mathcal{N}(\rvw_*)},\delta\rvw_0\right) = (O(1), O(\delta^{3/4})),\bm{\psi}\left(\frac{\ln\left({\kappa}/{\delta}\right)}{4} + \frac{\ln\left({\kappa}/{\delta}\right)}{2\mathcal{N}(\rvw_*)},\delta\rvw_0\right) = (2-O(\delta^{10}), O(1)).
	\end{align*}
	Thus, $\bm{\psi}\left(t,\delta\rvw_0\right) $ escapes from the saddle point at $(2,0)$ after $(O(1)+{\ln\left({1}/{\delta}\right)}/{2\mathcal{N}(\rvw_*)}) + {\ln\left({1}/{\delta}\right)}/{4})$ time has elapsed. This suggests that to analyze the gradient flow dynamics beyond the first saddle point, another $O({\ln\left({1}/{\delta}\right)})$ term needs to be added in the time. This observation may be helpful for future works that attempt to understand the gradient flow dynamics beyond the first saddle point.  
\end{example}
\subsubsection{Proof Outline of \Cref{main_thm_2hm}}
We first discuss three important lemmata, which are then used to prove \Cref{main_thm_2hm}. The first lemma proves the existence of $\rvp(t)$, for all $t\in (-\infty,\infty)$, by showing $\bm{\psi}\left(t+\frac{\ln\left({1}/{\delta}\right)}{2\mathcal{N}(\rvw_*)},\delta\rvw_*\right)$ is a Cauchy sequence, and also characterizes $\mathcal{L}(\rvp(0)). $

\begin{lemma}\label{lim_exists_2hm}
	Consider the setting of \Cref{main_thm_2hm}, then for any fixed $t\in (-\infty,\infty)$ and all sufficiently small $\delta_2\geq{\delta}_1>0$, there exists a $C>0$ such that
	\begin{align*}
	\left\|\bm{\psi}\left(t+\frac{\ln\left({1}/{{\delta}_1}\right)}{2\mathcal{N}(\rvw_*)},{\delta}_1\rvw_*\right) - \bm{\psi}\left(t+\frac{\ln\left({1}/{\delta_2}\right)}{2\mathcal{N}(\rvw_*)},\delta_2\rvw_*\right)\right\|_2 \leq C\delta_2,
	\end{align*}
	implying $\rvp(t)$  exists for all $t\in (-\infty,\infty)$. Furthermore, let  ${\delta}_1\to 0$ and $\delta_2 = \delta>0$, then  
	\begin{align*}
	\left\|\rvp(t)- \bm{\psi}\left(t+\frac{\ln\left({1}/{\delta}\right)}{2\mathcal{N}(\rvw_*)},\delta\rvw_*\right)\right\|_2 \leq C\delta.
	\end{align*}
	Also, $\mathcal{L}(\rvp(0)) \leq \mathcal{L}(\mathbf{0})-\eta$, for some $\eta>0$.
\end{lemma}
We next show  that $\bm{\psi}(t,\delta\rvw_0)$ gets sufficiently aligned with $\rvw_*$  while staying near the origin. 
\begin{lemma}\label{init_align_2hm}
	Consider the setting in \Cref{main_thm_2hm}, then there exists $T_1, a_1 >0$ such that for all sufficiently small $\delta>0$,
	\begin{align*}
	\left\|\bm{\psi}\left(T_1 + \frac{4\ln({1}/{\widetilde{\delta}})}{\Delta+8\mathcal{N}(\rvw_*)},\delta\rvw_0\right) - b_\delta\widetilde{\delta}^{\frac{\Delta}{\Delta + 8\mathcal{N}(\rvw_*)}}\rvw_* \right\|_2 \leq a_1 \widetilde{\delta}^{\frac{3\Delta}{\Delta + 8\mathcal{N}(\rvw_*)}},
	\end{align*}
	where $\widetilde{\delta}\in (A_2\delta-A_1\delta^3, A_2\delta+A_1\delta^3)$, for some positive constants $A_1, A_2$. Also, $b_\delta\in [\kappa_1,\kappa_2]$, for some $\kappa_2\geq \kappa_1>0$, depends on $\delta$.
\end{lemma}
If we define $T_{\widetilde{\delta}} \coloneqq T_1 + {4\ln({1}/{\widetilde{\delta}})}/({\Delta+8\mathcal{N}(\rvw_*)})$, then the above lemma implies 
\begin{align*}
\left\|\bm{\psi}\left(T_{\widetilde{\delta}},\delta\rvw_0\right)\right\|_2 = O\left(\delta^{\frac{\Delta}{\Delta + 8\mathcal{N}(\rvw_*)}}\right), 
\frac{\bm{\psi}\left(T_{\widetilde{\delta}},\delta\rvw_0\right)^\top\rvw_*}{\left\|\bm{\psi}\left(T_{\widetilde{\delta}},\delta\rvw_0\right)\right\|_2} = 1- O\left(\delta^{\frac{2\Delta}{\Delta + 8\mathcal{N}(\rvw_*)}}\right),
\end{align*}
for all sufficiently small $\delta>0$. Therefore, in the early stages of training, $\bm{\psi}(t,\delta\rvw_0)$ remains small and converges in direction to $\rvw_*$. Note that, the above equation implies a stronger directional convergence than Lemma \ref{lemma_early_dc}, since there $\epsilon$ is small but fixed (does not depend on $\delta$). However, compared to Lemma \ref{lemma_early_dc}, here $\rvw_*$ is assumed to be a \emph{second-order} KKT point.

The following lemma allows us to combine Lemma \ref{init_align_2hm} and Lemma \ref{lim_exists_2hm} in order to prove \Cref{main_thm_2hm}. It essentially shows that if initialized near the origin and sufficiently aligned with $\rvw_*$, then gradient flow remains close to $\rvp(t)$.
\begin{lemma}\label{init_dl3_2hm}
	Consider the setting in \Cref{main_thm_2hm}. Suppose there exists a constant $C_1>0$ such that for every $\delta>0$, ${\rva}_\delta$ is a vector that satisfies $\|{\rva}_\delta - \delta\rvw_*\|_2\leq C_1\delta^3$. Then, for any fixed $\widetilde{T}\in (-\infty,\infty)$, there exists a constant $C>0$ such that for all sufficiently small $\delta>0$, 
	\begin{align*}
	\left\|\bm{\psi}\left(t+ \frac{\ln\left({1}/{\delta}\right)}{2\mathcal{N}(\rvw_*)},{\rva}_\delta\right) - \bm{\psi}\left(t+ \frac{\ln\left({1}/{\delta}\right)}{2\mathcal{N}(\rvw_*)},\delta\rvw_*\right)  \right\|_2 \leq C\delta, \text{ for all }t\in [-\widetilde{T},\widetilde{T}].
	\end{align*}
\end{lemma}
\begin{proof}\textbf{of  Theorem \ref{main_thm_2hm}:} From Lemma \ref{init_align_2hm}, we know
\begin{align*}
\left\|\bm{\psi}\left(T_1 + \frac{4\ln({1}/{\widetilde{\delta}})}{\Delta+8\mathcal{N}(\rvw_*)},\delta\rvw_0\right) - b_\delta\widetilde{\delta}^{\frac{\Delta}{\Delta + 8\mathcal{N}(\rvw_*)}}\rvw_* \right\|_2 \leq a_1 \widetilde{\delta}^{\frac{3\Delta}{\Delta + 8\mathcal{N}(\rvw_*)}},
\end{align*}
for some $T_1, a_1>0$. Define $\bar{\delta} = b_\delta\widetilde{\delta}^{\frac{\Delta}{\Delta + 8\mathcal{N}(\rvw_*)}}$, then the above equation implies
\begin{align*}
\left\|\bm{\psi}\left(T_1 + \frac{4\ln({1}/{\widetilde{\delta}})}{\Delta+8\mathcal{N}(\rvw_*)},\delta\rvw_0\right) - \bar{\delta}\rvw_* \right\|_2 \leq a_2\bar{\delta}^3,
\end{align*}
where $a_2=a_1/\kappa_1^3\geq a_1/b_\delta^3.$ Next, using Lemma \ref{init_dl3_2hm}, for any fixed $\widetilde{T}\in (-\infty,\infty)$, there exists a constant $\widetilde{C}_1>0$ such that for all sufficiently small $\delta>0$,
\begin{align*}
\left\|\bm{\psi}\left(t+ \frac{\ln\left({1}/{\bar{\delta}}\right)}{2\mathcal{N}(\rvw_*)},\bm{\psi}\left(T_1 + \frac{4\ln({1}/{\widetilde{\delta}})}{\Delta+8\mathcal{N}(\rvw_*)},\delta\rvw_0\right)\right) - \bm{\psi}\left(t+ \frac{\ln\left({1}/{\bar{\delta}}\right)}{2\mathcal{N}(\rvw_*)},\bar{\delta}\rvw_*\right)  \right\|_2 \leq \widetilde{C}_1\bar{\delta},
\end{align*}
for all $t\in [-\widetilde{T},\widetilde{T}]$, and from Lemma \ref{lim_exists_2hm}, there exists a $\widetilde{C}_2>0$ such that
\begin{align*}
\left\|\rvp(t)- \bm{\psi}\left(t+\frac{\ln\left({1}/{\bar{\delta}}\right)}{2\mathcal{N}(\rvw_*)},\bar{\delta}\rvw_*\right)\right\|_2 \leq \widetilde{C}_2\bar{\delta}, \text{ for all } t\in [-\widetilde{T},\widetilde{T}].
\end{align*}
Now, since
\begin{align*}
&\bm{\psi}\left(t+ \frac{\ln\left({1}/{\bar{\delta}}\right)}{2\mathcal{N}(\rvw_*)},\bm{\psi}\left(T_1 + \frac{4\ln({1}/{\widetilde{\delta}})}{\Delta+8\mathcal{N}(\rvw_*)},\delta\rvw_0\right)\right)\\
&= \bm{\psi}\left(t+ \frac{\ln\left({1}/{\bar{\delta}}\right)}{2\mathcal{N}(\rvw_*)}+T_1 + \frac{4\ln({1}/{\widetilde{\delta}})}{\Delta+8\mathcal{N}(\rvw_*)},\delta\rvw_0\right)\\
&= \bm{\psi}\left(t+ T_1 +\frac{\ln(1/b_\delta)}{2\mathcal{N}(\rvw_*)}+\left(\frac{\Delta}{2\mathcal{N}(\rvw_*)}+4\right) \frac{\ln({1}/{\widetilde{\delta}})}{\Delta+8\mathcal{N}(\rvw_*)},\delta\rvw_0\right)\\
&= \bm{\psi}\left(t+ T_1 +\frac{\ln(1/b_\delta)}{2\mathcal{N}(\rvw_*)}+ \frac{\ln({1}/{\widetilde{\delta}})}{2\mathcal{N}(\rvw_*)},\delta\rvw_0\right),
\end{align*}
we have that for all $t\in [-\widetilde{T},\widetilde{T}]$ and for all sufficiently small $\delta>0$, 
\begin{align*}
&\left\|\bm{\psi}\left(t+ T_1 +\frac{\ln(1/b_\delta)}{2\mathcal{N}(\rvw_*)}+ \frac{\ln({1}/{\widetilde{\delta}})}{2\mathcal{N}(\rvw_*)},\delta\rvw_0\right) - \rvp(t)\right\|_2 \leq \left\|\bm{\psi}\left(t+ \frac{\ln\left({1}/{\bar{\delta}}\right)}{2\mathcal{N}(\rvw_*)},\bar{\delta}\rvw_*\right)  - \rvp(t)\right\|_2 \\
& + \left\|\bm{\psi}\left(t+ T_1 +\frac{\ln(1/b_\delta)}{2\mathcal{N}(\rvw_*)}+ \frac{\ln({1}/{\widetilde{\delta}})}{2\mathcal{N}(\rvw_*)},\delta\rvw_0\right) - \bm{\psi}\left(t+ \frac{\ln\left({1}/{\bar{\delta}}\right)}{2\mathcal{N}(\rvw_*)},\bar{\delta}\rvw_*\right) \right\|_2\\
&\leq \widetilde{C}_1\bar{\delta}+\widetilde{C}_2\bar{\delta} \leq \widetilde{C}\delta^{\frac{\Delta}{\Delta + 8\mathcal{N}(\rvw_*)}},
\end{align*}
where $\widetilde{C}$ is a positive constant. The last inequality is true since $\widetilde{\delta}\leq A_2\delta+A_1\delta^3\leq 2A_2\delta$, for all sufficiently small $\delta$. Thus, the proof is complete.
\end{proof}
\subsection{Deep Homogeneous Neural Networks}\label{sec_Lhm}
We now present our main result describing the dynamics of gradient flow for $L$-homogeneous neural networks after it escapes from origin, where $L>2$.
\begin{theorem}\label{main_thm_Lhm}
	Suppose $\mathcal{H}$ is $L$-homogeneous, where $L>2$. Let $\rvw_0\in\mathcal{S}(\rvw_*)$, where $\rvw_*$ is a $\Delta-$second-order positive KKT point of  \cref{ncf_gn_const}, for some $\Delta>0$. Then, for any fixed $\widetilde{T}\in  (-\infty,\infty)$, there exists a $\widetilde{C}>0$  such that for all sufficiently small $\delta>0,$
	\begin{align}
	&\left\|\bm{\psi}\left(t+ \frac{T_1}{\delta^{L-2}} + \left(\frac{1/b_\delta^{L-2}}{L(L-2)\mathcal{N}(\rvw_*)} - T\right){\widetilde{\delta}^{\frac{-(L-2)\Delta}{2L^2\mathcal{N}(\rvw_*)+\Delta}}}+\frac{T}{\widetilde{\delta}^{L-2}},\delta\rvw_0\right) - \rvp(t) \right\|_2\nonumber \\
	&\leq \widetilde{C}\delta^{\frac{\Delta}{2L^2\mathcal{N}(\rvw_*) +\Delta}}, 
	\label{err_bd_Lhm}
	\end{align} 
	for all $t\in [-\widetilde{T},\widetilde{T}]$,  where $T_1>0$ is a constant,  $\widetilde{\delta}\in (A_2\delta-A_1\delta^{L+1}, A_2\delta+A_1\delta^{L+1})$ for some constants $A_1, A_2>0$,  $T\geq \frac{1}{L(L-2)\mathcal{N}(\rvw_*)}$, and $b_\delta^{L-2}\in \left[\frac{1}{TL(L-2)\mathcal{N}(\rvw_*)},1\right]$ depends on $\delta$. Moreover,
	\begin{align*}
	\rvp(t) := \lim_{\delta\rightarrow 0} \bm{\psi}\left(t+\frac{1/\delta^{L-2}}{L(L-2)\mathcal{N}(\rvw_*)},\delta\rvw_*\right), 
	\end{align*}
	\sloppypar \noindent which exists for all $t\in (-\infty,\infty)$, and $\rvp(t) = \bm{\psi}\left(t,\rvp(0)\right),$ where $\rvp(0) =\lim\limits_{\delta\to 0} \bm{\psi}\left(\frac{1/\delta^{L-2}}{L(L-2)\mathcal{N}(\rvw_*)},\delta\rvw_*\right) $ and $\mathcal{L}(\rvp(0)) \leq  \mathcal{L}(\mathbf{0}) - \eta,  $ for some $\eta>0$.
\end{theorem}
Overall, the result here is similar to two-homogeneous case, except for the fact that it takes longer time to escape from the origin. For small $\delta$, $\bm{\psi}\left(t,\delta\rvw_*\right)$ would have escaped from the origin after $\frac{1/\delta^{L-2}}{L(L-2)\mathcal{N}(\rvw_*)}+O(1)$ time has elapsed, which is reflected in the definition of $\rvp(t)$. From \cref{err_bd_Lhm}, we can say that the trajectory of  $\bm{\psi}\left(t, \delta\rvw_0\right)$  and $\bm{\psi}\left(t, \delta\rvw_*\right)$ will approximately be the same after escaping from the origin, for arbitrarily long time if the initialization is sufficiently small. The proof follows a similar approach to the two-homogeneous case, with key differences arising from the fact that in deep homogeneous networks, the gradient flow of NCF can become unbounded at a finite time. See the proof for more details.

Next, we can use the above theorem to determine the saddle point encountered by gradient flow after escaping from the origin, in a similar way as for two-homogeneous neural networks. We start with the case when $\rvp(t)$ is bounded.

\begin{corollary}\label{cor_sq_loss_Lhm}
		Consider the setting of \Cref{main_thm_Lhm}. Suppose  $\rvp(t)$ is bounded for all $t\geq 0$, and let $\rvp^*= \lim_{t\to \infty} \rvp(t)$, where $\nabla \mathcal{L}(\rvp^*) = 0$ Then, for any sufficiently small $\epsilon>0$, there exists a time $T_\epsilon$ such that for all sufficiently small $\delta>0$,
	\begin{align*}
	\left\|\bm{\psi}\left(T_\delta,\delta\rvw_0\right) - \rvp^*\right\|_2 \leq \epsilon, \left\|\nabla \mathcal{L}\left(\bm{\psi}\left(T_\delta,\delta\rvw_0\right)\right)\right\|_2 \leq \epsilon,
	\end{align*}
	where $T_\delta \coloneqq {T_\epsilon}+ \frac{T_1}{\delta^{L-2}} + \left.\left(\frac{1/b_\delta^{L-2}}{L(L-2)\mathcal{N}(\rvw_*)} - T\right) {\widetilde{\delta}^{\frac{-(L-2)\Delta}{2L^2\mathcal{N}(\rvw_*)+\Delta}}}\right.+\frac{T}{\widetilde{\delta}^{L-2}}$.
\end{corollary}
We next consider the case when $\rvp(t)$ can become unbounded.
\begin{corollary}\label{cor_log_loss_Lhm}
	Consider the setting of \Cref{main_thm_Lhm}. Suppose $\lim_{t\to \infty} \rvp(t)/\|\rvp(t)\|_2 = \rvp^*$ and $\lim_{t\to \infty} \|\rvp(t)\| = \infty$, such that $\lim_{t\to\infty} \nabla \mathcal{L}(\rvp(t)) = \mathbf{0}$ and $\lim_{\alpha\to\infty}\nabla \mathcal{L}(\alpha\rvp^*) = 0$.  Then, for any sufficiently small $\epsilon>0$, there exists a time $T_\epsilon$ such that for all sufficiently small $\delta>0$,
	\begin{align*}
	&\left\|\bm{\psi}\left(T_\delta,\delta\rvw_0\right)\right\|_2 \geq \frac{1}{2\epsilon}, \frac{\bm{\psi}\left(T_\delta,\delta\rvw_0\right)^\top\rvp^*}{\left\|\bm{\psi}\left(T_\delta,\delta\rvw_0\right)\right\|_2 }\geq 1-\epsilon, \text{ and }\left\|\nabla \mathcal{L}\left(\bm{\psi}\left(T_\delta,\delta\rvw_0\right)\right)\right\|_2 \leq \epsilon,
	\end{align*}
	where $T_\delta \coloneqq {T_\epsilon}+ \frac{T_1}{\delta^{L-2}} + \left.\left(\frac{1/b_\delta^{L-2}}{L(L-2)\mathcal{N}(\rvw_*)} - T\right) {\widetilde{\delta}^{\frac{-(L-2)\Delta}{2L^2\mathcal{N}(\rvw_*)+\Delta}}}\right.+\frac{T}{\widetilde{\delta}^{L-2}}$.
\end{corollary}
\begin{remark}
	Theorem \ref{main_thm_2hm} and Theorem \ref{main_thm_Lhm} assume $\rvw_0\in\mathcal{S}(\rvw_*)$, where $\rvw_*$ is a second-order positive KKT point of  \cref{ncf_gn_const}. If this is not satisfied, three scenarios arise: $(i)$ $\rvw_0\in\mathcal{S}(\rvw_*)$, where  $\rvw_*$ is a first-order positive KKT point, but not second-order, $(ii)$  $\rvw_0\in\mathcal{S}(\rvw_*)$, where  $\rvw_*$ is a zero KKT point, or $(iii)$ $\phi(t,\rvw_0)$ converges to the origin. In the last two cases, the gradient flow may not escape the origin. We are unable to handle the first case, and leave it as a future direction. Perhaps this case does not occur typically, since for many problems gradient descent almost surely avoids first-order saddle points \citep{lee1_esc_saddle, lee2_esc_saddle}.
\end{remark}
\begin{remark}
	\Cref{main_thm_2hm} and \Cref{main_thm_Lhm} implies that the trajectories of $\bm{\psi}(t,\delta\rvw_0)$ and $\bm{\psi}(t,\delta\rvw_*)$ stay close after escaping the origin. One might attempt to prove this using the result of \cite{kumar_dc, early_dc} stated in Lemma \ref{lemma_early_dc}, which shows for $L=2$ that, for small $\epsilon>0$, there exists a time $T$ such that $\bm{\psi}(T,\delta\rvw_0)^\top\rvw_*/\|\bm{\psi}(T,\delta\rvw_0)\|_2 = 1- O(\epsilon),$ for all small $\delta>0$. To explore this idea further, suppose $\rvw_0 = \rvw_*+\epsilon\rvb$, where $\rvb \in \rvw_*^\perp$, that is, approximate directional convergence holds at the initialization itself. Since $\bm{\psi}(\cdot,\cdot)$ varies continuously with initialization and $\|\delta\rvw_0 - \delta\rvw_*\|_2 = \delta\epsilon$, which is small because $\delta$ and $\epsilon$ are small, it can be shown that $\bm{\psi}(t,\delta\rvw_*)$ and $\bm{\psi}(t,\delta\rvw_0)$ remain close for some time after the training begins. However, that time may not be long enough to ensure that $\bm{\psi}(t,\delta\rvw_*)$ and $\bm{\psi}(t,\delta\rvw_0)$ remain close after they escape from the origin, making this approach unsuccessful (see Appendix \ref{additonal_results} for more details).
	\label{remark_diff}
\end{remark}
\section{Implications for Feed-Forward Neural Networks}\label{sec_implications}
This section discusses the implications of the above theorems for feed-forward neural network. Specifically, we use the above theorems and the results of \cite{early_dc}, which studies the KKT points of the NCF, to show how sparsity structure is preserved among the weights after gradient flow escapes from the origin.

 We first introduce the notion of \emph{zero-preserving subset}, which is a subset of the weights of a neural network that remain zero throughout training, if they are zero at initialization. 
\begin{definition}
	Consider the training setup of \Cref{sec_ps}, and recall $\bm{\psi}(t,\rvw(0))$ denotes the solution of
	\begin{equation*}
	\dot{\rvw} = -\nabla \mathcal{L}(\rvw),
	\end{equation*}
	where $\rvw(0)$ is the initialization. If $\rvw_z$ is a vector containing subset of entries in $\rvw$,\footnote{Throughout this section, we will abuse the notation slightly and refer to a vector as a set containing its entries and vice-versa.}  then $\bm{\psi}_{\rvw_z}(t,\rvw(0))$ denotes the evolution of weights belonging to $\rvw_z$. We define $\rvw_z$ to be a zero-preserving subset of  $\rvw$ under the following condition: if $\|\bm{\psi}_{\rvw_z}(0,\rvw(0)) \|_2= 0$, then $\|\bm{\psi}_{\rvw_z}(t,\rvw(0))\|_2= 0$, for all $t\in(-\infty,\infty)$.
\end{definition} 
Note that there could be a subset of weights which are non-zero at initialization, but become zero at some stage of the training and remain zero after that. There could also be subset of weights which are zero at initialization, but become non-zero during training. Such set of weights do not satisfy the definition of zero-preserving subset. The weights belonging to zero-preserving subset are zero at initialization, and they remain zero throughout training. 

It may seem that the training data plays a role in determining the zero-preserving subsets, that is, for a fixed architecture, different training data may lead to different sets of zero-preserving subsets. While this is possible, we show that for some architectures, certain zero-preserving subsets do not change with the training data. In such cases, zero-preserving subsets should be viewed as a property of the architecture, rather than the training data.\\\\
For $L\geq 2$, the output of an $L-$layer feed-forward neural network $\mathcal{H}$ is
\begin{equation}
\mathcal{H}(\rvx;\rmW_1,\cdots,\rmW_L) = \rmW_L\sigma(\rmW_{L-1}\cdots\sigma(\rmW_1\rvx)\cdots),
\label{L_layer_feed}
\end{equation}
where $\rmW_l\in \sR^{k_{l}\times k_{l-1}}$, $k_0 = d$ and $k_L = 1.$ The activation function $\sigma:\sR\to\sR$ is applied co-ordinate wise. Note that, if $\sigma(x) = \max(x,\alpha x)^p$, for some $p\in\sN$ and $\alpha\in \sR$, then the above neural network is homogeneous. We use $\rmW_l[:,j]$ and $\rmW_l[j,:]$ to denote the $j$-th column and $j$-th row of $\rmW_l$ respectively. Also, $\rmW_l[j,:]$ contains  incoming  weights from the $j$-th neuron in the $l$-th layer, and  $\rmW_{l+1}[:,j]$ contains outgoing weights from the same neuron. 

The following lemma describes the zero-preserving subsets of feed-forward neural networks for arbitrary training data.
\begin{lemma}\label{zp_subs_ff}
	Let $\mathcal{H}$ be an $L-$layer feed-forward neural network as in \cref{L_layer_feed}, where $\sigma(\cdot)$ is locally Lipschitz, continuously differentiable\footnote{We assume differentiability for simplicity, the result can also be proved for ReLU activation.} and $\sigma(0) = 0$. Suppose $\rvw_z$ is a subset of the weights such that 
	\begin{equation*}
	\rvw_z = \bigcup\limits_{l=1}^{L-1}\left(\rmW_l[j,:] \cup \rmW_{l+1}[:,j], j\in \mathcal{S}_l\right),
	\end{equation*}
	where $\mathcal{S}_l$ is an arbitrary subset of $[k_l]$, for all $l\in [L-1]$, then $\rvw_z$ is a zero-preserving subset.
\end{lemma} 
The assumption $\sigma(0) = 0$ is satisfied by all homogeneous activation functions. According to the above lemma, the zero-preserving subset is a collection of rows and columns of the weight matrices such that if $\rmW_l[j,:]\in\rvw_z$, then $\rmW_{l+1}[:,j]\in\rvw_z$, and vice-versa. Another way to view the zero-preserving subset is to look at the hidden neurons. The zero-preserving subset is formed by combining \emph{all} the incoming and outgoing weights of certain subset of hidden neurons; some examples are provided in \Cref{fig:zp_set}. Also, note that the zero-preserving subset contains only rows of $\rmW_1$ and only columns of $\rmW_L$, whereas it could contain rows and columns of other weight matrices.

\begin{figure}[ht]
	\centering
	\begin{subfigure}[b]{0.45\textwidth}
		\centering
		\begin{tikzpicture}
		
		\tikzstyle{input neuron}=[circle, draw, minimum size=1cm]
		\tikzstyle{hidden neuron}=[circle, draw, minimum size=1cm]
		\tikzstyle{hidden neuron color}=[circle, draw, fill=gray!50, minimum size=1cm]
		\tikzstyle{output neuron}=[circle, draw, minimum size=1cm]
		\tikzstyle{annot} = [text width=4em, text centered]
		
		\def\numInputs{2}
		\def\numHidden{3}
		\def\numOutputs{1}
		
		\foreach \i in {1,...,\numInputs}
		\node[input neuron] (I-\i) at (0, -\i) {};

		\foreach \i in {1}
		\node[hidden neuron] (H-\i) at (2, -\i + 0.5) {};
		\foreach \i in {2,...,\numHidden}
		\node[hidden neuron color] (H-\i) at (2, -\i + 0.5) {};
		
		\node[output neuron] (O-1) at (4, -1.5) {};
		
		\foreach \i in {1,...,\numInputs}
		\foreach \j in {1}
		\draw[->] (I-\i) -- (H-\j);
		\foreach \i in {1,...,\numInputs}
		\foreach \j in {2,...,\numHidden}
		\draw[->,dashed] (I-\i) -- (H-\j);
		
		\foreach \i in {1}
		\draw[->] (H-\i) -- (O-1);
		\foreach \i in {2,...,\numHidden}
		\draw[->,dashed] (H-\i) -- (O-1);
		
		\node[annot, above of=I-1, node distance=1cm] {Input layer};
		\node[annot, above of=H-1, node distance=1cm] {Hidden layer};
		\node[annot, above of=O-1, node distance=1cm] {Output layer};
		
		\end{tikzpicture}
		\caption{Single hidden layer}
		\label{fig:neural1}
	\end{subfigure}
	\begin{subfigure}[b]{0.45\textwidth}
		\centering
		\begin{tikzpicture}
		
		\tikzstyle{input neuron}=[circle, draw, minimum size=1cm]
		\tikzstyle{hidden neuron}=[circle, draw, minimum size=1cm]
		\tikzstyle{hidden neuron color}=[circle, draw, fill=gray!50, minimum size=1cm]
		\tikzstyle{output neuron}=[circle, draw, minimum size=1cm]
		\tikzstyle{annot} = [text width=4em, text centered]
		
		\def\numInputs{2}
		\def\numHiddenA{3}
		\def\numHiddenB{3}
		\def\numOutputs{1}
		
		\foreach \i in {1,...,\numInputs}
		\node[input neuron] (I-\i) at (0, -\i) {};
		
		\foreach \i in {1}
		\node[hidden neuron] (H1-\i) at (2, -\i + 0.5) {};
		\foreach \i in {2,...,\numHiddenA}
		\node[hidden neuron color]  (H1-\i) at (2, -\i + 0.5) {};
		
		\foreach \i in {1}
		\node[hidden neuron] (H2-\i) at (4, -\i + 0.5) {};
		\foreach \i in {2,...,\numHiddenB}
		\node[hidden neuron color] (H2-\i) at (4, -\i + 0.5) {};
		
		\node[output neuron] (O-1) at (6, -1.5) {};
		
		\foreach \i in {1,...,\numInputs}
		\foreach \j in {1}
		\draw[->] (I-\i) -- (H1-\j);
		\foreach \i in {1,...,\numInputs}
		\foreach \j in {2,...,\numHiddenA}
		\draw[->,dashed] (I-\i) -- (H1-\j);
		
		\foreach \i in {2,...,\numHiddenA}
		\foreach \j in {1,...,\numHiddenB}
		\draw[->,dashed] (H1-\i) -- (H2-\j);
		\foreach \i in {1}
		\foreach \j in {2,...,\numHiddenB}
		\draw[->,dashed] (H1-\i) -- (H2-\j);
		\foreach \i in {1}
		\foreach \j in {1}
		\draw[->] (H1-\i) -- (H2-\j);
		
		\foreach \i in {2,...,\numHiddenB}
		\draw[->,dashed] (H2-\i) -- (O-1);
		\foreach \i in {1}
		\draw[->] (H2-\i) -- (O-1);
		
		\node[annot, above of=I-1, node distance=1cm] {Input layer};
		\node[annot, above of=H1-1, node distance=1cm] {Hidden layer 1};
		\node[annot, above of=H2-1, node distance=1cm] {Hidden layer 2};
		\node[annot, above of=O-1, node distance=1cm] {Output layer};
		
		\end{tikzpicture}
		\caption{Two hidden layer}
		\label{fig:neural2}
	\end{subfigure}
	\caption{The weights corresponding to the dashed arrows, or equivalently, all the incoming and outgoing weights of the hidden neurons in gray, form a zero-preserving subset. }
	\label{fig:zp_set}
\end{figure}
Now, recall that our goal in this section is to show that, for feed-forward homogeneous neural networks trained via gradient flow with small initialization, the sparsity structure emerging among the weights just before escaping the origin is preserved post-escape. Briefly, our approach to achieve this goal is as follows: From \Cref{main_thm_2hm} and \Cref{main_thm_Lhm}, we know that for all sufficiently small initializations, gradient flow escapes the origin along the same path as gradient flow with small initial weights and initial direction along a KKT point of the constrained NCF. Now, in \citet{early_dc}, KKT points of the constrained NCF for feed-forward homogeneous networks are studied, and for these KKT points, it can be shown that all the rows and columns of the hidden weights with zero magnitude form a zero-preserving subset. Combining this fact with the result of \Cref{main_thm_2hm} and \Cref{main_thm_Lhm}, we will show that if the initial weights lie in a stable set of a positive KKT point, then the sparsity structure will approximately be preserved after gradient flow escapes the origin. We next discuss this approach in greater detail. \\\\
We start with the following lemma from \citet[Lemma 5]{early_dc}. For feed-forward homogeneous networks, it highlights a key property of KKT points of the constrained NCF.
\begin{lemma}\label{bal_r1_kkt}
	Let $\mathcal{H}$ be an $L-$layer feed-forward neural network as in \cref{L_layer_feed}, where $\sigma(x) = \max(x,\alpha x)^p$, for some $p\in\sN$ and $\alpha\in \sR$. Let $\left(\overline{\rmW}_{1},\cdots \overline{\rmW}_{L}\right)$be a positive KKT point of 
	\begin{align}
	\max_{\rmW_1,\cdots,\rmW_L} \mathcal{N}(\rmW_1,\cdots,\rmW_L)\coloneqq\widetilde{\rvy}^\top\mathcal{H}(\rmX;\rmW_1,\cdots,\rmW_L),  \text{ s.t. } \sum_{i=1}^L\|\rmW_i\|_F^2 = 1.
	\label{ncf_const_ff}
	\end{align}
	Then,
	\begin{equation*}
	\text{diag}(\overline{\rmW}_l\overline{\rmW}_l^\top) = p\cdot\text{diag}(\overline{\rmW}_{l+1}^\top\overline{\rmW}_{l+1}), \text{ for all }l\in [L-1],
	\end{equation*}
	or equivalently, $\|\overline{\rmW}_l[j,:]\|_2^2 = p\|\overline{\rmW}_{l+1}[:,j]\|_2^2$, for all $j\in [k_l]$ and $l\in [L-1]$.
\end{lemma}
In \citet{early_dc}, authors empirically observed that, for feed-forward homogeneous neural networks, the weights tend to converge towards KKT points of the constrained NCF that have many rows and columns of the hidden weights with zero norm, during the early stages of training. In fact, for $p\geq 2$ (polynomial Leaky ReLU), the hidden weights usually had only one row and column with non-zero norm.
   
Now, using these observations and the above lemma, we describe the zero-preserving subset emerging from the sparsity structure in positive KKT points, and its consequence on dynamics of gradient flow beyond the origin. Since  $\|\overline{\rmW}_l[j,:]\|_2^2 = p\|\overline{\rmW}_{l+1}[:,j]\|_2^2$, if \mbox{$\|\overline{\rmW}_l[j,:]\|_2$} is $0$, then $\|\overline{\rmW}_{l+1}[:,j]\|_2$ is $0$, and vice versa. Therefore, all the rows and columns of a KKT point with zero norm will form a zero-preserving subset. Consequently, for gradient flow with initial direction along such KKT points, by the definition of zero-preserving subset, all the rows and columns with zero norm at initialization will remain zero during training, implying the sparsity structure is preserved during training. 

Furthermore, from \Cref{main_thm_2hm} and \Cref{main_thm_Lhm}, we know that the trajectory of gradient flow with small initialization after escaping the origin is approximately equal to the gradient flow with small initial weights and initial direction along a KKT point of the NCF. Hence, for a feed-forward homogeneous neural network, if direction of the initial weights are in a stable set of a positive KKT point, then the sparsity structure will be approximately preserved after escaping the origin. We formalize this idea in the following theorem, where we use $\rmW_{1:L}^0\coloneqq(\rmW_1^0,\cdots,\rmW_L^0)$ and $\overline{\rmW}_{1:L} \coloneqq (\overline{\rmW}_1,\cdots,\overline{\rmW}_L)$, to denote the initial direction of the weights and the KKT point of the constrained NCF respectively.

 \begin{theorem}\label{main_thm_sp}
 	Let $\mathcal{H}$ be an $L-$layer feed-forward neural network as in \cref{L_layer_feed}, where $\sigma(x) = \max(x,\alpha x)^p$, for some $p\in\sN, p\geq 2$ and $\alpha\in \sR$, or $p=1, \alpha = 1$. Suppose $\rmW_{1:L}^0 \in \mathcal{S}(\overline{\rmW}_{1:L}) $, where $\overline{\rmW}_{1:L}$ is a $\Delta-$second-order positive KKT point of \cref{ncf_const_ff}. Define $\overline{\mathcal{N}} \coloneqq \mathcal{N}(\overline{\rmW}_{1:L})$, and let $\rvw_z$ be the following subset of the weights:
 	\begin{align*}
 	\rvw_z = \bigcup\limits_{l=1}^{L-1}\bigg(\left\{\rmW_l[j,:]: \left\|\overline{\rmW}_l[j,:]\right\|_2 = 0, j\in [k_l]\right\} \cup \left\{\rmW_{l+1}[:,j]: \left\|\overline{\rmW}_{l+1}[:,j]\right\|_2 = 0, j\in [k_l]\right\}\bigg).
 	\end{align*} 
 	\begin{itemize}
 		\item If $\mathcal{H}$ is two-homogeneous, then, for all sufficiently small $\delta>0$,  
 		\begin{align}
 		&\left\|\bm{\psi}\left(T_{\widetilde{\delta}}^1, \delta\rmW_{1:L}^0\right) \right\|_2 \leq b_1 {\delta}^{\frac{\Delta}{\Delta + 8\overline{\mathcal{N}}}}, \text{ and } \frac{\left\|\bm{\psi}_{\rvw_z}\left(T_{\widetilde{\delta}}^1, \delta\rmW_{1:L}^0\right) \right\|_2}{\left\|\bm{\psi}\left( T_{\widetilde{\delta}}^1, \delta\rmW_{1:L}^0\right) \right\|_2} \leq b_2 {\delta}^{\frac{2\Delta}{\Delta + 8\overline{\mathcal{N}}}},
 		\label{init_algn_2hm_ff}
 		\end{align}
 		where $b_1,b_2>0$ are some constants, $T_{\widetilde{\delta}}^1 = T_1 + {4\ln({1}/{\widetilde{\delta}})}/(\Delta+8\overline{\mathcal{N}})$ and $T_1,\widetilde{\delta}$ are the same as defined in Theorem \ref{main_thm_2hm}. Furthermore, for any fixed $\widetilde{T}>0$, there exists a $\widetilde{C}>0$ such that, for all sufficiently small $\delta>0$ and for all $t\in [-\widetilde{T},\widetilde{T}]$,  
 		\begin{align}
 		\left\|\bm{\psi}_{\rvw_z}\left(t+T_{\widetilde{\delta}}^2 , \delta\rmW_{1:L}^0\right) \right\|_2 \leq \widetilde{C}\delta^\frac{\Delta}{\Delta+8\overline{\mathcal{N}} },
 		\label{err_bd_2hm_ff}
 		\end{align} 
 		where $T_{\widetilde{\delta}}^2 = T_1 +\frac{\ln(1/b_\delta)}{2\overline{\mathcal{N}}}+ \frac{\ln({1}/{\widetilde{\delta}})}{2\overline{\mathcal{N}}}$, and  $b_\delta$ is the same as defined in Theorem \ref{main_thm_2hm}. 
 		
 		\item  If $\mathcal{H}$ is $L$-homogeneous, for some $L>2$, then, for all sufficiently small $\delta>0$,  
 		\begin{align}
 		\left\|\bm{\psi}\left(T_{\widetilde{\delta}}^1 , \delta\rmW_{1:L}^0\right) \right\|_2 \leq b_1 {\delta}^{\frac{\Delta}{\Delta + 2L^2\overline{\mathcal{N}}}}, \text{ and } \frac{\left\|\bm{\psi}_{\rvw_z}\left(T_{\widetilde{\delta}}^1, \delta\rmW_{1:L}^0\right) \right\|_2}{\left\|\bm{\psi}\left(T_{\widetilde{\delta}}^1, \delta\rmW_{1:L}^0\right) \right\|_2} \leq b_2 {\delta}^{\frac{L\Delta}{\Delta + 2L^2\overline{\mathcal{N}}}},
 		\label{init_algn_Lhm_ff}
 		\end{align}
 		where $b_1,b_2>0$ are some constants, $T_{\widetilde{\delta}}^1 = \frac{T_1}{\delta^{L-2}} + \frac{T}{\widetilde{\delta}^{L-2}}\left(1 - \widetilde{\delta}^{\frac{2(L-2)L^2\overline{\mathcal{N}}}{2L^2\overline{\mathcal{N}}+\Delta}}\right),$and $T_1,T,\widetilde{\delta}$ are the same as defined in Theorem \ref{main_thm_Lhm}. Furthermore, for any fixed $\widetilde{T}>0$, there exists a $\widetilde{C}>0$ such that, for all sufficiently small $\delta>0$ and for all $t\in [-\widetilde{T},\widetilde{T}]$,  
 		\begin{align}
 		&\left\|\bm{\psi}_{\rvw_z}\left(t+T_{\widetilde{\delta}}^2,\delta\rmW_{1:L}^0\right) \right\|_2 \leq \widetilde{C}\delta^{\frac{\Delta}{2L^2\overline{\mathcal{N}} +\Delta}},
 		\label{err_bd_Lhm_ff}
 		\end{align} 
 		where $T_{\widetilde{\delta}}^2 =  \frac{T_1}{\delta^{L-2}} + \left(\frac{1/b_\delta^{L-2}}{L(L-2)\overline{\mathcal{N}}} - T\right){\widetilde{\delta}^{\frac{-(L-2)\Delta}{2L^2\overline{\mathcal{N}}+\Delta}}}+\frac{T}{\widetilde{\delta}^{L-2}}$, and $b_\delta$ is the same as defined in Theorem \ref{main_thm_Lhm}. 
 	\end{itemize}
 \end{theorem}
In the above theorem, $\overline{\rmW}_{1:L}$ is a positive KKT point of the constrained NCF, and $\rvw_z$ contain all the rows and columns of $\overline{\rmW}_{1:L}$ that have zero norm,\footnote{More precisely, all the rows of $(\overline{\rmW}_1,\cdots,\overline{\rmW}_{L-1})$ and all the columns of $(\overline{\rmW}_2,\cdots,\overline{\rmW}_{L})$ with zero norm.} which is a zero-preserving subset, as discussed above. For two-homogeneous networks, \cref{init_algn_2hm_ff} implies that at time $T_{\widetilde{\delta}}^1$, the norm of the weights are small, but weights belonging to $\rvw_z$ have much smaller magnitude. Therefore, a sparsity structure emerges among the weights before gradient flow escapes the origin. We mainly use Lemma \ref{init_align_2hm} to get \cref{init_algn_2hm_ff} which, we recall, implies that during the early stages of training, the weights remain small and converge in direction towards $\overline{\rmW}_{1:L}$. Since $\overline{\rmW}_{1:L}$ has a sparsity structure, specifically, weights belonging to $\rvw_z$ have zero magnitude, therefore, the same sparsity structure would emerge in the weights of the neural network during the early stages of training.

From \Cref{main_thm_2hm}, we know that gradient flow escapes the origin after $T_{\widetilde{\delta}}^2$ time has elapsed. Thus, \cref{err_bd_2hm_ff} implies that weights belonging to $\rvw_z$ remain small even after gradient flow escapes the origin. Hence, the sparsity structure that emerges among the weights before escaping the origin is preserved post-escape as well. To get \cref{err_bd_2hm_ff}, we rely on \Cref{main_thm_2hm} and $\rvw_z$ being a zero-preserving subset. 

For deep homogeneous networks, similar results are provided in \cref{init_algn_Lhm_ff} and \cref{err_bd_Lhm_ff}. Additionally, using the above theorem and following the proofs of Corollaries \ref{cor_sq_loss_2hm}, \ref{cor_log_loss_2hm}, \ref{cor_sq_loss_Lhm}, and \ref{cor_log_loss_Lhm}, it can be shown that the sparsity structure is preserved until the first saddle point encountered by gradient flow after escaping the origin.
\subsection{Numerical Experiments}
We next conduct experiments to validate the above theorem, with results presented in Figures \ref{fig:3_layer_nn}, \ref{fig:2_layer_rl} and \ref{fig:3_layer_rl}. The setting for these experiments is similar to Figure \ref{fig:two_layer_nn}, the weights are trained using gradient descent with small initialization and until they escape the origin and reach the next saddle point. Each figure contains both the evolution of training loss and the plot of weights, before escaping the origin and after reaching the next saddle point. The code is available at \href{https://github.com/akk0135/beyond_origin_experiments}{\texttt{github.com/akk0135/beyond\_origin\_experiments}}.\\\\
\textbf{Square activation function. }In \Cref{fig:two_layer_nn}, we observed that the weights of a two-layer network with square activation became sparse before gradient descent escaped the origin, with only one row of $\rmW$ and $\rvv$ being non-zero. In accordance with \Cref{main_thm_sp}, this sparsity structure was preserved after escaping the origin as well. Similarly, \Cref{fig:3_layer_nn} shows that for a three-layer neural network with square activation, before gradient descent escapes the origin, a single row of $\rmW_1$, and a single entry of $\rmW_2$, $\rvv$ are non-zero, and this sparsity structure is preserved even after gradient descent escapes the origin and reaches the next saddle point. Notably, the rows and columns with zero norm form a zero-preserving subset.\\\\
\textbf{ReLU activation function. }Although our results exclude ReLU networks, we explore two- and three-layer ReLU networks in \Cref{fig:2_layer_rl} and \Cref{fig:3_layer_rl}. Note that, before escaping the origin, the weights become approximately sparse, but, compared to square activation, a greater number of rows and columns of the weights are non-zero. This aligns with observations from \cite{early_dc}, which noted that for ReLU networks, gradient descent tend to converge towards a KKT point of the constrained NCF with multiple non-zero rows and columns. Next, a close look at the weights seems to suggest that rows and columns of relatively small size stay small, even after the gradient descent escapes the origin and reaches the next saddle point. Also, the rows and columns with relatively small size seems to follow the definition of zero-preserving subset, that is, if the $j$-th row of $\rmW_l$ is small, then the $j$-th column of $\rmW_{l+1}$ is small, and vice-versa. For example, in the two-layer case, the first two rows of $\rmW$ and first two columns of $\rvv^\top$ stay small at iteration $i_1$ and $i_2$. In the three-layer case, rows 15-17 of $\rmW_1$ and columns 15-17 of  $\rmW_2$, and rows 2-5 of $\rmW_2$ and columns 2-5 of  $\rvv^\top$ stay small at iteration $i_1$ and $i_2$. 

These observations suggest that the sparsity structure is also preserved for ReLU networks, even though our results exclude ReLU activation. Additional experiments are provided in \Cref{sec:add_exp}. These observations can be useful for future work in this area.
\begin{figure}[htbp]
	\centering
	\begin{minipage}[c]{0.3\textwidth}
		\centering
		\includegraphics[width=\textwidth]{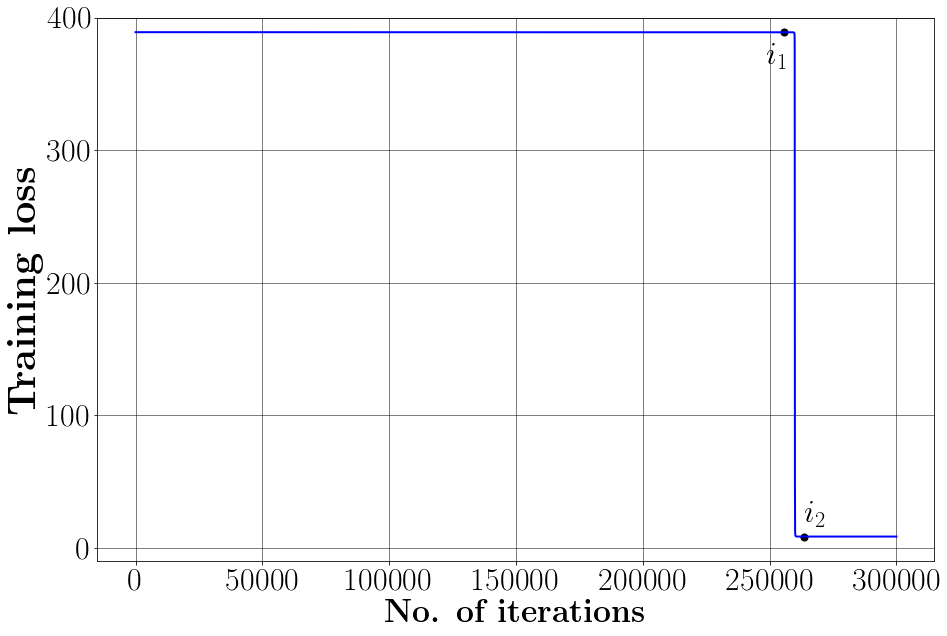}
		\\ (a) Evolution of training loss with iterations
	\end{minipage}\hspace{1cm}%
	\begin{minipage}[c]{0.6\textwidth}
		\centering
		\includegraphics[width=\textwidth]{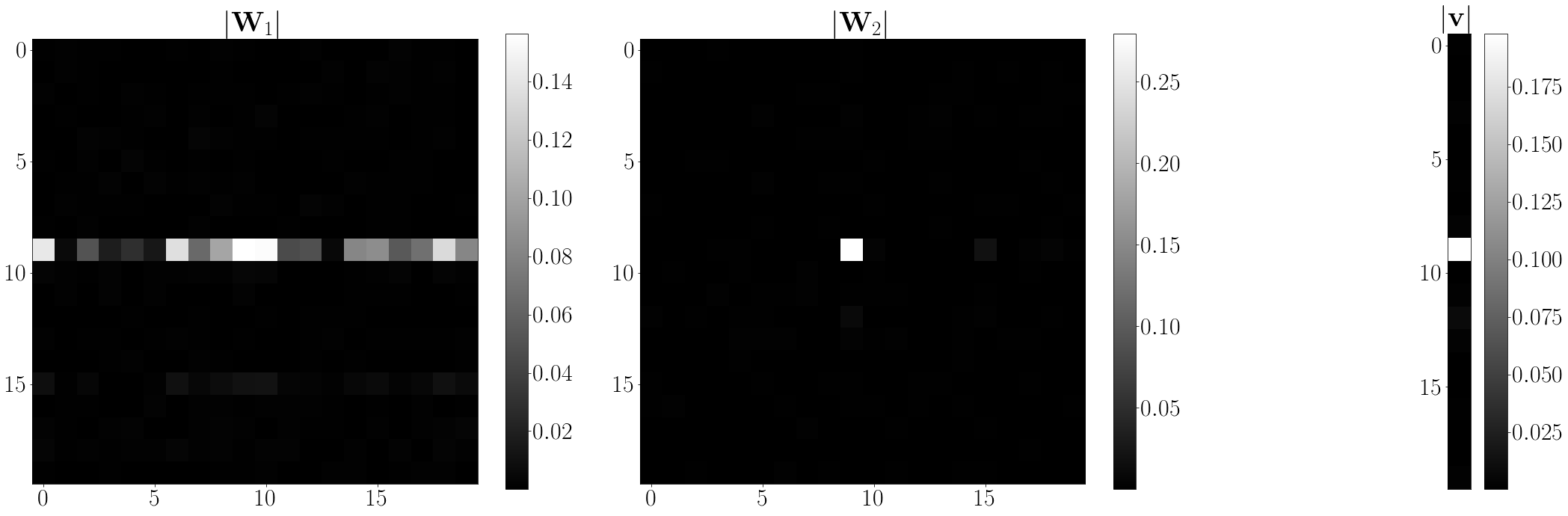}
		\\ (b) Weights at iteration $i_1$ \\[1ex]
		\includegraphics[width=\textwidth]{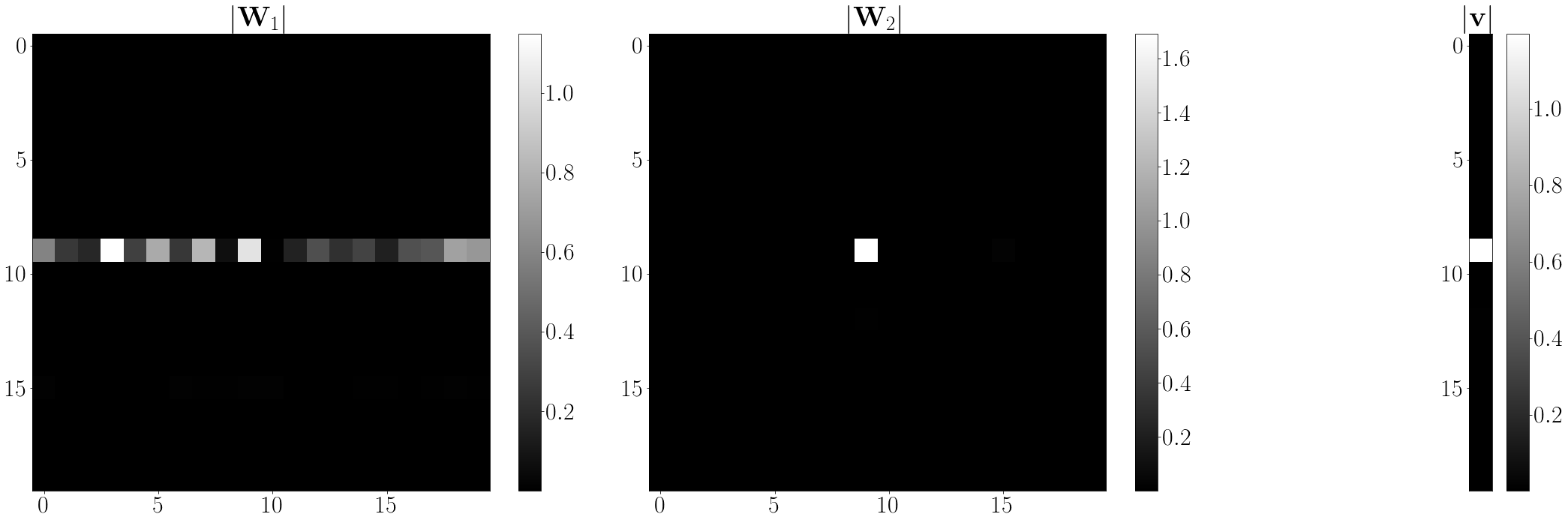}
		\\ (c) Weights at iteration $i_2$
	\end{minipage}
	\caption{We train a three-layer neural network whose output is $\rvv^\top\sigma(\rmW_2\sigma(\rmW_1\rvx)) ,$ where $\sigma(x) = x^2$ (\textbf{square activation}), and $\rvv\in \mathbb{R}^{20},\rmW_2,\rmW_1  \in \mathbb{R}^{20 \times 20}$ are the trainable weights. The sparsity structure is preserved upon escaping from the origin.}
	\label{fig:3_layer_nn}
\end{figure}
\begin{figure}[htbp]
	\centering
	\begin{minipage}[c]{0.3\textwidth}
		\centering
		\includegraphics[width=\textwidth]{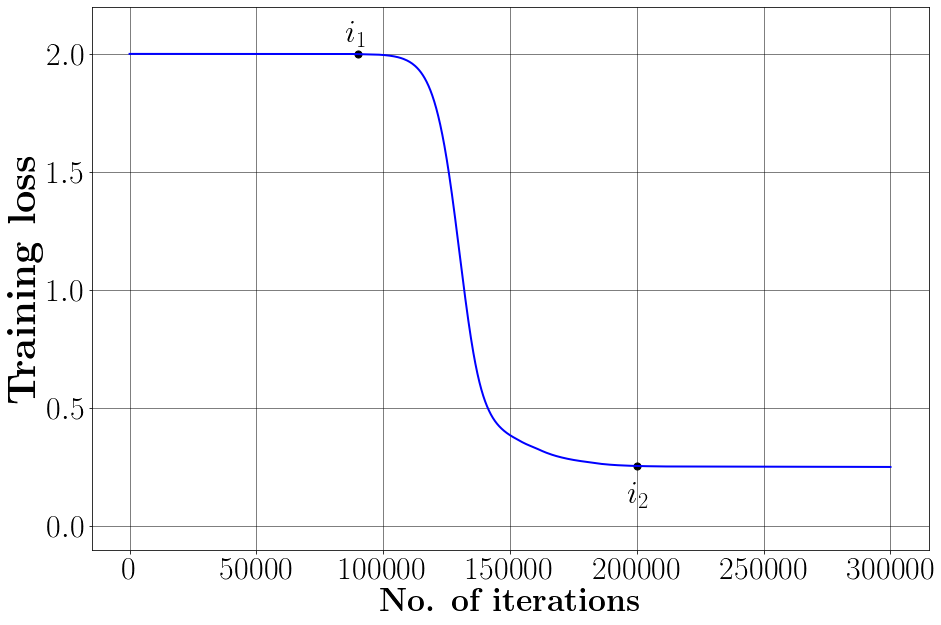}
		\\ (a) Evolution of training loss with iterations
	\end{minipage}
	\begin{minipage}[c]{0.6\textwidth}
		\centering
		\includegraphics[height = 0.35\textwidth, width=0.5\textwidth]{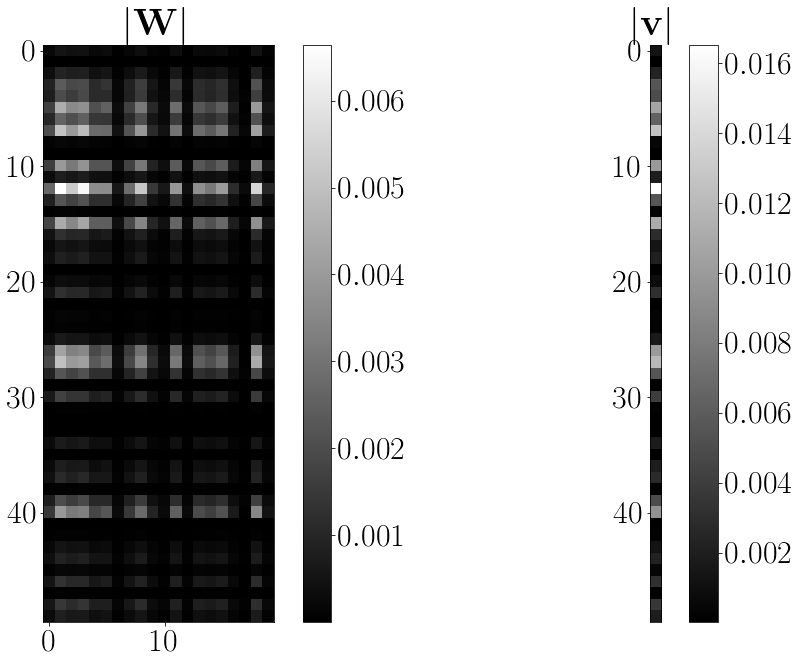}
		\\ (b) Weights at iteration $i_1$ \\[1ex]
		\includegraphics[height = 0.35\textwidth, width=0.5\textwidth]{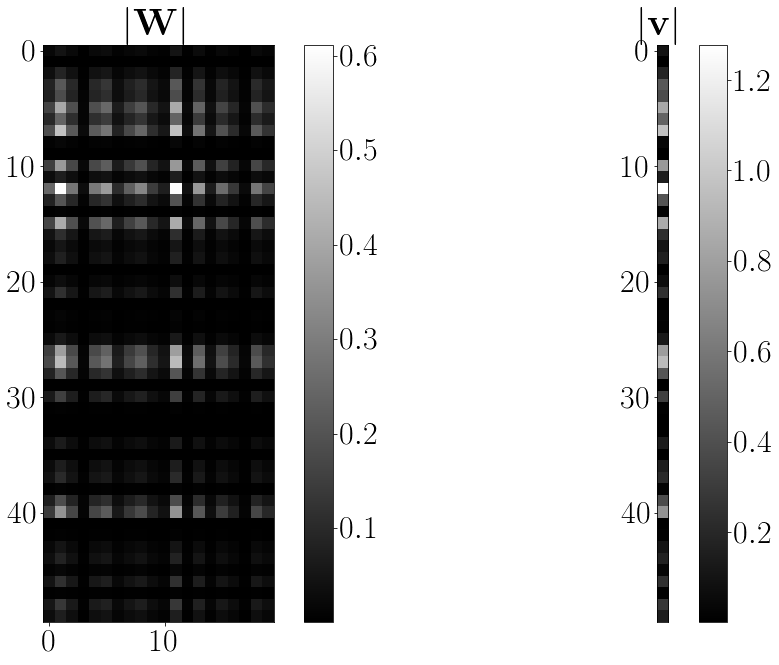}
		\\ (c) Weights at iteration $i_2$
	\end{minipage}
	\caption{We train a two-layer neural network whose output is $\rvv^\top\sigma(\rmW\rvx),$ where $\sigma(x) = \max(x,0)$ (\textbf{ReLU activation}), and  $\rvv\in \mathbb{R}^{50},\rmW  \in \mathbb{R}^{50 \times 20}$ are the trainable weights. As in the previous example, the sparsity structure is preserved upon escaping from the origin.}
	\label{fig:2_layer_rl}
\end{figure}

\begin{figure}[htbp]
	\centering
	\begin{minipage}[c]{0.3\textwidth}
		\centering
		\includegraphics[width=\textwidth]{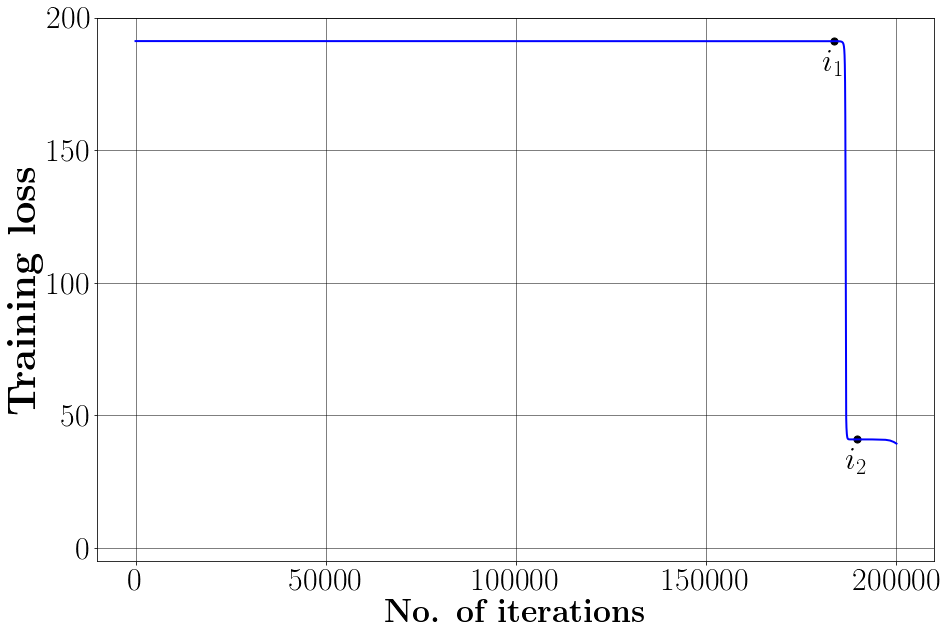}
		\\ (a) Evolution of training loss with iterations
	\end{minipage}\hspace{1cm}%
	\begin{minipage}[c]{0.6\textwidth}
		\centering
		\includegraphics[width=\textwidth]{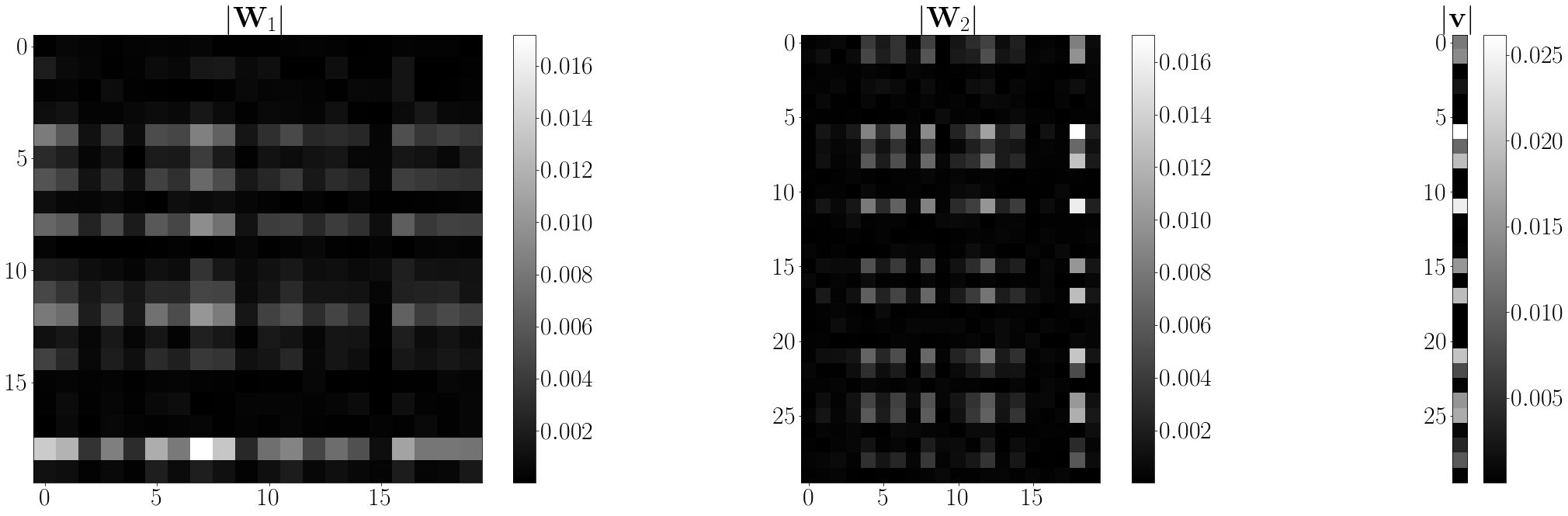}
		\\ (b) Weights at iteration $i_1$ \\[1ex]
		\includegraphics[width=\textwidth]{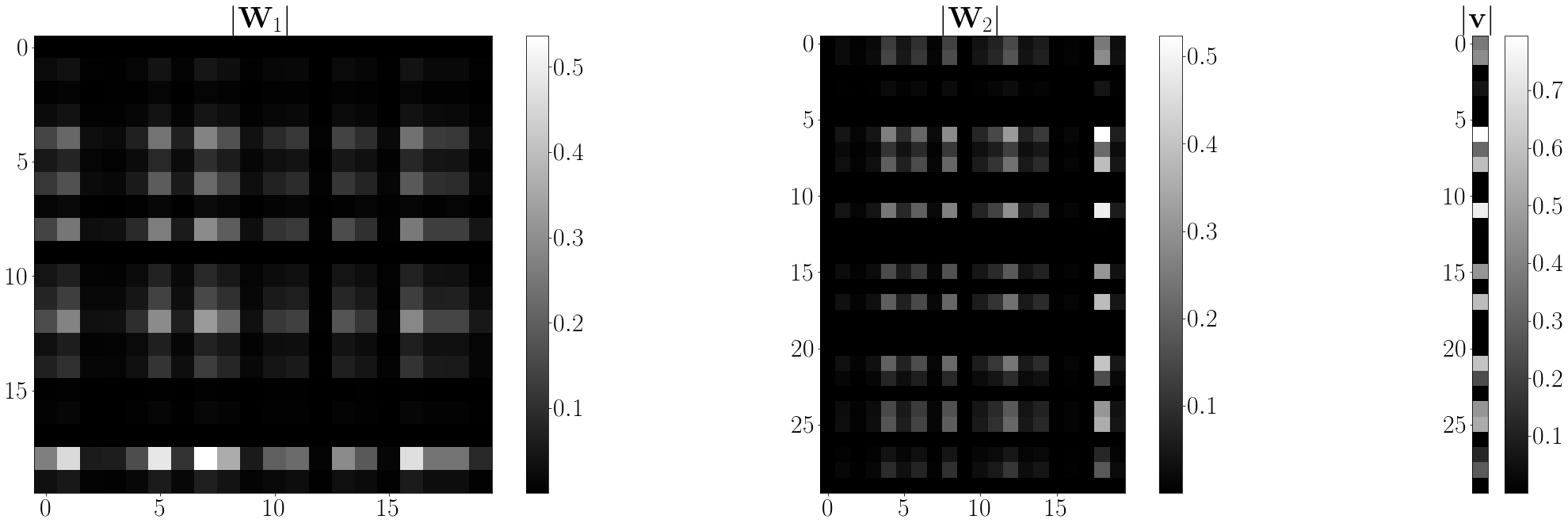}
		\\ (c) Weights at iteration $i_2$
	\end{minipage}
	\caption{We train a three-layer neural network whose output is $\rvv^\top\sigma(\rmW_2\sigma(\rmW_1\rvx)) ,$ where $\sigma(x) = \max(x,0)$ (\textbf{ReLU activation}), and $\rvv\in \mathbb{R}^{30},\rmW_2 \in \mathbb{R}^{30 \times 20} ,\rmW_1  \in \mathbb{R}^{20 \times 20}$ are the trainable weights.  This highlights the preservation of sparsity structure after escaping from the origin even for deeper ReLU networks.}
	\label{fig:3_layer_rl}
\end{figure}
\section{Conclusion and Future Directions}\label{sec_conc}
This paper studied the gradient flow dynamics of homogeneous neural networks which are trained with small initialization. We showed that for all sufficiently small initializations, the gradient flow escapes along the same path as the gradient flow with small initial weights and initial direction along a KKT point of the NCF. Next, we studied the gradient flow dynamics of feed-forward homogeneous neural networks and showed that the sparsity structure that emerges among the weights before escaping the origin is preserved post-escape.  

Our work describes a segment of the gradient flow dynamics beyond the origin, specifically up to the first saddle point encountered by gradient flow after escaping the origin. Understanding gradient flow dynamics beyond that saddle point is an important direction for future research; \Cref{sec:bey_saddle} provides some observations in this regard. Also, our results hold for neural networks with locally Lipschitz gradient, which excludes ReLU neural networks. Extending our results for such neural networks would be a valuable next step.
\acks{The authors graciously acknowledge resource support from the Minnesota Supercomputing Institute (MSI), and financial support in the form of gift funding from InterDigital.}\\
\newpage

\appendix
\noindent \textbf{Organization of the Appendix: } Appendix \ref{appendix_key} contains key lemmata useful to prove our main results. The proofs omitted from Section \ref{sec_2hm}, \ref{sec_Lhm} and \ref{sec_implications} are in Appendix \ref{appendix_2hm},  \ref{appendix_Lhm} and \ref{appendix_imp}, respectively. Appendix \ref{additonal_results} contains additional results and discussion, while \Cref{sec:add_exp} reports further experiments. Finally, \Cref{sec:bey_saddle} presents empirical observations on the training dynamics after escaping the first saddle point.
\section{Key Lemmata}\label{appendix_key}
The following lemma, also known as Euler's theorem, states two important properties of homogeneous functions \citep[Theorem B.2]{Lyu_ib}.
\begin{lemma}\label{euler_thm}
	Let $F:\sR^d\rightarrow \sR$ be locally Lipschitz, differentiable and L-positively homogeneous for some $L>0$. Then,
	\begin{itemize}
		\item For any $\rvx\in\sR^d$ and $c\geq 0$, $\nabla F(c\rvx) = c^{L-1}F(\rvx)$.
		\item For any $\rvx\in\sR^d$ , $\rvx^\top\nabla F(\rvx)= LF(\rvx).$
	\end{itemize}
\end{lemma}
Some properties of the second-order KKT points of the constrained NCF are derived next. 
\begin{lemma}\label{second_kkt}
	Suppose $\mathcal{H}$ is $L-$homogeneous, where $L\geq 2$.  Let $\rvw_*$ be a $\Delta-$second-order positive KKT point of the constrained NCF in \cref{ncf_gn_const}. Then,
	\begin{equation}
	\nabla\mathcal{N}(\rvw_*) = L\mathcal{N}(\rvw_*)\rvw_*, L\mathcal{N}(\rvw_*)-\rvb^\top\nabla^2\mathcal{N}(\rvw_*)\rvb \geq \Delta, \forall \rvb \in \rvw_*^\perp, \text{ and }
	\label{1_kkt}
	\end{equation}
	\begin{equation}
	\|\nabla^2\mathcal{N}(\rvw_*)\|_2\leq L(L-1)\mathcal{N}(\rvw_*).
	\label{kkt_op_norm}
	\end{equation}
\end{lemma}
\begin{proof}
	For \cref{ncf_gn_const}, the Lagrangian is
	\begin{equation}
	L(\rvw, \lambda) = \mathcal{N}(\rvw) + \lambda(1 - \|\rvw\|_2^2).
	\end{equation}
	Since $\rvw_*$ is a second-order KKT point, $\|\rvw_*\|_2^2 = 1$, 
	\begin{equation}
	\mathbf{0} = \nabla_\rvw L(\rvw_*, \lambda_*) = \nabla\mathcal{N}(\rvw_*) - 2\lambda_*\rvw_*,
	\label{1_kkt_pf}
	\end{equation}
	and
	\begin{equation}
	\Delta\leq -\rvb^\top\nabla^2_\rvw L(\rvw_*, \lambda_*)\rvb =2\lambda_* -  \rvb^\top\nabla^2 \mathcal{N}(\rvw_*)\rvb, \forall \rvb \in \rvw_*^\perp.
	\label{2_kkt_pf}
	\end{equation}
	Multiplying \cref{1_kkt_pf} by $\rvw_*^\top$ from the left, and using Lemma \ref{euler_thm} and $\|\rvw_*\|_2^2 = 1$, we get that
	\begin{equation*}
	\mathbf{0} = L\mathcal{N}(\rvw_*) - 2\lambda_*.
	\end{equation*}
	Using the above value of $\lambda_*$ in \cref{1_kkt_pf} and \cref{2_kkt_pf}  gives us \cref{1_kkt}. Next, since $\nabla \mathcal{N}(\rvw)$ is $(L-1)$-homogeneous, from Lemma \ref{euler_thm}, we have
	\begin{align*}
	\nabla^2 \mathcal{N}(\rvw_*)\rvw_* = (L-1)\nabla \mathcal{N}(\rvw_*).
	\end{align*}
	Since $\nabla \mathcal{N}(\rvw_*) = L\mathcal{N}(\rvw_*)\rvw_*$, the above equation implies that $\rvw_*$ is an eigenvector of $\nabla^2 \mathcal{N}(\rvw_*)$ with eigenvalue $L(L-1)\mathcal{N}(\rvw_*).$ Now, since $\nabla^2 \mathcal{N}(\rvw_*)$ is a symmetric matrix, other eigenvectors would be orthogonal to $\rvw_*$. However, for all $\rvb \in \rvw_*^\perp$, we know
	\begin{align*}
	\rvb^\top\nabla^2\mathcal{N}(\rvw_*)\rvb \leq L\mathcal{N}(\rvw_*) - \Delta \leq L(L-1)\mathcal{N}(\rvw_*).
	\end{align*}
	Hence, $\|\nabla^2 \mathcal{N}(\rvw_*)\|_2\leq L(L-1) \mathcal{N}(\rvw_*).$
\end{proof}
\subsection{Local Behavior at KKT point}
In this subsection, we derive some important inequalities satisfied by the weights near a $\Delta-$second-order KKT point of the constrained NCF. 
\begin{lemma}\label{approx_1}
	Suppose $\mathcal{H}$ is $L$-homogeneous, where $L\geq 2$. Let $\rvw_*$ be a $\Delta-$second-order positive KKT point of \cref{ncf_gn_const}, then for all sufficiently small $\epsilon>0$ and $\rvb\in \rvw_*^\perp$, we have
	\begin{align}
	&\mathcal{N}(\rvw_*+\epsilon \rvb) = \mathcal{N}(\rvw_*) + \frac{\epsilon^2}{2}\rvb^\top\nabla^2\mathcal{N}(\rvw_*)\rvb + o(\epsilon^2), \mbox{ and }\label{app_1}\\
	&\nabla\mathcal{N}(\rvw_*+\epsilon \rvb)^\top\rvw_* = L\mathcal{N}(\rvw_*) + \frac{(L-2)\epsilon^2}{2}\rvb^\top\nabla^2\mathcal{N}(\rvw_*)\rvb + o(\epsilon^2).\label{app_2}
	\end{align}
\end{lemma}
\begin{proof}
	Using Taylor's theorem, for all sufficiently small $\epsilon>0$, we have
	\begin{equation*}
	\mathcal{N}(\rvw_*+\epsilon \rvb) = \mathcal{N}(\rvw_*) + \epsilon\rvb^\top\nabla\mathcal{N}(\rvw_*) + \frac{\epsilon^2}{2}\rvb^\top\nabla^2\mathcal{N}(\rvw_*)\rvb + o(\epsilon^2).
	\end{equation*}
	From Lemma \ref{second_kkt}, we know $\nabla\mathcal{N}(\rvw_*)^\top\rvb = L\mathcal{N}(\rvw_*)\rvw_*^\top\rvb  = 0$. Using this fact in the above equation gives us \cref{app_1}. Next, note that
	\begin{align}
	\nabla\mathcal{N}(\rvw_*+\epsilon \rvb)^\top\rvw_* &= \nabla\mathcal{N}(\rvw_*+\epsilon \rvb)^\top(\rvw_*+\epsilon\rvb) - \epsilon\nabla\mathcal{N}(\rvw_*+\epsilon \rvb)^\top\rvb \nonumber\\
	&= L\mathcal{N}(\rvw_*+\epsilon \rvb) - \epsilon\nabla\mathcal{N}(\rvw_*+\epsilon \rvb)^\top\rvb,
	\label{gn_pf}
	\end{align}
	where the last equality follows from Lemma \ref{euler_thm}. Now, using \cref{app_1}, we have
	\begin{equation*}
	L\mathcal{N}(\rvw_*+\epsilon \rvb) = L\mathcal{N}(\rvw_*) + \frac{L\epsilon^2}{2}\rvb^\top\nabla^2\mathcal{N}(\rvw_*)\rvb + o(\epsilon^2).
	\end{equation*}
	From Taylor's theorem, for all sufficiently small $\epsilon>0$, we have
	\begin{equation*}
	\nabla\mathcal{N}(\rvw_*+\epsilon \rvb)^\top\rvb  = \nabla\mathcal{N}(\rvw_*)^\top\rvb + \epsilon\rvb^\top\nabla^2\mathcal{N}(\rvw_*)\rvb + o(\epsilon) = \epsilon\rvb^\top\nabla^2\mathcal{N}(\rvw_*)\rvb + o(\epsilon),
	\end{equation*}
	where we used $\nabla\mathcal{N}(\rvw_*)^\top\rvb  = 0$. Combining the above two equalities with \cref{gn_pf} gives us
	\begin{equation*}
	\nabla\mathcal{N}(\rvw_*+\epsilon \rvb)^\top\rvw_* = L\mathcal{N}(\rvw_*) + \frac{(L-2)\epsilon^2}{2}\rvb^\top\nabla^2\mathcal{N}(\rvw_*)\rvb + o(\epsilon^2),
	\end{equation*}
	which completes the proof.
\end{proof}
\begin{lemma}\label{inn_pd_ineq}
	Suppose $\mathcal{H}$ is $L$-homogeneous, where $L\geq 2$. Let $\rvw_*$ be a $\Delta-$second-order positive KKT point of \cref{ncf_gn_const}, then there exists a sufficiently small $\gamma>0$ such that for any unit-norm vector $\rvw$ satisfying $\rvw^\top\rvw_* \geq 1 - \gamma$, and for all $t_2\geq t_1\geq 0$, the following holds:
	\begin{equation}
	\left(\nabla \mathcal{N}(t_1\rvw) -  \nabla \mathcal{N}(t_2\rvw_*)\right)^\top(t_1\rvw - t_2\rvw_*) - L(L-1)\mathcal{N}(\rvw_*)t_2^{L-2}\|t_1\rvw - t_2\rvw_*\|_2^2\leq 0 .
	\label{inn_pd1}
	\end{equation}
\end{lemma}
\begin{proof}
	The proof is trivial if $t_1 = 0$, since in that case
	\begin{equation*}
	\left(\nabla \mathcal{N}(t_1\rvw) -  \nabla \mathcal{N}(t_2\rvw_*)\right)^\top(t_1\rvw - t_2\rvw_*) = \nabla \mathcal{N}(t_2\rvw_*)^\top( t_2\rvw_*) = t_2^LL\mathcal{N}(\rvw_*),
	\end{equation*}
	and $L(L-1)\mathcal{N}(\rvw_*)t_2^{L-2}\|t_1\rvw - t_2\rvw_*\|_2^2 = L(L-1)\mathcal{N}(\rvw_*)t_2^{L} \geq t_2^LL\mathcal{N}(\rvw_*).$
	Therefore, we assume $t_1 > 0$. Now, define $t  = t_2/t_1$, then \cref{inn_pd1} can be written as
	\begin{equation}
	\left(\nabla \mathcal{N}(\rvw) -  \nabla \mathcal{N}(t\rvw_*)\right)^\top(\rvw - t\rvw_*)  - L(L-1)\mathcal{N}(\rvw_*)t^{L-2}\|\rvw - t\rvw_*\|_2^2\leq 0.
	\label{inn_pd1_re}
	\end{equation}
	We will show that the above inequality is true for all $t\geq 1$ to prove the lemma. We will proceed by considering two cases: $L = 2$ and $L> 2$.\\
	\textbf{Case 1 ($L=2$):} Since $\rvw_*$ is a $\Delta$-second-order KKT point, from Lemma \ref{second_kkt} we know
	\begin{equation}
	\nabla\mathcal{N}(\rvw_*) = 2\mathcal{N}(\rvw_*)\rvw_*, \text{ and } \Delta+\rvb^\top\nabla^2\mathcal{N}(\rvw_*)\rvb \leq 2\mathcal{N}(\rvw_*), \forall \rvb \in \rvw_*^\perp.
	\label{kkt_cond_pf}
	\end{equation}
	For the sake of brevity, we define $\lambda_* =  2\mathcal{N}(\rvw_*)$ and $b_* = \rvb^\top\nabla^2\mathcal{N}(\rvw_*)\rvb $ here onward in this proof. Using Lemma \ref{euler_thm}, homogeneity and $\|\rvw\|_2 = 1$, the LHS of \cref{inn_pd1_re} can be written as
	\begin{align*}
	&2\mathcal{N}(\rvw) + 2t^2\mathcal{N}(\rvw_*) -  t\nabla\mathcal{N}(\rvw_*)^\top\rvw - t\nabla \mathcal{N}(\rvw)^\top\rvw_* - 2\mathcal{N}(\rvw_*)(1+t^2 - 2t\rvw^\top \rvw_*)\\
	& =  2\mathcal{N}(\rvw) +\lambda_*t^2 - \lambda_*t\rvw_*^\top\rvw - t\nabla \mathcal{N}(\rvw)^\top\rvw_* - \lambda_*(1+t^2 - 2t\rvw^\top \rvw_*)\\
	& =  2\mathcal{N}(\rvw) - \lambda_* - t\nabla \mathcal{N}(\rvw)^\top\rvw_* + \lambda_*t\rvw_*^\top\rvw,
	\end{align*}
	where the first equality uses \cref{kkt_cond_pf}, and simplifying the first equality leads to the second. Since $\rvw$ has unit norm, we define $\rvw = (\rvw_*+\epsilon\rvb)/\sqrt{1+\epsilon^2}$, where $\rvb\in \rvw_*^\perp$. Putting this choice of $\rvw$ in the above equation gives us
	\begin{equation*}
	\frac{2\mathcal{N}(\rvw_*+\epsilon\rvb)}{(1+\epsilon^2)} - \lambda_* - \frac{t\nabla \mathcal{N}(\rvw_*+\epsilon\rvb)^\top\rvw_*}{\sqrt{1+\epsilon^2}} + \frac{\lambda_*t\rvw_*^\top(\rvw_*+\epsilon\rvb)}{\sqrt{1+\epsilon^2}}.
	\end{equation*}
	To complete the proof, we will show that the above quantity is less than zero for all sufficiently small $\epsilon>0$. Using Lemma \ref{approx_1} and $\|\rvw_*\|_2 = 1$, the above quantity becomes
	\begin{equation*}
	\frac{2\mathcal{N}(\rvw_*)+\epsilon^2b_* + o(\epsilon^2)}{(1+\epsilon^2)} - \lambda_* - \frac{2t\mathcal{N}(\rvw_*) + o(\epsilon^2)}{\sqrt{1+\epsilon^2}} + \frac{\lambda_*t}{\sqrt{1+\epsilon^2}},
	\end{equation*}
	which can be further simplified to get
	\begin{equation*}
	\frac{\lambda_*+\epsilon^2b_*}{(1+\epsilon^2)} - \lambda_*+o(\epsilon^2).
	\end{equation*}
	Since $1/(1+\epsilon^2) = 1-\epsilon^2 + o(\epsilon^2)$, the above quantity can be written as  
	\begin{equation*}
	\lambda_*(1-\epsilon^2)+\epsilon^2b_*  - \lambda_* + o(\epsilon^2)  = \epsilon^2(b_*-\lambda_*) + o(\epsilon^2) \leq -\Delta\epsilon^2 + o(\epsilon^2).
	\end{equation*}
	Since $\Delta>0$, the proof is complete.\\
	\textbf{Case 2 ($L>2$):} Since $\rvw_*$ is a $\Delta-$second-order KKT point, from Lemma \ref{second_kkt} we have
	\begin{equation}
	\nabla\mathcal{N}(\rvw_*) = L\mathcal{N}(\rvw_*)\rvw_*, \text{ and } \Delta+\rvb^\top\nabla^2\mathcal{N}(\rvw_*)\rvb \leq L\mathcal{N}(\rvw_*), \forall \rvb \in \rvw_*^\perp.
	\label{L_kkt_cond_pf}
	\end{equation}
	For the sake of brevity, we define $\lambda_* =  L\mathcal{N}(\rvw_*)$ and  $b_* = \rvb^\top\nabla^2\mathcal{N}(\rvw_*)\rvb$ here onward in this proof. Using Lemma \ref{euler_thm}, homogeneity and $\|\rvw\|_2 = 1$, the LHS of \cref{inn_pd1_re} can be written as
	\begin{align*}
	&L\mathcal{N}(\rvw) + Lt^L\mathcal{N}(\rvw_*) -  t^{L-1}\nabla \mathcal{N}(\rvw_*)^\top\rvw - t\nabla \mathcal{N}(\rvw)^\top\rvw_*\\
	& - t^{L-2}L(L-1)\mathcal{N}(\rvw_*)(1+t^2 - 2t\rvw^\top \rvw_*)\nonumber\\
	& =  L\mathcal{N}(\rvw) +\lambda_*t^L - \lambda_*t^{L-1}\rvw_*^\top\rvw - t\nabla \mathcal{N}(\rvw)^\top\rvw_* - \lambda_*t^{L-2}(L-1)(1+t^2 - 2t\rvw^\top \rvw_*)\\
	& = L\mathcal{N}(\rvw) - \lambda_*((L-1)t^{L-2} + (L-2)t^L) - t\nabla \mathcal{N}(\rvw)^\top\rvw_* + \lambda_*(2L-3)t^{L-1}\rvw_*^\top\rvw,
	\end{align*}
	where the first equality uses \cref{L_kkt_cond_pf}, and simplifying the first equality leads to the second. Define $\rvw = (\rvw_*+\epsilon\rvb)/\sqrt{1+\epsilon^2}$, where $\rvb\in \rvw_*^\perp$. Putting this choice of $\rvw$ in the above equation gives us
	\begin{equation*}
	\frac{L\mathcal{N}(\rvw_*+\epsilon\rvb)}{(1+\epsilon^2)^{L/2}} - \lambda_*((L-1)t^{L-2} + (L-2)t^L) - \frac{t\nabla \mathcal{N}(\rvw_*+\epsilon\rvb)^\top\rvw_*}{({1+\epsilon^2})^{(L-1)/2}} + \frac{(2L-3)\lambda_*t^{L-1}}{\sqrt{1+\epsilon^2}}.
	\end{equation*}
	To complete the proof, we will show that the above quantity is less than zero for all sufficiently small $\epsilon>0$. Using Lemma \ref{approx_1} and $\|\rvw_*\|_2 = 1$, we further get
	\begin{equation*}
	\frac{\lambda_*+\frac{L\epsilon^2}{2}b_*+ o(\epsilon^2)}{(1+\epsilon^2)^{L/2}} - \lambda_*((L-1)t^{L-2} + (L-2)t^L) - \frac{t\lambda_* +\frac{t(L-2)\epsilon^2}{2}b_* + o(\epsilon^2)}{({1+\epsilon^2})^{(L-1)/2}} + \frac{\lambda_*(2L-3)t^{L-1}}{\sqrt{1+\epsilon^2}}.
	\end{equation*}
	Since $1/(1+\epsilon^2)^k = 1-k\epsilon^2 + o(\epsilon^2)$, the above quantity can be written as  
	\begin{align*}
	&\lambda_*\left(1-\frac{L\epsilon^2}{2}\right)+\frac{L\epsilon^2b_*}{2}  - \lambda_*((L-1)t^{L-2} + (L-2)t^L) -  t\lambda_*\left(1-\frac{(L-1)\epsilon^2}{2}\right)\\
	& - \frac{t(L-2)b_*\epsilon^2}{2} + \lambda_*(2L-3)t^{L-1}\left(1-\frac{\epsilon^2}{2}\right) + o(\epsilon^2)\\
	& = \lambda_*(1 - (L-1)t^{L-2} - (L-2)t^L-t + (2L-3)t^{L-1}) \\
	& + \frac{\epsilon^2}{2}(-L\lambda_* + Lb_*+t(L-1)\lambda_* - t(L-2)b_* -\lambda_*(2L-3)t^{L-1}) + o(\epsilon^2).
	\end{align*}
	Now, define $a(t) = (1+(L-2)t^{L-1}-(L-1)t^{L-2})$, then 
	\begin{equation*}
	(1-t)a(t) = 1 - (L-1)t^{L-2} - (L-2)t^L-t + (2L-3)t^{L-1}.
	\end{equation*}
	Also, note that $a(1) = 0$, and $a'(t) = (L-2)(L-1)(t^{L-2}-t^{L-3}) \geq 0,$ for all $t\geq 1.$ Hence, $a(t)\geq 0$ and an increasing function, for all $t\geq 1$.
	
	Next, let $b(t) = -L\lambda_* + Lb_*+t(L-1)\lambda_* - t(L-2)b_* -\lambda_*(2L-3)t^{L-1}$. Thus, our goal is to that show for all sufficiently small $\epsilon>0$,
	\begin{equation}
	(1-t)a(t) \lambda_*+ b(t)\epsilon^2/2 +o(\epsilon^2) \leq 0.
	\label{eq_ineq}
	\end{equation}
	We first consider the case when $t\in [1,L/(L-2)]$. In this case, since $a(t)\geq 0$, for all $t\geq 1$, we have $(1-t)a(t)\leq 0$. Next 
	\begin{align*}
	b(t) & =  -L\lambda_* + (L- t(L-2))b_*+t(L-1)\lambda_* -\lambda_*(2L-3)t^{L-1}\\
	&\leq  -L\lambda_* + (L-t(L-2))(\lambda_*-\Delta) +  t(L-1)\lambda_* -\lambda_*(2L-3)t^{L-1}\\
	&= -t(L-2)\lambda_* + t(L-1)\lambda_* - \lambda_*(2L-3)t^{L-1} - (L-t(L-2))\Delta\\
	&\leq t\lambda_* -  \lambda_*(2L-3)t^{L-1} \leq t\lambda_*(4-2L) \leq \lambda_*(4-2L) ,
	\end{align*}
	where the first inequality follows from \cref{L_kkt_cond_pf}, and the second uses $\Delta>0$. The last two inequalities are true since $t\geq 1.$ Therefore,  
	\begin{equation*}
	(1-t)a(t) \lambda_*+ b(t)\epsilon^2/2 +o(\epsilon^2) \leq \lambda_*(4-2L)\epsilon^2 + o(\epsilon^2).
	\end{equation*}
	Since $L>2$, the above inequality implies \cref{eq_ineq} is true for all sufficiently small $\epsilon>0$.
	
	We next consider $t\geq L/(L-2)$. Define $g(t) = (1-t)a(t) \lambda_*+ b(t)\epsilon^2/2$. From the above discussion, we know
	\begin{equation*}
	g(L/(L-2)) \leq  \lambda_*(4-2L)\epsilon^2 .
	\end{equation*}
	Now, if we can show that for all sufficiently small $\epsilon>0$, $g(t)\leq g(L/(L-2))$, for all $t\geq L/(L-2)$, then
	\begin{equation*}
	(1-t)a(t) \lambda_*+ b(t)\epsilon^2/2 +o(\epsilon^2) \leq \lambda_*(4-2L)\epsilon^2 + o(\epsilon^2), \text{ for all } t\geq L/(L-2).
	\end{equation*}
	Thus, \cref{eq_ineq} is true for all sufficiently small $\epsilon>0$. To complete the proof, we next show that $g'(t)\leq 0$, for all $t\geq L/(L-2)$ and all sufficiently small $\epsilon>0$. Note that
	\begin{equation*}
	g'(t) = (1-t)a'(t)\lambda_* - a(t) \lambda_*+ b'(t)\epsilon^2/2.
	\end{equation*}
	Since $a'(t)\geq 0$ and $t\geq 1$, $(1-t)a'(t)\lambda_* \leq 0$. Further, $a'(t)\geq 0$ implies $a(t)\geq a(L/(L-2))$, for all $t\geq L/(L-2)$. Since $a(L/L-2) = (1+(L/L-2)^{L-2}) \geq 1$, we have that
	\begin{align*}
	- a(t) \lambda_*+ b'(t)\epsilon^2/2 &= - a(t) \lambda_*+ ((L-1)\lambda_* - (L-2)b_* - (2L-3)(L-1)\lambda_*t^{L-2})\epsilon^2/2\\
	&\leq -\lambda_* - (L-2)b_*\epsilon^2/2 +(L-1)\lambda_*(1 - (2L-3)t^{L-2})\epsilon^2/2 \\
	&\leq -\lambda_* - (L-2)b_*\epsilon^2/2  \leq 0,
	\end{align*}
	where the first inequality uses $a(t)\geq a(L/(L-2))\geq 1$. The second inequality uses $L>2$ and $t\geq 1$. The last inequality is true for all sufficiently small $\epsilon>0$. Thus, $g'(t)\leq 0.$
\end{proof}
\begin{lemma}\label{lips_ineq}
	Suppose $\mathcal{H}$ is $L-$homogeneous, where $L\geq 2$, and $\rvw_*$ is a $\Delta-$second-order positive KKT point of \cref{ncf_gn_const}. Then there exists a sufficiently small $\gamma>0$ such that for all unit-norm vectors $\rvw$ satisfying $\rvw^\top\rvw_*\geq1-\gamma$ the following inequalities hold:
	\begin{align}
	&\rvw_*^\top\nabla\mathcal{N}(\rvw) - L\mathcal{N}(\rvw)\rvw_*^\top\rvw - \frac{\Delta}{2}\|\rvw - \rvw_*\|_2^2 \geq 0, \text{ and }\label{kl_ineq}\\
	&\mathcal{N}(\rvw) -  \mathcal{N}(\rvw_*) + \frac{\Delta}{4}\|\rvw - \rvw_*\|_2^2 \leq 0.\label{lips_ineq1}
	\end{align}
\end{lemma}
\begin{proof}
	Let $\rvw = (\rvw_*+\epsilon\rvb)/\sqrt{1+\epsilon^2}$, where $\rvb\in \rvw_*^\perp$. Putting this $\rvw$ in \cref{kl_ineq}  gives us
	\begin{equation}
	\frac{\rvw_*^\top\nabla\mathcal{N}(\rvw_*+\epsilon\rvb)}{(1+\epsilon^2)^{(L-1)/2}} - \frac{L\mathcal{N}(\rvw_*+\epsilon\rvb)\rvw_*^\top(\rvw_*+\epsilon\rvb)}{(1+\epsilon^2)^{(L+1)/2}} - \frac{\Delta}{2}(2-2\rvw^\top \rvw_*).
	\label{kl_pf}
	\end{equation}
	To complete the proof, we will show that the above quantity is greater than zero for all sufficiently small $\epsilon>0$. Note that, for all sufficiently small $\epsilon>0$,
	\begin{equation*}
	\rvw^\top \rvw_* = {1}/{\sqrt{1+\epsilon^2}} = 1-{\epsilon^2}/{2} + o(\epsilon^2),
	\end{equation*}
	and from Lemma \ref{approx_1}, we have
	\begin{align*}
	&\nabla\mathcal{N}(\rvw_*+\epsilon \rvb)^\top\rvw_* = L\mathcal{N}(\rvw_*) + \frac{(L-2)\epsilon^2}{2}\rvb^\top\nabla^2\mathcal{N}(\rvw_*)\rvb + o(\epsilon^2) \text{ and }\\
	&\mathcal{N}(\rvw_*+\epsilon \rvb) = \mathcal{N}(\rvw_*) + \frac{\epsilon^2}{2}\rvb^\top\nabla^2\mathcal{N}(\rvw_*)\rvb + o(\epsilon^2).
	\end{align*}
	Let  $\lambda_* =  L\mathcal{N}(\rvw_*)$, $b_* = \rvb^\top\nabla^2\mathcal{N}(\rvw_*)\rvb$, then \cref{kl_pf} can be simplified to get
	\begin{align*}
	&\lambda_*\left(1-\frac{(L-1)\epsilon^2}{2}\right) + \frac{(L-2)\epsilon^2}{2}b_* + o(\epsilon^2) - \lambda_*\left(1-\frac{(L+1)\epsilon^2}{2}\right) - \frac{L\epsilon^2}{2}b_*- \frac{\Delta\epsilon^2}{2}\\
	&= \frac{\epsilon^2}{2}\left(-(L-1)\lambda_*+ (L-2)b_* + (L+1)\lambda_* - Lb_* -\Delta \right) + o(\epsilon^2) \\
	&= \frac{\epsilon^2}{2}\left(  2\lambda_* - 2b_* -\Delta \right) + o(\epsilon^2) \geq \Delta\epsilon^2/2 + o(\epsilon^2).
	\end{align*}
	The last inequality uses $b_*\leq \lambda_*-\Delta$. Since $\Delta>0$, the above inequality proves \cref{kl_ineq}.\\
	Similarly, \cref{lips_ineq1} can be written as
	\begin{align*}
	&\frac{\mathcal{N}(\rvw_*+\epsilon \rvb)}{(1+\epsilon^2)^{L/2}} - \mathcal{N}(\rvw_*)+\frac{\Delta(2 -2\rvw^\top \rvw_* )}{4}\\
	&=  \mathcal{N}(\rvw_*)\left(1-\frac{L\epsilon^2}{2}\right) + \frac{\epsilon^2}{2}\rvb^\top\nabla^2\mathcal{N}(\rvw_*)\rvb - \mathcal{N}(\rvw_*) + \frac{\Delta\epsilon^2}{4} + o(\epsilon^2)\\
	&= \frac{\epsilon^2}{2}(b_* - L \mathcal{N}(\rvw_*) + \Delta/2) + o(\epsilon^2) \leq \frac{-\Delta\epsilon^2}{4} + o(\epsilon^2).
	\end{align*}
	The last inequality uses $b_*\leq \lambda_*-\Delta$. Since $\Delta>0$, the above inequality proves \cref{lips_ineq1}.
\end{proof}
\subsection{Gradient Flow Dynamics of NCF Near a KKT point}
In this subsection, we describe the gradient flow dynamics of NCF when initialized near a second-order positive KKT point.
\begin{lemma}\label{rate_ineq_2hm}
	Suppose $\mathcal{H}$ is $2-$homogeneous, and $\rvw_*$ is a $\Delta-$second-order positive KKT point of the constrained NCF in \cref{ncf_gn_const}. Let $\rvw(t)$ be the solution of 
	\begin{equation*}
	\dot{\rvw} = \nabla\mathcal{N}(\rvw), \rvw(0) = \rvw_0,
	\end{equation*}
	where $\|\rvw_0\|_2 = 1.$ There exists a sufficiently small $\gamma>0$ such that if $\rvw_*^\top\rvw_0 > 1-\gamma$, then
	\begin{align}
	\frac{\rvw_* ^\top\rvw(t)}{\|\rvw(t)\|_2} \geq 1 - e^{-t\Delta}\gamma, \forall t\geq 0, \text{ and }\rvw(t) = \rvg(t)e^{2t\mathcal{N}(\rvw_*)},
	\label{rt_2}
	\end{align}
	where $\|\rvg(t)\|_2 \in [\kappa_1,\kappa_2]$, for some $\kappa_2\geq\kappa_1>0$.
\end{lemma}
\begin{proof}
	We chose a sufficiently small $\gamma>0$ such that 
	\begin{equation}
	\mathcal{N}(\rvw_*)\geq \mathcal{N}(\rvw) > 0 \text{ and } \rvw_*^\top\nabla\mathcal{N}(\rvw) - 2\mathcal{N}(\rvw)\rvw_*^\top\rvw - \frac{\Delta}{2}\|\rvw - \rvw_*\|_2^2 \geq 0,
	\label{2hm_ineq_rt}
	\end{equation}
	for all unit-norm vector $\rvw$ satisfying $\rvw_*^\top\rvw > 1-\gamma$. The above conditions are possible since $\mathcal{N}(\rvw_*) > 0$ and due to Lemma \ref{lips_ineq}.
	
	Now, since $\rvw_0^\top\rvw_*>1-\gamma$, we have $\mathcal{N}(\rvw_0) > 0$. From \cite[Lemma C.4]{kumar_dc} , $\|\rvw(t)\|_2 \geq \|\rvw_0\|_2>0, $ for all $t\geq 0$. We next show that $\rvw_*^\top\rvw(t)/\|\rvw(t)\|_2>1-\gamma$, for all $t\geq 0$. For the sake of contradiction, suppose $\overline{T}>0$ be the smallest time such that $\rvw_*^\top\rvw(\overline{T})/\|\rvw(\overline{T})\|_2 = 1-\gamma.$ Then, for all $t\in [0,\overline{T})$, we have 
	\begin{align}
	\frac{d}{dt}\left(\frac{\rvw_* ^\top\rvw(t)}{\|\rvw(t)\|_2}\right) &=\rvw_*^\top\left(\rmI - \frac{\rvw\rvw^\top}{\|\rvw\|_2^2}\right)\frac{\nabla \mathcal{N}(\rvw) }{\|\rvw\|_2}\nonumber\\
	&= \rvw_*^\top\nabla \mathcal{N}(\rvw) /\|\rvw\|_2 - 2\rvw_*^\top\rvw\mathcal{N}(\rvw)/\|\rvw\|_2^3\nonumber\\
	&= \rvw_*^\top\nabla \mathcal{N}(\rvw/\|\rvw\|_2)  - 2\rvw_*^\top(\rvw/\|\rvw\|_2)\mathcal{N}(\rvw/\|\rvw\|_2)\nonumber\\
	& \geq  \frac{\Delta}{2}\left\|\frac{\rvw(t)}{\|\rvw(t)\|_2} - \rvw_*\right\|_2^2 = \Delta\left(1-\frac{\rvw_* ^\top\rvw(t)}{\|\rvw(t)\|_2}\right) \geq 0,
	\label{rt_ang}
	\end{align}
	where the first inequality follows from \cref{2hm_ineq_rt}. The above equation implies
	\begin{equation*}
	\left(\frac{\rvw_* ^\top\rvw(\overline{T})}{\|\rvw(\overline{T})\|_2}\right) \geq \left(\frac{\rvw_* ^\top\rvw(0)}{\|\rvw(0)\|_2}\right) > 1-\gamma,
	\end{equation*}
	which leads to a contradiction. Now, since  $\rvw_*^\top\rvw(t)/\|\rvw(t)\|_2>1-\gamma$, for all $t\geq 0$, from \cref{rt_ang} we have
	\begin{equation*}
	\left(1-\frac{\rvw_* ^\top\rvw(t)}{\|\rvw(t)\|_2}\right) \leq e^{-t\Delta}\left(1-\frac{\rvw_* ^\top\rvw(0)}{\|\rvw(0)\|_2}\right) \leq e^{-t\Delta}\gamma,
	\end{equation*}
	which proves the first part of \cref{rt_2}. Now, if $\rvw(t) = \rvg(t)e^{2t\mathcal{N}(\rvw_*)}$, then
	\begin{equation*}
	\nabla\mathcal{N}(\rvg) e^{2t\mathcal{N}(\rvw_*)} = \nabla\mathcal{N}(\rvw)  = \dot{\rvw} = \dot{\rvg}e^{2t\mathcal{N}(\rvw_*)}+2 \mathcal{N}(\rvw_*)\rvg(t)e^{2t\mathcal{N}(\rvw_*)},
	\end{equation*}
	which implies
	\begin{equation*}
	\dot{\rvg}+2 \mathcal{N}(\rvw_*)\rvg = \nabla\mathcal{N}(\rvg) .
	\end{equation*}
	Multiplying the above equation by $\rvg^\top$ from the left, we get
	\begin{equation*}
	\frac{1}{2}\frac{d\|\rvg\|_2^2}{dt}+2 \mathcal{N}(\rvw_*)\|\rvg\|_2^2 = 2\mathcal{N}(\rvg) = 2\mathcal{N}(\rvg/\|\rvg\|_2)\|\rvg\|_2^2 ,
	\end{equation*}
	which implies
	\begin{equation*}
	\|\rvg(t)\|_2^2 = \|\rvg(0)\|_2^2 e^{4\int_{0}^t \left(\mathcal{N}(\rvg(s)/\|\rvg(s)\|_2) - \mathcal{N}(\rvw_*)\right) ds}. 
	\end{equation*}
	Since $\mathcal{N}(\rvg(t)/\|\rvg(t)\|_2) = \mathcal{N}(\rvw(t)/\|\rvw(t)\|_2)  \leq \mathcal{N}(\rvw_*), $ for all $t\geq 0$, we have 
	\begin{equation*}
	\|\rvg(t)\|_2^2 \leq \|\rvg(0)\|_2^2 :=\kappa_2.
	\end{equation*}
	Next, since $\mathcal{N}(\rvw)$ is locally Lipschitz, there exists a $\mu>0$ such that, for all $t\geq0$,
	\begin{equation*}
	\mathcal{N}(\rvw_*) - \mathcal{N}(\rvg(t)/\|\rvg(t)\|_2) = |\mathcal{N}(\rvw_*) - \mathcal{N}(\rvg(t)/\|\rvg(t)\|_2)| \leq \mu\|\rvg(t)/\|\rvg(t)\|_2 - \rvw_*\|_2,
	\end{equation*}
	therefore, for $t\geq 0$, using the above equation and first part of \cref{rt_2}, we get
	\begin{align*}
	{\int_{0}^t \left(\mathcal{N}(\rvg(s)/\|\rvg(s)\|_2) - \mathcal{N}(\rvw_*)\right) ds} &\geq \int_{0}^t -\mu\|\rvg(s)/\|\rvg(s)\|_2 - \rvw_*\|_2 ds\\
	&= \int_{0}^t -\mu\sqrt{\|\rvw(s)/\|\rvw(s)\|_2 - \rvw_*\|_2^2} ds\\
	&= \int_{0}^t -\mu\sqrt{2-2\rvw_*^\top\rvw(s)/\|\rvw(s)\|_2} ds\\
	&\geq -\mu\sqrt{2\gamma}\int_{0}^t e^{-s\Delta/2} ds \geq -2\mu\sqrt{2\gamma}/\Delta.
	\end{align*}
	Hence, $ \|\rvg(t)\|_2^2 \geq \|\rvg(0)\|_2^2e^{-8\mu\sqrt{2\gamma}/\Delta} :=\kappa_1,	$ which completes the proof. 
\end{proof}

\begin{lemma}\label{rate_ineq}
	Suppose $\mathcal{H}$ is $L-$homogeneous, where $L>2$, and $\rvw_*$ is a $\Delta-$second-order positive KKT point of \cref{ncf_gn_const}. Let $\rvw(t)$ be the solution of 
	\begin{equation*}
	\dot{\rvw} = \nabla\mathcal{N}(\rvw), \rvw(0) = \rvw_0,
	\end{equation*}
	where $\|\rvw_0\|_2 = 1.$ There exists a sufficiently small $\gamma>0$ such that if $\rvw_*^\top\rvw_0 >1-\gamma$, then $\mathcal{N}(\rvw_0) > 0$ and there exists a finite $T>0$ such that
	\begin{align*}
	\lim_{t\rightarrow T}\|\rvw(t)\|_2 = \infty, \text{ where } T\in \left[\frac{1}{L(L-2)\mathcal{N}(\rvw_*)}, \frac{1}{L(L-2)\mathcal{N}(\rvw_0)}\right], \text{ and }
	\end{align*}
	\begin{align*}
	\frac{\rvw_*^\top\rvw(t)}{\|\rvw(t)\|_2} \geq 1 - \gamma (1-t/T)^{\frac{\Delta}{L(L-2)\mathcal{N}(\rvw_*)}}.
	\end{align*}
	Further, for all $t\in [0,T)$, $\|\rvw(t)\|_2$ is an increasing function and   
	\begin{align*}
	\rvw(t) = \frac{\rvg(t)}{(T-t)^{1/(L-2)}},
	\end{align*}
	where $\|\rvg(t)\|_2 $ is a decreasing function for all $t\in [0,T)$, 
	\begin{align*}
	\|\rvg(0)\|_2^{L-2} = T, \text{ and } \lim_{t\rightarrow T} \|\rvg(t)\|_2^{L-2} = {1}/{(L(L-2)\mathcal{N}(\rvw_*))}.
	\end{align*}
\end{lemma}
\begin{proof}
	We begin by choosing a sufficiently small $\gamma>0$ such that 
	\begin{equation}\label{Lhm_ineq_rt}
	\mathcal{N}(\rvw_*)\geq \mathcal{N}(\rvw) > 0, \text{ and } \rvw_*^\top\nabla\mathcal{N}(\rvw) - L\mathcal{N}(\rvw)\rvw_*^\top\rvw - \frac{\Delta}{2}\|\rvw - \rvw_*\|_2^2 \geq 0,
	\end{equation}
	for all unit norm vector $\rvw$ satisfying $\rvw_*^\top\rvw \geq 1-\gamma$. The above conditions are possible since $\mathcal{N}(\rvw_*) > 0$ and due to Lemma \ref{lips_ineq}.
	
	Now, since $\rvw_0^\top\rvw_*>1-\gamma$, we have $\mathcal{N}(\rvw_0) > 0$. From \cite[Lemma 13]{early_dc}, $\|\rvw(t)\|_2 \geq \|\rvw_0\|_2>0, $ for all $t\geq 0$. Next, we show that $\|\rvw(t)\|_2$ becomes unbounded at some finite time. For this, define
	\begin{equation*}
	\widetilde{\mathcal{N}}(\rvw) \coloneqq \mathcal{N}\left(\frac{\rvw}{\|\rvw\|_2}\right)= \frac{\mathcal{N}(\rvw)}{\|\rvw\|_2^L}, \text{ then }
	\nabla\widetilde{\mathcal{N}}(\rvw) = \left(\mathbf{I}-\frac{\rvw\rvw^\top}{\|\rvw\|_2^2}\right)\frac{\nabla  \mathcal{N}(\rvw)}{\|\rvw\|_2^L}.
	\end{equation*}
	Now, note that
	\begin{equation*}
	\frac{d\widetilde{\mathcal{N}}(\rvw)}{dt} = \dot{\rvw}^\top\nabla\widetilde{\mathcal{N}}(\rvw) = \nabla  \mathcal{N}(\rvw)\left(\mathbf{I}-\frac{\rvw\rvw^\top}{\|\rvw\|_2^2}\right)\frac{\nabla  \mathcal{N}(\rvw)}{\|\rvw\|_2^L} \geq 0,
	\end{equation*}
	which implies $\widetilde{\mathcal{N}}(\rvw(t))$ increases with time. Hence,
	\begin{equation}
	\frac{1}{2}\frac{d\|\rvw\|_2^2}{dt} = L\mathcal{N}(\rvw) = L\|\rvw\|_2^L\widetilde{\mathcal{N}}(\rvw) \geq L\|\rvw\|_2^L\mathcal{N}(\rvw_0)\geq 0 ,
	\label{Lhm_nm_bd}
	\end{equation}
	which implies $\|\rvw(t)\|_2$ is an increasing function, for all $t\geq 0$, and 
	\begin{equation*}
	\frac{d\|\rvw\|_2}{dt} \geq L\|\rvw\|_2^{L-1}\mathcal{N}(\rvw_0).
	\end{equation*}
	Taking $\|\rvw\|_2^{L-1}$ to the LHS and integrating from $0$ to $t$, we get
	\begin{equation*}
	\frac{1}{L-2}\left(\frac{1}{\|\rvw(0)\|_2^{L-2}} - \frac{1}{\|\rvw(t)\|_2^{L-2}}\right) \geq L\mathcal{N}(\rvw_0)t.
	\end{equation*}
	Using $\|\rvw(0)\|_2= 1$, and simplifying the above equation we get
	\begin{equation*}
	\|\rvw(t)\|_2^{L-2} \geq \frac{	1}{1- tL(L-2)\mathcal{N}(\rvw_0)}.
	\end{equation*}
	The above equation implies that for some $T\leq 1/(L(L-2)\mathcal{N}(\rvw_0))$, $\|\rvw(T)\|_2 = \infty.$ (We will show $T\geq 1/L(L-2)\mathcal{N}(\rvw_*)$ later.)
	
	We next derive related results for $\rvg(t)$. We know
	\begin{equation*}
	\frac{1}{2}\frac{d\|\rvw\|_2^2}{dt} = L\mathcal{N}(\rvw) = L\mathcal{N}(\rvw/\|\rvw\|_2)\|\rvw\|_2^L ,
	\end{equation*}
	which implies
	\begin{equation*}
	\frac{1}{\|\rvw\|_2^{L-1}}\frac{d\|\rvw\|_2}{dt} = L\mathcal{N}(\rvw/\|\rvw\|_2).
	\end{equation*}
	Integrating both sides from $0$ to $t\in (0,T)$, we get
	\begin{equation*}
	\frac{1}{L-2}\left(\frac{1}{\|\rvw(0)\|_2^{L-2}} - \frac{1}{\|\rvw(t)\|_2^{L-2}}\right) = \int_{0}^tL\mathcal{N}(\rvw(s)/\|\rvw(s)\|_2) ds.
	\end{equation*}
	Using $\|\rvw(0)\|_2 = 1$ and re-arranging the above equation gives us
	\begin{equation}
	1 - \int_{0}^tL(L-2)\mathcal{N}(\rvw(s)/\|\rvw(s)\|_2) ds = \frac{1}{\|\rvw(t)\|_2^{L-2}}.
	\label{nm_eq_L}
	\end{equation}
	Substituting $\rvw(t) = \rvg(t)/(T-t)^{1/(L-2)}$, we get
	\begin{equation*}
	\frac{1 - \int_{0}^tL(L-2)\mathcal{N}(\rvw(s)/\|\rvw(s)\|_2) ds}{(T-t)} = \frac{1}{\|\rvg(t)\|_2^{L-2}}.
	\end{equation*}
	Note that $\|\rvg(0)\|_2^{L-2} = T.$ Next, we compute $\|\rvg(T)\|_2$. In the LHS of the above equality, at $t=T$, the denominator is obviously $0$, and the numerator is also $0$ since, using \cref{nm_eq_L}, 
	\begin{equation*}
	1 - \int_{0}^TL(L-2)\mathcal{N}(\rvw(s)/\|\rvw(s)\|_2) ds = \lim_{t\rightarrow T}\frac{1}{\|\rvw(t)\|_2^{L-2}} = 0.
	\end{equation*}
	Therefore, using L'Hopital's rule,
	\begin{align*}
	\frac{1}{\|\rvg(T)\|_2^{L-2}} = \frac{\lim_{t\rightarrow T} -L(L-2)\mathcal{N}(\rvw(t)/\|\rvw(t)\|_2)}{-1} =  \lim_{t\rightarrow T} L(L-2)\mathcal{N}(\rvw(t)/\|\rvw(t)\|_2).
	\end{align*}
	We will soon show that $\lim_{t\rightarrow T}\rvw(t)/\|\rvw(t)\|_2 = \rvw_*$, which will complete the result for $\|\rvg(T)\|_2^{L-2}$. We now prove that $\|\rvg(t)\|_2$ is a decreasing function. For this, define
	\begin{equation*}
	h(t):= \frac{1 - \int_{0}^tL(L-2)\mathcal{N}(\rvw(s)/\|\rvw(s)\|_2) ds}{(T-t)} = \frac{1 - \int_{0}^tL(L-2)\widetilde{\mathcal{N}}(\rvw(s)) ds}{(T-t)}.
	\end{equation*} 
	We next show that $h'(t)\geq 0$, for all $t\in (0,T)$. This would imply $h(t)$ is increasing and therefore, $\|\rvg(t)\|_2$ is decreasing, for all for all $t\in (0,T)$. For all $t\in [0,T)$,
	\begin{align*}
	h'(t) &= \frac{-L(L-2)\widetilde{\mathcal{N}}(\rvw(t))(T-t) + (1 - \int_{0}^tL(L-2)\widetilde{\mathcal{N}}(\rvw(s))ds)}{(T-t)^2}\\
	& = \frac{1}{(T-t)}\left(-L(L-2)\widetilde{\mathcal{N}}(\rvw(t)) + h(t)\right).
	\end{align*}
	Note that 
	\begin{equation*}
	h(0) = 1/T \geq 0, \text{ and } h'(0) = \frac{1}{T}\left(-L(L-2)\mathcal{N}(\rvw_0) + \frac{1}{T}\right)  \geq 0,
	\end{equation*}
	where the last inequality follows since $T\leq 1/(L(L-2)\mathcal{N}(\rvw_0)).$ For the sake of contradiction, suppose there exists a $t_1\in (0,T)$ such that $h'(t_1) = -\epsilon$, for some $\epsilon>0$. Next, we compute $h''(t)$. Let $a(t)$ denote the numerator of $h'(t)$, then
	\begin{equation*}
	a'(t) = -L(L-2)\nabla\mathcal{N}(\rvw(t))^\top\left(\mathbf{I}-\frac{\rvw(t)\rvw(t)^\top}{\|\rvw(t)\|_2^2}\right)\frac{\nabla  \mathcal{N}(\rvw(t))}{\|\rvw\|_2^L} + h'(t).
	\end{equation*}
	Thus,
	\begin{align*}
	h''(t) &= \frac{a'(t)}{(T-t)} + \frac{a(t)}{(T-t)^2} = \frac{1}{T-t}\left(a'(t) + \frac{a(t)}{(T-t)}\right) = \frac{1}{T-t}\left(a'(t) + h'(t)\right)\\
	&= \frac{1}{(T-t)} \left(-L(L-2)\nabla\mathcal{N}(\rvw(t))^\top\left(\mathbf{I}-\frac{\rvw(t)\rvw(t)^\top}{\|\rvw(t)\|_2^2}\right)\frac{\nabla  \mathcal{N}(\rvw(t))}{\|\rvw\|_2^L} + 2h'(t)\right)
	\end{align*}
	Since $-L(L-2)\nabla\mathcal{N}(\rvw(t))^\top\left(\mathbf{I}-\frac{\rvw(t)\rvw(t)^\top}{\|\rvw(t)\|_2^2}\right)\frac{\nabla  \mathcal{N}(\rvw(t))}{\|\rvw\|_2^L} \leq 0$, the above equation implies
	\begin{equation*}
	h''(t) -\frac{2h'(t)}{(T-t)}\leq 0.
	\end{equation*}
	From Lemma \ref{gronwall_ineq}, for any $t\in (t_1,T)$, we get
	\begin{equation*}
	h'(t) \leq h'(t_1)/P(t), \text{ where } P(t) = 	e^{-\int_{t_1}^t 2/(T-s) ds} = (T-t)^2/(T-t_1)^2.
	\end{equation*}
	Hence, for any $t\in (t_1,T)$, 
	\begin{equation*}
	h'(t) \leq  {-\epsilon(T-t_1)^2}/(T-t)^2.
	\end{equation*}
	Integrating the above equation from $t_1$ to $t\in (t_1,T)$, we get
	\begin{equation*}
	h(t) - h(t_1) \leq -\epsilon(T-t_1)^2\left(\frac{1}{T-t} - \frac{1}{T-t_1}\right).
	\end{equation*}
	Now, if $t$ is chosen sufficiently close to $T$, then $h(t)$ becomes negative. This is  impossible since $h(t) = 1/\|\rvg(t)\|_2^{L-2}$, leading to a contradiction. Hence, $h'(t)\geq 0$, for all $t\in (0,T)$.
	
	We next show that $\rvw_*^\top\rvw(t)/\|\rvw(t)\|_2>1-\gamma$, for all $t\in [0,T)$. For the sake of contradiction, suppose $\overline{T}\in [0,T)$ is the smallest time such that $\rvw_*^\top\rvw(\overline{T})/\|\rvw(\overline{T})\|_2 = 1-\gamma.$ Then, for all $t\in [0,\overline{T})$,
	\begin{align}
	\frac{d}{dt}\frac{\rvw_*^\top\rvw(t)}{\|\rvw(t)\|_2}  &= \rvw_*^\top\left(\rmI - \frac{\rvw\rvw^\top}{\|\rvw\|_2^2}\right)\frac{\nabla \mathcal{N}(\rvw) }{\|\rvw\|_2}\nonumber\\
	&= \rvw_*^\top\nabla \mathcal{N}(\rvw) /\|\rvw\|_2 - L\rvw_*^\top\rvw\mathcal{N}(\rvw)/\|\rvw\|_2^3\nonumber\\
	&= \|\rvw\|_2^{L-2}\left(\rvw_*^\top\nabla \mathcal{N}(\rvw/\|\rvw\|_2)  - L\rvw_*^\top(\rvw/\|\rvw\|_2)\mathcal{N}(\rvw/\|\rvw\|_2)\right)\nonumber\\
	& \geq  \|\rvw\|_2^{L-2}\frac{\Delta}{2}\left\|\frac{\rvw(t)}{\|\rvw(t)\|_2} - \rvw_*\right\|_2^2 =  \|\rvw\|_2^{L-2}\Delta\left(1-\frac{\rvw_* ^\top\rvw(t)}{\|\rvw(t)\|_2}\right) \geq 0,\label{align_rt_Lhm}
	\end{align}
	where the first inequality follows from \cref{Lhm_ineq_rt}. The above equation implies
	\begin{equation*}
	\left(\frac{\rvw_* ^\top\rvw(\overline{T})}{\|\rvw(\overline{T})\|_2}\right) \geq \left(\frac{\rvw_* ^\top\rvw(0)}{\|\rvw(0)\|_2}\right) > 1-\gamma,
	\end{equation*}
	which leads to a contradiction. Now, from \cite[Lemma 2]{early_dc}, we know 
	\begin{equation*}
	\lim_{t\to T}{\rvw(t)}/{\|\rvw(t)\|_2} = \overline{\rvw},
	\end{equation*} 
	where $\overline{\rvw}$ is a first-order KKT point of  \cref{ncf_gn_const}. We next show that $\overline{\rvw} = \rvw_*$. Since $\overline{\rvw}$ is a first-order KKT point, from Lemma \ref{second_kkt}, we know $L\mathcal{N}(\overline{\rvw})\overline{\rvw} = \nabla\mathcal{N}(\overline{\rvw}).$ Also, since $\overline{\rvw}^\top\rvw^*\geq 1-\gamma$, from \cref{Lhm_ineq_rt}, we have
	\begin{equation*}
	0 \leq \rvw_*^\top\nabla\mathcal{N}(\overline{\rvw}) - L\mathcal{N}(\overline{\rvw})\rvw_*^\top\overline{\rvw} - \frac{\Delta}{2}\|\overline{\rvw} - \rvw_*\|_2^2 = - \frac{\Delta}{2}\|\overline{\rvw} - \rvw_*\|_2^2,
	\end{equation*}
	implying $\overline{\rvw} = \rvw_*$. Hence, $\lim\limits_{t\rightarrow T} \|\rvg(t)\|_2^{L-2} = {1}/{(L(L-2)\mathcal{N}(\rvw_*))} \geq \|\rvg(0)\|_2^{L-2} =1/T.$\\	
	Finally, since $\|\rvg(t)\|_2$ is a decreasing function, $\|\rvw(t)\|_2^{L-2} = \|\rvg(t)\|_2^{L-2}/(T-t) \geq \|\rvg(T)\|_2^{L-2}/(T-t) $. Let $\alpha:= \|\rvg(T)\|_2^{L-2}$, then, using \cref{align_rt_Lhm}, we have
	\begin{equation*}
	\frac{d}{dt}\left(1 - \frac{\rvw_*^\top\rvw(t)}{\|\rvw(t)\|_2}\right)  \leq  -\|\rvw\|_2^{L-2}\Delta\left(1-\frac{\rvw_* ^\top\rvw(t)}{\|\rvw(t)\|_2}\right) \leq -\frac{\alpha\Delta}{(T-t)}\left(1-\frac{\rvw_* ^\top\rvw(t)}{\|\rvw(t)\|_2}\right).
	\end{equation*} 
	Integrating the above equation from $0$ to $t\in (0,T)$, we get	
	\begin{equation*}
	\left(1 - \frac{\rvw_*^\top\rvw(t)}{\|\rvw(t)\|_2}\right) \leq \left(1 - \frac{\rvw_*^\top\rvw(0)}{\|\rvw(0)\|_2}\right)e^{-\int_0^t\frac{\alpha\Delta}{(T-s)} ds} \leq \gamma \left(1-\frac{t}{T}\right)^{\alpha\Delta},
	\end{equation*}
	which completes the proof.
\end{proof}
\begin{lemma}\label{err_bd_gf}
	Let $\rvp,\rvq\in \sR^d$. Then, for any fixed $\widetilde{T}>0$, there exists a $\widetilde{C}>0$  such that 
	\begin{align*}
	\|\bm{\psi}(t,\rvp)-\bm{\psi}(t,\rvq)\|_2 \leq \widetilde{C}\|\rvp-\rvq\|_2, \text{for all } t\in [-\widetilde{T},\widetilde{T}].
	\end{align*}
\end{lemma}
\begin{proof}
	Let $\rvu_1(t) = \bm{\psi}(t,\rvp)$ and $\rvu_2(t) = \bm{\psi}(t,\rvq)$. Then,
	\begin{equation*}
	\dot{\rvu}_1 = -\nabla\mathcal{L}(\rvu_1), \rvu_1(0) = \rvp, \dot{\rvu}_2 = -\nabla\mathcal{L}(\rvu_2), \rvu_2(0) = \rvq.
	\end{equation*}
	Since $\mathcal{L}(\cdot)$ has locally Lipshitz gradient, and $\widetilde{T}$ is fixed, we may assume that there exists a $\mu$ such that
	\begin{equation*}
	\|\nabla\mathcal{L}(\rvu_1(t)) - \nabla\mathcal{L}(\rvu_2(t))\|_2 \leq \mu\|\rvu_1(t)-\rvu_2(t)\|_2, \text{ for all } t\in[-\widetilde{T},\widetilde{T}].
	\end{equation*}
	Hence, for any $t\in [0,\widetilde{T}]$, we have
	\begin{equation*}
	\frac{1}{2}\frac{d}{dt}\|\rvu_1(t)-\rvu_2(t)\|_2^2 = -(\rvu_1-\rvu_2)^\top(\nabla\mathcal{L}(\rvu_1)-\nabla\mathcal{L}(\rvu_2)) \leq \mu\|\rvu_1(t)-\rvu_2(t)\|_2^2.
	\end{equation*}
	Integrating the above inequality from $0$ to $t\in [0,\widetilde{T}]$, we get
	\begin{equation}
		\|\rvu_1(t)-\rvu_2(t)\|_2\leq e^{\mu t}\|\rvu_1(0)-\rvu_2(0)\|_2 \leq e^{\mu \widetilde{T}}\|\rvu_1(0)-\rvu_2(0)\|_2.
		\label{pos_t_bd}
	\end{equation}
	Similarly, for any $t\in [-\widetilde{T},0]$, we have
	\begin{equation*}
	\frac{1}{2}\frac{d}{dt}\|\rvu_1(t)-\rvu_2(t)\|_2^2 = -(\rvu_1-\rvu_2)^\top(\nabla\mathcal{L}(\rvu_1)-\nabla\mathcal{L}(\rvu_2)) \geq -\mu\|\rvu_1(t)-\rvu_2(t)\|_2^2.
	\end{equation*}
	Integrating the above inequality from $t\in [-\widetilde{T},0]$ to $0$, we get
	\begin{equation}
	\|\rvu_1(t)-\rvu_2(t)\|_2\leq e^{-\mu t}\|\rvu_1(0)-\rvu_2(0)\|_2 \leq e^{\mu \widetilde{T}}\|\rvu_1(0)-\rvu_2(0)\|_2.
	\label{neg_t_bd}
	\end{equation}
	Combining \cref{pos_t_bd} and \cref{neg_t_bd} completes the proof.
\end{proof}	
\subsection{Auxiliary Lemmata}
In this subsection, we provide auxiliary lemmata which are crucial for the proofs.
\begin{lemma}\label{inc_dec}
	Suppose $\eta_1>\eta_2>0 $, then $f_1(t) = \frac{1-\eta_1t}{1-\eta_2t}$ and $f_2(t) = \frac{1}{1-\eta_1t}$ are monotonically decreasing and increasing functions, respectively, on $t\in [0,1/\eta_1)$.
\end{lemma}
\begin{proof}
	The claim follows since, for $t\in [0,1/\eta_1)$,
	\begin{equation*}
	f_1'(t) = \frac{-\eta_1(1-\eta_2t) + \eta_2(1-\eta_1t)}{(1-\eta_2t)^2} = \frac{\eta_2-\eta_1}{(1-\eta_2t)^2} \leq 0, \text{ and } f_2'(t) = \frac{\eta_1}{(1-\eta_1t)^2} \geq 0.
	\end{equation*}
\end{proof}
\begin{lemma}\label{int_factor}
	If $g(t) = b/(1-at)$, for some $a>0$. Then, for all $t\in [0,1/a)$, $e^{-\int_0^t g(s) ds} =  (1-at)^{b/a}.$
\end{lemma}
\begin{proof} Since $\int_0^t g(s) ds = {-b}\ln(1-at)/a,$ we have $	e^{-\int_0^t g(s) ds} = (1-at)^{b/a}.$
\end{proof}
\begin{lemma}\label{gronwall_ineq}
	If $x(t)$ satisfies $\dot{x} \leq g(t)x + h(t),$ then
	\begin{equation*}
	x(t) \leq \frac{1}{P(t)} \left(x(0) + \int_{0}^t P(s) h(s) ds \right), \text{ where } P(t) = 	e^{-\int_0^t g(s) ds}.
	\end{equation*}
\end{lemma}
\begin{proof}
	Multiplying by $e^{-\int_0^t g(s) ds}$ on both sides, we get
	\begin{equation*}
	e^{-\int_0^t g(s) ds} h(t) \geq e^{-\int_0^t g(s) ds} \dot{x} - e^{-\int_0^t g(s) ds}  g(t) x  = \frac{d}{dt}\left(e^{-\int_0^t g(s) ds} x\right).
	\end{equation*}
	Let $P(t) = e^{-\int_0^t g(s) ds} $, then integrating both sides from $0$ to $t$, we get
	\begin{equation*}
	\int_{0}^t P(s) h(s) ds \geq P(t) x(t) - x(0),
	\end{equation*}
	which can be rearranged to get the desired result.
\end{proof}
\section{Proofs Omitted from \Cref{sec_2hm}}\label{appendix_2hm}
This section contains the proof of Lemma \ref{lim_exists_2hm}, \ref{init_align_2hm} and \ref{init_dl3_2hm}, which were used to prove Theorem \ref{main_thm_2hm}. 
\begin{proof}\textbf{of Lemma \ref{init_dl3_2hm}:}
	We choose $\gamma>0$ sufficiently small such that for all unit-norm vector $\rvw$ satisfying $\rvw^\top\rvw_* \geq1-\gamma$, we have
	\begin{align}
	&\mathcal{N}(\rvw) \leq \mathcal{N}(\rvw_*), \|\nabla^2\mathcal{N}(\rvw)\|_2\leq 3\mathcal{N}(\rvw_*) \text{ and }
	\label{loc_ineq_2hm}\\
	&\left(\nabla \mathcal{N}(t_1\rvw) -  \nabla \mathcal{N}(t_2\rvw_*)\right)^\top\left( t_1\rvw - t_2\rvw_*\right) \leq 2\mathcal{N}(\rvw_*)\|t_1\rvw - t_2\rvw_*\|_2^2, \forall t_2\geq t_1\geq 0.
	\label{inn_pd_2hm}
	\end{align}
	The first inequality follows from Lemma \ref{lips_ineq}. The second inequality holds since, from Lemma \ref{second_kkt}, $\|\nabla^2\mathcal{N}(\rvw_*)\|_2 = 2\mathcal{N}(\rvw_*)$, and $\nabla^2\mathcal{N}(\rvw)$ is continuous in the neighborhood of $\rvw_*$. The third inequality follows from Lemma \ref{inn_pd_ineq}. 
	
	We further define $f(\rvw) = \mathcal{J}(\rmX;\rvw)^\top(\ell'(\mathcal{H}(\rmX;\rvw), \rvy) - \ell'(\mathbf{0}, \rvy))$. Then, note that
	\begin{equation*}
	\|f(\rvw)\|_2 \leq \| \mathcal{J}(\rmX;\rvw)\|_2\|\ell'(\mathcal{H}(\rmX;\rvw), \rvy) - \ell'(\mathbf{0}, \rvy)\|_2 \leq Kn\| \mathcal{J}(\rmX;\rvw)\|_2\|\mathcal{H}(\rmX;\rvw)\|_2 ,
	\end{equation*}
	where the first inequality follows from Cauchy-Schwartz inequality, and the second follows from the smoothness of the loss  function (see \Cref{ass_loss}). Note that the RHS in the above inequality is $3-$homogeneous in $\rvw$. Thus, there exists a $\beta>0$ such that
	\begin{equation}
	\|f(\rvw)\|_2 \leq  \beta\|\rvw\|_2^3, \text{ for all }\rvw\in \sR^k.
	\label{lips_res_2hm}
	\end{equation}
	Also, $ f(\rvw)$ is continuously differentiable in the neighborhood of $\rvw_*$. Thus, we may assume that for all vector $\rvw\in\sR^k$ that satisfy $\rvw_*^\top\rvw/\|\rvw\|_2 \geq 1-\gamma, $ we have
	\begin{align*}
	\|\nabla f(\rvw)\|_2 \leq \sum_{i=1}^n&\left\|\nabla^2\mathcal{H}(\rvx_i;\rvw) \right\|_2 \left|\ell'(\mathcal{H}(\rvx_i;\rvw), y_i) - \ell'({0}, y_i)\right| +\\
	& \left\|\nabla\mathcal{H}(\rvx_i;\rvw) \nabla\mathcal{H}(\rvx_i;\rvw) ^\top\right\|_2\left|\ell''(\mathcal{H}(\rvx_i;\rvw), y_i)\right| \\
	\leq K\sum_{i=1}^n&\|\nabla^2\mathcal{H}(\rvx_i;\rvw) \|_2 |\mathcal{H}(\rvx_i;\rvw) | + \|\nabla\mathcal{H}(\rvx_i;\rvw) \nabla\mathcal{H}(\rvx_i;\rvw) ^\top\|_2,
	\end{align*}
	where the second inequality uses smoothness of the loss function. Note that the final upper bound in the above equation is $2-$homogeneous. Thus, there exists a $\zeta>0$ such that for all vectors $\rvw\in\sR^k$ that satisfy $\rvw_*^\top\rvw/\|\rvw\|_2 \geq 1-\gamma, $ we have
	\begin{equation*}
	\|\nabla f(\rvw)\|_2 \leq \zeta\|\rvw\|_2^2.
	\end{equation*}
	Furthermore, from the mean value theorem, we have
	\begin{equation}
	\|f(\rvw_1) - f(\rvw_2)\|_2 \leq \|\nabla f(\tilde{\rvw})\|_2\|\rvw_1-\rvw_2\|_2 \leq \zeta\max(\|\rvw_2\|_2^2,\|\rvw_1\|_2^2)\|\rvw_1-\rvw_2\|_2,
	\label{lips_f_2hm}
	\end{equation}
	where $\rvw_*^\top\rvw_1/\|\rvw_1\|_2, \rvw_*^\top\rvw_2/\|\rvw_2\|_2 \geq 1-\gamma. $
	
	Now, let $\rvq(t) = \bm{\psi}\left(t,{\rva}_\delta\right)$, $\rvu(t) = \bm{\psi}\left(t,\delta\rvw_*\right)$, then $\|\rvq(0) - \rvu(0)\|_2=\left\|\rva_\delta- \delta\rvw_* \right\|_2 \leq C_1\delta^3,$ which implies $\|\rvq(0)\|_2 \leq \delta + C_1\delta^3$. Therefore, for all sufficiently small $\delta>0$, we have
	\begin{align*}
	\frac{\rvw_*^\top\rvq(0)}{\|\rvq(0)\|_2} &= \frac{\rvw_*^\top\rvu(0) + \rvw_*^\top(\rvq(0) - \rvu(0))}{\|\rvq(0)\|_2}\\
	& \geq \frac{\delta - \|\rvq(0) - \rvu(0)\|_2}{\|\rvq(0)\|_2}\geq \frac{\delta - C_1\delta^3}{\|\rvq_\delta(0)\|_2} \geq\frac{\delta - C_1\delta^3}{\delta + C_1\delta^3} > 1-\gamma.
	\end{align*}
	Define 
	\begin{equation*}
	\overline{T}_{1} = \min_{t\geq 0}\left\{t: \frac{\rvw_*^\top\rvq(t)}{\|\rvq(t)\|_2} = 1-\gamma\right\}.
	\end{equation*}
	Recall that $\rvq(t)$ satisfies
	\begin{equation*}
	\dot{\rvq} = -\mathcal{J}(\rmX;\rvq)^\top\ell'(\mathcal{H}(\rmX;\rvq), \rvy) = \nabla\mathcal{N}(\rvq)  - \mathcal{J}(\rmX;\rvq)^\top(\ell'(\mathcal{H}(\rmX;\rvq), \rvy) - \ell'(\mathbf{0}, \rvy)).
	\end{equation*}
	Multiplying the above equation by $\rvq^\top$ from the left we get
	\begin{equation*}
	\frac{1}{2}\frac{d\|\rvq\|_2^2}{dt} = 2\mathcal{N}(\rvq) - 2 \mathcal{H}(\rmX;\rvq)^\top(\ell'(\mathcal{H}(\rmX;\rvq), \rvy) - \ell'(\mathbf{0}, \rvy))
	\leq 2\mathcal{N}(\rvq) = 2\mathcal{N}(\rvq/\|\rvq\|_2)\|\rvq\|_2^2, 
	\end{equation*}
	where the first equality follows from Lemma \ref{euler_thm}, and the first inequality is due to convexity of the loss function. The second equality holds since $\mathcal{N}(\rvq)$ is two-homogeneous.
	
	Now, since $\rvw_*^\top\rvq(t)/\|\rvq(t)\|_2 \geq 1-\gamma$ for all $t\in [0,\overline{T}_1]$, from \cref{loc_ineq_2hm}, we get
	\begin{equation*}
	\frac{1}{2}\frac{d\|\rvq\|_2^2}{dt} \leq 2\mathcal{N}(\rvw_*)\|\rvq\|_2^2,
	\end{equation*}
	which implies
	\begin{equation}
	\|\rvq(t)\|_2 \leq \|\rvq(0)\|_2e^{2t\mathcal{N}(\rvw_*)}, \text{ for all } t\in [0,\overline{T}_1].
	\label{qt_bd_2hm}
	\end{equation}
	Let $\rvz(t) = e^{2t\mathcal{N}(\rvw_*)}\rvw_*$, which is the solution of
	\begin{equation*}
	\dot{\rvz} = \nabla \mathcal{N}(\rvz), \rvz(0) = \rvw_*.  
	\end{equation*}
	Note that, for $t\in [0,\overline{T}_1]$,
	\begin{align}
	&\left(\nabla \mathcal{N}\left(\frac{\rvq(t)}{\|\rvq(0)\|_2}\right) - \nabla\mathcal{N}(\rvz(t))\right)^\top\left(\frac{\rvq(t)}{\|\rvq(0)\|_2}- \rvz(t)\right)\nonumber\\
	= &\left(\nabla \mathcal{N}\left(\frac{\|\rvq(t)\|_2}{\|\rvq(0)\|_2}\frac{\rvq(t)}{\|\rvq(t)\|_2}\right) - \nabla\mathcal{N}\left({\rvz(t)}\right)\right)^\top\left(\frac{\|\rvq(t)\|_2}{\|\rvq(0)\|_2}\frac{\rvq(t)}{\|\rvq(t)\|_2}- {\rvz(t)}\right)\nonumber\\
	&\leq 2\mathcal{N}(\rvw_*)\left\|\frac{\|\rvq(t)\|_2}{\|\rvq(0)\|_2}\frac{\rvq(t)}{\|\rvq(t)\|_2}- \|\rvz(t)\|_2\frac{\rvz(t)}{\|\rvz(t)\|_2}\right\|^2 = 2\mathcal{N}(\rvw_*)\left\|\frac{\rvq(t)}{\|\rvq(0)\|_2}- {\rvz(t)}\right\|^2,
	\label{inn_pd_pf_2hm}
	\end{align}
	where the inequality follows from ${\|\rvq(t)\|_2}/{\|\rvq(0)\|_2}\leq e^{2t\mathcal{N}(\rvw_*)} = \|\rvz(t)\|_2$, $\rvw_*^\top\rvq(t)/\|\rvq(t)\|_2 \geq 1-\gamma$ for all $t\in [0,\overline{T}_1]$, and \cref{inn_pd_2hm}. Hence, for $t\in [0,\overline{T}_1]$,
	\begin{align*}
	&\frac{1}{2}\frac{d}{dt} \left\|\frac{\rvq(t)}{\|\rvq(0)\|_2} - \rvz(t) \right\|_2^2\\  
	&= \left(\frac{\rvq(t)}{\|\rvq(0)\|_2} - \rvz(t) \right)^\top\left(\frac{\dot{\rvq}}{\|\rvq(0)\|_2} - \dot{\rvz}\right)\\
	&= \left(\frac{\rvq(t)}{\|\rvq(0)\|_2} - \rvz(t) \right)^\top\left(\frac{\nabla\mathcal{N}({\rvq})}{\|\rvq(0)\|_2} - \nabla\mathcal{N}({\rvz})\right) - \left(\frac{\rvq(t)}{\|\rvq(0)\|_2} - \rvz(t) \right)^\top \frac{f(\rvq)}{\|\rvq(0)\|_2}\\
	&\leq 2\mathcal{N}(\rvw_*)\left\|\frac{\rvq(t)}{\|\rvq(0)\|_2}- {\rvz(t)}\right\|_2^2 + \beta\left\|\frac{\rvq(t)}{\|\rvq(0)\|_2}- {\rvz(t)}\right\|_2\frac{\|\rvq\|_2^3}{\|\rvq(0)\|_2},
	\end{align*}
	where the inequality follows from \cref{inn_pd_pf_2hm} and \cref{lips_res_2hm}. The above equation implies
	\begin{equation*}
	\frac{d}{dt} \left\|\frac{\rvq(t)}{\|\rvq(0)\|_2} - \rvz(t) \right\|_2  \leq 2\mathcal{N}(\rvw_*)\left\|\frac{\rvq(t)}{\|\rvq(0)\|_2} - \rvz(t) \right\|_2 +  \frac{\beta\|\rvq\|_2^3}{\|\rvq(0)\|_2}.
	\end{equation*}
	Using Lemma \ref{gronwall_ineq}, we get 
	\begin{align*}
	\left\|\frac{\rvq(t)}{\|\rvq(0)\|_2} - \rvz(t) \right\|_2 &\leq e^{2t\mathcal{N}(\rvw_*)}\left(\left\|\frac{\rvq(0)}{\|\rvq(0)\|_2} - \rvz(0) \right\|_2 + \int_{0}^te^{-2s\mathcal{N}(\rvw_*)}\beta\|\rvq(s)\|_2^3/\|\rvq(0)\|_2 ds\right) \\	 
	&\leq e^{2t\mathcal{N}(\rvw_*)}\left(\left\|\frac{\rvq(0)}{\|\rvq(0)\|_2} - \rvz(0) \right\|_2 + \beta\|\rvq(0)\|_2^2\int_{0}^te^{4s\mathcal{N}(\rvw_*)} ds\right)\\
	& \leq C_2^2\delta^2e^{6t\mathcal{N}(\rvw_*)},
	\end{align*}
	where the second equality uses \cref{qt_bd_2hm}. In the third inequality, $C_2$ is some sufficiently large constant, and it follows since, for all sufficiently small $\delta>0$, $\|\rvq(0)\|_2 = O(\delta)$ and 
	\begin{align*}
	\left\|\frac{\rvq(0)}{\|\rvq(0)\|_2} - \rvz(0) \right\|_2 = \frac{\|\rva_\delta - \|\rva_\delta\|_2\rvw_*\|_2}{\|\rva_\delta\|_2} \leq \frac{\|\rva_\delta - \delta\rvw_*\|_2 + \left|\delta - \|\rva_\delta\|_2\right|}{\delta-C_1\delta^3}\leq \frac{2C_1\delta^3}{\delta-C_1\delta^3} = O(\delta^2). 
	\end{align*}
	Define $\tau_q(t) \coloneqq \rvq(t)/\|\rvq(0)\|_2- \rvz(t) $, then, using the definition of $\overline{T}_{1}$, we have
	\begin{equation*}
	1-\gamma = \frac{\rvw_*^\top\rvq(\overline{T}_1)}{\|\rvq(\overline{T}_1)\|_2}  =  \frac{\rvw_*^\top\rvz(\overline{T}_1) + \rvw_*^\top\tau_q(\overline{T}_1) }{\|\rvq(\overline{T}_1)\|_2/\|\rvq(0)\|_2} \geq  \frac{e^{2\overline{T}_1\mathcal{N}(\rvw_*)} -  \|\tau_q(\overline{T}_1)\|_2 }{e^{2\overline{T}_1\mathcal{N}(\rvw_*)}} \geq 1 - {C_2^2\delta^2 e^{4\overline{T}_1\mathcal{N}(\rvw_*)} },
	\end{equation*}
	which implies $\overline{T}_{1}\geq \frac{1}{2\mathcal{N}(\rvw_*)}\ln\left(\frac{\sqrt{\gamma}}{C_2\delta}\right).$
	
	Next, we focus on $\rvu(t)$. Note that $\rvw_*^\top\rvu(0)/\|\rvu(0)\|_2  = 1>1-\gamma$, and $\rvu(t)$ satisfies
	\begin{equation*}
	\dot{\rvu} = -\mathcal{J}(\rmX;\rvu)^\top\ell'(\mathcal{H}(\rmX;\rvu), \rvy) = \nabla\mathcal{N}(\rvu)  - \mathcal{J}(\rmX;\rvu)^\top(\ell'(\mathcal{H}(\rmX;\rvu), \rvy) - \ell'(\mathbf{0}, \rvy)).
	\end{equation*}
	Thus, if we define
	\begin{equation*}
	\overline{T}_{2} = \min_{t\geq 0}\left\{t: \frac{\rvw_*^\top\rvu(t)}{\|\rvu(t)\|_2} = 1-\gamma\right\},
	\end{equation*}
	then, similar to $\rvq(t)$, we can show, for all $t\in [0,\overline{T}_2]$, 
	\begin{equation*}
	\|\rvu(t)\|_2 \leq \|\rvu(0)\|_2e^{2t\mathcal{N}(\rvw_*)}.
	\end{equation*}
	Also, there exists a sufficiently large constant $C_3$ such that 
	\begin{align*}
	\left\|\frac{\rvu(t)}{\|\rvu(0)\|_2} - \rvz(t) \right\|_2 &\leq e^{2t\mathcal{N}(\rvw_*)}\left(\left\|\frac{\rvu(0)}{\|\rvu(0)\|_2} - \rvz(0) \right\|_2 + \int_{0}^te^{-2s\mathcal{N}(\rvw_*)}\beta\|\rvu(s)\|_2^3/\|\rvu(0)\|_2 ds\right) \\	 
	& \leq C_3^2\delta^2e^{6t\mathcal{N}(\rvw_*)},
	\end{align*}
	and $\overline{T}_{2}\geq \frac{1}{2\mathcal{N}(\rvw_*)}\ln\left(\frac{\sqrt{\gamma}}{C_3\delta}\right).$
	
	Next, choose $C_4>0$ sufficiently small such that $C_4 < \sqrt{\gamma}\min(1/C_2,1/C_3)$ and $C_4^2\leq \mathcal{N}(\rvw_*)/(\zeta(1+C_1)^2)$. Then, for all $t\in [0,\frac{1}{2\mathcal{N}(\rvw_*)}\ln\left(\frac{C_4}{\delta}\right)]$,
	\begin{align}
	&	\frac{\rvw_*^\top\rvu(t)}{\|\rvu(t)\|_2} \geq 1-\gamma,	\frac{\rvw_*^\top\rvq(t)}{\|\rvq(t)\|_2} \geq 1-\gamma, \text{ and }\nonumber\\
	&\|\rvu(t)\|_2 \leq \|\rvu(0)\|_2e^{2t\mathcal{N}(\rvw_*)} \leq C_4, \|\rvq(t)\|_2 \leq \|\rvq(0)\|_2e^{2t\mathcal{N}(\rvw_*)} \leq C_4(1+C_1\delta^2) .
	\label{qu_bd}
	\end{align}
	Therefore, for $t\in \left[0,\frac{1}{2\mathcal{N}(\rvw_*)}\ln\left(\frac{C_4}{\delta}\right)\right]$, we have
	\begin{align*}
	\frac{1}{2}\frac{d}{dt} \left\|{\rvu(t)} - \rvq(t) \right\|_2^2 &= \left({\rvu} - \rvq\right)^\top\left({\dot{\rvu}} - \dot{\rvq}\right)\\
	& = \left({\rvu} - \rvq\right)^\top\left(\nabla\mathcal{N}(\rvu) - \nabla\mathcal{N}(\rvq)\right) - \left({\rvu}- \rvq\right)^\top\left(f(\rvu) - f(\rvq)\right)\\
	&\leq 3\mathcal{N}(\rvw_*)\left\|{\rvu} - \rvq \right\|_2^2 + \zeta\left\|{\rvu}- \rvq \right\|_2^2\max(\|\rvq\|_2^2,\|\rvu\|_2^2).\\
	& \leq (3\mathcal{N}(\rvw_*)+ \zeta C_4^2(1+C_1)^2)\left\|{\rvu(t)} - \rvq(t) \right\|_2^2\\
	&\leq 4\mathcal{N}(\rvw_*)\left\|{\rvu(t)} - \rvq(t) \right\|_2^2,
	\end{align*}	
	where the first inequality follows from combining mean value theorem with \cref{loc_ineq_2hm} and $0$-homogeneity of $\nabla^2\mathcal{N}(\rvw)$, and \cref{lips_f_2hm}. The second inequality follows from \cref{qu_bd} when $\delta\leq 1$. The last inequality uses our assumption on $C_4.$ Therefore,
	\begin{align*}
	\left\|\rvu\left(\frac{1}{2\mathcal{N}(\rvw_*)}\ln\left(\frac{C_4}{\delta}\right)\right) - \rvq\left(\frac{1}{2\mathcal{N}(\rvw_*)}\ln\left(\frac{C_4}{\delta}\right)\right)  \right\|_2 &\leq  \left\|\rvu(0) - \rvq(0) \right\|_2e^{\frac{4\mathcal{N}(\rvw_*)}{2\mathcal{N}(\rvw_*)}\ln\left(\frac{C_4}{\delta}\right)}\\
	& \leq C_1C_4^2\delta.
	\end{align*}	
	Since $\widetilde{T}$ is fixed, from Lemma \ref{err_bd_gf} and the above inequality, there exists a ${C}>0$ such that for all sufficiently small $\delta>0$, we have 
	\begin{align*}
	\left\|\bm{\psi}\left(t+ \frac{1}{2\mathcal{N}(\rvw_*)}\ln\left(\frac{1}{\delta}\right),{\rva}_\delta\right) - \bm{\psi}\left(t+ \frac{1}{2\mathcal{N}(\rvw_*)}\ln\left(\frac{1}{\delta}\right),\delta\rvw_*\right)  \right\|_2 \leq C\delta,
	\end{align*}
	for all $t\in [-\widetilde{T},\widetilde{T}]$, which completes the proof. 
\end{proof}
\begin{proof}\textbf{of Lemma \ref{lim_exists_2hm}: }
	We choose $\gamma>0$ sufficiently small such that for all unit-norm vectors $\rvw$ satisfying $\rvw^\top\rvw_* \geq 1-\gamma$, we have $\mathcal{N}(\rvw) \leq \mathcal{N}(\rvw_*)$ and
	\begin{equation}
	\left(\nabla \mathcal{N}(t_1\rvw) -  \nabla \mathcal{N}(t_2\rvw_*)\right)^\top\left( t_1\rvw - t_2\rvw_*\right) \leq 2\mathcal{N}(\rvw_*)\|t_1\rvw - t_2\rvw_*\|_2^2 , \forall t_2\geq t_1\geq 0.
	\label{2_inn_pd_2hm}
	\end{equation}
	The first and second inequality follow from Lemma \ref{lips_ineq} and Lemma \ref{inn_pd_ineq}, respectively. Let $\rvu_{\delta_1}(t) = \bm{\psi}\left(t, \delta_1\rvw_*\right) \text{ and } \rvu_{\delta_2}(t) = \bm{\psi}\left(t, \delta_2\rvw_*\right),$ where recall that $\delta_2\geq \delta_1>0$. Let $\rvz(t) = e^{2t\mathcal{N}(\rvw_*)}\rvw_*$, which is the solution of
	\begin{equation*}
	\dot{\rvz} = \nabla \mathcal{N}(\rvz), \rvz(0) = \rvw_*.  
	\end{equation*}
	Note that $\rvu_{\delta_1}(0) = \delta_1\rvw_*$. Define
	\begin{equation*}
	T^*_1 = \min_{t\geq 0}\left\{t: \frac{\rvw_*^\top\rvu_{\delta_1}(t)}{\|\rvu_{\delta_1}(t)\|_2} = 1-\gamma\right\}.
	\end{equation*}
	Then, as demonstrated in the proof of  Lemma \ref{init_dl3_2hm}, we can show
	\begin{equation*}
	\|\rvu_{\delta_1}(t)\|_2 \leq \|\rvu_{\delta_1}(0)\|_2e^{2t\mathcal{N}(\rvw_*)}.
	\end{equation*}
	Also, there exists a sufficiently large constant $C_2$ such that
	\begin{equation}
	\left\|\frac{\rvu_{\delta_1}(t)}{\|\rvu_{\delta_1}(0)\|_2} - \rvz(t) \right\|_2
	\leq C_2^2\delta_1^2e^{6t\mathcal{N}(\rvw_*)}, \text{ for all }t\in [0,T_1^*],\label{dt1_bd_2hm}
	\end{equation}
	and ${T}_{1}^*\geq \frac{1}{2\mathcal{N}(\rvw_*)}\ln\left(\frac{\sqrt{\gamma}}{C_2\delta_1}\right)$. Hence, if $\delta_2\leq \sqrt{\gamma}/C_2$, then ${T}_{1}^*\geq \frac{1}{2\mathcal{N}(\rvw_*)}\ln\left(\frac{\delta_2}{\delta_1}\right)$.
	
	Next, since $\delta_1\rvz\left(\frac{1}{2\mathcal{N}(\rvw_*)}\ln\left(\frac{\delta_2}{\delta_1}\right)\right) = \delta_2\rvw_*$, using \cref{dt1_bd_2hm} we get
	\begin{align}
	\left\|\rvu_{\delta_1}\left(\frac{\ln\left({\delta_2}/{\delta_1}\right)}{2\mathcal{N}(\rvw_*)}\right) - \delta_2\rvw_* \right\|_2 &= \left\|\rvu_{\delta_1}\left(\frac{1}{2\mathcal{N}(\rvw_*)}\ln\left(\frac{\delta_2}{\delta_1}\right)\right) - \delta_1\rvz\left(\frac{1}{2\mathcal{N}(\rvw_*)}\ln\left(\frac{\delta_2}{\delta_1}\right)\right) \right\|_2\nonumber \\
	&\leq \frac{C_2^2\delta_1^3 \delta_2^3 }{\delta_1^3} = C_2\delta_2^3.
	\label{d1d2_bd}
	\end{align}
	Now, note that 
	\begin{align*}
	\bm{\psi}\left(t+\frac{1}{2\mathcal{N}(\rvw_*)}\ln\left(\frac{1}{\delta_1}\right),\delta_1\rvw_*\right) &= \bm{\psi}\left(t+\frac{1}{2\mathcal{N}(\rvw_*)}\ln\left(\frac{1}{\delta_2}\right)+\frac{1}{2\mathcal{N}(\rvw_*)}\ln\left(\frac{\delta_2}{\delta_1}\right),\delta_1\rvw_*\right)\\
	= & \bm{\psi}\left(t+\frac{1}{2\mathcal{N}(\rvw_*)}\ln\left(\frac{1}{\delta_2}\right),\bm{\psi}\left(\frac{1}{2\mathcal{N}(\rvw_*)}\ln\left(\frac{\delta_2}{\delta_1}\right),\delta_1\rvw_*\right)\right)\\
	=&  \bm{\psi}\left(t+\frac{1}{2\mathcal{N}(\rvw_*)}\ln\left(\frac{1}{\delta_2}\right),\rvu_{\delta_1}\left(\frac{1}{2\mathcal{N}(\rvw_*)}\ln\left(\frac{\delta_2}{\delta_1}\right)\right)\right),
	\end{align*}
	Combining the above equality with \cref{d1d2_bd} and Lemma \ref{init_dl3_2hm}, we get that for any fixed $t\in (-\infty,\infty)$ and all sufficiently small $\delta_2,\delta_1>0$, there exists a constant $C>0$ such that
	\begin{equation*}
	\left\|\bm{\psi}\left(t+\frac{1}{2\mathcal{N}(\rvw_*)}\ln\left(\frac{1}{\delta_1}\right),\delta_1\rvw_*\right) - \bm{\psi}\left(t+\frac{1}{2\mathcal{N}(\rvw_*)}\ln\left(\frac{1}{\delta_2}\right),\delta_2\rvw_*\right)\right\|_2 \leq C\delta_2,
	\end{equation*}
	which implies $\rvp(t)$ exists for all $t\in (-\infty,\infty)$. 
	
	We next prove that $\mathcal{L}(\rvp(0)) \leq\mathcal{L}(\mathbf{0}) - \eta$, for some $\eta>0$. Let $\rvu(t) = \bm{\psi}(t,\delta\rvw_*)$ and $\rvz(t) = e^{2t\mathcal{N}(\rvw_*)}\rvw_*$, then from \cref{dt1_bd_2hm}, there exists a sufficiently large constant $B_1$ such that 
	\begin{align}
	\left\|{\rvu(t)}/{\delta} - \rvz(t) \right\|_2
	& \leq B_1^2\delta^2e^{6t\mathcal{N}(\rvw_*)}, \text{ for all }t\in [0,T_1],
	\label{diff_u_2hm}
	\end{align}
	where ${T}_{1}\geq \frac{1}{2\mathcal{N}(\rvw_*)}\ln\left(\frac{\sqrt{\gamma}}{B_1\delta}\right).$
	
	Define $\alpha\coloneqq\max_{\|\rvw\|_2 = 1} \|\mathcal{H}(\rmX,\rvw)\|_2$. From the convexity of the loss function, we know
	\begin{align}
	\mathcal{L}(\rvw) &\leq \mathcal{L}(\mathbf{0}) + \ell'(\mathcal{H}(\rmX,\rvw),\rvy)^\top\mathcal{H}(\rmX,\rvw)\nonumber\\
	&= \mathcal{L}(\mathbf{0}) + (\ell'(\mathcal{H}(\rmX,\rvw),\rvy)-\ell'(\mathbf{0},\rvy))^\top\mathcal{H}(\rmX,\rvw) + \ell'(\mathbf{0},\rvy)^\top\mathcal{H}(\rmX,\rvw)\nonumber\\
	& \leq \mathcal{L}(\mathbf{0}) + Kn\|\mathcal{H}(\rmX,\rvw)\|_2^2 - \mathcal{N}(\rvw) \leq \mathcal{L}(\mathbf{0}) + \alpha^2 Kn\|\rvw\|_2^4 - \mathcal{N}(\rvw),
	\label{loss_ub_2hm}
	\end{align}
	where the last two inequalities follow from smoothness of the loss function and the definition of $\alpha$. Now, choose $\epsilon\in(0,1)$ sufficiently small such that
	\begin{equation*}
	\epsilon\leq \frac{\sqrt{\gamma}}{B_1}, \text{ and } Kn\alpha^2(\sqrt{\epsilon}+\epsilon^{2.5}/B_1^2)^4 \leq \mathcal{N}(\rvw) -  \frac{\mathcal{N}(\rvw_*)}{2}, \text{ if } \|\rvw-\rvw_*\|_2\leq \epsilon^2B_1^2,
	\end{equation*}
	where the second inequality holds true since $\mathcal{N}(\rvw)$ is continuous and $\mathcal{N}(\rvw_*)>0$. Note that
	\begin{equation*}
	{T}_{1}\geq \frac{1}{2\mathcal{N}(\rvw_*)}\ln\left(\frac{\sqrt{\gamma}}{B_1\delta}\right)\geq \frac{1}{2\mathcal{N}(\rvw_*)}\ln\left(\frac{\epsilon}{\delta}\right)\coloneqq T_2.
	\end{equation*} 
	Hence, using \cref{diff_u_2hm},  $\|\rvu(T_2)\|_2\leq \epsilon+B_1^2\epsilon^3$. Also, if we define $\tau\coloneqq\rvu(T_2)-\delta\rvz(T_2)$, then $\|\tau\|_2 \leq B_1^2\epsilon^3$. Therefore, using \cref{loss_ub_2hm}, we get
	\begin{align*}
	\mathcal{L}(\rvu(T_2)) &\leq \mathcal{L}(\mathbf{0}) + \alpha^2 Kn\|\rvu(T_2)\|_2^4 - \mathcal{N}(\rvu(T_2))\\
	&\leq\mathcal{L}(\mathbf{0}) + \alpha^2 Kn(\epsilon+B_1^2\epsilon^3)^4 - \mathcal{N}(\rvu(T_2))\\
	&= \mathcal{L}(\mathbf{0}) + \epsilon^2(\alpha^2 Kn(\sqrt{\epsilon}+B_1^2\epsilon^{2.5})^4 - \mathcal{N}(\rvu(T_2)/\epsilon)) \leq  \mathcal{L}(\mathbf{0}) - \epsilon^2\mathcal{N}(\rvw_*)/2,
	\end{align*}
	where the last inequality follows from the choice of $\epsilon$ and since $\|\rvu(T_2)/\epsilon-\rvw_*\|_2\leq B_1^2\epsilon^2.$ If we define $\eta = \epsilon^2\mathcal{N}(\rvw_*)/2$, then the above equation implies, for all sufficiently small $\delta>0$,
	\begin{equation*}
	\mathcal{L}\left(\bm{\psi}\left(\frac{\ln(\epsilon)+\ln\left({1}/{\delta}\right)}{2\mathcal{N}(\rvw_*)},\delta\rvw_*\right)\right) = \mathcal{L}\left(\bm{\psi}\left(\frac{\ln\left({\epsilon}/{\delta}\right)}{2\mathcal{N}(\rvw_*)},\delta\rvw_*\right)\right)\leq \mathcal{L}(\mathbf{0}) - \eta,
	\end{equation*} 
	where note that $\epsilon\in (0,1)$ and fixed, and thus, $\eta$ is fixed. Since $\ln(\epsilon)<0$, and the loss decreases with time, the above equation implies, for all sufficiently small $\delta>0$,
	\begin{equation*}
	\mathcal{L}\left(\bm{\psi}\left(\frac{\ln\left({1}/{\delta}\right)}{2\mathcal{N}(\rvw_*)},\delta\rvw_*\right)\right) \leq\mathcal{L}\left(\bm{\psi}\left(\frac{\ln\left({\epsilon}/{\delta}\right)}{2\mathcal{N}(\rvw_*)},\delta\rvw_*\right)\right)\leq \mathcal{L}(\mathbf{0}) - \eta.
	\end{equation*} 
	Taking $\delta\to 0$, gives us $\mathcal{L}(\rvp(0)) \leq \mathcal{L}(\mathbf{0}) - \eta,$ which completes the proof.
\end{proof}
\begin{proof}\textbf{of Lemma \ref{init_align_2hm}:}
	We choose $\gamma>0$ sufficiently small such that for all unit-norm vectors $\rvw$ satisfying $\rvw^\top\rvw_* \geq 1-\gamma$, we have
	\begin{align}
	&\mathcal{N}(\rvw) \leq \mathcal{N}(\rvw_*), \|\nabla^2\mathcal{N}(\rvw)\|_2\leq 3\mathcal{N}(\rvw_*), \text{ and }\label{loc_min_2hm}\\
	&\left(\nabla \mathcal{N}(t_1\rvw) -  \nabla \mathcal{N}(t_2\rvw_*)\right)^\top\left(t_1\rvw - t_2\rvw_*\right) \leq 2\mathcal{N}(\rvw_*)\|t_1\rvw - t_2\rvw_*\|_2^2 , \forall t_2\geq t_1\geq 0.
	\label{inn_pd_2hm_3}
	\end{align}
	The first inequality follows from Lemma \ref{lips_ineq}. The second inequality holds since, from Lemma \ref{second_kkt}, $\|\nabla^2\mathcal{N}(\rvw_*)\|_2 = 2\mathcal{N}(\rvw_*)$, and $\nabla^2\mathcal{N}(\rvw)$ is continuous in the neighborhood of $\rvw_*$. The third inequality follows from Lemma \ref{inn_pd_ineq}. Define $f(\rvw) = \mathcal{J}(\rmX;\rvw)^\top(\ell'(\mathcal{H}(\rmX;\rvw), \rvy) - \ell'(\mathbf{0}, \rvy))$, then, as shown in the proof of Lemma \ref{init_dl3_2hm}, there exists a $\beta>0$ such that
	\begin{equation}
	\|f(\rvw)\|_2 \leq  \beta\|\rvw\|_2^3, \text{ for all }\rvw\in \sR^k.
	\label{lips_res_2hm_1}
	\end{equation}
	Let $\rvz_1(t)$ denote the solution of 
	\begin{equation*}
	\dot{\rvz} = \nabla\mathcal{N}(\rvz), \rvz(0) = \rvw_0.
	\end{equation*}
	Since $\rvw_0\in \mathcal{S}(\rvw_*)$, it follows that $\rvz_1(t)$ converges to $\rvw_*$ in direction. Thus, we can assume there exists some time $T_\gamma$ and a constant $B_1$ such that
	\begin{equation*}
	\frac{\rvw_*^\top\rvz_1(T_\gamma)}{\|\rvz_1(T_\gamma)\|_2} = 1-\frac{\gamma^2}{8}, \text{ and }\|\rvz_1(t)\|_2\leq B_1, \text{ for all }t\in [0,T_\gamma].
	\end{equation*}
	Let $\mu_1 \coloneqq \max_{\|\rvw\|_2 = 1}\mathcal{N}(\rvw),$ and $\rvw(t) \coloneqq \bm{\psi}\left(t,\delta\rvw_0\right),$	then $\rvw(t)$ is the solution of 
	\begin{equation*}
	\dot{\rvw} = \nabla\mathcal{N}(\rvw)  - \mathcal{J}(\rmX;\rvw)^\top(\ell'(\mathcal{H}(\rmX;\rvw), \rvy) - \ell'(\mathbf{0}, \rvy)), \rvw(0) = \delta\rvw_0.
	\end{equation*}
	Multiplying the above equation by $\rvw^\top$ from the left and using convexity of $\ell(\cdot,\cdot)$, we get
	\begin{equation*}
	\frac{1}{2}\frac{d\|\rvw\|_2^2}{dt} =  2\mathcal{N}(\rvw) - 2 \mathcal{H}(\rmX;\rvw)^\top(\ell'(\mathcal{H}(\rmX;\rvw), \rvy) - \ell'(\mathbf{0}, \rvy))
	\leq 2\mathcal{N}(\rvw) \leq \mu_1\|\rvw\|_2^2,
	\end{equation*}
	which implies $\|\rvw(t)\|_2\leq \delta e^{\mu_1t}$. Now, since $\nabla\mathcal{N}(\rvw)$ is locally Lipschitz, and for all $t\in [0,T_\gamma]$, $\|\rvz_1(t)\|_2\leq B_1$ and $\|\rvw(t)\|_2/\delta\leq  e^{\mu_1T_\gamma} := B_2$, there exists a $\mu_2>0$ such that
	\begin{equation*}
	\|\nabla\mathcal{N}({\rvw(t)}/{\delta} ) - \nabla\mathcal{N}(\rvz_1(t))\|_2 \leq \mu_2\|\rvw(t)/\delta-\rvz_1(t)\|_2, \forall t\in [0,T_\gamma].
	\end{equation*}
	Hence, for all $t\in[0,T_\gamma]$,
	\begin{align*}
	\frac{1}{2}\frac{d}{dt} \left\|\frac{\rvw(t)}{\delta} - \rvz_1(t) \right\|_2^2 &= \left(\frac{\rvw}{\delta} - \rvz_1\right)^\top\left(\frac{\dot{\rvw}}{\delta}  - \dot{\rvz}_1\right)\\
	& = \left(\frac{\rvw}{\delta} - \rvz_1\right)^\top\left(\nabla\mathcal{N}({\rvw}/{\delta} ) - \nabla\mathcal{N}(\rvz_1)\right) - \left(\frac{\rvw}{\delta} - \rvz_1\right)^\top\frac{f(\rvw)}{\delta}\\
	&\leq \mu_2\left\|\frac{\rvw}{\delta}  - \rvz_1 \right\|_2^2 + \beta\left\|\frac{\rvw}{\delta}  - \rvz_1 \right\|_2\|\rvw\|_2^3/\delta,
	\end{align*}	
	which implies
	\begin{equation*}
	\frac{d}{dt} \left\|\frac{\rvw(t)}{\delta} - \rvz_1(t) \right\|_2
	\leq \mu_2\left\|\frac{\rvw}{\delta}  - \rvz_1(t) \right\|_2  + \beta\|\rvw\|_2^3/\delta.
	\end{equation*}	
	Using Lemma \ref{gronwall_ineq}, we get
	\begin{align}
	\left\|\frac{\rvw(T_\gamma)}{\delta} - \rvz_1(T_\gamma) \right\|_2 &\leq e^{\mu_2 T_\gamma}\int_{0}^{T_\gamma}e^{-\mu_2 s}\beta\|\rvw(s)\|_2^3/\delta ds 	 \nonumber\\
	&\leq e^{\mu_2 T_\gamma}\int_{0}^{T_\gamma} \beta\|\rvw(s)\|_2^3/\delta ds \leq \beta\delta^2e^{\mu_2 T_\gamma}\int_{0}^{T_\gamma} e^{3\mu_1s} ds \leq A_1\delta^2,
	\label{wtg_bd_2hm}
	\end{align}
	where $A_1$ is a sufficiently large constant. Let $\tau_1 \coloneqq {\rvw(T_\gamma)}/{\delta} - \rvz_1(T_\gamma)$, then, for all sufficiently small $\delta>0$, we have
	\begin{align*}
	\frac{\rvw_*^\top\rvw(T_\gamma)}{\|\rvw(T_\gamma)\|_2} = \frac{\rvw_*^\top\rvz_1(T_\gamma) + \rvw_*^\top\tau_1}{\|\rvw(T_\gamma)\|_2/\delta} &\geq \frac{(1-\gamma^2/8)\|\rvz_1(T_\gamma)\|_2 - A_1\delta^2}{\|\rvw(T_\gamma)\|_2/\delta}\\
	&\geq \frac{(1-\gamma^2/8)\|\rvz_1(T_\gamma)\|_2 - A_1\delta^2}{\|\rvz_1(T_\gamma)\|_2 +A_1\delta^2} \\
	& = 1 - \frac{\gamma^2\|\rvz_1(T_\gamma)\|_2/8 - 2A_1\delta^2}{\|\rvz_1(T_\gamma)\|_2+A_1\delta^2}\geq 1-\gamma^2/4.
	\end{align*}
	Let $A_2 = \|\rvz_1(T_\gamma)\|_2$, where note that $A_2$ is a constant that does not depend on $\delta$. Define $\rvw_1 = \rvw(T_\gamma)/\| \rvw(T_\gamma)\|_2$ and $\widetilde{\delta} = \| \rvw(T_\gamma)\|_2$. From \cref{wtg_bd_2hm}, we know
	\begin{equation*}
	\widetilde{\delta}\in (A_2\delta - A_1\delta^3, A_2\delta + A_1\delta^3),
	\end{equation*}
	thus, $\widetilde{\delta}$ can be made sufficiently small by choosing $\delta$ sufficiently small. Next, let
	$\rvu(t) = \bm{\psi}(t+T_\gamma,\delta\rvw_0)$. Then,
	\begin{equation*}
	\rvu(t) = \bm{\psi}(t+T_\gamma,\delta\rvw_0) = \bm{\psi}(t,\bm{\psi}(T_\gamma,\delta\rvw_0)) = \bm{\psi}(t,\rvw(T_\gamma)) = \bm{\psi}(t,\widetilde{\delta} \rvw_1),
	\end{equation*} 
	and $\rvu(t)$ is the solution of 
	\begin{equation*}
	\dot{\rvu} =\nabla\mathcal{N}(\rvu)  - \mathcal{J}(\rmX;\rvu)^\top(\ell'(\mathcal{H}(\rmX;\rvu), \rvy) - \ell'(\mathbf{0}, \rvy)), \rvu(0) = \widetilde{\delta} \rvw_1.
	\end{equation*}
	Since $\rvw_*^\top\rvu(0)/\|\rvu(0)\|_2 = \rvw_*^\top\rvw_1 \geq 1-\gamma^2/4 > 1-\gamma$, we define 
	$$T^* = \min_{t\geq 0}\left\{t: \frac{\rvw_*^\top\rvu(t)}{\|\rvu(t)\|_2} = 1-\gamma\right\}.$$
	Now, since $\rvw_*^\top\rvu(t)/\|\rvu(t)\|_2 \geq 1-\gamma$, for all $t\in [0,T^* ]$, from \cref{loc_min_2hm}, we have
	\begin{equation*}
	\frac{1}{2}\frac{d\|\rvu(t)\|_2^2}{dt} = 2\mathcal{N}(\rvu) - 2 \mathcal{H}(\rmX;\rvu)^\top(\ell'(\mathcal{H}(\rmX;\rvu), \rvy) - \ell'(\mathbf{0}, \rvy)) \leq 2\mathcal{N}(\rvw_*)\|\rvu(t)\|_2^2,
	\end{equation*}
	which implies
	\begin{equation*}
	\|\rvu(t)\|_2 \leq 	\|\rvu(0)\|_2 e^{2t\mathcal{N}(\rvw_*)} = \widetilde{\delta}e^{2t\mathcal{N}(\rvw_*)} .
	\end{equation*}
	Next, let $\rvz_2(t) = e^{2t\mathcal{N}(\rvw_*)}\rvw_*$, which is the solution of
	\begin{equation}
	\dot{\rvz} = \nabla\mathcal{N}(\rvz), \rvz(0) = \rvw_*.
	\end{equation}
	Note that, for $t\in [0,T^*]$,
	\begin{align}
	&\left(\nabla \mathcal{N}\left(\frac{\rvu(t)}{\|\rvu(0)\|_2}\right) - \nabla\mathcal{N}(\rvz_2(t))\right)^\top\left(\frac{\rvu(t)}{\|\rvu(0)\|_2}- \rvz_2(t)\right)\nonumber\\
	&= \left(\nabla \mathcal{N}\left(\frac{\|\rvu(t)\|_2}{\|\rvu(0)\|_2}\frac{\rvu(t)}{\|\rvu(t)\|_2}\right) - \nabla\mathcal{N}\left({\rvz_2(t)}\right)\right)^\top\left(\frac{\|\rvu(t)\|_2}{\|\rvu(0)\|_2}\frac{\rvu(t)}{\|\rvu(t)\|_2}- {\rvz_2(t)}\right)\nonumber\\
	&\leq 2\mathcal{N}(\rvw_*)\left\|\frac{\|\rvu(t)\|_2}{\|\rvu(0)\|_2}\frac{\rvu(t)}{\|\rvu(t)\|_2}- \|\rvz_2(t)\|_2\frac{\rvz_2(t)}{\|\rvz_2(t)\|_2}\right\|^2 = 2\mathcal{N}(\rvw_*)\left\|\frac{\rvu(t)}{\|\rvu(0)\|_2}- {\rvz_2(t)}\right\|^2,
	\label{lm_init_inn_pd_pf_2hm}
	\end{align}
	where the inequality follows from ${\|\rvu(t)\|_2}/{\|\rvu(0)\|_2}\leq e^{2t\mathcal{N}(\rvw_*)} = \|\rvz_2(t)\|_2$, $\rvw_*^\top\rvu(t)/\|\rvu(t)\|_2 \geq 1-\gamma$ for all $t\in [0,T^*]$, and \cref{inn_pd_2hm_3}. Then, for all $t\in [0,T^*]$, we have
	\begin{align*}
	\frac{1}{2}\frac{d}{dt} \left\|\frac{\rvu(t)}{\widetilde{\delta}} - \rvz_2(t) \right\|_2^2&= \left(\frac{\rvu}{\widetilde{\delta}} - \rvz_2\right)^\top\left(\frac{\dot{\rvu}}{\widetilde{\delta}} - \dot{\rvz}_2\right)\\
	& = \left(\frac{\rvu}{\widetilde{\delta}} - \rvz_2\right)^\top\left(\nabla\mathcal{N}(\rvu/\widetilde{\delta}) - \nabla\mathcal{N}(\rvz_2)\right) - \left(\frac{\rvu}{\widetilde{\delta}} - \rvz_2\right)^\top\frac{f(\rvu)}{\widetilde{\delta}}\\
	&\leq 2\mathcal{N}(\rvw_*)\left\|\frac{\rvu}{\widetilde{\delta}} - \rvz_2 \right\|_2^2 + \beta\left\|\frac{\rvu}{\widetilde{\delta}} - \rvz_2 \right\|_2\|\rvu\|_2^3/\widetilde{\delta}.
	\end{align*}	
	From the above equation, we have
	\begin{equation*}
	\frac{d}{dt} \left\|\frac{\rvu}{\widetilde{\delta}} - \rvz_2 \right\|_2\leq 2\mathcal{N}(\rvw_*)\left\|\frac{\rvu}{\widetilde{\delta}} - \rvz_2 \right\|_2 +  \beta\|\rvu\|_2^3/\widetilde{\delta},
	\end{equation*}
	which implies
	\begin{align*}
	\left\|\frac{\rvu(t)}{\widetilde{\delta}} - \rvz_2(t) \right\|_2 &\leq e^{2t\mathcal{N}(\rvw_*)}\left(\|\rvw_1-\rvw_*\|_2 + \int_{0}^te^{-2s\mathcal{N}(\rvw_*)}\beta\|\rvu(s)\|_2^3/\widetilde{\delta} ds\right) \\
	&  \leq e^{2t\mathcal{N}(\rvw_*)}\|\rvw_1-\rvw_*\|_2 + \beta\widetilde{\delta}^2 e^{2t\mathcal{N}(\rvw_*)}\int_{0}^t e^{4s\mathcal{N}(\rvw_*)} ds \\
	& \leq e^{2t\mathcal{N}(\rvw_*)}\|\rvw_1-\rvw_*\|_2 + {A_3\widetilde{\delta}^2 e^{6t\mathcal{N}(\rvw_*)}},
	\end{align*}
	where $A_3$ is a sufficiently large constant. Now, define $\overline{T} = \frac{4\ln(1/\widetilde{\delta})}{\Delta + 8\mathcal{N}(\rvw_*)} $, then we claim that $\overline{T} < T^*$, for all sufficiently small $\delta>0$. For the sake of contradiction, let  $T^*\leq\overline{T}$. Define $\tau_2(t) = {\rvu(t)} - \widetilde{\delta}\rvz_2(t)$, then 
	\begin{equation*}
	\|\tau_2(T^*)\|_2 = \|{\rvu(T^*)} - \widetilde{\delta}\rvz_2(T^*)\|_2 \leq \widetilde{\delta}e^{2T^*\mathcal{N}(\rvw_*)}\|\rvw_1-\rvw_*\|_2 + A_3\widetilde{\delta}^3 e^{6T^*\mathcal{N}(\rvw_*)}.
	\end{equation*}
	Dividing by $\widetilde{\delta}e^{2T^*\mathcal{N}(\rvw_*)}$ on both sides and using $T^*\leq\overline{T}$, we get
	\begin{equation*}
	\frac{\|\tau_2(T^*)\|_2}{\widetilde{\delta}e^{2T^*\mathcal{N}(\rvw_*)}} \leq \|\rvw_1-\rvw_*\|_2 + {A_3\widetilde{\delta}^2 e^{4\overline{T}\mathcal{N}(\rvw_*)} }.
	\end{equation*}
	Next, by definition of $T^*$, we have
	\begin{align*}
	1-\gamma = \frac{\rvw_*^\top\rvu(T^*)}{\|\rvu(T^*)\|_2} &\geq \frac{\rvw_*^\top\rvu(T^*)}{\widetilde{\delta}e^{2T^*\mathcal{N}(\rvw_*)}} = \frac{\widetilde{\delta} e^{2T^*\mathcal{N}(\rvw_*)} + \rvw_*^\top\tau_2(T^*)}{\widetilde{\delta} e^{2T^*\mathcal{N}(\rvw_*)}}\geq 1 - \frac{\|\tau_2(T^*)\|_2}{ \widetilde{\delta}e^{2T^*\mathcal{N}(\rvw_*)}}.
	\end{align*}
	Now, since $\rvw_1^\top\rvw_* \geq 1 - \gamma^2/4$,
	\begin{equation*}
	\|\rvw_1-\rvw_*\|_2^2 = 2-2\rvw_1^\top\rvw_* \leq \gamma^2/2.
	\end{equation*}
	Hence, 
	\begin{align*}
	1-\gamma  \geq   1-\gamma/\sqrt{2} - {A_3\widetilde{\delta}^2 e^{4\overline{T}\mathcal{N}(\rvw_*)}} = 1-\gamma/\sqrt{2} - {A_3\widetilde{\delta}^2}/{\widetilde{\delta}^{\frac{16\mathcal{N}(\rvw_*)}{\Delta + 8\mathcal{N}(\rvw_*)}}},
	\end{align*}
	which implies
	\begin{equation*}
	\gamma(1-1/\sqrt{2}) \leq {A_3\widetilde{\delta}^2}/{\widetilde{\delta}^{\frac{16\mathcal{N}(\rvw_*)}{\Delta + 8\mathcal{N}(\rvw_*)}}}=  A_3\widetilde{\delta}^\frac{2\Delta}{\Delta + 8\mathcal{N}(\rvw_*)}.
	\end{equation*}
	The above inequality can not be true if $\delta$ is chosen sufficiently small, leading to a contradiction. Hence, $\overline{T}<T^*$.
	Next, let $\rvz_3(t)$ be the the solution of
	\begin{equation}
	\dot{\rvz} = \nabla\mathcal{N}(\rvz), \rvz(0) = \rvw_1,
	\end{equation}
	Since $\rvw_1^\top\rvw_*\geq 1-\gamma^2/4>1-\gamma$, then we may assume without loss of generality that $\gamma>0$ is sufficiently small such that Lemma \ref{rate_ineq_2hm} is applicable and hence,
	\begin{equation}
	\frac{\rvw_* ^\top\rvz_3(t)}{\|\rvz_3(t)\|_2} \geq 1 - e^{-t\Delta}\gamma, \text{ for all }t\geq 0, \text{ and } \rvz_3(t) = \rvg(t)e^{2t\mathcal{N}(\rvw_*)},
	\end{equation}
	where $\|\rvg(t)\|_2 \in [\kappa_1,\kappa_2]$, for some $\kappa_2\geq\kappa_1>0$. By definition of $\overline{T}$, we have
	\begin{equation*}
	\frac{\rvw_* ^\top\rvz_3(\overline{T})}{\|\rvz_3(\overline{T})\|_2} \geq 1-\widetilde{\delta}^{\frac{4\Delta}{\Delta + 8\mathcal{N}(\rvw_*)}}\gamma, \text{ and }	\|\rvz_3(\overline{T})\|_2 = \frac{\|\rvg(\overline{T})\|_2}{\widetilde{\delta}^{\frac{ 8\mathcal{N}(\rvw_*)}{\Delta + 8\mathcal{N}(\rvw_*)}}},
	\end{equation*}
	which implies
	\begin{equation*}
	\left\|\frac{\rvz_3(\overline{T})}{\|\rvz_3(\overline{T})\|_2}- \rvw_*  \right\|_2^2 = 2-\frac{2\rvw_* ^\top\rvz_3(\overline{T})}{\|\rvz_3(\overline{T})\|_2} \leq 2\widetilde{\delta}^{\frac{4\Delta}{\Delta + 8\mathcal{N}(\rvw_*)}}\gamma.
	\end{equation*}
	Multiplying by $\widetilde{\delta}^2\|\rvz_3(\overline{T})\|_2^2$ on both sides gives us
	\begin{align}
	\left\| \widetilde{\delta}\rvz_3(\overline{T}) - \widetilde{\delta}\|\rvz_3(\overline{T})\|_2\rvw_*   \right\|_2^2 &\leq 2\widetilde{\delta}^2\|\rvz_3(\overline{T})\|_2^2\delta^{\frac{4\Delta}{\Delta + 8\mathcal{N}(\rvw_*)}}\gamma\nonumber\\
	&\leq {2\|\rvg(\overline{T})\|_2^2\gamma}\frac{\widetilde{\delta}^2\widetilde{\delta}^{\frac{4\Delta}{\Delta+ 8\mathcal{N}(\rvw_*)}}}{\widetilde{\delta}^{\frac{ 16\mathcal{N}(\rvw_*)}{\Delta + 8\mathcal{N}(\rvw_*)}}}= 2{\|\rvg(\overline{T})\|_2^2\gamma}\widetilde{\delta}^{\frac{6\Delta}{\Delta + 8\mathcal{N}(\rvw_*)}}.
	\label{bd_zt}
	\end{align}
	Next, note that for $t\in [0,\overline{T}]$, $	{\rvw_* ^\top\rvz_3(t)}/{\|\rvz_3(t)\|_2}, {\rvw_* ^\top\rvu(t)}/{\|\rvu(t)\|_2}\geq 1-\gamma $. Since $\nabla^2\mathcal{N}(\rvw)$ is $0$-homogeneous, from the mean value theorem and \cref{loc_min_2hm}, for all $t\in [0,\overline{T}]$, we have
	\begin{align*}
	\left\|\left(\frac{\rvu(t)}{\widetilde{\delta}} - \rvz_3(t)\right)^\top\left(\nabla\mathcal{N}(\rvu(t)/\widetilde{\delta}) - \nabla\mathcal{N}(\rvz_3(t))\right) \right\|_2 \leq 3\mathcal{N}(\rvw_*)\left\|\frac{\rvu(t)}{\widetilde{\delta}} - \rvz_3(t) \right\|_2^2. 
	\end{align*}
	Hence, for $t\in [0,\overline{T}]$, we have
	\begin{align*}
	\frac{1}{2}\frac{d}{dt} \left\|\frac{\rvu(t)}{\widetilde{\delta}} - \rvz_3(t) \right\|_2^2 &= \left(\frac{\rvu}{\widetilde{\delta}} - \rvz_3\right)^\top\left(\frac{\dot{\rvu}}{\widetilde{\delta}} - \dot{\rvz}_3\right)\\
	& = \left(\frac{\rvu}{\widetilde{\delta}} - \rvz_3\right)^\top\left(\nabla\mathcal{N}(\rvu/\widetilde{\delta}) - \nabla\mathcal{N}(\rvz_3)\right) - \left(\frac{\rvu}{\widetilde{\delta}} - \rvz_3\right)^\top\frac{f(\rvw)}{\widetilde{\delta}}\\
	&\leq 3\mathcal{N}(\rvw_*)\left\|\frac{\rvu(t)}{\widetilde{\delta}} - \rvz_3(t) \right\|_2^2 + \beta\left\|\frac{\rvu(t)}{\widetilde{\delta}} - \rvz_3(t) \right\|_2\|\rvu\|_2^3/\widetilde{\delta}.
	\end{align*}	
	From the above equation, we have
	\begin{equation*}
	\frac{d}{dt} \left\|\frac{\rvu(t)}{\widetilde{\delta}} - \rvz_3(t) \right\|_2 \leq 3\mathcal{N}(\rvw_*)\left\|\frac{\rvu(t)}{\widetilde{\delta}} - \rvz_3(t) \right\|_2+  \beta\|\rvu(t)\|_2^3/\widetilde{\delta},
	\end{equation*}
	which implies
	\begin{align*}
	\left\|\frac{\rvu(t)}{\widetilde{\delta}} - \rvz_3(t) \right\|_2  &\leq e^{3\mathcal{N}(\rvw_*) t}\int_{0}^te^{-3\mathcal{N}(\rvw_*) s}\beta\|\rvu(s)\|_2^3/\widetilde{\delta} ds\\
	& 	 \leq \beta\widetilde{\delta}^2 e^{3\mathcal{N}(\rvw_*)t}\int_{0}^t e^{3\mathcal{N}(\rvw_*) s} ds \leq {{A_4}\widetilde{\delta}^2 e^{6t\mathcal{N}(\rvw_*)} },
	\end{align*}
	where $A_4$ is a sufficiently large constant. Putting $t=\overline{T}$ in the above equation gives us
	\begin{equation*}
	\left\|{\rvu(\overline{T})} - \widetilde{\delta}\rvz_3(\overline{T}) \right\|_2 \leq A_4\widetilde{\delta}^3 e^{6\overline{T}\mathcal{N}(\rvw_*)} = A_4\widetilde{\delta}^3/\widetilde{\delta}^{\frac{24\mathcal{N}(\rvw_*)}{\Delta+8\mathcal{N}(\rvw_*)}}= A_4\widetilde{\delta}^{\frac{3\Delta}{\Delta+8\mathcal{N}(\rvw_*)}}.
	\end{equation*}
	From \cref{bd_zt} and the above equation, we get
	\begin{align*}
	\left\|{\rvu(\overline{T})} - \widetilde{\delta}\|\rvz_3(\overline{T})\|_2\rvw_* \right\|_2 &\leq \left\|{\rvu(\overline{T})} - \widetilde{\delta}\rvz_3(\overline{T}) \right\|_2 + \left\| \widetilde{\delta}\|\rvz_3(\overline{T})\|_2\rvw_*  - \widetilde{\delta}\rvz_3(\overline{T}) \right\|_2\\
	&\leq \left(A_4+{\|\rvg(\overline{T})\|_2\sqrt{2\gamma}}\right)\widetilde{\delta}^{\frac{3\Delta}{\Delta+8\mathcal{N}(\rvw_*)}}.
	\end{align*}
	Also, note that
	\begin{equation*}
	\widetilde{\delta}\|\rvz_3(\overline{T})\|_2=\|\rvg(\overline{T})\|_2{\widetilde{\delta}}/{\widetilde{\delta}^{\frac{ 8\mathcal{N}(\rvw_*)}{\Delta + 8\mathcal{N}(\rvw_*)}}} = \|\rvg(\overline{T})\|_2\widetilde{\delta}^{\frac{\Delta}{\Delta+8\mathcal{N}(\rvw_*)}}.
	\end{equation*}
	Thus, using the above two equations and the definition of $\rvu(t)$, we get
	\begin{equation*}
	\left\|\bm{\psi}(\overline{T}+T_\gamma, \delta\rvw_0) - \|\rvg(\overline{T})\|_2\widetilde{\delta}^{\frac{\Delta}{\Delta+8\mathcal{N}(\rvw_*)}} \rvw_*\right\|_2 \leq \left(A_4+{\|\rvg(\overline{T})\|_2\sqrt{2\gamma}}\right)\widetilde{\delta}^{\frac{3\Delta}{\Delta+8\mathcal{N}(\rvw_*)}}.
	\end{equation*}
	Since $\overline{T}$ depends on $\delta$, $\|\rvg(\overline{T})\|_2$ may also depend on $\delta$, but $\|\rvg(\overline{T})\|_2\in [\kappa_1,\kappa_2]$. Hence, for all sufficiently small $\delta>0$,
	\begin{equation*}
	\left\|\bm{\psi}(\overline{T}+T_\gamma, \delta\rvw_0) - \|\rvg(\overline{T})\|_2\widetilde{\delta}^{\frac{\Delta}{\Delta+8\mathcal{N}(\rvw_*)}} \rvw_*\right\|_2 \leq C\widetilde{\delta}^{\frac{3\Delta}{\Delta+8\mathcal{N}(\rvw_*)}},
	\end{equation*}
	where $C$ is a sufficiently large constant, which completes the proof.
\end{proof}
\begin{proof}\textbf{of Corollary \ref{cor_sq_loss_2hm}: } Since  $\rvp(t)$ is bounded for all $t\geq 0$, there exists a constant $B>0$ such that $\|\rvp(t)\|_2\leq B$, for all $t\geq 0$. Moreover, since $\mathcal{L}(\cdot)$ has locally Lipschitz gradient, there exists a constant $\widetilde{A}>0$ such that, if $\|\rvw_1\|_2, \|\rvw_2\|_2\leq 2B$, then
	\begin{equation}
	\|\nabla\mathcal{L}(\rvw_1)-\nabla\mathcal{L}(\rvw_2)\|_2\leq \widetilde{A}\|\rvw_1-\rvw_2\|_2.
	\label{lips_sq_2hm}
	\end{equation} 
	Since $\rvp^*= \lim_{t\to \infty} \rvp(t)$ and $\nabla\mathcal{L}(\rvp^*) = \mathbf{0}$, for any $\epsilon\in (0,B)$, there exists a $T_\epsilon$ such that
	\begin{equation}
	\|\rvp(T_\epsilon)-\rvp_*\|_2\leq \epsilon/2 \text{ and } \|\nabla\mathcal{L}(\rvp(T_\epsilon))\|_2 \leq \epsilon/2.
	\label{err_sq_2hm}
	\end{equation}
	Since $T_\epsilon$ does not depend on $\delta$, therefore, from Theorem \ref{main_thm_2hm}, for all sufficiently small $\delta>0$ and for all $t\in [-T_\epsilon,T_\epsilon]$, 
	\begin{equation*}
	\left\|\bm{\psi}\left(t+T_1 +\frac{\ln(1/b_\delta)}{2\mathcal{N}(\rvw_*)}+ \frac{\ln({1}/{\widetilde{\delta}})}{2\mathcal{N}(\rvw_*)},\delta\rvw_0\right) - \rvp(t)\right\|_2 \leq \widetilde{C}\delta^{\frac{\Delta}{\Delta+8\mathcal{N}(\rvw_*)}} \leq \epsilon/2.
	\end{equation*} 
	Putting $t=T_\epsilon$ in the above equation, and using  \cref{err_sq_2hm}, we get
	\begin{equation*}
	\left\|\bm{\psi}\left(T_\delta,\delta\rvw_0\right) - \rvp^*\right\|_2 \leq \epsilon,
	\end{equation*}
	where $T_\delta\coloneqq T_\epsilon+T_1 +\frac{\ln(1/b_\delta)}{2\mathcal{N}(\rvw_*)}+ \frac{\ln({1}/{\widetilde{\delta}})}{2\mathcal{N}(\rvw_*)}$. Next, since $\|\bm{\psi}\left(T_\delta,\delta\rvw_0\right) \|_2\leq B+\epsilon/2\leq 2B$, using \cref{lips_sq_2hm} and \cref{err_sq_2hm}, we get
	\begin{align*}
	\left\|\nabla \mathcal{L}\left(\bm{\psi}\left(T_\delta,\delta\rvw_0\right)\right)\right\|_2 &\leq 	\left\|\nabla \mathcal{L}(\rvp(T_\epsilon))\right\|_2 + \left\|\nabla \mathcal{L}\left(\bm{\psi}\left(T_\delta,\delta\rvw_0\right)\right)-\nabla \mathcal{L}(\rvp(T_\epsilon))\right\|_2\\
	  &\leq {\epsilon}/{2}+ \widetilde{A}\widetilde{C}\delta^{\frac{\Delta}{\Delta+8\mathcal{N}(\rvw_*)}}\leq \epsilon,
	\end{align*}
	where the final inequality is true for all sufficiently small $\delta>0$, thus completing the proof.
\end{proof}
\begin{proof}\textbf{of Corollary \ref{cor_log_loss_2hm}: } Since $\lim_{t\to \infty} \rvp(t)/\|\rvp(t)\|_2 = \rvp^*$, $\lim_{t\to \infty} \|\rvp(t)\|_2 = \infty$, and $\lim_{t\to\infty} \nabla \mathcal{L}(\rvp(t)) = \mathbf{0}$, therefore, for any $\epsilon\in (0,1)$, we can choose a $T_\epsilon$ such that
	\begin{equation}
	\rvp(T_\epsilon)^\top\rvp^*/\|\rvp(T_\epsilon)\|_2 \geq 1-\epsilon/2, \|\rvp(T_\epsilon)\|_2\geq 1/\epsilon,  \text{ and } \|\nabla\mathcal{L}(\rvp(T_\epsilon))\|_2 \leq \epsilon/2.
	\label{err_log_2hm}
	\end{equation}
	Let $B_\epsilon \coloneqq  \max_{t\in [0,T_\epsilon]}\|\rvp(t)\|_2$. Then, since $\mathcal{L}(\cdot)$ has locally Lipschitz gradient, there exists a constant $\widetilde{A}_\epsilon>0$ such that, if $\|\rvw_1\|_2, \|\rvw_2\|_2\leq B_\epsilon+\epsilon$, it follows that
	\begin{equation}
	\|\nabla\mathcal{L}(\rvw_1)-\nabla\mathcal{L}(\rvw_2)\|_2\leq \widetilde{A}_\epsilon\|\rvw_1-\rvw_2\|_2,
	\label{lips_log_2hm}
	\end{equation} 
	where note that $\widetilde{A}_\epsilon$ depends on $\epsilon$. Since $T_\epsilon$ does not depend on $\delta$, from Theorem \ref{main_thm_2hm}, for all sufficiently small $\delta>0$ and for all $t\in [-T_\epsilon,T_\epsilon]$, 
	\begin{equation}
	\left\|\bm{\psi}\left(t+T_1 +\frac{\ln(1/b_\delta)}{2\mathcal{N}(\rvw_*)}+ \frac{\ln({1}/{\widetilde{\delta}})}{2\mathcal{N}(\rvw_*)},\delta\rvw_0\right) - \rvp(t)\right\|_2 \leq \widetilde{C}\delta^{\frac{\Delta}{\Delta+8\mathcal{N}(\rvw_*)}} \leq \epsilon/2.
	\label{err_bd_log_2hm}
	\end{equation} 
	Putting $t=T_\epsilon$ in the above equation, and using  \cref{err_log_2hm}, we get, for all sufficiently small $\delta$,
	\begin{equation*}
	\left\|\bm{\psi}\left(T_\delta,\delta\rvw_0\right)\right\|_2 \geq \|\rvp(T_\epsilon)\|_2 - \widetilde{C}\delta^{\frac{\Delta}{\Delta+8\mathcal{N}(\rvw_*)}} \geq {1}/{\epsilon} - \widetilde{C}\delta^{\frac{\Delta}{\Delta+8\mathcal{N}(\rvw_*)}}\geq {1}/{(2\epsilon)} ,
	\end{equation*}
	where $T_\delta\coloneqq T_\epsilon+T_1 +\frac{\ln(1/b_\delta)}{2\mathcal{N}(\rvw_*)}+ \frac{\ln({1}/{\widetilde{\delta}})}{2\mathcal{N}(\rvw_*)}$. We also have
	\begin{equation*}
	\frac{\bm{\psi}\left(T_\delta,\delta\rvw_0\right)^\top\rvp^*}{\left\|\bm{\psi}\left(T_\delta,\delta\rvw_0\right)\right\|_2 } \geq \frac{\rvp(T_\epsilon)^\top\rvp_* - \widetilde{C}\delta^{\frac{\Delta}{\Delta+8\mathcal{N}(\rvw_*)}}}{\left\|\bm{\psi}\left(T_\delta,\delta\rvw_0\right)\right\|_2} \geq \frac{(1-\epsilon/2)\|\rvp(T_\epsilon)\|_2 - \widetilde{C}\delta^{\frac{\Delta}{\Delta+8\mathcal{N}(\rvw_*)}}}{\left\|\rvp(T_\epsilon)\right\|_2 + \widetilde{C}\delta^{\frac{\Delta}{\Delta+8\mathcal{N}(\rvw_*)}}}\geq 1-\epsilon,
	\end{equation*}
	where the first inequality uses \cref{err_bd_log_2hm}. The second inequality uses \cref{err_bd_log_2hm} and \cref{err_log_2hm}. The final inequality is true for all sufficiently small $\delta>0$. Next, since $\|\bm{\psi}\left(T_\delta,\delta\rvw_0\right) \|_2\leq B_\epsilon+\epsilon/2$, using \cref{lips_log_2hm} and \cref{err_log_2hm}, we get
	\begin{align*}
	\left\|\nabla \mathcal{L}\left(\bm{\psi}\left(T_\delta,\delta\rvw_0\right)\right)\right\|_2 &\leq 	\left\|\nabla \mathcal{L}(\rvp(T_\epsilon))\right\|_2 + \left\|\nabla \mathcal{L}\left(\bm{\psi}\left(T_\delta,\delta\rvw_0\right)\right)-\nabla \mathcal{L}(\rvp(T_\epsilon))\right\|_2\\
		&\leq {\epsilon}/{2}+ \widetilde{A}_\epsilon\widetilde{C}\delta^{\frac{\Delta}{\Delta+8\mathcal{N}(\rvw_*)}}\leq \epsilon,
	\end{align*}
		where the final inequality is true for all sufficiently small $\delta>0$, thus completing the proof.
\end{proof}
\section{Proof Omitted from \Cref{sec_Lhm}}\label{appendix_Lhm}
We first prove three important lemmata, which are then used to prove \Cref{main_thm_Lhm}.
\begin{lemma}\label{init_dl3_Lhm}
	Consider the setting in \Cref{main_thm_Lhm}. Suppose there exists a constant $C_1>0$ such that for every $\delta>0$, ${\rva}_\delta$ is a vector that satisfies $\|{\rva}_\delta - \delta\rvw_*\|_2\leq C_1\delta^{L+1}$. Then, for any fixed $\widetilde{T}\in (-\infty,\infty)$, there exists a constant $C>0$ such that for all sufficiently small $\delta>0$, 
	\begin{equation*}
	\left\|\bm{\psi}\left(t+ \frac{1/\delta^{L-2}}{L(L-2)\mathcal{N}(\rvw_*)},{\rva}_\delta\right) - \bm{\psi}\left(t+ \frac{1/\delta^{L-2}}{L(L-2)\mathcal{N}(\rvw_*)},\delta{\rvw}_*\right) \right\|_2 \leq C\delta, \text{ for all } t\in[-\widetilde{T},\widetilde{T}].
	\end{equation*}
\end{lemma}
\begin{proof}
	We choose $\gamma>0$ sufficiently small such that for all unit-norm vector $\rvw$ satisfying $\rvw^\top\rvw_* \geq1-\gamma$, we have
	\begin{align}
	&\mathcal{N}(\rvw) \leq \mathcal{N}(\rvw_*), \|\nabla^2\mathcal{N}(\rvw)\|_2\leq L(L-1/2)\mathcal{N}(\rvw_*), \text{ and }
	\label{loc_ineq_Lhm}\\
	&(\nabla \mathcal{N}(t_1\rvw) -  \nabla \mathcal{N}(t_2\rvw_*))^\top\left( t_1\rvw - t_2\rvw_*\right) \leq L(L-1)\mathcal{N}(\rvw_*)t_2^{L-2}\|t_1\rvw - t_2\rvw_*\|_2^2, \forall t_2\geq t_1\geq 0.
	\label{inn_pd_Lhm}
	\end{align}
	The first inequality follows from Lemma \ref{lips_ineq}. The second inequality holds since, from Lemma \ref{second_kkt}, $\|\nabla^2\mathcal{N}(\rvw_*)\|_2 = L(L-1)\mathcal{N}(\rvw_*)$, and $\nabla^2\mathcal{N}(\rvw)$ is continuous in the neighborhood of $\rvw_*$. The third inequality follows from Lemma \ref{inn_pd_ineq}. 
	
	We further define $f(\rvw) = \mathcal{J}(\rmX;\rvw)^\top(\ell'(\mathcal{H}(\rmX;\rvw), \rvy) - \ell'(\mathbf{0}, \rvy))$. Then, note that
	\begin{equation*}
	\|f(\rvw)\|_2 \leq \| \mathcal{J}(\rmX;\rvw)\|_2\|(\ell'(\mathcal{H}(\rmX;\rvw), \rvy) - \ell'(\mathbf{0}, \rvy)\|_2 \leq Kn\| \mathcal{J}(\rmX;\rvw)\|_2\|\mathcal{H}(\rmX;\rvw), \rvy)\|_2 ,
	\end{equation*}
	where the first inequality follows from the Cauchy-Schwartz inequality, and the second follows from the smoothness of the loss  function (see \Cref{ass_loss}). Note that the RHS in the above inequality is $(2L-1)-$homogeneous in $\rvw$. Thus, there exists a $\beta>0$ such that
	\begin{equation}
	\|f(\rvw)\|_2 \leq  \beta\|\rvw\|_2^{2L-1}, \text{ for all }\rvw\in \sR^k.
	\label{lips_res_Lhm}
	\end{equation}
	Also, $ f(\rvw)$ is continuously differentiable in the neighborhood of $\rvw_*$. Thus, in a similar way as in the proof of Lemma \ref{init_dl3_2hm}, we may assume that for all vectors $\rvw\in\sR^k$ that satisfy $\rvw_*^\top\rvw/\|\rvw\|_2 \geq 1-\gamma, $ we have
	\begin{equation*}
	\|\nabla f(\rvw)\|_2 \leq  K\sum_{i=1}^n\|\nabla^2\mathcal{H}(\rvx_i;\rvw) \|_2 |\mathcal{H}(\rvx_i;\rvw) | + \|\nabla\mathcal{H}(\rvx_i;\rvw) \nabla\mathcal{H}(\rvx_i;\rvw) ^\top\|_2.
	\end{equation*}
	Note that the final upper bound in the above equation is $(2L-2)$-homogeneous. Thus, there exists a $\zeta>0$ such that for all vector $\rvw\in\sR^k$ satisfying $\rvw_*^\top\rvw/\|\rvw\|_2 \geq 1-\gamma, $ we have
	\begin{equation*}
	\|\nabla f(\rvw)\|_2 \leq \zeta\|\rvw\|_2^{2L-2}.
	\end{equation*}
	Furthermore, from the mean value theorem, we have
	\begin{align}
	\|f(\rvw_1) - f(\rvw_2)\|_2 &\leq \|\nabla f(\tilde{\rvw})\|_2\|\rvw_1-\rvw_2\|_2\nonumber \\
	&\leq \zeta\max(\|\rvw_2\|_2^{2L-2},\|\rvw_1\|_2^{2L-2})\|\rvw_1-\rvw_2\|_2,
	\label{lips_f_Lhm}
	\end{align}
	where $\rvw_*^\top\rvw_1/\|\rvw_1\|_2, \rvw_*^\top\rvw_2/\|\rvw_2\|_2 \geq1-\gamma. $
	
	Now, let $\rvq(t) = \bm{\psi}(t,\rva_\delta)$ and $\rvu(t) = \bm{\psi}(t,\delta\rvw_*)$. Then $\|\rvq(0) - \rvu(0)\|_2 \leq C_1\delta^{L+1},$ which implies $\|\rvq(0)\|_2\leq \delta+C_1\delta^{L+1}$. Therefore, for all sufficiently small $\delta>0$, we have
	\begin{equation*}
	\frac{\rvw_*^\top\rvq(0)}{\|\rvq(0)\|_2} = \frac{\rvw_*^\top\rvu(0) + \rvw_*^\top(\rvq(0) - \rvu(0))}{\|\rvq(0)\|_2} \geq\frac{\delta - C_1\delta^{L+1}}{\|\rvq(0)\|_2} \geq \frac{\delta - C_1\delta^{L+1}}{\delta+C_1\delta^{L+1}}   > 1-\gamma.
	\end{equation*}
	Define 
	$$\overline{T}_{1} = \min_{t\geq 0}\left\{t: \frac{\rvw_*^\top\rvq(t)}{\|\rvq(t)\|_2} = 1-\gamma\right\}.$$
	Recall that  $\rvq(t)$ satisfies
	\begin{equation*}
	\dot{\rvq} = -\mathcal{J}(\rmX;\rvq)^\top\ell'(\mathcal{H}(\rmX;\rvq), \rvy)  =  \nabla\mathcal{N}(\rvq)  - \mathcal{J}(\rmX;\rvq)^\top(\ell'(\mathcal{H}(\rmX;\rvq), \rvy) - \ell'(\mathbf{0}, \rvy)).
	\end{equation*}
	Multiplying the above equation by $\rvq^\top$ from the left we get
	\begin{equation*}
	\frac{1}{2}\frac{d\|\rvq\|_2^2}{dt} = L\mathcal{N}(\rvq) - L\mathcal{H}(\rmX;\rvq)^\top(\ell'(\mathcal{H}(\rmX;\rvq), \rvy) - \ell'(\mathbf{0}, \rvy))\leq L\mathcal{N}(\rvq) = L\mathcal{N}(\rvq/\|\rvq\|_2)\|\rvq\|_2^L,
	\end{equation*}
	where the first equality follows from Lemma \ref{euler_thm}, the first inequality is due to convexity of the loss function, and the second equality holds since $\mathcal{N}(\rvq)$ is $L$-homogeneous.
	
	Now, since $\rvw_*^\top\rvq(t)/\|\rvq(t)\|_2 \geq 1-\gamma,$ for all $t\in [0,\overline{T}_1]$, from \cref{loc_ineq_Lhm}, we get
	\begin{equation*}
	\frac{d\|\rvq\|_2^2}{dt} \leq L\mathcal{N}(\rvw_*)\|\rvq\|_2^{L},
	\end{equation*}
	which implies
	\begin{equation*}
	\|\rvq(t)\|_2^{L-2} \leq \frac{\|\rvq(0)\|_2^{L-2}}{1-t\|\rvq(0)\|_2^{L-2}L(L-2)\mathcal{N}(\rvw_*)}, \text{ for all }t\in [0,\overline{T}_{1} ].
	\end{equation*}
	Define $\rvs_q(t) = \frac{1}{\delta}\rvq\left(\frac{t}{\delta^{L-2}}\right)$ and $\eta = (1+C_1\delta^L)^{L-2}$, then, for all $t\in [0,\overline{T}_{1} ]$, 
	\begin{equation}
	\|\rvs_q(t)\|_2^{L-2} \leq \frac{(1+C_1\delta^L)^{L-2}}{1-t(1+C_1\delta^L)^{L-2}L(L-2)\mathcal{N}(\rvw_*)} =\frac{\eta}{1-t\eta L(L-2)\mathcal{N}(\rvw_*)},
	\label{bd_sq_Lhm}
	\end{equation}
	where we used $\|\rvq(0)\|_2\leq \delta(1+C_1\delta^{L})$. Also, since
	\begin{align*}
	\frac{d\rvs_q}{dt} = \frac{1}{\delta^{L-1}}\dot{\rvq}\left(\frac{t}{\delta^{L-2}}\right) &=  - \frac{1}{\delta^{L-1}}\mathcal{J}\left(\rmX;\rvq\left(\frac{t}{\delta^{L-2}}\right)\right)^\top\ell'\left(\mathcal{H}\left(\rmX;\rvq\left(\frac{t}{\delta^{L-2}}\right)\right), \rvy\right)\\
	&=  - \mathcal{J}\left(\rmX;\frac{1}{\delta}\rvq\left(\frac{t}{\delta^{L-2}}\right)\right)^\top\ell'\left(\mathcal{H}\left(\rmX;\rvq\left(\frac{t}{\delta^{L-2}}\right)\right), \rvy\right)\\
	& =  - \mathcal{J}\left(\rmX;\rvs_q(t)\right)^\top\ell'\left(\mathcal{H}\left(\rmX;\delta\rvs_q\left(t\right)\right), \rvy\right), 
	\end{align*}
	where the second equality follows from $(L-1)$-homogeneity of $\mathcal{J}\left(\rmX;\rvw\right)$, we have that $\rvs_q(t)$ is the solution of
	\begin{equation*}
	\dot{\rvs} = \nabla\mathcal{N}(\rvs)  - \mathcal{J}(\rmX;\rvs)^\top(\ell'(\mathcal{H}(\rmX;\delta\rvs), \rvy) - \ell'(\mathbf{0}, \rvy)) = \nabla\mathcal{N}(\rvs)  - f(\delta\rvs)/\delta^{L-1}, \rvs(0) = \rvq(0)/\delta.
	\end{equation*}
	Let $\rvz(t) = \frac{\eta^{\frac{1}{L-2}}\rvw_*}{(1-t\eta L(L-2)\mathcal{N}(\rvw_*))^{1/(L-2)}}$, which is the solution of
	\begin{equation*}
	\dot{\rvz} = \nabla \mathcal{N}(\rvz), \rvz(0) = \eta^{\frac{1}{L-2}}\rvw_*.
	\end{equation*}
	Note that, for $t\in [0,\overline{T}_1]$,
	\begin{align}
	&\left(\nabla \mathcal{N}\left({\rvs_q(t)}\right) - \nabla\mathcal{N}(\rvz(t))\right)^\top\left({\rvs_q(t)} - \rvz(t)\right)\nonumber\\
	&= \left(\nabla \mathcal{N}\left({\|\rvs_q(t)\|_2}\frac{\rvs_q(t)}{\|\rvs_q(t)\|_2}\right) - \nabla\mathcal{N}\left({\rvz(t)}\right)\right)^\top\left({\|\rvs_q(t)\|_2}\frac{\rvs_q(t)}{\|\rvs_q(t)\|_2}- {\rvz(t)}\right)\nonumber\\
	&\leq L(L-1)\mathcal{N}(\rvw_*)\|\rvz(t)\|_2^{L-2}\left\|{\|\rvs_q(t)\|_2}\frac{\rvs_q(t)}{\|\rvs_q(t)\|_2}- \|\rvz(t)\|_2\frac{\rvz(t)}{\|\rvz(t)\|_2}\right\|^2\nonumber\\
	& = L(L-1)\mathcal{N}(\rvw_*)\|\rvz(t)\|_2^{L-2}\left\|{\rvs_q(t)}- {\rvz(t)}\right\|^2,
	\end{align}
	where the inequality follows since, from \cref{bd_sq_Lhm}, ${\|\rvs_q(t)\|_2}\leq  \|\rvz(t)\|_2$, $\rvw_*^\top\rvq(t)/\|\rvq(t)\|_2 \geq 1-\gamma$ for all $t\in [0,\overline{T}_1]$, and from \cref{inn_pd_2hm}. Hence, for $t\in [0,\overline{T}_1]$,
	\begin{align*}
	\frac{1}{2}\frac{d}{dt} \left\|\rvs_q(t) - \rvz(t) \right\|_2^2&= \left(\rvs_q - \rvz\right)^\top\left(\dot{\rvs}_q - \dot{\rvz}\right)\\
	& = \left(\rvs_q - \rvz\right)^\top\left(\nabla\mathcal{N}(\rvs_q) - \nabla\mathcal{N}(\rvz)\right) - \left(\rvs_q - \rvz\right)^\top f(\delta\rvs_q)/\delta^{L-1}\\
	\leq L(L-1)&\mathcal{N}(\rvw_*)\|\rvz(t)\|_2^{L-2}\left\|\rvs_q(t) - \rvz(t) \right\|_2^2  + \beta\delta^L\left\|\rvs_q(t) - \rvz(t) \right\|_2\|\rvs_q(t)\|_2^{2L-1}.
	\end{align*}	
	Next, since for all $t\in [0,\overline{T}_1]$, 
	\begin{equation}
	\|\rvs_q(t)\|_2^{L-2} \leq \|\rvz(t)\|_2^{L-2}  = \frac{\eta}{1-t\eta L(L-2)\mathcal{N}(\rvw_*)},\label{bd_nm_s}
	\end{equation}
	we have
	\begin{equation*}
	\frac{d}{dt} \left\|\rvs_q(t) - \rvz(t) \right\|_2 \leq  
	\frac{\eta L(L-1)\mathcal{N}(\rvw_*)}{(1-t\eta L(L-2)\mathcal{N}(\rvw_*))}\left\|\rvs_q(t) - \rvz(t) \right\|_2 +  \beta\delta^{L}\|\rvz(t)\|_2^{2L-1}.
	\end{equation*}
	Using Lemma \ref{gronwall_ineq}, we get
	\begin{equation*}
	\left\|\rvs_q(t) - \rvz(t) \right\|_2 \leq \frac{1}{P(t)}\left(	\left\|\rvs_q(0) - \rvz(0) \right\|_2  + \beta\delta^{L}\int_0^t P(s)\|\rvz(t)\|_2^{2L-1} ds\right), 
	\end{equation*}
	where $P(t) = e^{-\int_0^t b(s) ds}$ and $b(t) =  \frac{\eta L(L-1)\mathcal{N}(\rvw_*)}{(1-t\eta L(L-2)\mathcal{N}(\rvw_*))}$.  Using Lemma \ref{int_factor}, we have
	\begin{equation*}
	P(t) = \left(1-t\eta L(L-2)\mathcal{N}(\rvw_*)\right)^{(L-1)/(L-2)},
	\end{equation*}
	which implies
	\begin{align*}
	\int_0^t P(s)\|\rvz(s)\|_2^{2L-1} ds &\leq \eta^{\frac{2L-1}{L-2}}\int_0^t \left(1-s\eta L(L-2)\mathcal{N}(\rvw_*)\right)^{\frac{L-1}{L-2}- \frac{2L-1}{L-2}} ds\\
	&= \eta^{\frac{2L-1}{L-2}}\int_0^t 1/\left(1-s\eta L(L-2)\mathcal{N}(\rvw_*)\right)^{\frac{L}{L-2}} ds\\
	&= \eta^{\frac{2L-1}{L-2}}\frac{L-2}{2\eta L(L-2)\mathcal{N}(\rvw_*)}\left( \frac{1}{\left(1-t\eta L(L-2)\mathcal{N}(\rvw_*)\right)^{\frac{2}{L-2}}} - 1\right)\\
	&\leq  \frac{\eta^{\frac{L+1}{L-2}}}{2L\mathcal{N}(\rvw_*)}\left( \frac{1}{\left(1-t\eta L(L-2)\mathcal{N}(\rvw_*)\right)^{\frac{2}{L-2}}}\right).
	\end{align*}
	Hence, for all $t\in [0,\overline{T}_1]$ and all sufficiently small $\delta>0$,
	\begin{align}
	\left\|\rvs_q(t) - \rvz(t) \right\|_2
	&\leq \frac{\left\|\rvs_q(0) - \rvz(0) \right\|_2 }{\left(1-t\eta L(L-2)\mathcal{N}(\rvw_*)\right)^{\frac{L-1}{L-2}}}+ \frac{\beta\delta^L\eta^{\frac{L+1}{L-2}}}{2L\mathcal{N}(\rvw_*)}\left( \frac{1}{\left(1-t\eta L(L-2)\mathcal{N}(\rvw_*)\right)^{\frac{L+1}{L-2}}}\right)\nonumber\\
	&\leq \frac{C_2\delta_2^L}{\left(1-t\eta L(L-2)\mathcal{N}(\rvw_*)\right)^{\frac{L+1}{L-2}}},\label{diff_s_z}
	\end{align}
	\sloppypar \noindent where $C_2$ is a sufficiently large constant, and the last inequality follows since  $\left(1-t\eta L(L-2)\mathcal{N}(\rvw_*)\right)^{\frac{L+1}{L-2}} \leq \left(1-t\eta L(L-2)\mathcal{N}(\rvw_*)\right)^{\frac{L-1}{L-2}}$ and, for all sufficiently small $\delta>0$,
	\begin{align*}
	\left\|{\rvs_q(0)}- \rvz(0) \right\|_2 = \left\|\frac{\rva_\delta}{\delta} - (1+C_1\delta^L)\rvw_* \right\|_2 \leq \left\|\frac{\rva_\delta}{\delta} - \rvw_* \right\|_2 + \left\|C_1\delta^L\rvw_* \right\|_2 = O(\delta^L).
	\end{align*} 
	Let $\tau_q(t) = \rvs_q(t) - \rvz(t) $, and define 
	\begin{equation*}
	h_1(t) = \frac{\eta^{1/(L-2)}}{\left(1-t\eta L(L-2)\mathcal{N}(\rvw_*)\right)^{1/(L-2)}}, h_2(t) = \frac{1}{\left(1-tL(L-2)\mathcal{N}(\rvw_*)\right)}.
	\end{equation*}
	By definition of $\overline{T}_1 $,
	\begin{align*}
	1-\gamma = \frac{\rvw_*^\top\rvs_q(\overline{T}_1)}{\|\rvs_q(\overline{T}_1)\|_2} \geq \frac{\rvw_*^\top\rvs_q(\overline{T}_1)}{h_1(\overline{T}_1)} = \frac{\rvw_*^\top\rvz(\overline{T}_1) + \rvw_*^\top\tau_q(\overline{T}_1)}{h_1(\overline{T}_1)} &\geq \frac{h_1(\overline{T}_1) -  \|\tau_p(\overline{T}_1)\|_2}{h_1(\overline{T}_1)}\\
	\geq 1 - &\frac{C_2\delta^L/\eta^{1/(L-2)}}{\left(1-\overline{T}_1\eta L(L-2)\mathcal{N}(\rvw_*)\right)^{\frac{L}{L-2}}},
	\end{align*}
	which implies, for some sufficiently large constant $\widetilde{C}_2$,
	\begin{align*}
	1 - \eta\overline{T}_1L(L-2)\mathcal{N}(\rvw_*) \leq \frac{C_2^{\frac{L-2}{L}}\delta^{L-2}/\eta^{1/L}}{\gamma^{\frac{L-2}{L}}} \leq \widetilde{C}_2\delta^{L-2}.
	\end{align*}
	Therefore, there exists a sufficiently large constant ${C}_3$ such that
	\begin{align*}
	\overline{T}_1 \geq \frac{1-\widetilde{C}_2\delta^{L-2}}{\eta L(L-2)\mathcal{N}(\rvw_*)} \geq \frac{1-{C}_3\delta^{L-2}}{ L(L-2)\mathcal{N}(\rvw_*)}, 
	\end{align*}
	where the second inequality follows since, for all sufficiently small $\delta>0$,
	\begin{align*}
	\frac{1-\widetilde{C}_2\delta^{L-2}}{(1+C_1\delta^L)^{L-2}}\geq 	\frac{1-\widetilde{C}_2\delta^{L-2}}{(1+2C_1(L-2)\delta^L)} \geq (1-\widetilde{C}_2\delta^{L-2})(1-2C_1(L-2)\delta^L)\geq 1-C_3\delta^{L-2},
	\end{align*}
	where ${C}_3$ is a sufficiently large constant. Hence,
	\begin{equation}
	\frac{\rvw_*^\top\rvs_q(t)}{\|\rvs_q(t)\|_2} \geq 1-\gamma, \text{ for all } t\in \left[0,\frac{\left(1-\delta^{L-2}C_3\right)}{L(L-2)\mathcal{N}(\rvw_*)}\right].
	\label{align_t_s}
	\end{equation}
	Next, recall $\rvu(t) = \bm{\psi}(t,\delta\rvw_*)$, and define $\rvs_u(t) = \frac{1}{\delta}\rvu\left(\frac{t}{\delta^{L-2}}\right)$. Note that
	\begin{equation*}
	\frac{\rvw_*^\top\rvs_u(0)}{\|\rvs_u(0)\|_2} = 1   > 1-\gamma.
	\end{equation*}
	Define 
	$$\overline{T}_{3} \coloneqq \min_{t\geq 0}\left\{t: \frac{\rvw_*^\top\rvs_u(t)}{\|\rvs_u(t)\|_2} = 1-\gamma\right\}.$$
	Then, following in a similar way as in the proof of $\rvs_q(t)$ above, we can show that there exists a sufficiently large constant $C_4$ such that, for all sufficiently small $\delta>0$,
	\begin{equation*}
	\frac{\rvw_*^\top\rvs_u(t)}{\|\rvs_u(t)\|_2} \geq 1-\gamma, \text{ for all } t\in \left[0,\frac{\left(1-\delta^{L-2}C_4\right)}{L(L-2)\mathcal{N}(\rvw_*)}\right], \text{ and }
	\end{equation*}
	\begin{equation*}
	\|\rvs_u(t)\|_2^{L-2} \leq \frac{\|\rvs_u(0)\|_2^{L-2}}{1-t\|\rvs_u(0)\|_2^{L-2}L(L-2)\mathcal{N}(\rvw_*)} = \frac{1}{1-tL(L-2)\mathcal{N}(\rvw_*)}. 
	\end{equation*}
	Now, choose $C_5$ large enough such that  $C_5^{L/(L-2)}>\zeta2^{(2L-2)/(L-2)}/(L\mathcal{N}(\rvw_*))$ and $C_5\geq \max(C_3,C_4)$, and define $\overline{T}_4 \coloneqq \frac{\left(1-\delta^{L-2}C_5\right)}{L(L-2)\mathcal{N}(\rvw_*)}$. Then, for all $t\in \left[0,\overline{T}_4\right]$, we have
	\begin{align*}
	&{\rvw_*^\top\rvs_q(t)}/{\|\rvs_q(t)\|_2},{\rvw_*^\top\rvs_u(t)}/{\|\rvs_u(t)\|_2} \geq 1-\gamma, 
	\\
	&\|\rvs_u(t)\|_2^{L-2} \leq \frac{1}{1-tL(L-2)\mathcal{N}(\rvw_*)} \leq \frac{1}{C_5\delta^{L-2}}, \text{ and }\\
	&\|\rvs_q(t)\|_2^{L-2} \leq \frac{\eta}{1-t\eta L(L-2)\mathcal{N}(\rvw_*)} = \frac{1}{1/\eta - 1+C_5\delta^{L-2}} \leq  \frac{2}{C_5\delta^{L-2}},
	\end{align*}
	where the last inequality follows since, for all sufficiently small $\delta>0$,
	\begin{equation}
	\frac{1}{\eta} = \frac{1}{(1+C_1\delta^L)^{L-2}}  \geq \frac{1}{1+2C_1(L-2)\delta^L}  \geq 1-2C_1(L-2)\delta^L \geq 1-C_5\delta^{L-2}/2.
	\label{eta_bd}
	\end{equation}
	Therefore, for $t\in [0,\overline{T}_4]$, we have
	\begin{align*}
	&\frac{1}{2}\frac{d}{dt} \left\|{\rvs_{q}(t)} - \rvs_u(t) \right\|_2^2 = \left({\rvs_{q}(t)} - \rvs_u\right)^\top\left({\dot{\rvs}_{q}} - \dot{\rvs}_u\right)\\
	& = \left({\rvs_{q}(t)} - \rvs_u\right)^\top\left(\nabla\mathcal{N}(\rvs_{q}) - \nabla\mathcal{N}(\rvs_u)\right) -\left({\rvs_{q}}- \rvs_u\right)^\top\left(f(\delta\rvs_{q}) - f(\delta\rvs_u)\right)/\delta^{L-1}\\
	&\leq (L(L-1/2)\mathcal{N}(\rvw_*))\left\|{\rvs_{q}(t)} - \rvs_u(t) \right\|_2^2\max(\|\rvs_q(t)\|_2^{L-2},\|\rvs_{u}(t)\|_2^{L-2}) + \\ &\ \ \ \ \zeta\delta^L\left\|{\rvs_{q}(t)}- \rvs_u(t) \right\|_2^2\max(\|\rvs_u(t)\|_2^{2L-2},\|\rvs_{q}(t)\|_2^{2L-2})\\
	& \leq \left(L(L-1/2)\mathcal{N}(\rvw_*) + \frac{\zeta 2^{L/(L-2)}}{C_5^{L/(L-2)}}\right)\left\|{\rvs_{q}(t)} - \rvs_u \right\|_2^2\max(\|\rvs_q(t)\|_2^{L-2},\|\rvs_{u}(t)\|_2^{L-2})\\
	& \leq L^2\mathcal{N}(\rvw_*)\left\|{\rvs_{q}(t)} - \rvs_u \right\|_2^2\max(\|\rvs_q(t)\|_2^{L-2},\|\rvs_{u}(t)\|_2^{L-2}),
	\end{align*}	
	where the first inequality follows from \cref{loc_ineq_Lhm}, $\nabla^2\mathcal{N}(\rvw) $ being $(L-2)$-homogeneous and \cref{lips_f_Lhm}. The second inequality uses $\delta^L\max(\|\rvs_q(t)\|_2^{L},\|\rvs_{u}(t)\|_2^{L}) \leq 2^{L/(L-2)}/C_5^{L/(L-2)}$, for all $t\in [0,\overline{T}_4]$. The last inequality follows from our assumption on $C_5.$ Now, let $\kappa =L^2\mathcal{N}(\rvw_*) $, then
	\begin{equation*}
	\frac{d}{dt} \left\|{\rvs_{q}(t)} - \rvs_u(t) \right\|_2 \leq \frac{\kappa\eta\left\|{\rvs_{q}(t)} - \rvs_u \right\|_2}{1-t\eta L(L-2)\mathcal{N}(\rvw_*)}.
	\end{equation*}
	Using Lemma \ref{gronwall_ineq}, we get
	\begin{equation*}
	\left\|\rvs_q(t) - \rvs_u(t) \right\|_2 \leq \frac{1}{P_1(t)}\left(	\left\|\rvs_q(0) - \rvs_u(0) \right\|_2 \right), 
	\end{equation*}
	where $P_1(t) = e^{-\int_0^t b_1(s) ds}$ and $b_1(t) =  \frac{\eta \kappa}{(1-t\eta L(L-2)\mathcal{N}(\rvw_*))}$.  From Lemma \ref{int_factor}, we know
	\begin{equation*}
	P_1(t) = \left(1-t\eta L(L-2)\mathcal{N}(\rvw_*)\right)^{L/(L-2)}.
	\end{equation*}
	Now, using \cref{eta_bd}, there exists a constant $C_6$ such that
	\begin{equation*}
	P_1(\overline{T}_4) = \left(1-\eta(1-\delta^{L-2}C_5) \right)^{L/(L-2)} \geq \left(\frac{\eta C_5\delta^{L-2}}{2}\right)^{L/(L-2)} \geq  C_6\delta^L.
	\end{equation*}
	Hence,
	\begin{equation*}
	\left\|\rvs_q(\overline{T}_4) - \rvs_u(\overline{T}_4) \right\|_2 \leq \frac{1}{C_6\delta^L}\left(	\left\|\rva_\delta/\delta - \rvw_* \right\|_2 \right) \leq C_1/C_6,
	\end{equation*}
	which implies
	\begin{equation*}
	\left\|\bm{\psi}\left( \frac{\left({1}/{\delta^{L-2}}-C_5\right)}{L(L-2)\mathcal{N}(\rvw_*)},{\rva}_\delta\right) - \bm{\psi}\left( \frac{\left({1}/{\delta^{L-2}}-C_5\right)}{L(L-2)\mathcal{N}(\rvw_*)},\delta{\rvw}_*\right) \right\|_2 \leq C_1\delta/C_6.
	\end{equation*}
	Since $\widetilde{T}$ is fixed, from Lemma \ref{err_bd_gf} and the above inequality, there exists a $C>0$ such that for all sufficiently small $\delta$, we have
	\begin{equation*}
	\left\|\bm{\psi}\left(t+ \frac{1}{L(L-2)\mathcal{N}(\rvw_*)}\left(\frac{1}{\delta^{L-2}}\right),{\rva}_\delta\right) - \bm{\psi}\left(t+ \frac{1}{L(L-2)\mathcal{N}(\rvw_*)}\left(\frac{1}{\delta^{L-2}}\right),\delta{\rvw}_*\right) \right\|_2 \leq C\delta,
	\end{equation*} 
	for all $t\in [-\widetilde{T},\widetilde{T}]$, which completes the proof.
\end{proof}
\begin{lemma}\label{lim_exists_Lhm}
	Consider the setting of \Cref{main_thm_2hm}. For any fixed $t\in (-\infty,\infty)$ and all sufficiently small $\delta_2\geq\delta_1>0$, there exists a $C>0$ such that
	\begin{equation*}
	\left\|\bm{\psi}\left(t+\frac{{1}/{\delta_1^{L-2}}}{L(L-2)\mathcal{N}(\rvw_*)},\delta_1\rvw_*\right) - \bm{\psi}\left(t+\frac{{1}/{\delta_2^{L-2}}}{L(L-2)\mathcal{N}(\rvw_*)},\delta_2\rvw_*\right)\right\|_2 \leq C\delta_2,
	\end{equation*}
	which implies $\rvp(t)$ exists for all $t\in (-\infty,\infty)$. Furthermore, let  $\delta_1\rightarrow 0$ and $\delta_2 = \delta$, then  
	\begin{equation*}
	\left\|\rvp(t)- \bm{\psi}\left(t+\frac{{1}/{\delta^{L-2}}}{L(L-2)\mathcal{N}(\rvw_*)},\delta\rvw_*\right)\right\|_2 \leq C\delta.
	\end{equation*}
	Also, $\mathcal{L}(\rvp(0)) \leq \mathcal{L}(\mathbf{0}) - \eta$, for some $\eta>0$.
\end{lemma}
\begin{proof}
	We choose $\gamma>0$ sufficiently small such that for all unit-norm vectors $\rvw$ satisfying $\rvw^\top\rvw_*\geq 1-\gamma$, we have $\mathcal{N}(\rvw) \leq \mathcal{N}(\rvw_*)$, and for all $t_2\geq t_1\geq 0,$
	\begin{equation*}
	\left(\nabla \mathcal{N}(t_1\rvw) -  \nabla \mathcal{N}(t_2\rvw_*)\right)^\top\left( t_1\rvw - t_2\rvw_*\right)\leq L(L-1)\mathcal{N}(\rvw_*)t_2^{L-2}\|t_1\rvw - t_2\rvw_*\|_2^2 .
	\end{equation*}
	The first and second inequalities follow from Lemma \ref{lips_ineq} and Lemma \ref{inn_pd_ineq}, respectively. For the sake of brevity, let $\rvu_1(t) = \bm{\psi}(t,{\delta}_1\rvw_*),  \rvu_2(t) = \bm{\psi}(t,\delta_2\rvw_*) \text{ and } \rvs_1(t) = \frac{1}{\delta_1}\rvu_1\left(\frac{t}{\delta_1^{L-2}}\right), \rvs_2(t) = \frac{1}{\delta_2}\rvu_2\left(\frac{t}{\delta_2^{L-2}}\right),$ where recall that $\delta_2\geq\delta_1>0$. Also, let
	\begin{equation*}
	\rvz(t) = \frac{\rvw_*}{(1-tL(L-2)\mathcal{N}(\rvw_*))^{\frac{1}{L-2}}},
	\end{equation*}
	which is the solution of 
	\begin{equation*}
	\dot{\rvz} = \nabla\mathcal{N}(\rvz), \rvz(0)=\rvw_*.
	\end{equation*}
	Now, $\rvs_1(0) = \rvw_*$. Define 
	$$T^*_1 = \min_{t\geq 0}\left\{t: \frac{\rvw_*^\top\rvs_1(t)}{\|\rvs_1(t)\|_2} = 1-\gamma\right\}.$$
	Then, following a a similar way as in the proof of Lemma \ref{init_dl3_Lhm} to get \cref{diff_s_z}, \cref{align_t_s}, and \cref{bd_nm_s}, we can show that there exists a $C_1, C_2>0$ such that, for all $t\in [0,\frac{\left(1-\delta_1^{L-2}C_1\right)}{L(L-2)\mathcal{N}(\rvw_*)}]$,
	\begin{equation}
	\left\|\rvs_1(t) - \rvz(t) \right\|_2 \leq   \frac{C_2\delta_1^L}{\left(1-t L(L-2)\mathcal{N}(\rvw_*)\right)^{\frac{L+1}{L-2}}},
	\frac{\rvw_*^\top\rvs_1(t)}{\|\rvs_1(t)\|_2} \geq 1-\gamma, \text{ and }
	\label{diff_u_Lhm}
	\end{equation}
	\begin{equation*}
	\|\rvs_1(t)\|_2^{L-2} \leq \frac{1}{1-tL(L-2)\mathcal{N}(\rvw_*)},
	\end{equation*}
	for all sufficiently small $\delta_1>0$. Now, suppose $\delta_2^{L-2}\leq \frac{1}{C_1}$, then
	\begin{equation*}
	T_2^*\coloneqq\frac{1}{L(L-2)\mathcal{N}(\rvw_*)}\left(1-\frac{\delta_1^{L-2}}{\delta_2^{L-2}}\right) \leq \frac{\left(1-\delta_1^{L-2}C_1\right)}{L(L-2)\mathcal{N}(\rvw_*)}.
	\end{equation*}
	Next, note that
	\begin{equation*}
	\delta_1\rvz\left(T_2^*\right) = \frac{\rvw_*}{(1-T_2^*L(L-2)\mathcal{N}(\rvw_*))^{\frac{1}{L-2}}} =  \frac{\delta_1\rvw_*}{(1 - (1-\delta_1^{L-2}/\delta_2^{L-2}))^{\frac{1}{L-2}} } =  \delta_2\rvw_*, \text{ and } 
	\end{equation*}
	\begin{equation*}
	\|\rvs_1(T_2^*) - \rvz(T_2^*) \|_2 \leq \left( \frac{C_2\delta_1^L}{\left(1-T_2^* L(L-2)\mathcal{N}(\rvw_*)\right)^{\frac{L+1}{L-2}}}\right) =  \frac{C_2\delta_2^{L+1}}{\delta_1}.
	\end{equation*}
	Hence,
	\begin{equation}
	\left\|\bm{\psi}\left(\frac{{1}/{\delta_1^{L-2}}-{1}/{\delta_2^{L-2}}}{L(L-2)\mathcal{N}(\rvw_*)},\delta_1\rvw_*\right) - \delta_2\rvw_*\right\|_2  = \|\delta_1\rvs_1(T_2^*) - \delta_1\rvz(T_2^*) \|_2 \leq C_2\delta_2^{L+1}.
	\label{bd_diff_dlt}
	\end{equation}
	Now, note that
	\begin{equation*}
	\bm{\psi}\left(t + \frac{{1}/{\delta_1^{L-2}}}{L(L-2)\mathcal{N}(\rvw_*)},\delta_1\rvw_*\right) = \bm{\psi}\left(t + \frac{{1}/{\delta_2^{L-2}}}{L(L-2)\mathcal{N}(\rvw_*)},\bm{\psi}\left(\frac{{1}/{\delta_1^{L-2}}-{1}/{\delta_2^{L-2}}}{L(L-2)\mathcal{N}(\rvw_*)},\delta_1\rvw_*\right)\right) .
	\end{equation*}  
	Combining the above equality with \cref{bd_diff_dlt} and Lemma \ref{init_dl3_Lhm}, we get that for any fixed $t\in (-\infty,\infty)$, there exist a constant $C>0$ such that
	\begin{equation*}
	\left\|\bm{\psi}\left(t+\frac{{1}/{\delta_1^{L-2}}}{L(L-2)\mathcal{N}(\rvw_*)},\delta_1\rvw_*\right) - \bm{\psi}\left(t+\frac{{1}/{\delta_2^{L-2}}}{L(L-2)\mathcal{N}(\rvw_*)},\delta_2\rvw_*\right)\right\|_2 \leq C\delta_2,
	\end{equation*}
	which implies $\rvp(t)$ exists for all $t\in (-\infty,\infty)$.
	
	We next prove that $\mathcal{L}(\rvp(0)) \leq \mathcal{L}(\mathbf{0}) - \eta$, for some $\eta>0$. Let $\rvu(t) = \bm{\psi}(t,\delta\rvw_*)$ and $\rvz(t) = {\rvw_*}/{(1-tL(L-2)\mathcal{N}(\rvw_*))^{\frac{1}{L-2}}}.$ From \cref{diff_u_Lhm}, there exist constants $B_1, B_2$ such that 
	\begin{equation}
	\left\|\frac{\rvu(t/\delta^{L-2})}{\delta} - \rvz(t) \right\|_2
	\leq\frac{B_2\delta^L}{\left(1-t L(L-2)\mathcal{N}(\rvw_*)\right)^{\frac{L+1}{L-2}}}, \text{ for all }t\in \left[0,\frac{\left(1-B_1\delta^{L-2}\right)}{L(L-2)\mathcal{N}(\rvw_*)}\right].
	\label{diff_u_Lhm1}
	\end{equation}
	Define $\alpha\coloneqq\max_{\|\rvw\|_2 = 1} \|\mathcal{H}(\rmX,\rvw)\|_2$. Similar to \cref{loss_ub_2hm}, we have
	\begin{equation}
	\mathcal{L}(\rvw) \leq \mathcal{L}(\mathbf{0}) + Kn\|\mathcal{H}(\rmX,\rvw)\|_2^2 - \mathcal{N}(\rvw) \leq \mathcal{L}(\mathbf{0}) + \alpha^2 Kn\|\rvw\|_2^{2L} - \mathcal{N}(\rvw).
	\label{loss_ub_Lhm}
	\end{equation}
	Now, choose $\epsilon>0$ sufficiently small such that
	\begin{equation*}
	\epsilon\leq \frac{1}{B_1}, \text{ and } \alpha^2 Kn({\epsilon}^\frac{1}{2(L-2)}+B_1\epsilon^\frac{L+1/2}{L-2})^{2L}  \leq \mathcal{N}(\rvw)-\frac{\mathcal{N}(\rvw_*)}{2}, \text{ if } \|\rvw-\rvw_*\|_2\leq \epsilon^\frac{L}{L-2}B_1,
	\end{equation*}
	where the second inequality holds true since $\mathcal{N}(\rvw)$ is continuous. Note that
	\begin{equation*}
	\frac{\left(1-B_1\delta^{L-2}\right)}{L(L-2)\mathcal{N}(\rvw_*)}\geq \frac{\left(1-\delta^{L-2}/\epsilon\right)}{L(L-2)\mathcal{N}(\rvw_*)}\coloneqq T_2.
	\end{equation*} 
	Hence, using \cref{diff_u_Lhm1},  $\|\rvu(T_2/\delta^{L-2})\|_2\leq \epsilon^{1/(L-2)}+B_2\epsilon^{(L+1)/(L-2)}$. Also, if we define $\tau\coloneqq\rvu(T_2/\delta^{L-2})-\delta\rvz(T_2/\delta^{L-2})$, then, $\|\tau\|_2 \leq B_2\epsilon^{(L+1)/(L-2)}$. Therefore, using \cref{loss_ub_Lhm}, we get
	\begin{align*}
	&\mathcal{L}\left(\rvu\left(\frac{T_2}{\delta^{L-2}}\right)\right) \leq \mathcal{L}(\mathbf{0}) + \alpha^2 Kn\left\|\rvu\left(\frac{T_2}{\delta^{L-2}}\right)\right\|_2^{2L} - \mathcal{N}\left(\rvu\left(\frac{T_2}{\delta^{L-2}}\right)\right),\\
	&\leq\mathcal{L}(\mathbf{0}) + \alpha^2 Kn(\epsilon^{1/(L-2)}+B_2\epsilon^{(L+1)/(L-2)})^{2L} - \mathcal{N}\left(\rvu\left(\frac{T_2}{\delta^{L-2}}\right)\right)\\
	&= \mathcal{L}(\mathbf{0}) + \epsilon^\frac{L}{L-2}\left(\alpha^2 Kn({\epsilon}^\frac{1}{2(L-2)}+B_1\epsilon^\frac{L+1/2}{L-2})^{2L} - \mathcal{N}\left(\frac{1}{\epsilon^\frac{1}{L-2}}\rvu\left(\frac{T_2}{\delta^{L-2}}\right)\right)\right)\\
	& \leq  \mathcal{L}(\mathbf{0}) - \epsilon^\frac{L}{L-2}\frac{\mathcal{N}(\rvw_*)}{2},
	\end{align*}
	where the last inequality follows from the choice of $\epsilon$ and since $\|\rvu(T_2/\delta^{L-2})/\epsilon^{1/(L-2)}-\rvw_*\|_2 =\tau/\epsilon^{1/(L-2)} \leq B_1\epsilon^{L/(L-2)}.$ Let $\eta = \epsilon^{L/(L-2)}\mathcal{N}(\rvw_*)/2$, then, from the above equation,
	\begin{equation*}
	\mathcal{L}\left(\bm{\psi}\left(\frac{\left(1/\delta^{L-2}-\epsilon\right)}{L(L-2)\mathcal{N}(\rvw_*)},\delta\rvw_*\right)\right) \leq \mathcal{L}(\mathbf{0}) - \eta,
	\end{equation*} 
	where note that $\epsilon>0$ and fixed, and thus, $\eta$ is fixed. Since the loss decreases with time, the above equation implies
	\begin{equation*}
	\mathcal{L}\left(\bm{\psi}\left(\frac{1/\delta^{L-2}}{L(L-2)\mathcal{N}(\rvw_*)},\delta\rvw_*\right)\right)  \leq\mathcal{L}\left(\bm{\psi}\left(\frac{\left(1/\delta^{L-2}-\epsilon\right)}{L(L-2)\mathcal{N}(\rvw_*)},\delta\rvw_*\right)\right) \leq \mathcal{L}(\mathbf{0}) - \eta.
	\end{equation*} 
	Taking $\delta\to 0$, gives us $\mathcal{L}(\rvp(0)) \leq \mathcal{L}(\mathbf{0}) - \eta,$ which completes the proof.
\end{proof}
\begin{lemma}\label{init_align_Lhm}
	Consider the setting in \Cref{main_thm_Lhm}. There exists $T_1, a_1 \geq 0$ such that for all sufficiently small $\delta>0$,
	\begin{equation*}
	\left\|\bm{\psi}\left(\frac{T_1}{\delta^{L-2}} + \frac{T}{\widetilde{\delta}^{L-2}}\left(1 - \widetilde{\delta}^{\frac{2(L-2)L^2\mathcal{N}(\rvw_*)}{2L^2\mathcal{N}(\rvw_*)+\Delta}}\right),\delta\rvw_0\right) - b_\delta\widetilde{\delta}^{\frac{\Delta}{2L^2\mathcal{N}(\rvw_*) +\Delta}}\rvw_* \right\|_2 \leq a_1\widetilde{\delta}^{\frac{(L+1)\Delta}{2L^2\mathcal{N}(\rvw_*) + \Delta}},
	\end{equation*}
	where $\widetilde{\delta}\in (A_2\delta-A_1\delta^{L+1},A_2\delta+A_1\delta^{L+1})$, for some positive constants $A_2, A_1$. Also, $T\geq \frac{1}{L(L-2)\mathcal{N}(\rvw_*)}$, and $b_\delta^{L-2}\in \left[\frac{1}{TL(L-2)\mathcal{N}(\rvw_*)},1\right]$ depends on $\delta$.
\end{lemma}
\begin{proof}
	Choose $\alpha\in(0,1)$ sufficiently small such that 
	\begin{equation}
	0<\frac{(L(L-1)\mathcal{N}(\rvw_*)+\alpha)(1+\alpha)^{L-2}}{L(L-2)(\mathcal{N}(\rvw_*)-\alpha)} +1 < \frac{2L-1}{L-2}.
	\label{alp_choice}
	\end{equation}
	Next, we choose $\gamma>0$ sufficiently small such that for all unit-norm vector $\rvw$ satisfying $\rvw^\top\rvw_*\geq1-\gamma$, we have
	\begin{align}
	&\mathcal{N}(\rvw_*)-\alpha\leq\mathcal{N}(\rvw)\leq\mathcal{N}(\rvw_*), \|\nabla^2\mathcal{N}(\rvw)\|_2 \leq L(L-1)\mathcal{N}(\rvw_*)+\alpha,
	\label{hess_bd_Lhm}\\
	&\rvw_*^\top\nabla\mathcal{N}(\rvw) - L\mathcal{N}(\rvw)\rvw_*^\top\rvw - \frac{\Delta}{2}\|\rvw - \rvw_*\|_2^2 \geq 0,\label{kl_ineq_pf_Lhm}
	\end{align}
	and for all $t_2\geq t_1\geq 0$, 
	\begin{equation}
	\left(\nabla \mathcal{N}(t_1\rvw) -  \nabla \mathcal{N}(t_2\rvw_*)\right)^\top\left( t_1\rvw - t_2\rvw_*\right) \leq L(L-1)\mathcal{N}(\rvw_*)t_2^{L-2}\|t_1\rvw - t_2\rvw_*\|_2^2 .
	\label{inn_pd_Lhm_1}
	\end{equation}
	The first and second inequalities follow from Lemma \ref{lips_ineq} and Lemma \ref{second_kkt}. The third and fourth inequalities follow from Lemma \ref{lips_ineq} and Lemma \ref{inn_pd_ineq}. Define $f(\rvw) = \mathcal{J}(\rmX;\rvw)^\top(\ell'(\mathcal{H}(\rmX;\rvw), \rvy) - \ell'(\mathbf{0}, \rvy))$, then, as shown in the proof of Lemma \ref{init_dl3_Lhm}, there exists a $\beta>0$ such that
	\begin{equation}
	\|f(\rvw)\|_2 \leq  \beta\|\rvw\|_2^{2L-1}, \text{ for all }\rvw\in \sR^d.
	\label{lips_res_Lhm_1}
	\end{equation}
	Let $\rvz_1(t)$ denote the solution of
	\begin{equation}
	\dot{\rvz} = \nabla\mathcal{N}(\rvz), \rvz(0) = \rvw_0.
	\end{equation}
	Since $\rvw_0\in\mathcal{S}(\rvw_*)$, $\rvz_1(t)$ converges to $\rvw_*$ in direction. Thus, we can assume	 there exists some time $T_\gamma$ and a constant $B_1$ such that
	\begin{equation*}
	\frac{\rvw_*^\top\rvz_1(T_\gamma)}{\|\rvz_1(T_\gamma)\|_2} = 1-\gamma^2/8 \text{ and } \|\rvz_1(t)\|_2\leq B_1, \text{ for all }t\in [0,T_\gamma].
	\end{equation*} 
	Define $\rvw(t) = \bm{\psi}(t,\delta\rvw_0) \text{ and } \rvs_w(t) = \frac{1}{\delta}\rvw\left(\frac{t}{\delta^{L-2}}\right),$ then $\rvs_w(t)$ is the solution of
	\begin{equation*}
	\dot{\rvs} =  \nabla\mathcal{N}(\rvs)  - \mathcal{J}(\rmX;\rvs)^\top(\ell'(\mathcal{H}(\rmX;\delta\rvs), \rvy) - \ell'(\mathbf{0}, \rvy)) = \nabla\mathcal{N}(\rvs)  -f(\delta\rvs)/\delta^{L-1}, \rvs(0) = \rvw_0.	
	\end{equation*}
	Note that $\|\rvs_w(0)-\rvz_1(0)\|_2 = 0$. Define
	\begin{equation*}
	T_1^* = \min_{t\geq 0}\{t:\|\rvs_w(t) - \rvz_1(t)\|_2 = \gamma^2/4\}.
	\end{equation*}
	Then, for all $t\in [0,T_1^*]$, $\|\rvs_w(t) - \rvz_1(t)\|_2 \leq \gamma^2/4$. Next, we show that for all sufficiently small $\delta>0$, $T_1^*> T_\gamma$. For the sake of contradiction, suppose $T_1^*\leq T_\gamma$.  Let $B_2 = B_1+\gamma^2/4$, then $\|\rvs_w(t)\|_2\leq B_2$, for all $t\in [0,T_1^*]$. Since $\nabla\mathcal{N}(\rvs)$ is locally Lipschitz, there exists a $\mu_1>0$ such that
	\begin{equation*}
	\|\nabla\mathcal{N}(\rvs_w(t))-\nabla\mathcal{N}(\rvz_1(t))\|_2 \leq \mu_1\|\rvs_w(t)-\rvz_1(t)\|_2, \text{ for all } t\in [0,T_1^*].
	\end{equation*}
	Hence, for all $t\in [0,T_1^*]$,
	\begin{align*}
	\frac{1}{2}\frac{d}{dt}\left\|\rvs_w(t)-\rvz_1(t)\right\|_2^2 &= (\rvs_w-\rvz_1)^\top(\dot{\rvs}_w-\dot{\rvz}_1)\\
	&= (\rvs_w-\rvz_1)^\top(\nabla\mathcal{N}(\rvs_w)-\nabla\mathcal{N}(\rvz_1)) -  (\rvs_w-\rvz_1)^\top f(\delta\rvs_w)/\delta^{L-1}\\
	&\leq \mu_1\|\rvs_w-\rvz_1\|_2^2 + \beta\delta^L\|\rvs_w-\rvz_1\|_2\|\rvs_w\|_2^{2L-1} \\
	&\leq \mu_1\|\rvs_w-\rvz_1\|_2^2 + \beta\delta^L B_2^{2L-1}\|\rvs_w-\rvz_1\|_2,
	\end{align*}
	which implies
	\begin{equation*}
	\frac{d}{dt}\left\|\rvs_w(t)-\rvz_1(t)\right\|_2 \leq \mu_1\|\rvs_w-\rvz_1\|_2 + \beta\delta^L B_2^{2L-1}.
	\end{equation*}
	Hence, using Lemma \ref{gronwall_ineq}, we get
	\begin{equation*}
	\|\rvs_w(t)-\rvz_1(t)\|_2  \leq \beta\delta^L B_2^{2L-1}e^{\mu_1t}\int_0^t e^{-\mu_1 s}ds\leq B_3\delta^Le^{\mu_1t},
	\end{equation*}
	where $B_3$ is a sufficiently large constant. From the definition of $T_1^*$, we get
	\begin{equation*}
	\gamma^2/4=\|\rvs_w(T_1^*)-\rvz_1(T_1^*)\|_2 \leq B_3\delta^Le^{\mu_1T_1^*} \leq B_3\delta^Le^{\mu_1T_\gamma},
	\end{equation*} 
	where the second inequality uses $T_1^*\leq T_\gamma$.  Clearly, the above inequality is not possible if $\delta$ is sufficiently small, leading to a contradiction. Hence, $T_1^*> T_\gamma$. Let $A_1 = B_3e^{\mu_1T_\gamma}$, then
	\begin{equation}
	\|\rvs_w(T_\gamma)-\rvz_1(T_\gamma)\|_2 \leq A_1\delta^L.
	\label{bd_sw}
	\end{equation}
	Next, let $\tau_1(t) = \rvs_w(t) - \rvz_1(t)$, then, for all sufficiently small $\delta>0$, we have
	\begin{align*}
	\frac{\rvw_*^\top\rvs_w(T_\gamma)}{\|\rvs_w(T_\gamma)\|_2} = \frac{\rvw_*^\top\rvz_1(T_\gamma) +\rvw_*^\top\tau_1(T_\gamma)}{\|\rvs_w(T_\gamma)\|_2}&\geq \frac{(1-\gamma^2/8)\|\rvz_1(T_\gamma)\|_2 -A_1\delta^L}{\|\rvs_w(T_\gamma)\|_2}\\
	&\geq \frac{(1-\gamma^2/8)\|\rvz_1(T_\gamma)\|_2 -A_1\delta^L}{\|\rvz_1(T_\gamma)\|_2 + A_1\delta^L} \geq 1-\gamma^2/4.
	\end{align*}
	Let $A_2 = \|\rvz_1(T_\gamma)\|_2,$ where note that $A_2$ is a constant that does not depend on $\delta$. Define $\rvw_1 = \rvs_w(T_\gamma)/\|\rvs_w(T_\gamma)\|_2$ and $\widetilde{\delta} = \delta\|\rvs_w(T_\gamma)\|_2$. From \cref{bd_sw}, we have
	\begin{equation*}
	\widetilde{\delta}\in (A_2\delta-A_1\delta^{L+1}, A_2\delta+A_1\delta^{L+1}),
	\end{equation*}
	thus, $\widetilde{\delta}$ can be made sufficiently small by choosing $\delta$ sufficiently small. Next, let $\rvu(t) = \bm{\psi}(t+T_\gamma/\delta^{L-2},\delta\rvw_0)$. Then,
	\begin{equation*}
	\rvu(t) = \bm{\psi}\left(t+\frac{T_\gamma}{\delta^{L-2}},\delta\rvw_0\right) = \bm{\psi}\left(t,\bm{\psi}\left(\frac{T_\gamma}{\delta^{L-2}},\delta\rvw_0\right)\right) = \bm{\psi}\left(t,\delta\|\rvs_w(T_\gamma)\|_2\rvw_1\right) = \bm{\psi}\left(t,\widetilde{\delta}\rvw_1\right).
	\end{equation*} 
	Define $\rvs_u(t) = \frac{1}{\widetilde{\delta}}\rvu\left(\frac{t}{\widetilde{\delta}^{L-2}}\right),$ then $\rvs_u(t)$ is the solution of 
	\begin{equation*}
	\dot{\rvs} = \nabla\mathcal{N}(\rvs)  - \mathcal{J}(\rmX;\rvs)^\top(\ell'(\mathcal{H}(\rmX;\widetilde{\delta}\rvs), \rvy) - \ell'(\mathbf{0}, \rvy)) =  \nabla\mathcal{N}(\rvs)  - f(\widetilde{\delta}\rvs)/\widetilde{\delta}^{L-1}, \rvs(0) =  \rvw_1.
	\end{equation*}
	Define $\rvz_2(t)$ to be the solution of
	\begin{equation}
	\dot{\rvz} = \nabla\mathcal{N}(\rvz), \rvz(0) = \rvw_1.
	\end{equation}
	Since  $\rvz_2(0)^\top\rvw_*\geq 1-\gamma^2/4>1-\gamma$, we may assume without loss of generality that $\gamma>0$ is sufficiently small such that Lemma \ref{rate_ineq} is applicable. Hence, there exists a $T>0$ such that
	\begin{equation}
	\left(1 - \frac{\rvw_*^\top\rvz_2(t)}{\|\rvz_2(t)\|_2}\right)  \leq  \gamma\left(1-\frac{t}{T}\right)^{\frac{\Delta}{L(L-2)\mathcal{N}(\rvw_*)}} \text{ and } \|\rvz_2(t)\|_2 = \frac{\|\rvg(t)\|_2}{(T-t)^{\frac{1}{L-2}}},
	\label{z2_alg_nm_bd}
	\end{equation}
	for all $t\in [0,T)$, where $\|\rvz_2(t)\|$ is an increasing function and $\|\rvg(t)\|_2$ is a decreasing function in $[0,T)$, and $	\|\rvg(0)\|_2^{L-2} = T, \|\rvg(T)\|_2^{L-2} = 1/(L(L-2)\mathcal{N}(\rvw_*)).$	Also, from Lemma \ref{rate_ineq} and \cref{hess_bd_Lhm}, we have
	\begin{equation}
	T\leq \frac{1}{L(L-2)\mathcal{N}(\rvw_1)} \leq \frac{1}{L(L-2)(\mathcal{N}(\rvw_*)-\alpha)}.
	\label{bd_g0}
	\end{equation}
	Next, define $\eta = \frac{2(L-2)L^2\mathcal{N}(\rvw_*)}{2L^2\mathcal{N}(\rvw_*)+\Delta}$ and  $\overline{T} = T(1 - \widetilde{\delta}^{\eta})$, then
	\begin{equation*}
	\left(1 - \frac{\rvw_*^\top\rvz_2(\overline{T} )}{\|\rvz_2(\overline{T} )\|_2}\right) \leq {\gamma\widetilde{\delta}^{\frac{2L\Delta}{2L^2\mathcal{N}(\rvw_*)+\Delta}}}, \text{ and } \|\rvz_2(\overline{T})\|_2^{L-2} = \frac{\|\rvg(\overline{T})\|_2^{L-2}}{T\widetilde{\delta}^\eta}.	
	\end{equation*}
	Therefore,
	\begin{equation}
	\left\|\frac{\rvz_2(\overline{T} )}{\|\rvz_2(\overline{T} )\|_2} - \rvw_*\right\|_2^2 = 2\left(1-\frac{\rvw_*^\top\rvz_2(\overline{T} )}{\|\rvz_2(\overline{T} )\|_2}\right)\leq {2\gamma\widetilde{\delta}^{\frac{2L\Delta}{2L^2\mathcal{N}(\rvw_*)+\Delta}}}.
	\label{nm_diff_bd}
	\end{equation}
	Define $\mu_2 = (L(L-1)\mathcal{N}(\rvw_*)+\alpha)(1+\alpha)^{L-2}$, and let
	\begin{equation*}
	C_1 = \frac{2\beta(1+\alpha)^{2L-1}}{\left(\frac{L+1}{L-2} - \mu_2\|\rvg(T)\|_2^{L-2}\right)}\frac{\|\rvg(0)\|_2^{\frac{2L-1}{L-2}}}{\|\rvg(T)\|_2^{\frac{L+1}{L-2}}}>0, 
	\end{equation*}
	where the inequality holds since, from \cref{alp_choice}, $\frac{L+1}{L-2} - \mu_2\|\rvg(T)\|_2^{L-2} >0$. Now, since $\rvw_*^\top\rvs_u(0)/\|\rvs_u(0)\|_2 \geq 1-\gamma^2/4$ and $ 0 = \left\|\rvs_u({0}) - \rvz_2({0})\right\|_2 < C_1\widetilde{\delta}^L\|\rvz_2({0})\|_2^{L+1}$, we define
	\begin{equation*}
	T_2^* = \min_{t\geq 0}\left\{t:\frac{\rvw_*^\top\rvs_u(t)}{\|\rvs_u(t)\|_2}= 1-\gamma\right\}, T_3^* = \min_{t\geq 0}\{t:\left\|\rvs({t}) - \rvz_2({t})\right\|_2 = C_1\widetilde{\delta}^L\|\rvz_2({t})\|_2^{L+1}\}, 
	\end{equation*} 
	and $T_4^* = \min(T_2^*,T_3^*).$
	Thus, for all $t\in [0,{T}_4^*]$, we have
	\begin{equation*}
	\left\|\rvs_u({t}) - \rvz_2({t})\right\|_2 \leq C_1\widetilde{\delta}^L\|\rvz_2({t})\|_2^{L+1}, \text{ and } \frac{\rvs_u({t})^\top\rvw_*}{\|\rvs_u(t)\|_2} \geq 1-\gamma.
	\end{equation*}
	We next show that $\overline{T}\leq T_4^*$, for all sufficiently small $\delta>0$. For the sake of contradiction, let $\overline{T}> T_4^*$. Now, since $\|\rvz_2(t)\|_2$ increases with time, for all $t\in [0,T_4^*]$, we have
	\begin{equation*}
	C_1\widetilde{\delta}^L\|\rvz_2({t})\|_2^{L} \leq C_1\widetilde{\delta}^L\|\rvz_2(\overline{T})\|_2^{L} \leq \frac{C_1\|\rvg(\overline{T})\|_2^L}{T^{\frac{L}{L-2}}}\widetilde{\delta}^{L - \frac{L\eta}{(L-2)}} = \frac{C_1\|\rvg(\overline{T})\|_2^L}{T^{\frac{L}{L-2}}}\widetilde{\delta}^{ \frac{L\Delta}{2L^2\mathcal{N}(\rvw_*)+\Delta}} \leq \alpha, 
	\end{equation*}
	for all sufficiently small $\delta>0$. Since $\left\|\rvs_u({t}) - \rvz_2({t})\right\|_2 \leq C_1\widetilde{\delta}^L\|\rvz_2({t})\|_2^{L+1} $, we get
	\begin{equation*}
	\left\|\rvs_u({t}) - \rvz_2({t})\right\|_2  \leq \alpha\|\rvz_2({t})\|_2, \text{ for all }t\in [0,T_4^*],
	\end{equation*}
	which implies
	\begin{equation}
	0<(1-\alpha)\left\| \rvz_2({t})\right\|_2\leq \left\|\rvs_u({t}) \right\|_2\leq (1+\alpha)\left\| \rvz_2({t})\right\|_2.
	\label{su_bd_Lhm}
	\end{equation}
	Then, for all $t\in [0,T_4^*]$,
	\begin{align*}
	\frac{d}{dt}\frac{\rvw_*^\top\rvs_u(t)}{\|\rvs_u(t)\|_2}&= \rvw_*^\top\left(\rmI - \frac{\rvs_u\rvs_u^\top}{\|\rvs_u\|_2^2}\right)\frac{\dot{\rvs}_u }{\|\rvs_u\|_2}\\
	&= \rvw_*^\top\left(\rmI - \frac{\rvs_u\rvs_u^\top}{\|\rvs_u\|_2^2}\right)\frac{\nabla \mathcal{N}(\rvs_u) }{\|\rvs_u\|_2} -  \rvw_*^\top\left(\rmI - \frac{\rvs_u\rvs_u^\top}{\|\rvs_u\|_2^2}\right)\frac{f(\widetilde{\delta}\rvs_u)}{\widetilde{\delta}^{L-1}\|\rvs_u\|_2} \\
	&\geq \frac{\rvw_*^\top\nabla \mathcal{N}(\rvs_u)}{\|\rvs_u\|_2} - \frac{L\rvw_*^\top\rvs_u\mathcal{N}(\rvs_u)}{\|\rvs_u\|_2^3} - \beta\widetilde{\delta}^L\|\rvs_u\|_2^{2L-2}\\
	&= \|\rvs_u\|_2^{L-2}\left(\rvw_*^\top\nabla \mathcal{N}\left(\frac{\rvs_u}{\|\rvs_u\|_2}\right)  - L\rvw_*^\top\left(\frac{\rvs_u}{\|\rvs_u\|_2}\right)\mathcal{N}\left(\frac{\rvs_u}{\|\rvs_u\|_2}\right) \right)-\beta\widetilde{\delta}^L\|\rvs_u\|_2^{2L-2}\\
	& \geq  \|\rvs_u\|_2^{L-2}\frac{\Delta}{2}\left\|\frac{\rvs_u(t)}{\|\rvs_u(t)\|_2} - \rvw_*\right\|_2^2 -\beta\widetilde{\delta}^L\|\rvs_u\|_2^{2L-2}\\
	& \geq  -\beta\widetilde{\delta}^L\|\rvs_u\|_2^{2L-2}\geq -\beta(1+\alpha)^{2L-2}\widetilde{\delta}^L\|\rvz_2(t)\|_2^{2L-2},
	\end{align*}
	where the second inequality uses \cref{kl_ineq_pf_Lhm}, and the last inequality uses \cref{su_bd_Lhm}. Hence,
	\begin{align*}
	\frac{\rvw_*^\top\rvs_u(t)}{\|\rvs_u(t)\|_2}&\geq \frac{\rvw_*^\top\rvs_u(0)}{\|\rvs_u(0)\|_2} - \beta\widetilde{\delta}^L(1+\alpha)^{2L-2}\int_0^t\frac{\|\rvg(s)\|_2^{2L-2}ds}{(T-s)^{\frac{2L-2}{L-2}}}\\
	&\geq \frac{\rvw_*^\top\rvs_u(0)}{\|\rvs_u(0)\|_2} - \beta\|\rvg(0)\|_2^{2L-2}\widetilde{\delta}^L(1+\alpha)^{2L-2}\int_0^t\frac{ds}{(T-s)^{\frac{2L-2}{L-2}}}\\
	&\geq \frac{\rvw_*^\top\rvs_u(0)}{\|\rvs_u(0)\|_2} - \frac{\beta(L-2)\|\rvg(0)\|_2^{2L-2}\widetilde{\delta}^L(1+\alpha)^{2L-2}}{L(T-t)^{\frac{L}{L-2}}},
	\end{align*}
	where the second inequality holds since $\rvg(t)$ is a decreasing function. Since $\overline{T}> T_4^*$, we get
	\begin{align}
	\frac{\rvw_*^\top\rvs_u( T_4^*)}{\|\rvs_u( T_4^*)\|_2} - \frac{\rvw_*^\top\rvs_u(0)}{\|\rvs_u(0)\|_2} &\geq -\frac{(L-2)\beta\|\rvg(0)\|_2^{2L-2}\widetilde{\delta}^L(1+\alpha)^{2L-2}}{L(T-\overline{T})^{\frac{L}{L-2}}}\nonumber\\
	& = -\frac{(L-2)\beta\|\rvg(0)\|_2^{2L-2}\widetilde{\delta}^L(1+\alpha)^{2L-2}}{LT^{\frac{L}{L-2}}\widetilde{\delta}^{\frac{\eta L}{L-2}}}.\label{ineq1_Lhm}
	\end{align}
	Next, note that for $t\in [0,T_4^*]$, using the mean value theorem, we have
	\begin{equation*}
	\left\|\nabla\mathcal{N}(\rvs_u(t)) - \nabla\mathcal{N}(\rvz_2(t))\right\|_2 \leq \|\nabla^2\mathcal{N}(\rvr)\|_2\left\|\rvs_u(t)- \rvz_2(t)\right\|_2, 
	\end{equation*}
	where $\rvr = (1-\lambda)\rvs_u(t) + \lambda\rvz_2(t)$, for some $\lambda\in (0,1)$. Since $\nabla^2\mathcal{N}(\cdot)$ is $(L-2)$-homogeneous,
	\begin{align*}
	\|\nabla^2\mathcal{N}(\rvr)\|_2 &= \|\nabla^2\mathcal{N}(\rvr/\|\rvr\|_2)\|_2\|\rvr\|_2^{L-2}\\
	&\leq \|\nabla^2\mathcal{N}(\rvr/\|\rvr\|_2)\|_2\max(\|\rvz_2(t)\|_2^{L-2},\|\rvs_u(t)\|_2^{L-2})\\
	&\leq \|\nabla^2\mathcal{N}(\rvr/\|\rvr\|_2)\|_2(1+\alpha)^{L-2}\|\rvz_2(t)\|_2^{L-2}.
	\end{align*}
	From \cref{hess_bd_Lhm}, we know $\|\nabla^2\mathcal{N}(\rvr/\|\rvr\|_2)\|_2 \leq L(L-1)\mathcal{N}(\rvw_*)+\alpha$. Thus, for all $t\in [0,T_4^*]$,
	\begin{equation*}
	\left\|\nabla\mathcal{N}(\rvs_u(t)) - \nabla\mathcal{N}(\rvz_2(t))\right\|_2 \leq (L(L-1)\mathcal{N}(\rvw_*)+\alpha)(1+\alpha)^{L-2}\|\rvz_2(t)\|_2^{L-2}\left\|\rvs_u(t)- \rvz_2(t)\right\|_2.
	\end{equation*}
	Recall that $\mu_2 = (L(L-1)\mathcal{N}(\rvw_*)+\alpha)(1+\alpha)^{L-2}$, then
	\begin{align*}
	\frac{1}{2}\frac{d}{dt} \left\|{\rvs_u(t)} - \rvz_2(t) \right\|_2^2 &= \left({\rvs_u}- \rvz_2\right)^\top\left(\dot{\rvs}_u - \dot{\rvz}_2\right)\\
	& = \left({\rvs_u}{} - \rvz_2\right)^\top\left(\nabla\mathcal{N}(\rvs_u) - \nabla\mathcal{N}(\rvz_2)\right) - \left({\rvs_u}{} - \rvz_2\right)^\top f(\widetilde{\delta}\rvs_u)/\widetilde{\delta}^{L-1}\\
	&\leq \mu_2\|\rvz_2(t)\|_2^{L-2}\left\|{\rvs_u(t)}{} - \rvz_2(t) \right\|_2^2 + \beta\widetilde{\delta}^L\left\|{\rvs_u(t)} - \rvz_2(t) \right\|_2\|\rvs_u(t)\|_2^{2L-1}\\
	\leq \frac{\mu_2\|\rvg(t)\|_2^{L-2}}{(T-t)}&\left\|{\rvs_u(t)}{} - \rvz_2(t) \right\|_2^2 + \beta(1+\alpha)^{2L-1}\widetilde{\delta}^L\left\|{\rvs_u(t)}{} - \rvz_2(t) \right\|_2\|\rvz_2(t)\|_2^{2L-1}.
	\end{align*}	
	Define $\beta_1 = \beta(1+\alpha)^{2L-1}$, then, from the above equation, we have
	\begin{equation*}
	\frac{d}{dt} \left\|{\rvs_u(t)}{} - \rvz_2(t) \right\|_2 \leq  \frac{\mu_2\|\rvg(t)\|_2^{L-2}}{(T-t)}\left\|{\rvs_u(t)}{} - \rvz_2(t) \right\|_2 + \beta_1\widetilde{\delta}^L\|\rvz_2(t)\|_2^{2L-1}.
	\end{equation*}
	Using Lemma \ref{gronwall_ineq},  we get
	\begin{equation*}
	\left\|\rvs_u(t) - \rvz_2(t) \right\|_2 \leq \frac{\beta_1\widetilde{\delta}^L}{P(t)}\int_0^t P(s)\|\rvz_2(t)\|_2^{2L-1} ds,
	\end{equation*}
	where $P(t) = e^{-\int_0^t b(s) ds}$ and $b(t) =  \frac{\mu_2\|\rvg(t)\|_2^{L-2}}{(T-t)}$. The above equation implies that
	\begin{align*}
	\frac{\left\|\rvs_u(t) - \rvz_2(t) \right\|_2}{\|\rvz_2(t)\|_2^{L+1}} &\leq \frac{\beta_1\widetilde{\delta}^{L}}{P(t)\|\rvz(t)\|_2^{L+1} }\int_0^t P(s)\|\rvz(t)\|_2^{2L-1} ds\\
	&= {\beta_1\widetilde{\delta}^L}\frac{\int_0^tP(s)\|\rvg(s)\|_2^{\frac{2L-1}{L-2}}/(T-t)^{\frac{2L-1}{L-2}}}{P(t)\|\rvg(t)\|_2^{\frac{L+1}{L-2}}/(T-t)^{\frac{L+1}{L-2}}}\\
	&\leq \frac{\beta_1\|\rvg(0)\|_2^{\frac{2L-1}{L-2}}\widetilde{\delta}^L}{\|\rvg(T)\|_2^{\frac{L+1}{L-2}}}\frac{\int_0^tP(s)/(T-t)^{\frac{2L-1}{L-2}}}{P(t)/(T-t)^{\frac{L+1}{L-2}}},
	\end{align*}
	where in the last inequality we used $\|\rvg(T)\|_2\leq \|\rvg(s)\|_2\leq \|\rvg(0)\|_2$, for all $s\in [0,T)$. Let
	\begin{equation*}
	h(t) \coloneqq \frac{\int_0^tP(s)/(T-s)^{\frac{2L-1}{L-2}}}{P(t)/(T-t)^{\frac{L+1}{L-2}}}, \text{ for all } t\in [0,T).
	\end{equation*}
	We will later show that $h(t)$ is an increasing function in $[0,T)$ and 
	\begin{equation*}
	\lim_{t\rightarrow T} h(t) = \frac{1}{\left(\frac{L+1}{L-2} - \mu_2\|\rvg(T)\|_2^{L-2}\right)}.
	\end{equation*}
	Assuming the above statement is true and since $T>\overline{T}>T_4^*$, we have
	\begin{equation}
	\frac{\left\|\rvs_u(T_4^*) - \rvz_2(T_4^*) \right\|_2}{\beta_1\widetilde{\delta}^L\|\rvz_2(T_4^*)\|_2^{L+1}} \leq  \frac{h(T_4^*)\|\rvg(0)\|_2^{\frac{2L-1}{L-2}}}{\|\rvg(T)\|_2^{\frac{L+1}{L-2}}} \leq \frac{h(T)\|\rvg(0)\|_2^{\frac{2L-1}{L-2}}}{\|\rvg(T)\|_2^{\frac{L+1}{L-2}}} = \frac{C_1}{2\beta_1}.
	\label{ineq2_Lhm}
	\end{equation}
	Now, by definition of $T_4^*$, either
	\begin{equation*}
	\frac{\rvw_*^\top\rvs_u(T_4^*)}{\|\rvs_u(T_4^*)\|_2}= 1-\gamma \text{ or } \left\|\rvs_u({T_4^*}) - \rvz_2({T_4^*})\right\|_2 = C_1\widetilde{\delta}^L\|\rvz_2({T_4^*})\|_2^{L+1}. 
	\end{equation*}
	However, from \cref{ineq1_Lhm},
	\begin{align*}
	\frac{\rvw_*^\top\rvs_u( T_4^*)}{\|\rvs_u( T_4^*)\|_2}&\geq  \frac{\rvw_*^\top\rvs_u(0)}{\|\rvs_u(0)\|_2}-\frac{(L-2)\beta\|\rvg(0)\|_2^{2L-2}\widetilde{\delta}^L(1+\alpha)^{2L-2}}{LT^{\frac{L}{L-2}}\widetilde{\delta}^{\frac{\eta L}{L-2}}}\\
	&\geq 1-\gamma^2/4 - \frac{(L-2)\beta\|\rvg(0)\|_2^{2L-2}\widetilde{\delta}^{\frac{L\Delta}{L^2\mathcal{N}(\rvw_*)+\Delta}}(1+\alpha)^{2L-2}}{LT^{\frac{L}{L-2}}}> 1-\gamma/2,
	\end{align*}
	for all sufficiently small $\delta>0$. Also, from \cref{ineq2_Lhm}, we have
	\begin{equation*}
	C_1\widetilde{\delta}^L = \frac{\left\|\rvs_u(T_4^*) - \rvz_2(T_4^*) \right\|_2}{\|\rvz_2(T_4^*)\|_2^{L+1}}  \leq C_1\widetilde{\delta}^L/2.
	\end{equation*}
	Hence, there is a contradiction, which implies $T_4^*\geq \overline{T}$. Therefore,
	\begin{align*}
	\|\rvu(\overline{T}/\widetilde{\delta}^{L-2}) - \widetilde{\delta}\|\rvz_2(\overline{T})\|_2\rvw_*\|_2  &= \|\widetilde{\delta}\rvs_u(\overline{T}) - \widetilde{\delta}\|\rvz_2(\overline{T})\|_2\rvw_*\|_2\\
	&\leq \widetilde{\delta}\|\rvs_u(\overline{T}) - \rvz_2(\overline{T})\|_2+ \widetilde{\delta}\|\rvz_2(\overline{T}) - \|\rvz_2(\overline{T})\|_2\rvw_*\|_2\\
	&\leq C_1\widetilde{\delta}^{L+1}\|\rvz_2(\overline{T})\|_2^{L+1} + \widetilde{\delta}\|\rvz_2(\overline{T})\|_2\|\rvz_2(\overline{T})/\|\rvz_2(\overline{T}) \|_2 - \rvw_*\|_2\\
	&\leq  \frac{C_1\widetilde{\delta}^{L+1}\|\rvg(\overline{T})\|_2^{L+1}}{T^{\frac{L+1}{L-2}}\widetilde{\delta}^\frac{\eta(L+1)}{L-2}} + \frac{\sqrt{2\gamma}\|\rvg(\overline{T})\|_2\widetilde{\delta}\widetilde{\delta}^{\frac{L\Delta}{2L^2\mathcal{N}(\rvw_*)+\Delta}}}{T^{\frac{1}{L-2}}\widetilde{\delta}^\frac{\eta}{L-2}}\\
	& \leq  C_2\widetilde{\delta}^{\frac{(L+1)\Delta}{2L^2\mathcal{N}(\rvw_*)+\Delta}} +  C_3\widetilde{\delta}^{\frac{(L+1)\Delta}{2L^2\mathcal{N}(\rvw_*)+\Delta}},
	\end{align*}
	for some constants $C_2,C_3$ and for all sufficiently small $\delta>0$, where the third inequality follows from \cref{z2_alg_nm_bd} and \cref{nm_diff_bd}.  Also, note that
	\begin{equation*}
	\widetilde{\delta}\|\rvz_2(\overline{T})\|_2 =  \widetilde{\delta}{\|\rvg(\overline{T})\|_2}/{(T^{\frac{1}{L-2}}\widetilde{\delta}^\frac{\eta}{L-2})}= \|\rvg(\overline{T})\|_2\widetilde{\delta}^{\frac{\Delta}{2L^2\mathcal{N}(\rvw_*)+\Delta}}/T^{\frac{1}{L-2}}.
	\end{equation*}
	Thus, using the above two equations and the definition of $\rvu(t)$, we get
	\begin{equation*}
	\left\|\bm{\psi}({T}_\gamma/\delta^{L-2}+\overline{T}/\widetilde{\delta}^{L-2},\delta\rvw_0) - \|\rvg(\overline{T})\|_2\widetilde{\delta}^{\frac{\Delta}{2L^2\mathcal{N}(\rvw_*)+\Delta}}/T^{\frac{1}{L-2}}\right\|_2\leq a_1\widetilde{\delta}^{\frac{(L+1)\Delta}{2L^2\mathcal{N}(\rvw_*)+\Delta}},
	\end{equation*}
	where $a_1$ is a sufficiently large constant. Note that, since $\overline{T}$ depends on $\delta$, $\|\rvg(\overline{T})\|_2$ may also depend on $\delta$, but $\|\rvg(\overline{T})\|_2^{L-2}\in [1/(L(L-2)\mathcal{N}(\rvw_*)), T]$.
	
	We next prove $h(t)$ is an increasing function by showing  $h'(t)\geq 0$ for $t\in [0,T)$. Note that
	\begin{equation*}
	P'(t) = -\mu_2 \|\rvg(t)\|_2^{L-2}P(t)/(T-t).
	\end{equation*}
	From the quotient rule of differentiation, the denominator of $h'(t)$ is $P^2(t)/(T-t)^{\frac{2L+2}{L-2}}$, and the  numerator can be written as
	\begin{align*}
	&\frac{P(t)}{(T-t)^{\frac{2L-1}{L-2}}}\frac{P(t)}{(T-t)^{\frac{L+1}{L-2}}} - \left(\int_0^t\frac{P(s)}{(T-s)^{\frac{2L-1}{L-2}}} ds\right)\left(\frac{-\mu_2 \|\rvg(t)\|_2^{L-2}P(t)}{(T-t)^{\frac{2L-1}{L-2}}} + \frac{L+1}{L-2}\frac{P(t)}{(T-t)^{\frac{2L-1}{L-2}}}\right)\\
	&= 	\frac{P^2(t)}{(T-t)^{\frac{3L}{L-2}}}\left(1 - \left(\frac{L+1}{L-2} - \mu_2\|\rvg(t)\|_2^{L-2}\right)h(t)\right).
	\end{align*}
	Let $r(t) = {\left(\frac{L+1}{L-2} - \mu_2\|\rvg(t)\|_2^{L-2}\right)}$, then $h'(t) = \frac{1-r(t)h(t)}{(T-t)},$ and $r(t)$ is an increasing function since $\|\rvg(t)\|_2$ is a decreasing function. Next, note that
	\begin{equation*}
	1-r(t)h(t) = \frac{{P(t)}/{(T-t)^{\frac{L+1}{L-2}}}- r(t)\int_0^t{P(s)}/{(T-s)^{\frac{2L-1}{L-2}}} ds}{{P(t)}/{(T-t)^{\frac{L+1}{L-2}}}}.
	\end{equation*}
	Now, the denominator is an increasing function since its derivative is
	\begin{equation*}
	\left(\frac{-\mu_2 \|\rvg(t)\|_2^{L-2}P(t)}{(T-t)^{\frac{2L-1}{L-2}}} + \frac{L+1}{L-2}\frac{P(t)}{(T-t)^{\frac{2L-1}{L-2}}}\right) = \frac{P(t)r(t)}{(T-t)^{\frac{2L-1}{L-2}}} \geq 0,
	\end{equation*}
	where the inequality is true since, for all $t\in (0,T)$, from \cref{alp_choice} and \cref{bd_g0}, we have
	\begin{equation*}
	\frac{L+1}{L-2} - \mu_2\|\rvg(t)\|_2^{L-2} \geq \frac{L+1}{L-2} - \mu_2\|\rvg(0)\|_2^{L-2} \geq \frac{L+1}{L-2} - \frac{\mu_2}{L(L-2)(\mathcal{N}(\rvw_*)-\alpha)} \geq 0.
	\end{equation*}  
	Also, the numerator is a decreasing function since its derivative is
	\begin{equation*}
	\frac{P(t)r(t)}{(T-t)^{\frac{2L-1}{L-2}}} - \frac{P(t)r(t)}{(T-t)^{\frac{2L-1}{L-2}}} - r'(t)\int_0^t\frac{P(s)}{(T-s)^{\frac{2L-1}{L-2}}} ds =  - r'(t)\int_0^t\frac{P(s)}{(T-s)^{\frac{2L-1}{L-2}}} ds  \leq 0,
	\end{equation*}
	where the inequality is true since $r(t)$ is an increasing function. 
	Hence, $1-r(t)h(t)$ is a decreasing function in $(0,T)$. Now, for the sake of contradiction, let $h'(t_1) = -\epsilon < 0$, for some $t_1\in (0,T)$. Then, we may assume $1-r(t_1)h(t_1) = -\tilde{\epsilon} < 0$. Since $1-r(t)h(t)$ is a decreasing function in $(0,T)$, $1-r(t)h(t) \leq -\tilde{\epsilon}$, for all $t\geq t_1$. Hence,
	\begin{equation*}
	h(t) - h(t_1) \leq \int_{t_1}^t \frac{-\tilde{\epsilon}}{(T-s)} ds = -\tilde{\epsilon}\ln\left(\frac{T-t_1}{T-t}\right), \text{ for all } t\in (t_1,T).
	\end{equation*}
	Now, if $t$ is chosen sufficiently close to $T$, then $h(t)$ will become negative, leading to a contradiction. Hence, $h(t)$ is an increasing function in $(0,T)$.
	
	We next compute $\lim_{t\to T} h(t)$. We first show that 
	\begin{equation*}
	\lim_{t\rightarrow T} \int_0^tP(s)/(T-s)^{\frac{2L-1}{L-2}} = \infty \text{ and } \lim_{t\rightarrow T} P(t)/(T-t)^{\frac{L+1}{L-2}}= \infty.
	\end{equation*}
	From \cref{bd_g0}, we know $\|\rvg(t)\|_2^{L-2} \leq \|\rvg(0)\|_2^{L-2}\leq  \frac{1}{L(L-2)(\mathcal{N}(\rvw_*)-\alpha )}$. Hence,
	\begin{equation*}
	\frac{P(t)}{(T-t)^{\frac{L+1}{L-2}}} \geq \frac{e^{-\frac{\mu_2}{L(L-2)(\mathcal{N}(\rvw_*)-\alpha )}\int_0^t \frac{ds}{(T-s)}}}{(T-t)^{\frac{L+1}{L-2}}} = \left(\frac{T-t}{T}\right)^{\frac{\mu_2}{L(L-2)(\mathcal{N}(\rvw_*)-\alpha )}}{(T-t)^{\frac{-L-1}{L-2}}}.
	\end{equation*}
	Combining the above inequality with \cref{alp_choice}, we get $\lim_{t\rightarrow T} P(t)/(T-t)^{\frac{L+1}{L-2}}= \infty.$ Next,
	\begin{align*}
	&\int_0^t {P(s)}/{(T-s)^{\frac{2L-1}{L-2}}} ds \geq \frac{1}{T^{\frac{\mu_2}{L(L-2)(\mathcal{N}(\rvw_*)-\alpha )}}}\int_0^t \left(T-s\right)^{\left(\frac{\mu_2}{L(L-2)(\mathcal{N}(\rvw_*)-\alpha)} - \frac{2L-1}{L-2}\right)} ds\\
	&= \frac{\big( (T-t)^{\left(\frac{\mu_2}{L(L-2)(\mathcal{N}(\rvw_*)-\alpha)} -\frac{L+1}{L-2}\right)} - T^{\left(\frac{\mu_2}{L(L-2)(\mathcal{N}(\rvw_*)-\alpha) } -\frac{L+1}{L-2}\right)} \big)}{T^{\frac{\mu_2}{L(L-2)(\mathcal{N}(\rvw_*)-\alpha )}}\left(\frac{L+1}{L-2} - \frac{\mu_2}{L(L-2)(\mathcal{N}(\rvw_*)-\alpha)}\right)}.
	\end{align*}
	Combining the above inequality with \cref{alp_choice}, we get $\lim_{t\rightarrow T} \int_0^tP(s)/(T-s)^{\frac{2L-1}{L-2}} = \infty$.
	Thus, using L'Hopital rule, we have
	\begin{equation*}
	\lim_{t\rightarrow T} h(t) = \lim_{t\rightarrow T}\frac{P(t)/(T-t)^{\frac{2L-1}{L-2}}}{\frac{P(t)}{(T-t)^{\frac{2L-1}{L-2}}}{\left(\frac{L+1}{L-2} - \mu_2\|\rvg(t)\|_2^{L-2}\right)}} = \frac{1}{\left(\frac{L+1}{L-2} - \mu_2\|\rvg(T)\|_2^{L-2}\right)}.
	\end{equation*}	
\end{proof}
\begin{proof}\textbf{of Theorem \ref{main_thm_Lhm}: } From Lemma \ref{init_align_Lhm}, we have
	\begin{equation*}
	\left\|\bm{\psi}\left(\frac{T_1}{\delta^{L-2}} + \frac{T}{\widetilde{\delta}^{L-2}}\left(1 - \widetilde{\delta}^{\frac{2(L-2)L^2\mathcal{N}(\rvw_*)}{2L^2\mathcal{N}(\rvw_*)+\Delta}}\right),\delta\rvw_0\right) - b_\delta\widetilde{\delta}^{\frac{\Delta}{2L^2\mathcal{N}(\rvw_*) +\Delta}}\rvw_* \right\|_2 \leq a_1\widetilde{\delta}^{\frac{(L+1)\Delta}{2L^2\mathcal{N}(\rvw_*) + \Delta}},
	\end{equation*}
	for some $T_1, a_1>0$. Define $\bar{\delta} = b_\delta\widetilde{\delta}^{\frac{\Delta}{2L^2\mathcal{N}(\rvw_*)+\Delta}}$, then the above equation implies
	\begin{equation*}
	\left\|\bm{\psi}\left(\frac{T_1}{\delta^{L-2}} + \frac{T}{\widetilde{\delta}^{L-2}}\left(1 - \widetilde{\delta}^{\frac{2(L-2)L^2\mathcal{N}(\rvw_*)}{2L^2\mathcal{N}(\rvw_*)+\Delta}}\right),\delta\rvw_0\right)  - \bar{\delta}\rvw_* \right\|_2 \leq a_2\bar{\delta}^{L+1},
	\end{equation*}
	where $a_2\geq a_1/b_\delta^{L+1}$ is a large enough constant. Define $T_\delta \coloneqq  \frac{T_1}{\delta^{L-2}} + \frac{T\big(1 - \widetilde{\delta}^{\frac{2(L-2)L^2\mathcal{N}(\rvw_*)}{2L^2\mathcal{N}(\rvw_*)+\Delta}}\big)}{\widetilde{\delta}^{L-2}},$ 
	then, using Lemma \ref{init_dl3_Lhm}, for any fixed $\widetilde{T}\in (-\infty,\infty)$, there exists a constant $\widetilde{C}_1>0$ such that for all sufficiently small $\delta>0$, we have 
	\begin{equation*}
	\left\|\bm{\psi}\left(t+ \frac{1/\overline{\delta}^{L-2}}{L(L-2)\mathcal{N}(\rvw_*)},\bm{\psi}\left(T_\delta,\delta\rvw_0\right)\right) - \bm{\psi}\left(t+ \frac{1/\overline{\delta}^{L-2}}{L(L-2)\mathcal{N}(\rvw_*)},\overline{\delta}{\rvw}_*\right) \right\|_2 \leq \widetilde{C}_1\overline{\delta}, 
	\end{equation*}
	for all $t\in[-\widetilde{T},\widetilde{T}]$, and from Lemma \ref{lim_exists_Lhm}, there exists a $\widetilde{C}_2>0$ such that
	\begin{equation*}
	\left\|\rvp(t)- \bm{\psi}\left(t+\frac{{1}/{\overline{\delta}^{L-2}}}{L(L-2)\mathcal{N}(\rvw_*)},\overline{\delta}\rvw_*\right)\right\|_2 \leq \widetilde{C}_2\overline{\delta}, \text{ for all } t\in [-\widetilde{T},\widetilde{T}].
	\end{equation*}
	Since
	\begin{align*}
	\frac{1/\overline{\delta}^{L-2}}{L(L-2)\mathcal{N}(\rvw_*)} + T_\delta &= \frac{T_1}{\delta^{L-2}} + \frac{1/(b_\delta^{L-2}\widetilde{\delta}^{\frac{(L-2)\Delta}{2L^2\mathcal{N}(\rvw_*)+\Delta}})}{L(L-2)\mathcal{N}(\rvw_*)} + \frac{T}{\widetilde{\delta}^{L-2}} - \frac{T}{\widetilde{\delta}^{\frac{(L-2)\Delta}{2L^2\mathcal{N}(\rvw_*)+\Delta}}}\\
	&= \frac{T_1}{\delta^{L-2}} + \frac{T}{\widetilde{\delta}^{L-2}} + \left(\frac{1}{b_\delta^{L-2}L(L-2)\mathcal{N}(\rvw_*)} - T\right){\widetilde{\delta}^{\frac{-(L-2)\Delta}{2L^2\mathcal{N}(\rvw_*)+\Delta}}},
	\end{align*}
	it follows that
	\begin{align*}
	&\bm{\psi}\left(t+ \frac{1/\overline{\delta}^{L-2}}{L(L-2)\mathcal{N}(\rvw_*)},\bm{\psi}\left(T_\delta,\delta\rvw_0\right)\right)\\
	&= \bm{\psi}\left(t+ \frac{T_1}{\delta^{L-2}} + \frac{T}{\widetilde{\delta}^{L-2}} + \left(\frac{1}{b_\delta^{L-2}L(L-2)\mathcal{N}(\rvw_*)} - T\right){\widetilde{\delta}^{\frac{-(L-2)\Delta}{2L^2\mathcal{N}(\rvw_*)+\Delta}}},\delta\rvw_0\right).
	\end{align*}
	Hence, for all $t\in [-\widetilde{T},\widetilde{T}]$ and for all sufficiently small $\delta$, we have 
	\begin{align*}
	&\left\|\bm{\psi}\left(t+ \frac{T_1}{\delta^{L-2}} + \frac{T}{\widetilde{\delta}^{L-2}} + \left(\frac{1}{b_\delta^{L-2}L(L-2)\mathcal{N}(\rvw_*)} - T\right){\widetilde{\delta}^{\frac{-(L-2)\Delta}{2L^2\mathcal{N}(\rvw_*)+\Delta}}},\delta\rvw_0\right) - \rvp(t)\right\|_2 \\
	&\leq  \left\|\bm{\psi}\left(t+ \frac{1/\overline{\delta}^{L-2}}{L(L-2)\mathcal{N}(\rvw_*)},\bm{\psi}\left(T_\delta,\delta\rvw_0\right)\right) - \bm{\psi}\left(t+ \frac{{1}/{\overline{\delta}^{L-2}}}{L(L-2)\mathcal{N}(\rvw_*)},\bar{\delta}\rvw_*\right) \right\|_2 + \\
	& \left\|\bm{\psi}\left(t+ \frac{{1}/{\overline{\delta}^{L-2}}}{L(L-2)\mathcal{N}(\rvw_*)},\bar{\delta}\rvw_*\right)  - \rvp(t)\right\|_2 \leq\widetilde{C}_1\bar{\delta}+\widetilde{C}_2\bar{\delta} \leq \widetilde{C}\delta^{\frac{\Delta}{\Delta + 2L^2\mathcal{N}(\rvw_*)}},
	\end{align*}
	where $\widetilde{C}$ is a positive constant. The last inequality is true since $b_\delta\leq1$, and $\widetilde{\delta}\leq A_2\delta+A_1\delta^{L+1}\leq 2A_2\delta$, for all sufficiently small $\delta$. Thus, the proof is complete.
\end{proof}
\begin{proof}\textbf{of Corollary \ref{cor_sq_loss_Lhm}: } Since  $\rvp(t)$ is bounded for all $t\geq 0$, there exists a constant $B>0$ such that $\|\rvp(t)\|_2\leq B$, for all $t\geq 0$. Moreover, since $\mathcal{L}(\cdot)$ has locally Lipschitz gradient, there exists a constant $\widetilde{A}>0$ such that, if $\|\rvw_1\|_2, \|\rvw_2\|_2\leq 2B$, then
	\begin{equation}
	\|\nabla\mathcal{L}(\rvw_1)-\nabla\mathcal{L}(\rvw_2)\|_2\leq \widetilde{A}\|\rvw_1-\rvw_2\|_2.
	\label{lips_sq_Lhm}
	\end{equation} 
	Since $\rvp^*= \lim_{t\to \infty} \rvp(t)$ and $\nabla\mathcal{L}(\rvp^*) = \mathbf{0}$, for any $\epsilon\in (0,B)$, we can choose a $T_\epsilon$ such that
	\begin{equation}
	\|\rvp(T_\epsilon)-\rvp_*\|_2\leq \epsilon/2 \text{ and } \|\nabla\mathcal{L}(\rvp(T_\epsilon))\|_2 \leq \epsilon/2.
	\label{err_sq_Lhm}
	\end{equation}
	Since $T_\epsilon$ does not depend on $\delta$, from Theorem \ref{main_thm_Lhm}, for all sufficiently small $\delta>0$,
	\begin{align}
	&\left\|\bm{\psi}\left(t+\frac{T_1}{\delta^{L-2}} + \left(\frac{1/b_\delta^{L-2}}{L(L-2)\mathcal{N}(\rvw_*)} - T\right){\widetilde{\delta}^{\frac{-(L-2)\Delta}{2L^2\mathcal{N}(\rvw_*)+\Delta}}}+\frac{T}{\widetilde{\delta}^{L-2}},\delta\rvw_0\right) - \rvp(t)\right\|_2\nonumber \\
	&\leq \widetilde{C}\delta^{\frac{\Delta}{\Delta+2L^2\mathcal{N}(\rvw_*)}} \leq \epsilon/2, \text{ for all } t\in [-T_\epsilon,T_\epsilon].
	\label{err_psi_Lhm}
	\end{align} 
	Putting $t=T_\epsilon$ in the above equation, and using  \cref{err_sq_Lhm}, we get
	\begin{equation*}
	\left\|\bm{\psi}\left(T_\delta,\delta\rvw_0\right) - \rvp^*\right\|_2 \leq \epsilon,
	\end{equation*}
	where $T_\delta\coloneqq T_\epsilon+\frac{T_1}{\delta^{L-2}} + \left(\frac{1/b_\delta^{L-2}}{L(L-2)\mathcal{N}(\rvw_*)} - T\right){\widetilde{\delta}^{\frac{-(L-2)\Delta}{2L^2\mathcal{N}(\rvw_*)+\Delta}}}+\frac{T}{\widetilde{\delta}^{L-2}}$. Next, since $\|\bm{\psi}\left(T_\delta,\delta\rvw_0\right) \|_2\leq B+\epsilon/2\leq 2B$, using \cref{lips_sq_Lhm}, \cref{err_sq_Lhm} and \cref{err_psi_Lhm}, we get
	\begin{equation*}
	\left\|\nabla \mathcal{L}\left(\bm{\psi}\left(T_\delta,\delta\rvw_0\right)\right)\right\|_2 \leq \epsilon/2+ \widetilde{A}\widetilde{C}\delta^{\frac{\Delta}{\Delta+2L^2\mathcal{N}(\rvw_*)}}\leq \epsilon,
	\end{equation*}
	where the final inequality is true for all sufficiently small $\delta>0$. This completes the proof.
\end{proof}
\begin{proof}\textbf{of Corollary \ref{cor_log_loss_Lhm}: } Since $\lim_{t\to \infty} \rvp(t)/\|\rvp(t)\|_2 = \rvp^*$, $\lim_{t\to \infty} \|\rvp(t)\|_2 = \infty$, and $\lim_{t\to\infty} \nabla \mathcal{L}(\rvp(t)) = \mathbf{0}$, we have that for any $\epsilon\in (0,1)$, we can choose a $T_\epsilon$ such that
	\begin{equation}
	\rvp(T_\epsilon)^\top\rvp^*/\|\rvp(T_\epsilon)\|_2 \geq 1-\epsilon/2, \|\rvp(T_\epsilon)\|_2\geq 1/\epsilon,  \text{ and } \|\nabla\mathcal{L}(\rvp(T_\epsilon))\|_2 \leq \epsilon/2.
	\label{err_log_Lhm}
	\end{equation}
	Let $B_\epsilon \coloneqq  \max_{t\in [0,T_\epsilon]}\|\rvp(t)\|_2$. Then, since $\mathcal{L}(\cdot)$ has locally Lipschitz gradient, there exists a constant $\widetilde{A}_\epsilon>0$ such that, if $\|\rvw_1\|_2, \|\rvw_2\|_2\leq B_\epsilon+\epsilon$, then
	\begin{equation}
	\|\nabla\mathcal{L}(\rvw_1)-\nabla\mathcal{L}(\rvw_2)\|_2\leq \widetilde{A}_\epsilon\|\rvw_1-\rvw_2\|_2.
	\label{lips_log_Lhm}
	\end{equation} 
	Here, $\widetilde{A}_\epsilon$ depends on $\epsilon$. Since $T_\epsilon$ does not depend on $\delta$, from Theorem \ref{main_thm_Lhm}, for all sufficiently small $\delta>0$ and for all $t\in [-T_\epsilon,T_\epsilon]$, 
	\begin{align}
	&\left\|\bm{\psi}\left(t+\frac{T_1}{\delta^{L-2}} + \left(\frac{1/b_\delta^{L-2}}{L(L-2)\mathcal{N}(\rvw_*)} - T\right)\frac{1}{\widetilde{\delta}^{\frac{(L-2)\Delta}{2L^2\mathcal{N}(\rvw_*)+\Delta}}}+\frac{T}{\widetilde{\delta}^{L-2}},\delta\rvw_0\right) - \rvp(t)\right\|_2 \nonumber\\
	&\leq \widetilde{C}\delta^{\frac{\Delta}{\Delta+2L^2\mathcal{N}(\rvw_*)}} \leq \epsilon/2.
	\label{err_bd_log_Lhm}
	\end{align} 
	Putting $t=T_\epsilon$ in the above equation, and using  \cref{err_log_Lhm}, we get
	\begin{equation*}
	\left\|\bm{\psi}\left(T_\delta,\delta\rvw_0\right)\right\|_2 \geq \|\rvp(T_\epsilon)\|_2- \widetilde{C}\delta^{\frac{\Delta}{\Delta+2L^2\mathcal{N}(\rvw_*)}} \geq \frac{1}{\epsilon} - \widetilde{C}\delta^{\frac{\Delta}{\Delta+2L^2\mathcal{N}(\rvw_*)}}\geq \frac{1}{2\epsilon} ,
	\end{equation*}
	where $T_\delta\coloneqq T_\epsilon+\frac{T_1}{\delta^{L-2}} + \left(\frac{1/b_\delta^{L-2}}{L(L-2)\mathcal{N}(\rvw_*)} - T\right)\frac{1}{\widetilde{\delta}^{\frac{(L-2)\Delta}{2L^2\mathcal{N}(\rvw_*)+\Delta}}}+\frac{T}{\widetilde{\delta}^{L-2}}$, and the second inequality is true for all sufficiently small $\delta>0$. We also have
	\begin{equation*}
	\frac{\bm{\psi}\left(T_\delta,\delta\rvw_0\right)^\top\rvp^*}{\left\|\bm{\psi}\left(T_\delta,\delta\rvw_0\right)\right\|_2 } \geq \frac{\rvp(T_\epsilon)^\top\rvp_* - \widetilde{C}\delta^{\frac{\Delta}{\Delta+2L^2\mathcal{N}(\rvw_*)}}}{\left\|\bm{\psi}\left(T_\delta,\delta\rvw_0\right)\right\|_2} \geq \frac{(1-\epsilon/2)\|\rvp(T_\epsilon)\|_2 - \widetilde{C}\delta^{\frac{\Delta}{\Delta+2L^2\mathcal{N}(\rvw_*)}}}{\left\|\rvp(T_\epsilon)\right\|_2 + \widetilde{C}\delta^{\frac{\Delta}{\Delta+2L^2\mathcal{N}(\rvw_*)}}}\geq 1-\epsilon,
	\end{equation*}
	where the first inequality uses \cref{err_bd_log_Lhm}. The second inequality uses \cref{err_bd_log_Lhm} and \cref{err_log_Lhm}. The final inequality is true for all sufficiently small $\delta>0$. Next, since $\|\bm{\psi}\left(T_\delta,\delta\rvw_0\right) \|_2\leq B_\epsilon+\epsilon/2$, using \cref{lips_log_Lhm} and \cref{err_log_Lhm}, we get
	\begin{equation*}
	\left\|\nabla \mathcal{L}\left(\bm{\psi}\left(T_\delta,\delta\rvw_0\right)\right)\right\|_2 \leq \epsilon/2+ \widetilde{A}_\epsilon\widetilde{C}\delta^{\frac{\Delta}{\Delta+2L^2\mathcal{N}(\rvw_*)}} \leq \epsilon,
	\end{equation*}
	where the final inequality is true for all sufficiently small $\delta>0$. 
\end{proof}
\section{Proof Omitted from \Cref{sec_implications}}\label{appendix_imp}
\begin{proof}\textbf{of Lemma \ref{zp_subs_ff}: } We first compute the gradient of $\mathcal{H}(\rvx;\rmW_1,\cdots,\rmW_L) $ with respect to $\rmW_{l+1}[:,j] \text{ and } \rmW_l[j,:]$. Let
	\begin{align*}
	&\phi_0(\rvx) = \rvx \text{ and } \phi_l(\rvx) =  \sigma(\rmW_{l}\phi_{l-1}(\rvx)), \text{ for all } l\geq 1, \text{ and }\\
	&\psi_L(\rvz) = \rmW_L\rvz, \text{ and } \psi_l(\rvz) = \psi_{l+1}(\sigma(\rmW_{l}\rvz)), \text{ for all } l\leq L-1.
	\end{align*}
	Then
	\begin{align*}
	\mathcal{H}(\rvx;\rmW_1,\cdots,\rmW_L) &= \psi_{l+1}\left(\rmW_{l+1}\sigma(\rmW_l\phi_{l-1}(\rvx))\right)\\
	& = \psi_{l+1}\left(\sum_{p=1}^{k_l}\rmW_{l+1}[:,p]\sigma(\rmW_l[p,:]^\top\phi_{l-1}(\rvx))\right),
	\end{align*}
	which implies
	\begin{equation*}
	\frac{d \mathcal{H}(\rvx;\rmW_1,\cdots,\rmW_L)}{d \rmW_{l+1}[:,j]} = \nabla\psi_{l+1}\left(\sum_{p=1}^{k_l}\rmW_{l+1}[:,p]\sigma(\rmW_l[p,:]^\top\phi_{l-1}(\rvx))\right)\sigma(\rmW_l[j,:]^\top\phi_{l-1}(\rvx)),
	\end{equation*}
	and
	\begin{align*}
	&\frac{d \mathcal{H}(\rvx;\rmW_1,\cdots,\rmW_L)}{d \rmW_{l}[j,: ]} \\
	&= \rmW_{l+1}[:,j]^\top\nabla\psi_{l+1}\left(\sum_{p=1}^{k_l}\rmW_{l+1}[:,p]\sigma(\rmW_l[p,:]^\top\phi_{l-1}(\rvx))\right)\sigma'(\rmW_l[j,:]^\top\phi_{l-1}(\rvx))\phi_{l-1}(\rvx).
	\end{align*}
	Now, note that, if the weights are bounded, then there exists a $K>0$ such that
	\begin{equation}
	\left\|\frac{d \mathcal{H}(\rvx;\rmW_1,\cdots,\rmW_L)}{d \rmW_{l+1}[:,j]}\right\|_2 \leq K\|\rmW_l[j,:]\|_2, \left\|\frac{d \mathcal{H}(\rvx;\rmW_1,\cdots,\rmW_L)}{d\rmW_l[j,:]}\right\|_2 \leq K\| \rmW_{l+1}[:,j]\|_2.
	\label{bd_gd}
	\end{equation}
	The first inequality follows from Cauchy-Schwartz inequality and since $|\sigma(\rmW_l[j,:]^\top\phi_{l-1}(\rvx))| \leq K_1|\rmW_l[j,:]^\top\phi_{l-1}(\rvx)|$, for some $K_1>0$, which follows from $\sigma(\cdot)$ being locally Lipshitz and $\sigma(0) = 0$. The second inequality follows from Cauchy-Schwartz inequality.
	
	Now, let $(\rmW_1(t),\cdots,\rmW_L(t)) $ be the solution of 
	\begin{equation*}
	\dot{\rmW}_i = -\nabla_{\rmW_i} \mathcal{L}(\rmW_1,\cdots,\rmW_L), \text{ for all } i \in [L].
	\end{equation*}
	If $\rmW_{l+1}[:,j] , \rmW_l[j,:] \in \rvw_z$, for some $l\in [L-1]$ and $j\in [k_l]$, then $\|\rmW_{l+1}[:,j](0)\|_2= 0 = \|\rmW_l[j,:](0)\|_2$. Also, note that
	\begin{align*}
	&\dot{\rmW}_{l+1}[:,j] = -\sum_{i=1}^n\ell'\left(\mathcal{H}(\rvx_i;\rmW_1,\cdots,\rmW_L),y_i\right)\frac{d \mathcal{H}(\rvx_i;\rmW_1,\cdots,\rmW_L)}{d \rmW_{l+1}[:,j]}, \mbox{ and}\\
	&\dot{\rmW}_{l}[j,:] = -\sum_{i=1}^n\ell'\left(\mathcal{H}(\rvx_i;\rmW_1,\cdots,\rmW_L),y_i\right)\frac{d \mathcal{H}(\rvx_i;\rmW_1,\cdots,\rmW_L)}{d \rmW_{l}[j,:]}.
	\end{align*}
	Let $T\in (0,\infty)$, then for all $t\in [-T,T]$, we can assume $(\rmW_1(t),\cdots,\rmW_L(t)) $  are bounded. Since the output of the neural network and the loss function are bounded for bounded input, from \cref{bd_gd}, there exists a $B>0$ such that, for all $t\in [0,T]$, we have
	\begin{align*}
	&\frac{1}{2}\frac{d\|\rmW_{l+1}[:,j]\|_2^2}{dt} \leq \|\rmW_{l+1}[:,j]\|_2\|\dot{\rmW}_{l+1}[:,j]\|_2\leq B\|\rmW_{l+1}[:,j]\|_2\|\rmW_{l}[j,:]\|_2, \mbox{ and} \\
	&\frac{1}{2}\frac{d\|\rmW_{l}[j,:]\|_2^2}{dt} \leq \|\rmW_{l}[j,:]\|_2\|\dot{\rmW}_{l}[j,:]\|_2\leq B\|\rmW_{l}[j,:]\|_2\|\rmW_{l+1}[:,j]\|_2.
	\end{align*}
	Adding the above two equations, and since $2xy\leq x^2+y^2$, we get
	\begin{equation}
	\frac{d\left(\|\rmW_{l+1}[:,j]\|_2^2+\|\rmW_{l}[j,:]\|_2^2\right)}{dt} \leq 2B\left(\|\rmW_{l+1}[:,j]\|_2^2+\|\rmW_{l}[j,:]\|_2^2\right).
	\label{bd_nm_w}
	\end{equation}
	Integrating the above equation from $0$ to $t\in [0,T]$, we get
	\begin{equation*}
	\|\rmW_{l+1}[:,j](t)\|_2^2+\|\rmW_{l}[j,:](t)\|_2^2 \leq (\|\rmW_{l+1}[:,j](0)\|_2^2+\|\rmW_{l}[j,:](0)\|_2^2 )e^{2tB}.
	\end{equation*}
	Since $\|\rmW_{l+1}[:,j](0)\|_2 = 0 = \|\rmW_l[j,:](0)\|_2$, the above equation implies $\|\rmW_{l+1}[:,j](t)\|_2 = 0 = \|\rmW_l[j,:](t)\|_2$, for all $t\in [0,T]$. Similarly, there exists a $B>0$ such that, for all $t\in [-T,0]$, we have
	\begin{align*}
	&\frac{1}{2}\frac{d\|\rmW_{l+1}[:,j]\|_2^2}{dt} \geq -\|\rmW_{l+1}[:,j]\|_2\|\dot{\rmW}_{l+1}[:,j]\|_2\geq -B\|\rmW_{l+1}[:,j]\|_2\|\rmW_{l}[j,:]\|_2, \mbox{ and}\\
	&\frac{1}{2}\frac{d\|\rmW_{l}[j,:]\|_2^2}{dt} \geq -\|\rmW_{l}[j,:]\|_2\|\dot{\rmW}_{l}[j,:]\|_2\geq -B\|\rmW_{l}[j,:]\|_2\|\rmW_{l+1}[:,j]\|_2.
	\end{align*}
	Adding the above two equation gives us
	\begin{equation}
	\frac{d\left(\|\rmW_{l+1}[:,j]\|_2^2+\|\rmW_{l}[j,:]\|_2^2\right)}{dt} \geq  -2B\left(\|\rmW_{l+1}[:,j]\|_2^2+\|\rmW_{l}[j,:]\|_2^2\right).
	\end{equation}
	Integrating the above equation from $t\in [-T,0]$ to $0$, we get
	\begin{equation*}
	\|\rmW_{l+1}[:,j](t)\|_2^2+\|\rmW_{l}[j,:](t)\|_2^2 \leq (\|\rmW_{l+1}[:,j](0)\|_2^2+\|\rmW_{l}[j,:](0)\|_2^2 )e^{-2tB},
	\end{equation*}
	which implies $\|\rmW_{l+1}[:,j](t)\|_2 = 0 = \|\rmW_l[j,:](t)\|_2$, for all $t\in [-T,0]$.
\end{proof}

\begin{proof}\textbf{of Theorem \ref{main_thm_sp}: }Note that $\rvw_z$ is a zero-preserving subset. If $\rmW_l[j,:]\in \rvw_z$, for some $l\in [L-1]$ and $j\in [k_l]$, then $\|\overline{\rmW}_l[j,:]\|_2 = 0$, and from Lemma \ref{bal_r1_kkt}, $\|\overline{\rmW}_{l+1}[:,j]\|_2 = 0$, which implies  ${\rmW}_{l+1}[:,j]\in \rvw_z$. Similarly, if ${\rmW}_{l+1}[:,j]\in \rvw_z$, then $\rmW_l[j,:]\in \rvw_z$. Hence, $\rvw_z$ is a zero-preserving subset. 
	
	Throughout this proof, let $\overline{\rvw}$ denote the vector containing the entries of $(\overline{\rmW}_1,\cdots, \overline{\rmW}_{L})$, and $\overline{\rvw}_z$ be the vector containing the entries of $(\overline{\rmW}_1,\cdots, \overline{\rmW}_{L})$ that belong to $\rvw_z$. Now, consider the case when $\mathcal{H}$ is two-homogeneous. From Lemma \ref{init_align_2hm}, we know
	\begin{equation}
	\left\|\bm{\psi}\left(T_1 + \frac{4\ln({1}/{\widetilde{\delta}})}{\Delta+8\overline{\mathcal{N}}},\delta\rmW_{1:L}^0\right) - b_\delta\widetilde{\delta}^{\frac{\Delta}{\Delta + 8\overline{\mathcal{N}}}}\overline{\rvw}\right\|_2 \leq a_1 \widetilde{\delta}^{\frac{3\Delta}{\Delta + 8\overline{\mathcal{N}}}},
	\label{diff_init_align_2hm}
	\end{equation}
	where $T_1, a_1, b_\delta, \widetilde{\delta}$ are as defined in Lemma \ref{init_align_2hm}. Since $\|\overline{\rvw}\|_2 =1$, $b_\delta\geq \kappa_1>0$, and $\widetilde{\delta}\geq A_2\delta/2$ for all sufficiently small $\delta$,  there exists a $b_1$ such that
	\begin{equation}
	\left\|\bm{\psi}\left(T_1 + \frac{4\ln({1}/{\widetilde{\delta}})}{\Delta+8\overline{\mathcal{N}}},\delta\rmW_{1:L}^0\right) \right\|_2 \geq b_1 \widetilde{\delta}^{\frac{\Delta}{\Delta + 8\overline{\mathcal{N}}}}.
	\label{bd_psi_2hm}
	\end{equation} 
	Now, since $\|\overline{\rvw}_z\|_2 =0$, from \cref{diff_init_align_2hm}, we have
	\begin{equation}
	\left\|\bm{\psi}_{\rvw_z}\left(T_1 + \frac{4\ln({1}/{\widetilde{\delta}})}{\Delta+8\overline{\mathcal{N}}},\delta\rmW_{1:L}^0\right)\right\|_2 \leq a_1\widetilde{\delta}^{\frac{3\Delta}{\Delta + 8\overline{\mathcal{N}}}}.
	\label{bd_psi_wz_2hm}
	\end{equation}
	Dividing \cref{bd_psi_wz_2hm} by \cref{bd_psi_2hm} and using $\widetilde{\delta}\leq 2A_2\delta$, for all sufficiently small $\delta,$ we get \cref{init_algn_2hm_ff}. Next, from the definition of  zero-preserving subset, we know 
	\begin{equation*}
	\|\bm{\psi}_{\rvw_z}(t, \delta\overline{\rmW}_{1:L})\|_2 = 0, \text{ for all } t \in (-\infty,\infty).
	\end{equation*}
	Recall that
	\begin{equation*}
	\rvp(t) = \lim_{\delta\to 0} \bm{\psi}\left(t+\frac{\ln(1/\delta)}{2\overline{\mathcal{N}}},\delta\overline{\rmW}_{1:L}\right), \text{ and let } \rvp_{\rvw_z}(t) = \lim_{\delta\to 0} \bm{\psi}_{\rvw_z}\left(t+\frac{\ln(1/\delta)}{2\overline{\mathcal{N}}},\delta\overline{\rmW}_{1:L}\right).
	\end{equation*}
	From Theorem \ref{main_thm_2hm}, for all sufficiently small $\delta$ and for all $t\in [-\widetilde{T},\widetilde{T}]$, we have
	\begin{equation}
	\left\|\bm{\psi}\left(t+T_1 +\frac{\ln(1/b_\delta)}{2\overline{\mathcal{N}}}+ \frac{\ln({1}/{\widetilde{\delta}})}{2\overline{\mathcal{N}}},\delta\rmW_{1:L}^0\right) - \rvp(t) \right\|_2 \leq \widetilde{C}\delta^\frac{\Delta}{\Delta+8\overline{\mathcal{N}} }.
	\label{2hm_psi_bd}
	\end{equation} 
	Since, for any $t\in (-\infty,\infty)$ and $\delta>0$ we have $\|\bm{\psi}_{\rvw_z}(t + \ln(1/\delta)/2\overline{\mathcal{N}},\delta\overline{\rmW}_{1:L})\|_2 = 0,$ it follows that $\|\rvp_{\rvw_z}(t) \|_2 = 0$, for all $t\in (-\infty,\infty)$. Combining this fact with \cref{2hm_psi_bd} gives us \cref{err_bd_2hm_ff}. 
	
	The proof for $L$-homogeneous networks is similar, as shown next. From Lemma \ref{init_align_Lhm}, 
	\begin{equation}
	\left\|\bm{\psi}\left(T_{\widetilde{\delta}}^1,\delta\rmW_{1:L}^0\right) - b_\delta\widetilde{\delta}^{\frac{\Delta}{2L^2\overline{\mathcal{N}} +\Delta}}\overline{\rvw} \right\|_2 \leq a_1\widetilde{\delta}^{\frac{(L+1)\Delta}{2L^2\overline{\mathcal{N}} + \Delta}}, 
	\label{diff_init_align_Lhm}
	\end{equation}
	where $T_{\widetilde{\delta}}^1 =  \frac{T_1}{\delta^{L-2}} + \frac{T\big(1 - \widetilde{\delta}^{\frac{2(L-2)L^2\overline{\mathcal{N}}}{2L^2\overline{\mathcal{N}}+\Delta}}\big)}{\widetilde{\delta}^{L-2}},$ and $T_1,T,a_1,b_\delta,\widetilde{\delta}$ are as defined in Lemma \ref{init_align_Lhm}. Since $\|\overline{\rvw}\|_2 =1$, $b_\delta^{L-2}\geq 1/(TL(L-2)\overline{\mathcal{N}})>0$, and $\widetilde{\delta}\geq A_2\delta/2$, for all sufficiently small $\delta>0,$ there exists a $b_1$ such that
	\begin{equation}
	\left\|\bm{\psi}\left(T_{\widetilde{\delta}}^1,\delta\rmW_{1:L}^0\right) \right\|_2 \geq b_1 \widetilde{\delta}^{\frac{\Delta}{\Delta + 2L^2\overline{\mathcal{N}}}}.
	\label{bd_psi_Lhm}
	\end{equation} 
	Now, since $\|\overline{\rvw}_z\|_2 =0$, from \cref{diff_init_align_Lhm}, we have
	\begin{equation}
	\left\|\bm{\psi}_{\rvw_z}\left(T_{\widetilde{\delta}}^1,\delta\rmW_{1:L}^0\right)\right\|_2 \leq a_1\widetilde{\delta}^{\frac{(L+1)\Delta}{\Delta + 2L^2\overline{\mathcal{N}}}}.
	\label{bd_psi_wz_Lhm}
	\end{equation}
	Dividing \cref{bd_psi_wz_Lhm} by \cref{bd_psi_Lhm} and using $\widetilde{\delta}\leq 2A_2\delta$, for all sufficiently small $\delta,$ gives us \cref{init_algn_Lhm_ff}. Next, from the definition of  zero-preserving subset, we know 
	\begin{equation*}
	\|\bm{\psi}_{\rvw_z}(t,\delta\overline{\rmW}_{1:L})\|_2 = 0, \text{ for all } t \in (-\infty,\infty).
	\end{equation*}
	Recall that
	\begin{equation*}
	\rvp(t) = \lim_{\delta\to 0} \bm{\psi}\left(t+\frac{1/\delta^{L-2}}{L(L-2)\overline{\mathcal{N}}},\delta\overline{\rmW}_{1:L}\right), \end{equation*}
	and let
	\begin{equation*}
	\rvp_{\rvw_z}(t) = \lim_{\delta\to 0} \bm{\psi}_{\rvw_z}\left(t+\frac{1/\delta^{L-2}}{L(L-2)\overline{\mathcal{N}}},\delta\overline{\rmW}_{1:L}\right).
	\end{equation*}
	From Theorem \ref{main_thm_Lhm}, for all sufficiently small $\delta$ and for all $t\in [-\widetilde{T},\widetilde{T}]$, we have
	\begin{equation}
	\left\|\bm{\psi}\left(t+ T_{\widetilde{\delta}}^2,\delta\rmW_{1:L}^0\right) - \rvp(t) \right\|_2\leq \widetilde{C}\delta^{\frac{\Delta}{2L^2\overline{\mathcal{N}} +\Delta}},
	\end{equation} 
	where 
	\begin{equation*}
	T_{\widetilde{\delta}}^2 = \frac{T_1}{\delta^{L-2}} + \left(\frac{1/b_\delta^{L-2}}{L(L-2)\overline{\mathcal{N}}} - T\right){\widetilde{\delta}^{\frac{-(L-2)\Delta}{2L^2\overline{\mathcal{N}}\Delta}}}+\frac{T}{\widetilde{\delta}^{L-2}}.
	\end{equation*}
	Since, for any $t\in (-\infty,\infty)$ and $\delta>0$, $\|\bm{\psi}_{\rvw_z}(t + (1/\delta)^{L-2}/(L(L-2)\overline{\mathcal{N}}), \delta\overline{\rmW}_{1:L})\|_2 = 0,$ we have $\|\rvp_{\rvw_z}(t) \|_2 = 0$, for all $t\in (-\infty,\infty)$. Combining this fact with the above inequality gives us \cref{err_bd_Lhm_ff}. 
\end{proof}

\section{Additional Results}\label{additonal_results}
The following lemma describes an instance where $\bm{\psi}(t,\delta\rvw_*)$ does not escape from the origin, if $\rvw_*$ is zero KKT point of the constrained NCF in \cref{ncf_gn_const}.
\begin{lemma}\label{ex_not_escape}
	Let $\mathcal{H}$ be a single-layer squared ReLU network with single neuron, i.e, $\mathcal{H}(\rvx;\rvw) = \max(0,\rvw^\top\rvx)^2$. Let $\rvw_*\in\sS^{d-1}$ be such that $\rvw_*^\top\rvx_i < 0$, for all $i\in [n]$. Then, $\rvw_*$ is a first-order KKT point of \cref{ncf_gn_const} such that $\mathcal{N}(\rvw_*) = 0$. Also, $\bm{\psi}(t,\delta\rvw_*) = \delta\rvw_*$, for all $t\geq 0$.
\end{lemma}
\begin{proof}
	Since $\rvw_*^\top\rvx_i < 0$, $\mathcal{H}(\rvx_i;\rvw_*) = 0 $ and $\nabla\mathcal{H}(\rvx_i;\rvw_*) = \mathbf{0},$ for all $i\in [n]$, which implies $\mathcal{J}(\rmX;\rvw_*) = \mathbf{0}.$ Therefore, $\mathcal{N}(\rvw_*) = 0$ and $\nabla\mathcal{N}(\rvw_*) = \mathbf{0}$, implying $\rvw_*$ is a first-order KKT point of \cref{ncf_gn_const}.
	
	Next, note that $\nabla\mathcal{L}(\delta\rvw_*) = \mathcal{J}(\rmX;\rvw_*)^\top\ell'(\mathcal{H}(\rmX;\delta\rvw_*),\rvy ) = \mathbf{0},$ for all $\delta>0$. Thus, $\delta\rvw_*$ is a critical point of the training loss and therefore, $\bm{\psi}(t,\delta\rvw_*) = \delta\rvw_*$. 
\end{proof}
The above lemma essentially implies that if the weights are chosen such that the neuron is inactive for all the training data, then those weights are a first-order KKT point of \cref{ncf_gn_const}, and moreover, in such cases the gradient flow does not escape from the origin\\

\begin{proof}\textbf{of Example \ref{ex_2hm}: }
	For square loss, $-\ell'(\mathbf{0},\mathbf{y}) = 2\rvy$, hence, the NCF in this example is $\mathcal{N}(\rvw) = 8w_1^2 + 2w_2^2.$ Now, $\bm{\phi}(t,\rvw_0) = (\sqrt{2}e^{16t},\sqrt{2}e^{4t}),$ which implies
	\begin{equation*}
	\lim_{t\to\infty}\frac{\bm{\phi}(t,\rvw_0) }{\|\bm{\phi}(t,\rvw_0) \|_2} = (1,0) = \rvw_*.
	\end{equation*}
	Let $(w_1(t),w_2(t)) = \bm{\psi}(t,\delta\rvw_0)$, then
	\begin{equation*}
	\dot{w}_1 = -4w_1(w_1^2-4), w_1(0) = \delta/\sqrt{2}, \text{ and }\dot{w}_2 = -4w_2(w_2^2-1), w_2(0) = \delta/\sqrt{2}.
	\end{equation*}
	Therefore, $ w_1(t) = \frac{2\delta}{\sqrt{\delta^2 + (8-\delta^2)e^{-32t}}}$ and $w_2(t) = \frac{\delta}{\sqrt{\delta^2 + (2-\delta^2)e^{-8t}}}$. Next, let $(u_1(t),u_2(t)) = \bm{\psi}(t,\delta\rvw_*)$, then
	\begin{equation*}
	\dot{u}_1 = -4u_1(u_1^2-4), u_1(0) = \delta, \text{ and }\dot{u}_2 = -4u_2(u_2^2-1), u_2(0) = 0.
	\end{equation*}
	Therefore, $u_1(t) = \frac{2\delta}{\sqrt{\delta^2 + (4-\delta^2)e^{-32t}}}, u_2(t) = 0.$ Since $\mathcal{N}(\rvw_*) = 8$, 
	\begin{equation*}
	\rvp(0) = \lim_{\delta\to 0}\left(u_1\left(\frac{\ln\left({1}/{\delta}\right)}{2\mathcal{N}(\rvw_*)}\right), 0\right) = \lim_{\delta\to 0}\left(\frac{2\delta}{\sqrt{\delta^2 + (4-\delta^2)\delta^2}}, 0\right) = (2/\sqrt{5}, 0),
	\end{equation*}
	which implies $\rvp(t) = \left(\frac{2}{\sqrt{1 + 4e^{-32t}}}, 0\right)$.
\end{proof}
\textbf{Difficulty in proving Theorem \ref{main_thm_2hm} using Lemma \ref{lemma_early_dc}: } In Remark \ref{remark_diff}, we briefly discussed why \Cref{main_thm_2hm} can not be proved using Lemma \ref{lemma_early_dc}. Here, we elaborate more on that topic. 

Suppose $\rvw_0 = \rvw_*+\epsilon\rvb$, where $\rvb \in \rvw_*^\perp$, that is, approximate directional convergence holds at the initialization itself. We assume the training loss has a globally Lipschitz gradient with Lipschitz constant $\kappa > 0$ to simplify the argument, though in reality, the gradient is only locally Lipschitz, a weaker condition. Then,
\begin{align*}
&\frac{1}{2}\frac{d\|\bm{\psi}(t,\delta( \rvw_*+\epsilon\rvb)) - \bm{\psi}(t,\delta \rvw_*)\|_2^2}{dt} \\
&= -(\bm{\psi}(t,\delta( \rvw_*+\epsilon\rvb)) - \bm{\psi}(t,\delta \rvw_*))^\top(\nabla\mathcal{L}(\bm{\psi}(t,\delta( \rvw_*+\epsilon\rvb)) ) - \nabla\mathcal{L}(\bm{\psi}(t,\delta \rvw_*) ))\\
&\leq \kappa\|\bm{\psi}(t,\delta( \rvw_*+\epsilon\rvb)) - \bm{\psi}(t,\delta \rvw_*)\|_2^2,
\end{align*}
which implies
\begin{equation}
\|\bm{\psi}(t,\delta( \rvw_*+\epsilon\rvb)) - \bm{\psi}(t,\delta \rvw_*)\|_2\leq e^{\kappa t}\delta\epsilon.
\label{bd_psi_diff}
\end{equation}
The above bound becomes vacuous after $\ln(1/(\delta\epsilon))/\kappa$ time has elapsed, where note that $\epsilon$ is small but \emph{fixed}. However, from Theorem \ref{main_thm_2hm}, we know that $\bm{\psi}(t,\delta \rvw_*)$ escapes from the origin after $\ln(1/(\delta))/(2\mathcal{N}(\rvw_*))$ time has elapsed. Therefore, the above inequality can be useful in describing the gradient flow dynamics beyond the origin if, for all sufficiently small $\delta>0$, $\kappa \leq 2\mathcal{N}(\rvw_*)$, which we are not sure can be proven. The problem is further aggravated if we consider deep homogeneous neural networks. In that case, \cref{bd_psi_diff} still holds, but note that $\bm{\psi}(t,\delta \rvw_*)$ escapes from the origin after $(1/\delta^{L-2})/(L(L-2\mathcal{N}(\rvw_*)))$ time has elapsed, which is always greater than  $\ln(1/(\delta\epsilon))/\kappa$, for all all sufficiently small $\delta>0$ and $L>2$. 
\section{Additional Experiments}
\label{sec:add_exp}
\subsection{Four-layer ReLU Neural Network}
To further demonstrate the preservation of sparsity structure after weights escape from the origin, we conduct experiments with four-layer ReLU neural networks. The setup closely follows that of the earlier experiments: the weights are trained using gradient descent with small initialization, and training is continued until the weights escape the origin and reach the next saddle point.\\\\
In \Cref{fig:4l_relu}, similar to earlier experiments, the training set consists of 100 points sampled uniformly from the unit sphere in $\sR^{20}$, with the corresponding labels generated by a smaller four-layer network. We observe that certain rows and columns of the weight matrices become relatively small before the weights escape the origin, and they remain small even until reaching the next saddle point. Moreover, these rows and columns follow the definition of zero-preserving subset: if the $j$-th row of $\rmW_l$ is small, then the $j$-th column of $\rmW_{l+1}$ is small, and vice-versa. For instance, the first row of $\rmW_1$ and first column of $\rmW_2$, rows 16-18 of $\rmW_2$ and columns 16-18 of  $\rmW_3$, and rows 17-19 of $\rmW_3$ and columns 17-19 of  $\rvv^\top$ stay small at iteration $i_1$ and $i_2$. 
\begin{figure}[htbp]
	\centering
	\begin{minipage}[c]{0.3\textwidth}
		\centering
		\includegraphics[width=\textwidth]{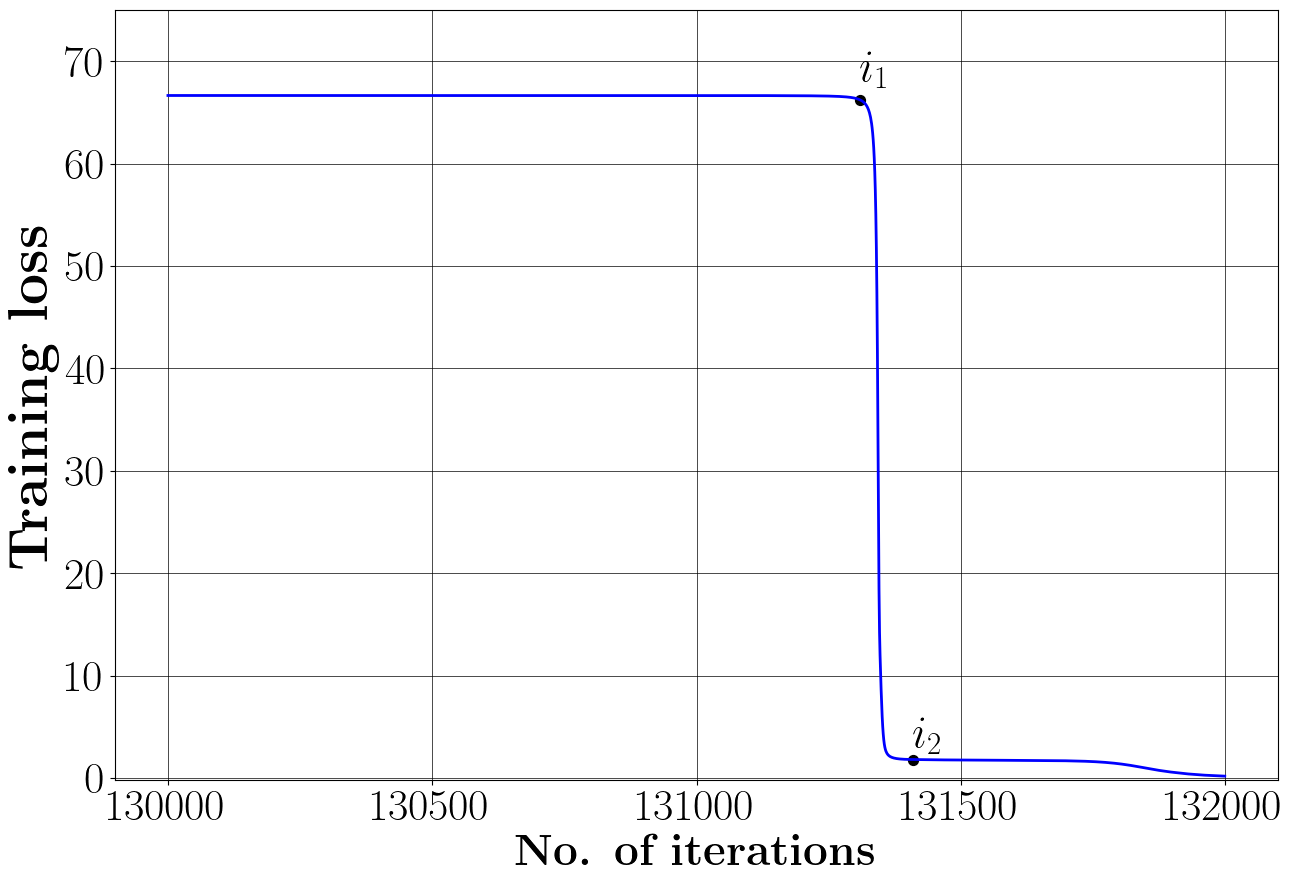}
		\\ (a) Evolution of training loss with iterations
	\end{minipage}\hspace{1cm}%
	\begin{minipage}[c]{0.6\textwidth}
		\centering
		\includegraphics[width=\textwidth]{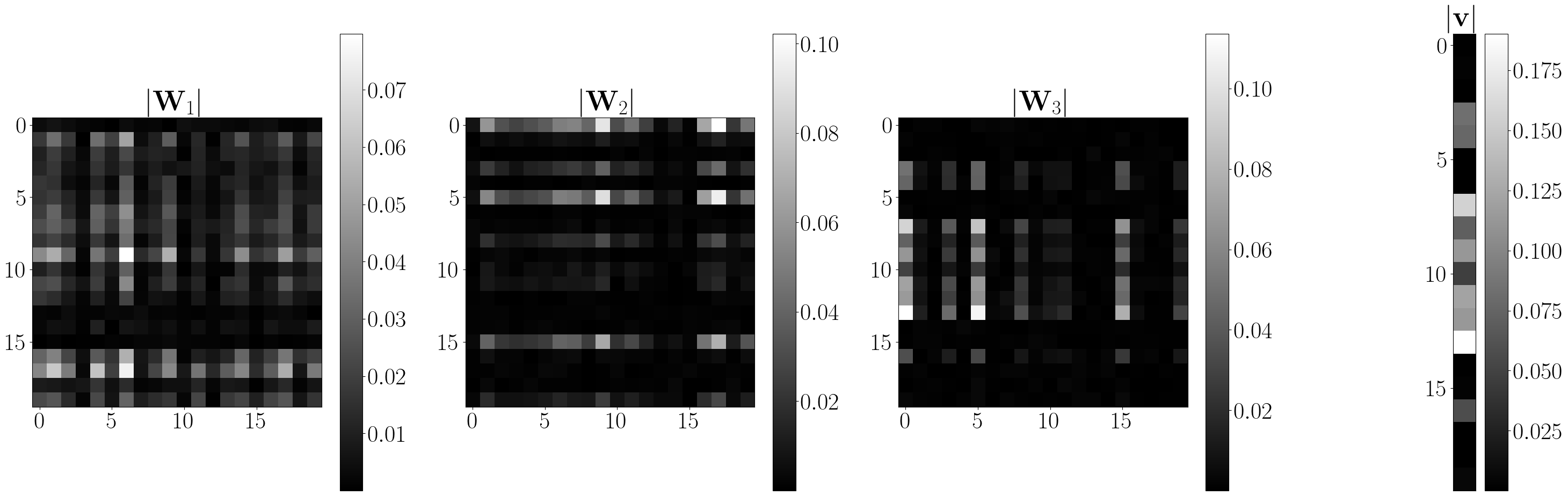}
		\\ (b) Weights at iteration $i_1$ \\[1ex]
		\includegraphics[width=\textwidth]{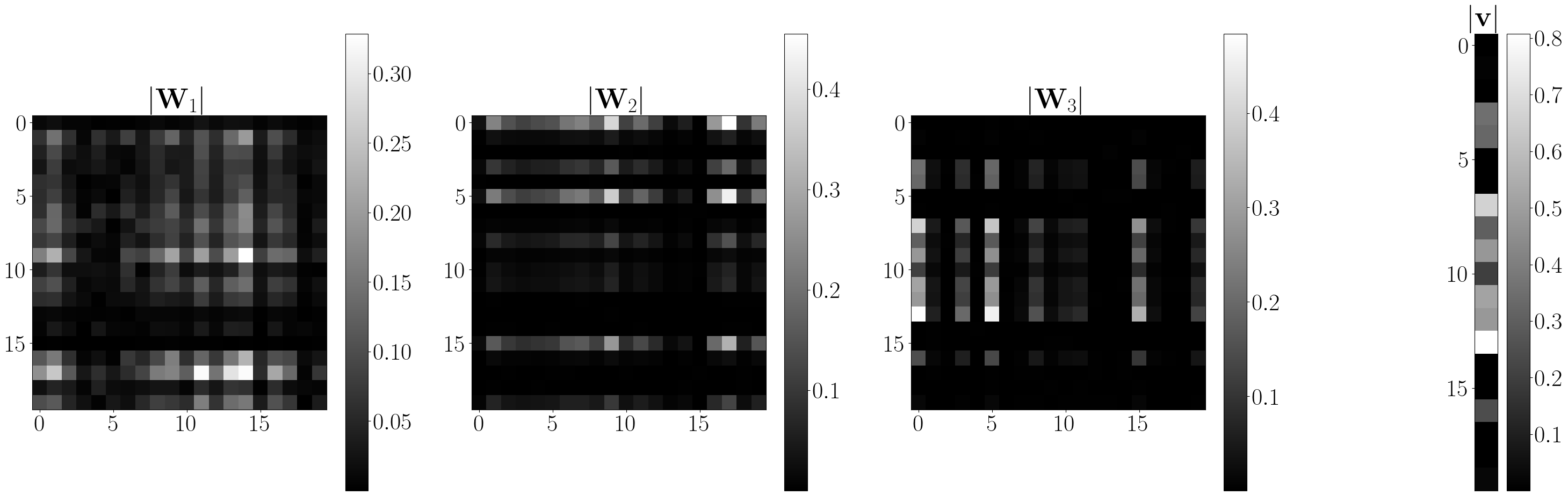}
		\\ (c) Weights at iteration $i_2$
	\end{minipage}
 	\caption{We train a four-layer neural network whose output is $\rvv^\top\sigma(\rmW_3\sigma(\rmW_2\sigma(\rmW_1\rvx))) ,$ where $\sigma(x) = \max(x,0)$, and $\rvv\in \mathbb{R}^{20},\rmW_3,\rmW_2,\rmW_1  \in \mathbb{R}^{20 \times 20}$ are the trainable weights. We observe that rows and columns of the weight matrices that become small near the origin remain small until gradient descent reaches the next saddle point, 
	demonstrating preservation of the sparsity structure.}
\label{fig:4l_relu}
\end{figure}

\noindent We next consider a more realistic setting in \Cref{fig:4_layer_rl_mnist}, where the task is to classify digits $0$ and $1$ from the MNIST dataset \citep{lecun_mnist}. Digits are labeled as $-1$ (for $0$) and $1$ (for $1$). The original $28\times 28$ images are down-sampled to size $14\times 14$, and each image is flattened into a vector of length $196$. Training is performed using a four-layer ReLU network with square loss. Once again, we observe that certain rows and columns of the weight matrices become relatively small before escaping the origin, and they remain small until reaching the next saddle point. Also, these rows and columns follow the definition of zero-preserving subset. For example, the second row of $\rmW_1$ and second column of $\rmW_2$, first five rows of $\rmW_2$ and first five columns of  $\rmW_3$, and first three rows of $\rmW_3$ and first three columns of  $\rvv^\top$ stay small at iteration $i_1$ and $i_2$. 

\begin{figure}[htbp]
	\centering
	\begin{minipage}[c]{0.3\textwidth}
		\centering
		\includegraphics[width=\textwidth]{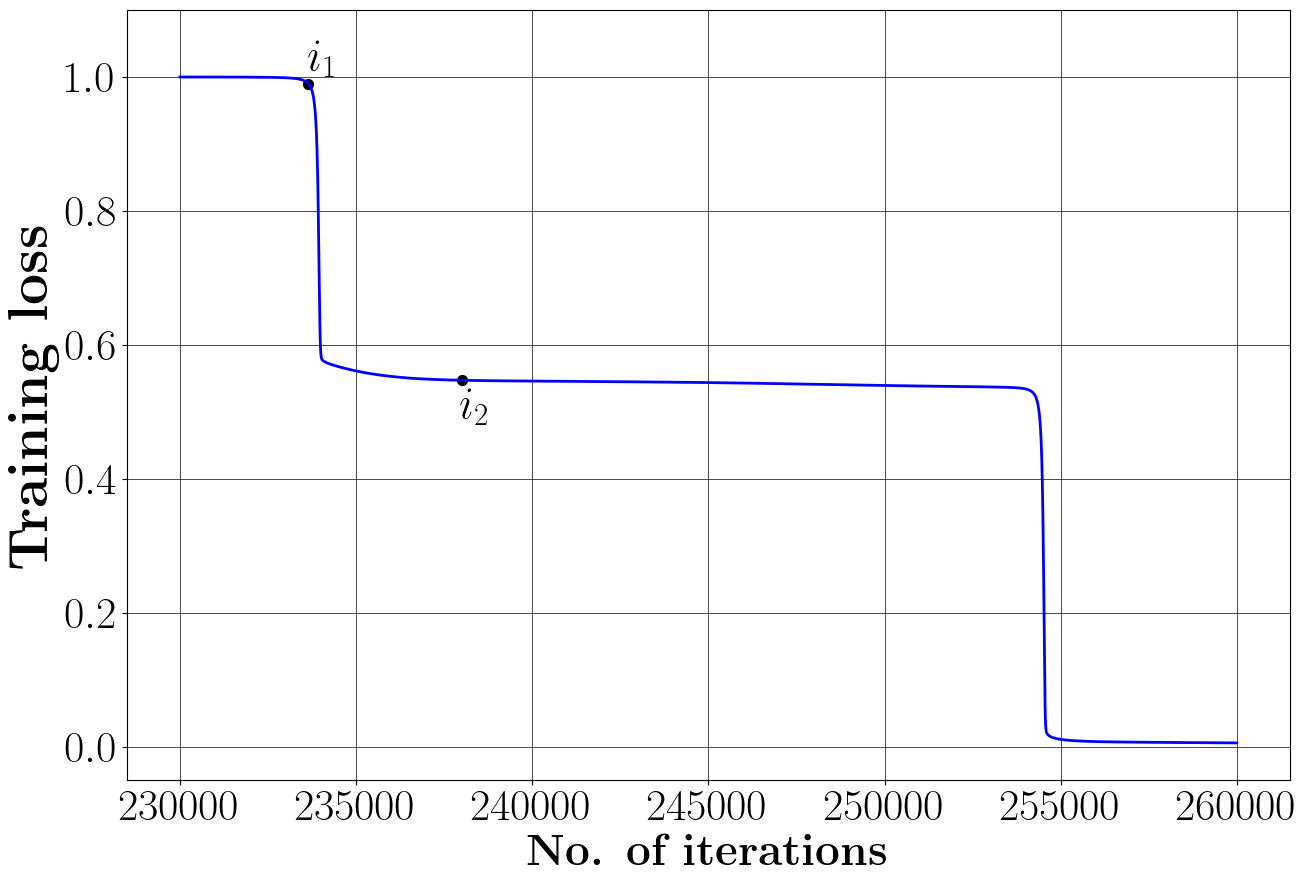}
		\\ (a) Evolution of training loss with iterations
	\end{minipage}\hspace{1cm}%
	\begin{minipage}[c]{0.6\textwidth}
		\centering
		\includegraphics[width=\textwidth]{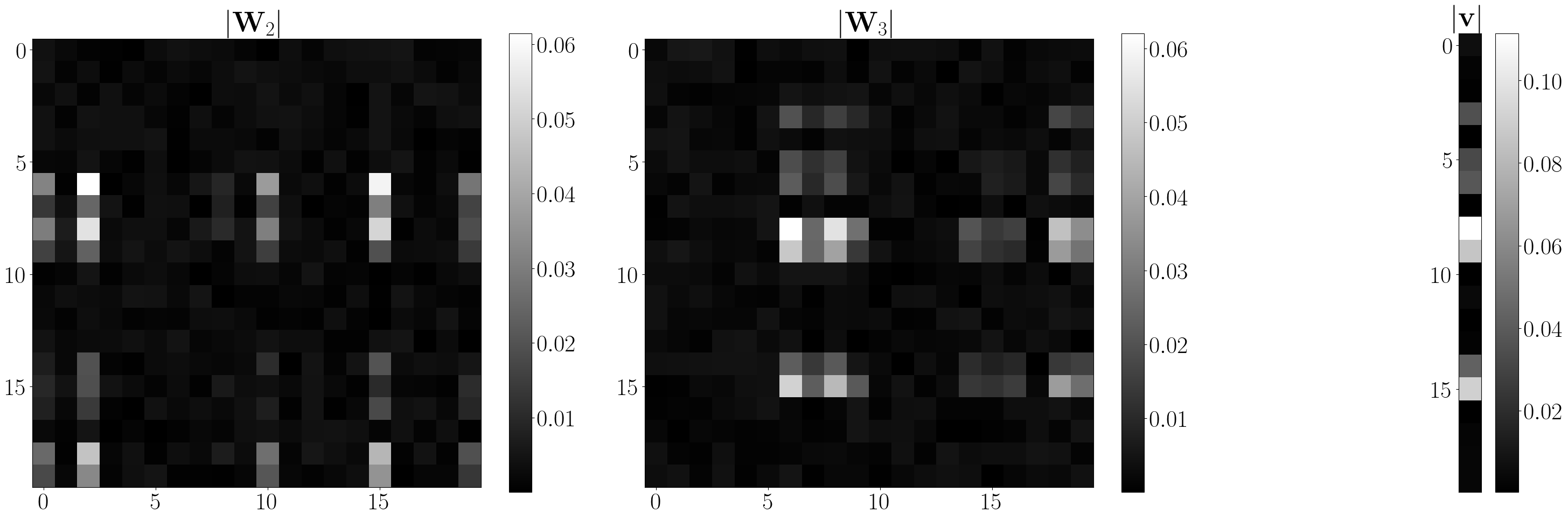}
		\\ (b) Weights of last three layers at iteration $i_1$ \\[1ex]
		\includegraphics[width=\textwidth]{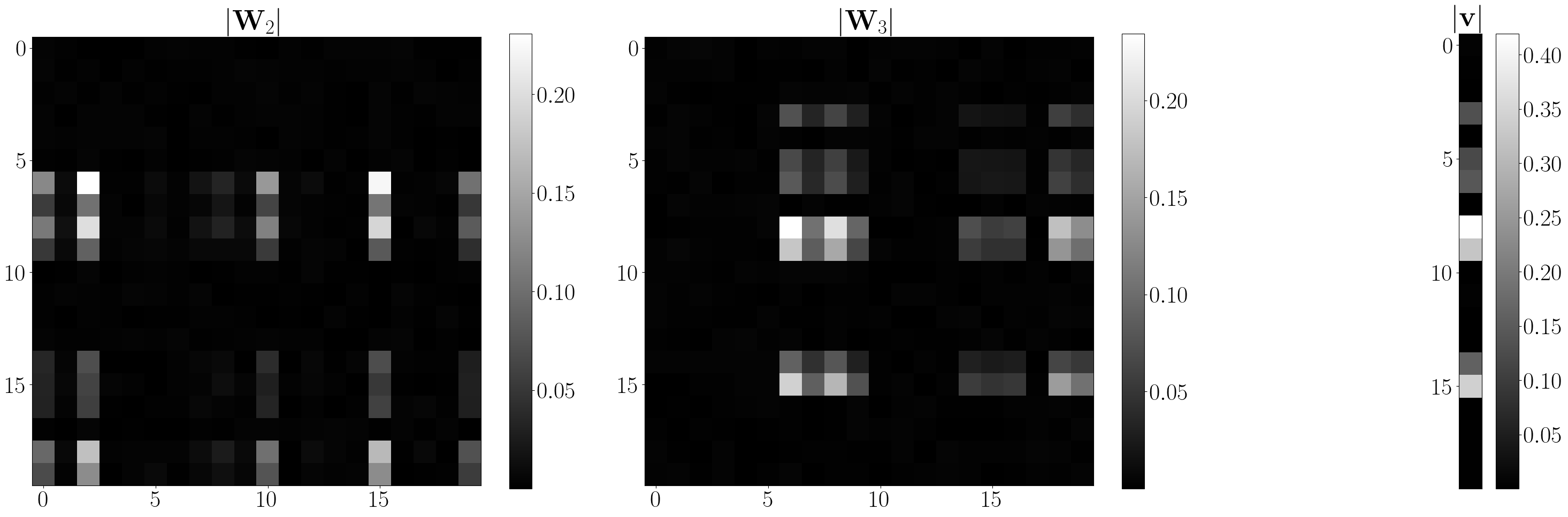}
		\\ (c) Weights of last three layers at iteration $i_2$
	\end{minipage}
	
	\vspace{0.3cm} 
	
	\begin{minipage}{0.9\textwidth}
		\centering
		\includegraphics[width=\textwidth]{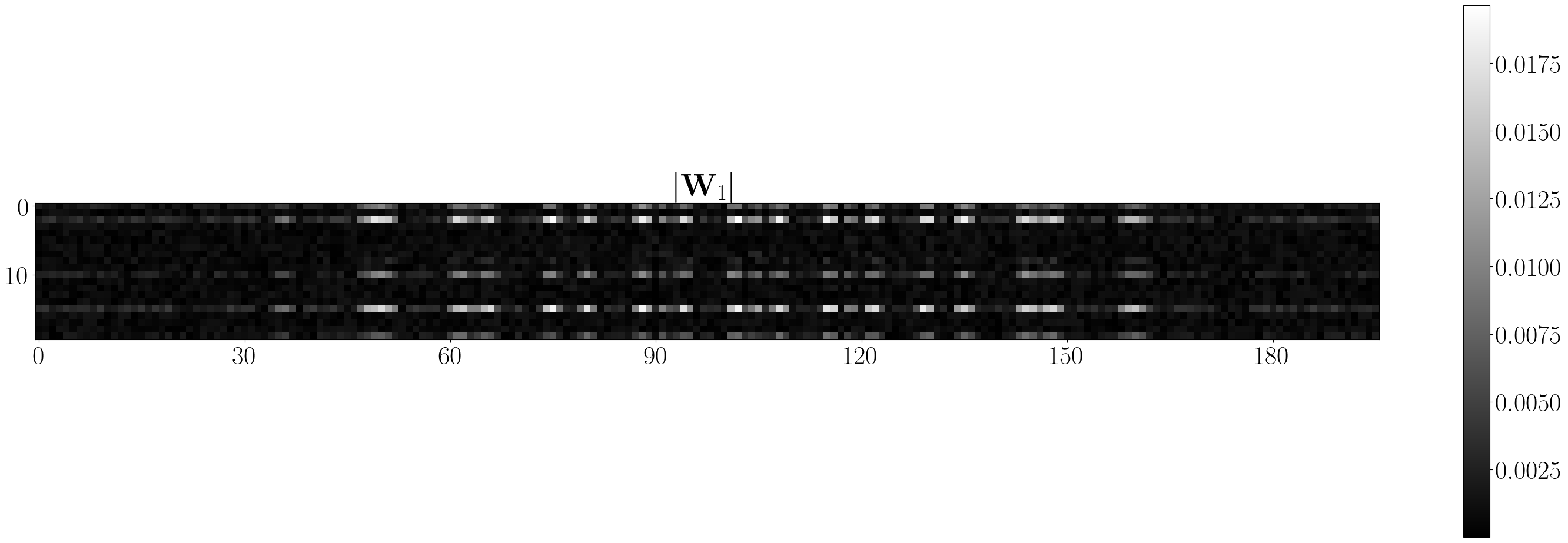} 
		\\ (d) Weights of first layer at iteration $i_1$
	\end{minipage}
	
	\vspace{0.3cm} 
	
	\begin{minipage}{0.9\textwidth}
		\centering
		\includegraphics[width=\textwidth]{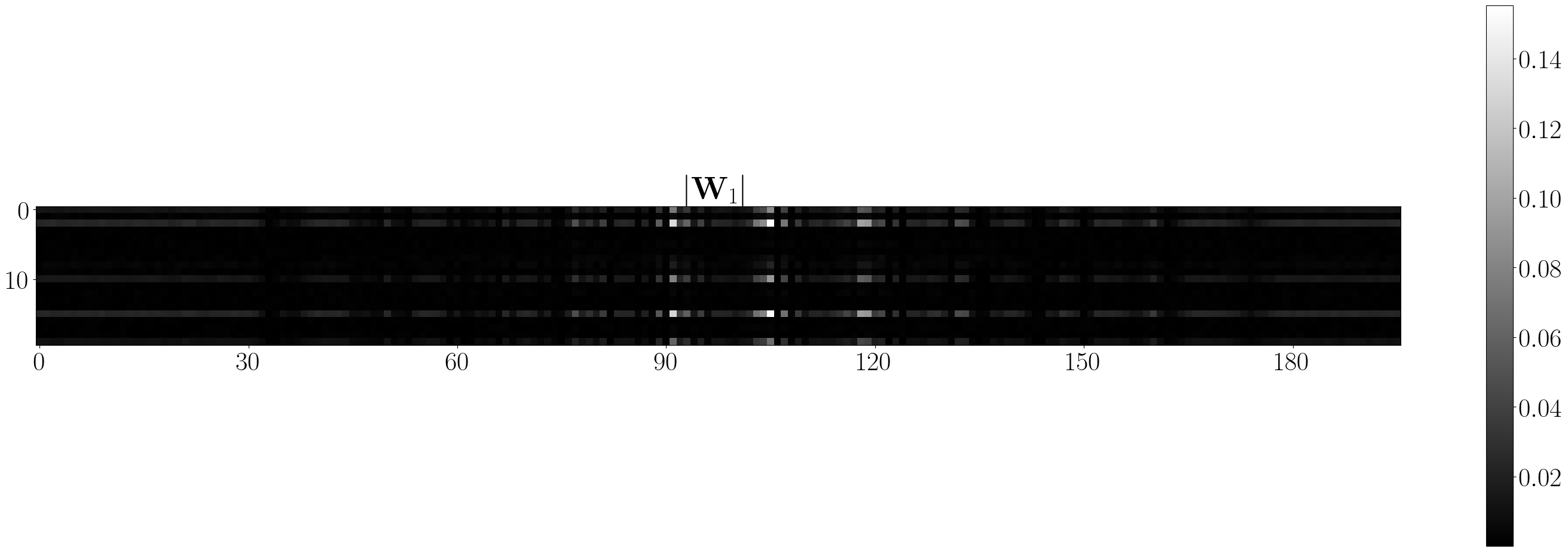} 
		\\ (e) Weights of first layer at iteration $i_2$
	\end{minipage}
	
	\caption{We train a four-layer neural network to classify digits $0$ and $1$ from the MNIST dataset. 
		The images are down-sampled by a factor of $2$ (from $28 \times 28$ to $14 \times 14$) and reshaped into vectors of length $196$. The network output is $\rvv^\top\sigma(\rmW_3\sigma(\rmW_2\sigma(\rmW_1\rvx))) ,$ where $\sigma(x) = \max(x,0)$, and $\rvv\in \mathbb{R}^{20},\rmW_3,\rmW_2 \in \mathbb{R}^{20 \times 20} ,\rmW_1  \in \mathbb{R}^{20 \times 196}$ are the trainable weights. The rows and columns of the weight matrices that become small near the origin remain small until the next saddle point.}
	\label{fig:4_layer_rl_mnist}
\end{figure}

\subsection{Non-homogeneous Activation Function}
Although our theoretical results are stated for homogeneous activations, we evaluate whether the preservation of the sparsity structure also occurs with non-homogeneous activations, specifically tanh and Gaussian Error Linear Unit (GELU) \citep{gelu}. In all experiments the weights are trained using gradient descent with small initialization, and training is continued until the weights escape the origin and reach the next saddle point.\\\\
\textbf{Tanh activation function. }We train two-, three-, and four-layer neural network with tanh activation function, where the results are depicted in \Cref{fig:2_layer_nn_tanh}, \Cref{fig:3_layer_nn_tanh} and \Cref{fig:4l_tanh}, respectively. The training set consists of 100 points sampled uniformly from the unit sphere in $\sR^{20}$, with the corresponding labels generated by a smaller network. In all three cases we observe the same qualitative behavior seen for homogeneous activations: certain rows and columns of the weight matrices become relatively small before the weights escape from the origin, and they remain small until reaching the next saddle point. Moreover, these rows and columns also follow the definition of zero-preserving subset. For instance, in the two-layer case, first row of $\rmW_1$ and first column of  $\rvv^\top$ stay small at iteration $i_1$ and $i_2$. In the three-layer case, last row of $\rmW_1$ and last column of $\rmW_2$, and row 27 of $\rmW_2$ and column 27 of  $\rvv^\top$ stay small at iteration $i_1$ and $i_2$. In the four-layer case, row 14 of $\rmW_1$ and column 14 of $\rmW_2$, the row 25 of $\rmW_2$ and column 25 of $\rmW_3$, and first row of $\rmW_3$ and first column of  $\rvv^\top$ stay small at iteration $i_1$ and $i_2$. \\\\
\textbf{GELU activation function. }We train two- and three-layer neural network with GELU activation function, where the results are depicted in \Cref{fig:2l_gelu} and \Cref{fig:3l_gelu}, respectively. The training set consists of 100 points sampled uniformly from the unit sphere in $\sR^{20}$, with the corresponding labels generated by a smaller network. In both cases, before the weights escape from the origin, certain rows and columns of the weight matrices become relatively small. However, as the weights escape from the origin and training further progresses, it appears that the sparsity among the weights increases, as even the rows and columns which were not small near the origin become small as the training progresses. For example, in the two-layer case, row 40-42 of $\rmW_1$ is not relatively small at iteration $i_1$, but it becomes quite small at iteration $i_3$ and beyond. In the three-layer case, first two rows of $\rmW_1$ are not relatively small at iteration $i_1$, but it becomes quite small at iteration $i_4$ and beyond.

This behavior contrasts with the above experiments using homogeneous activation and tanh activation.
While the precise reason behind this is not entirely clear to us, one possible explanation lies in the motivation behind GELU. As discussed in \citet{gelu}, GELU activation was derived by combining dropout with ReLU. Since dropout is known to encourage sparsity, perhaps this leads to increased sparsity in the weights after escaping from the origin.\\\\
Overall, these set of experiments indicate that the phenomenon of preservation of the sparsity structure   extends beyond homogeneous activation functions, though it may not hold in general. Exploring the conditions under which this phenomenon persists is an important direction for future research.

\begin{figure}[htbp]
	\centering
	\begin{minipage}[c]{0.3\textwidth}
		\centering
		\includegraphics[width=\textwidth]{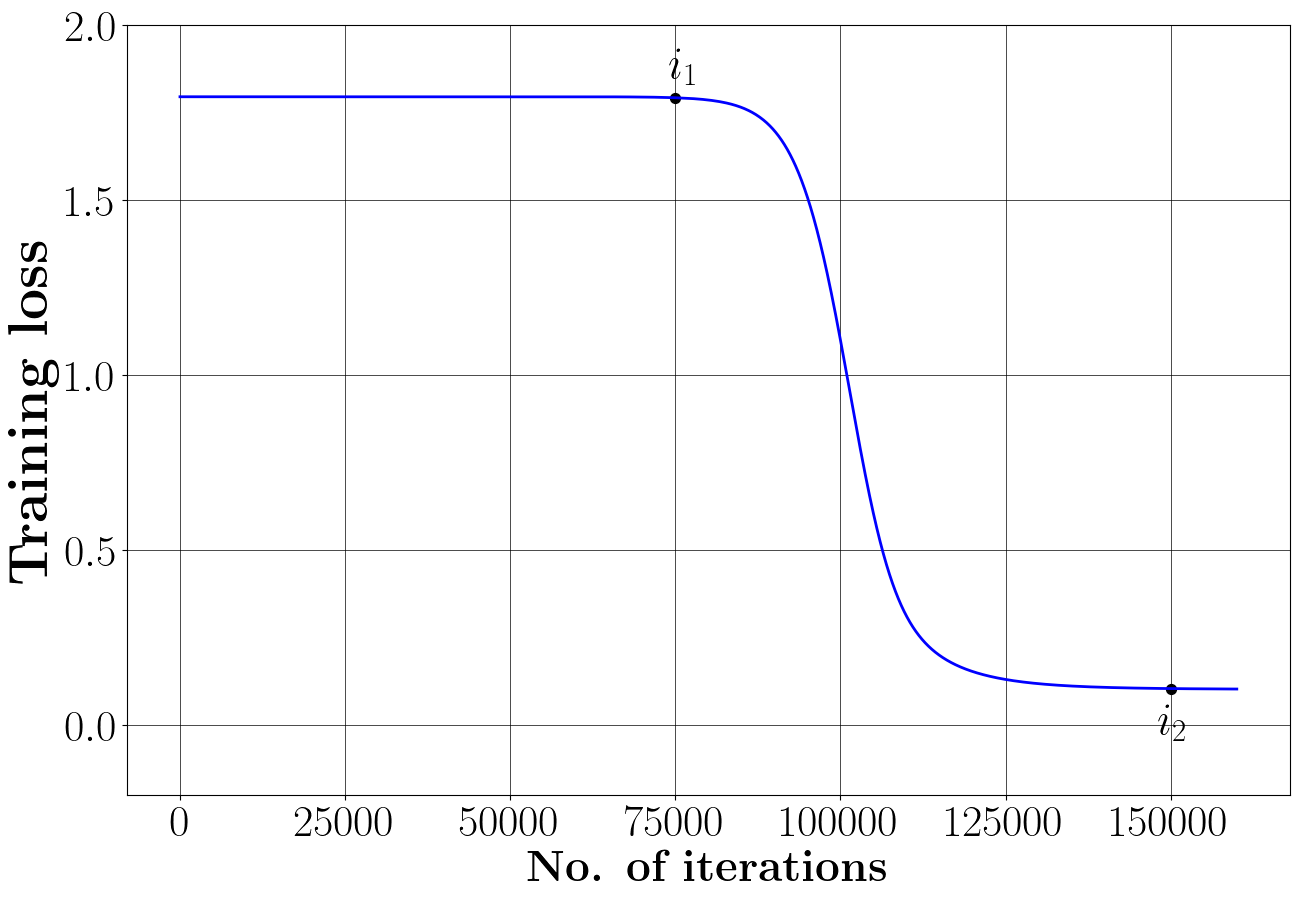}
		\\ (a) Evolution of training loss with iterations
	\end{minipage}\hspace{1cm}%
	\begin{minipage}[c]{0.6\textwidth}
		\centering
		\includegraphics[width=0.4\textwidth]{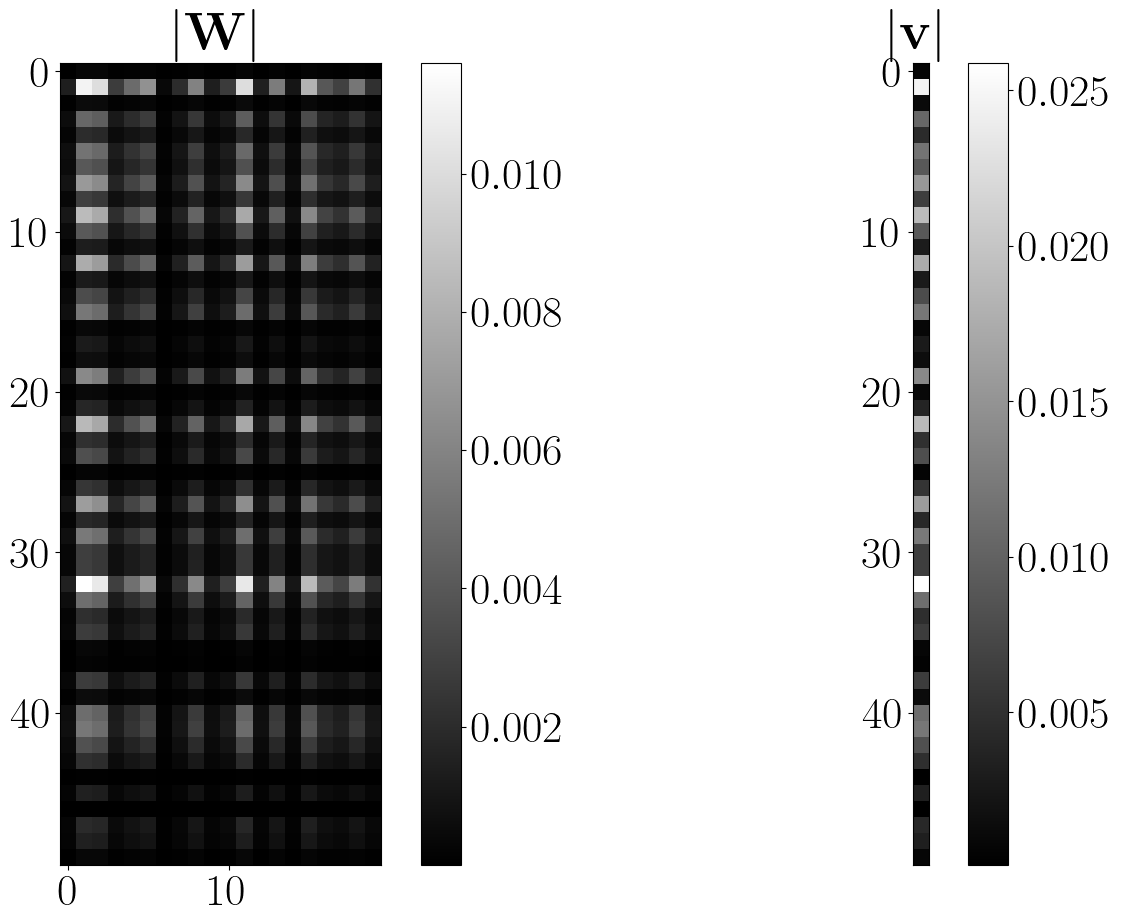}
		\\ (b) Weights at iteration $i_1$ \\[1ex]
		\includegraphics[width=0.4\textwidth]{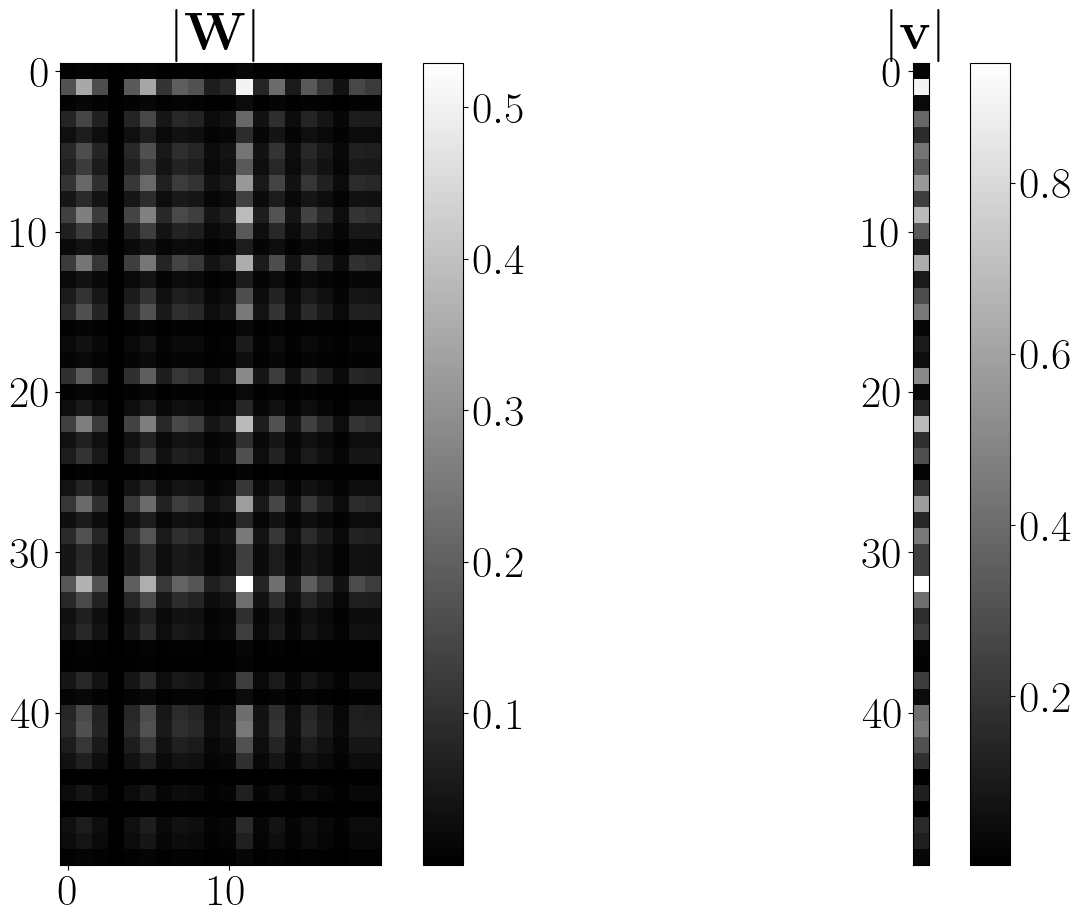}
		\\ (c) Weights at iteration $i_2$
	\end{minipage}
	\caption{We train a two-layer neural network whose output is $\rvv^\top\sigma(\rmW_1\rvx) ,$ where $\sigma(x) ={ \rm tanh}(x)$, and $\rvv\in \mathbb{R}^{50},\rmW_1  \in \mathbb{R}^{50 \times 20}$ are the trainable weights. The sparsity structure that emerges among the weights before before escaping the origin is preserved post-escape and until reaching the next saddle point.}
	\label{fig:2_layer_nn_tanh}
\end{figure}
\begin{figure}[htbp]
	\centering
	\begin{minipage}[c]{0.3\textwidth}
		\centering
		\includegraphics[width=\textwidth]{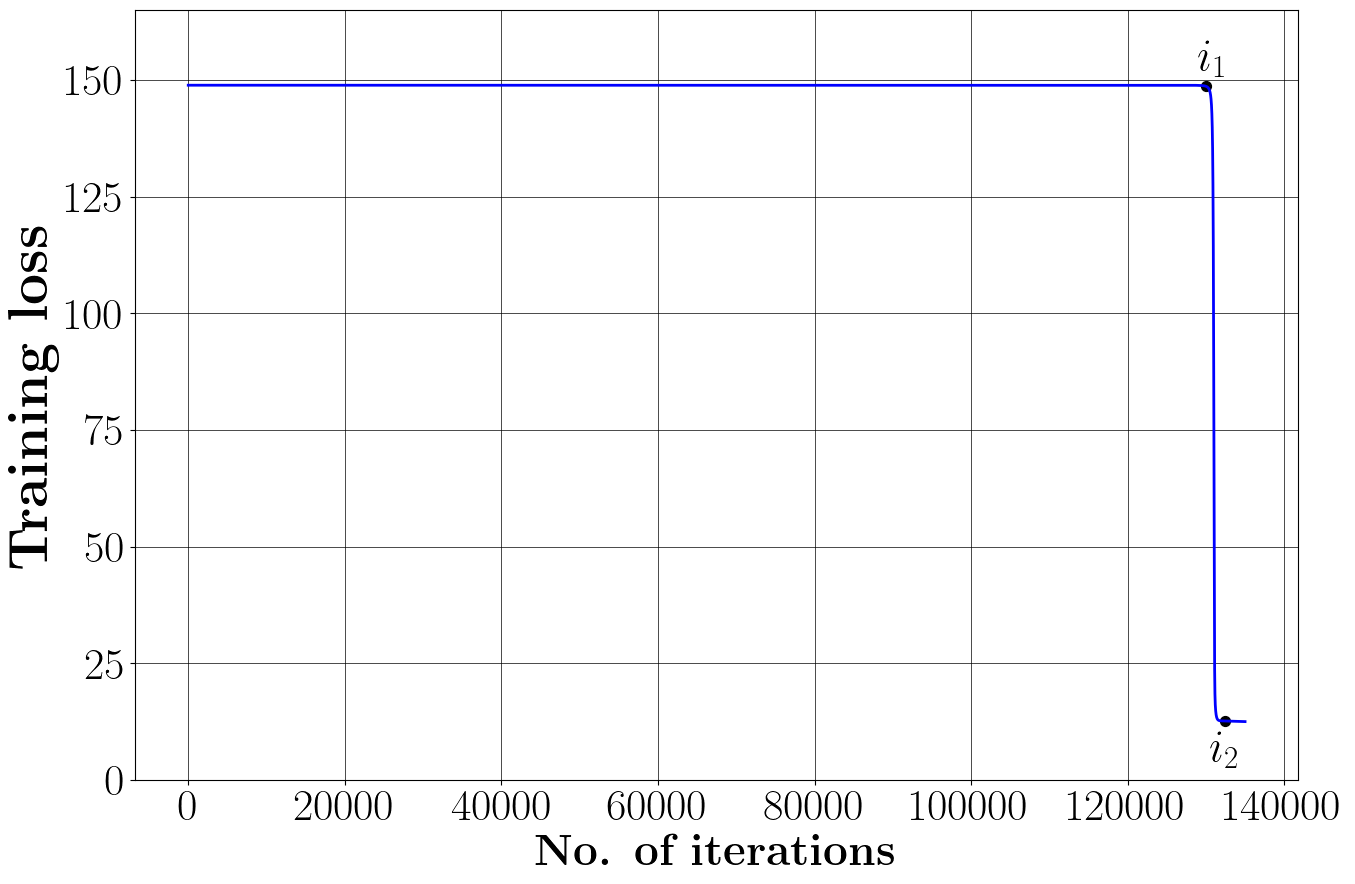}
		\\ (a) Evolution of training loss with iterations
	\end{minipage}\hspace{1cm}%
	\begin{minipage}[c]{0.6\textwidth}
		\centering
		\includegraphics[width=\textwidth]{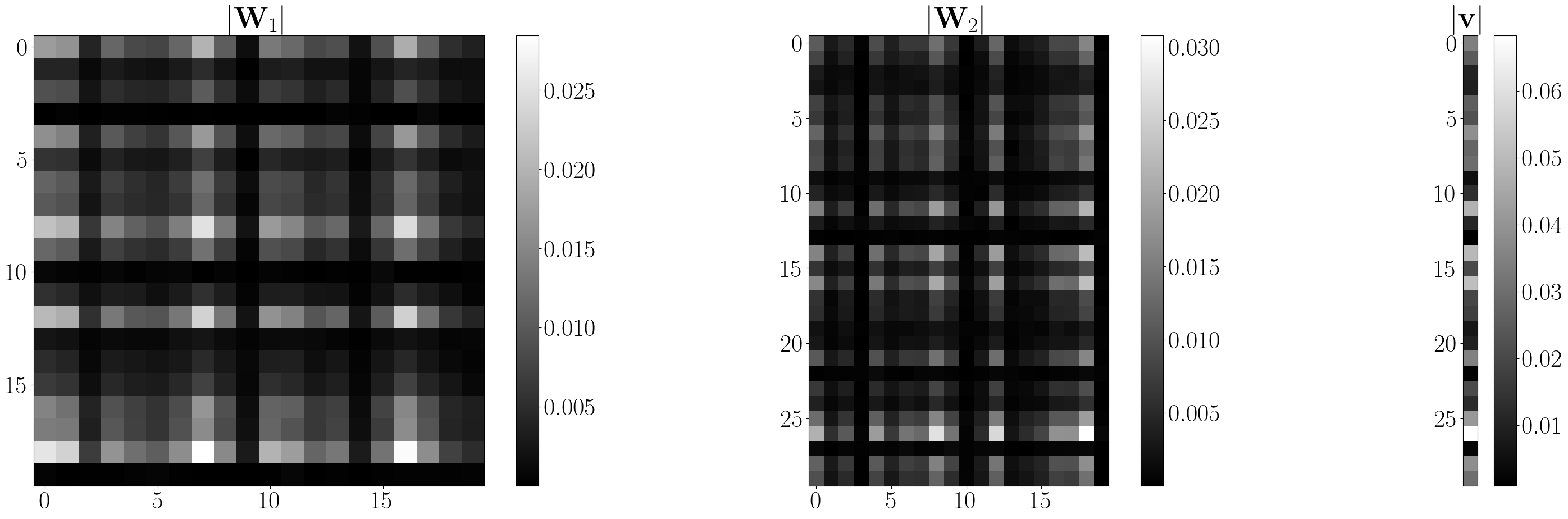}
		\\ (b) Weights at iteration $i_1$ \\[1ex]
		\includegraphics[width=\textwidth]{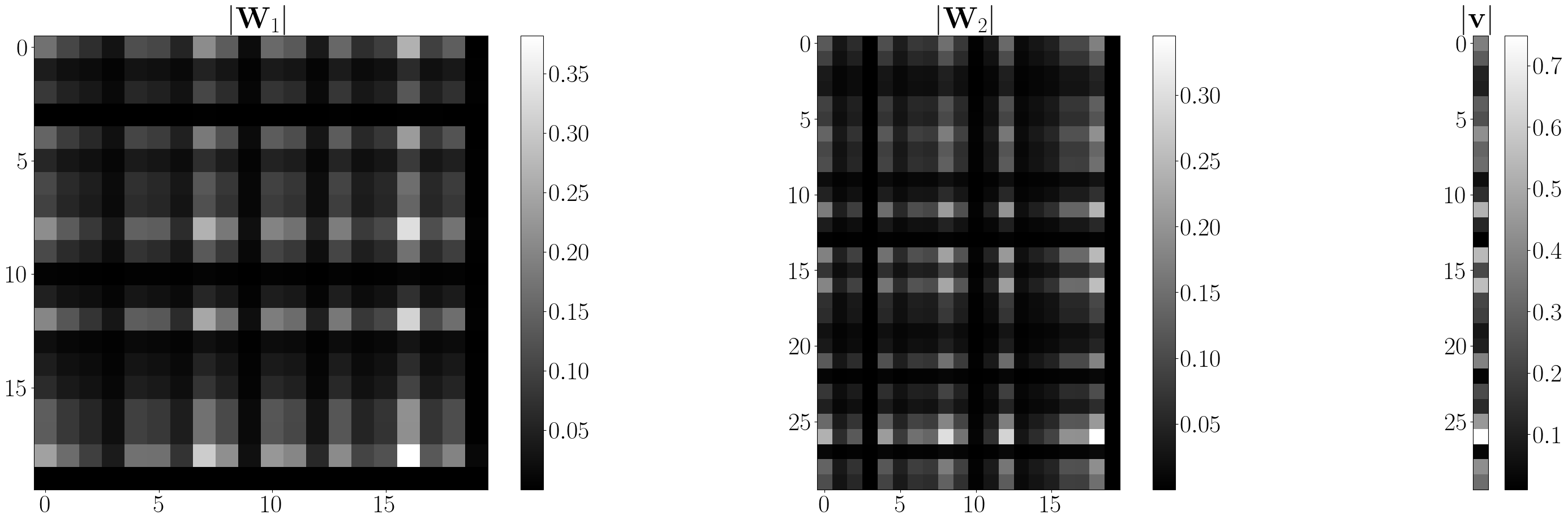}
		\\ (c) Weights at iteration $i_2$
	\end{minipage}
	\caption{We train a three-layer neural network whose output is $\rvv^\top\sigma(\rmW_2\sigma(\rmW_1\rvx)) ,$ where $\sigma(x) ={ \rm tanh}(x)$, and $\rvv\in \mathbb{R}^{30},\rmW_2\in \mathbb{R}^{30 \times 20},\rmW_1  \in \mathbb{R}^{20 \times 20}$ are the trainable weights. The sparsity structure is preserved upon escaping from the origin.}
	\label{fig:3_layer_nn_tanh}
\end{figure}

\begin{figure}[htbp]
	\centering
	\begin{minipage}[c]{0.3\textwidth}
		\centering
		\includegraphics[width=\textwidth]{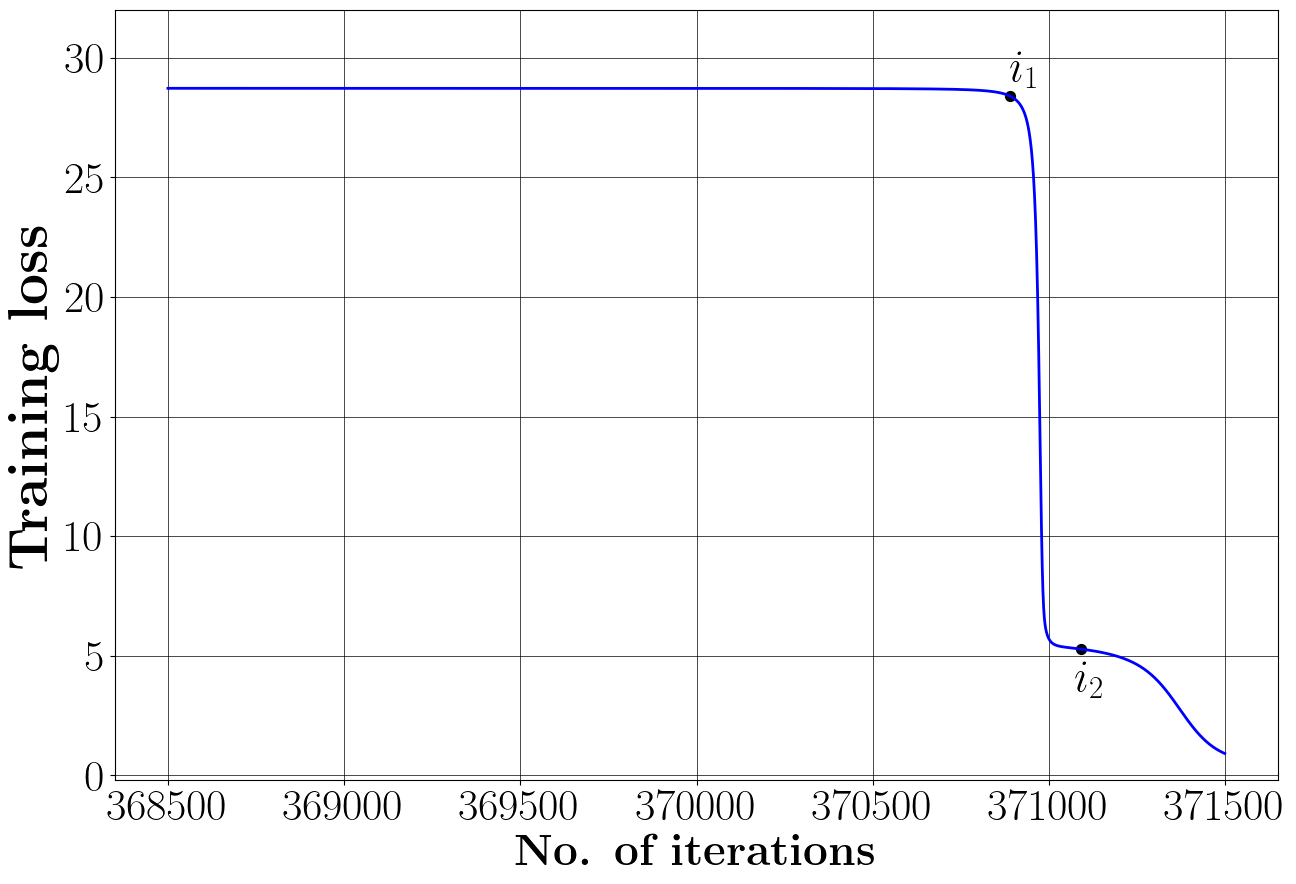}
		\\ (a) Evolution of training loss with iterations
	\end{minipage}\hspace{1cm}%
	\begin{minipage}[c]{0.6\textwidth}
		\centering
		\includegraphics[width=\textwidth]{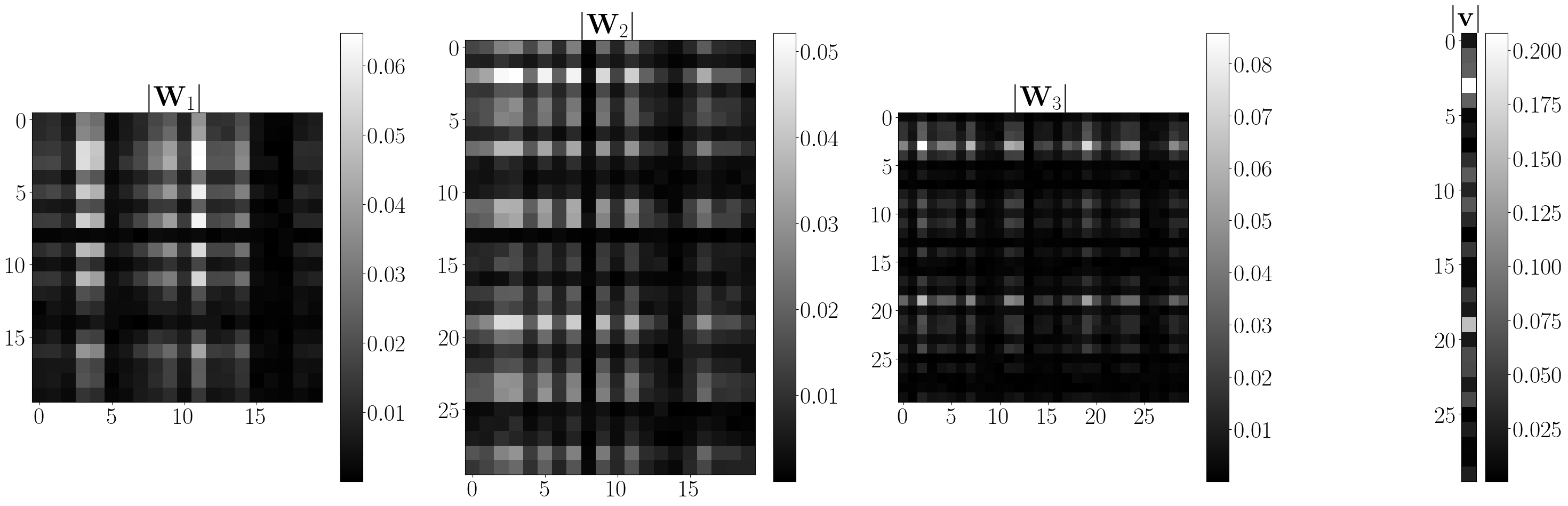}
		\\ (b) Weights at iteration $i_1$ \\[1ex]
		\includegraphics[width=\textwidth]{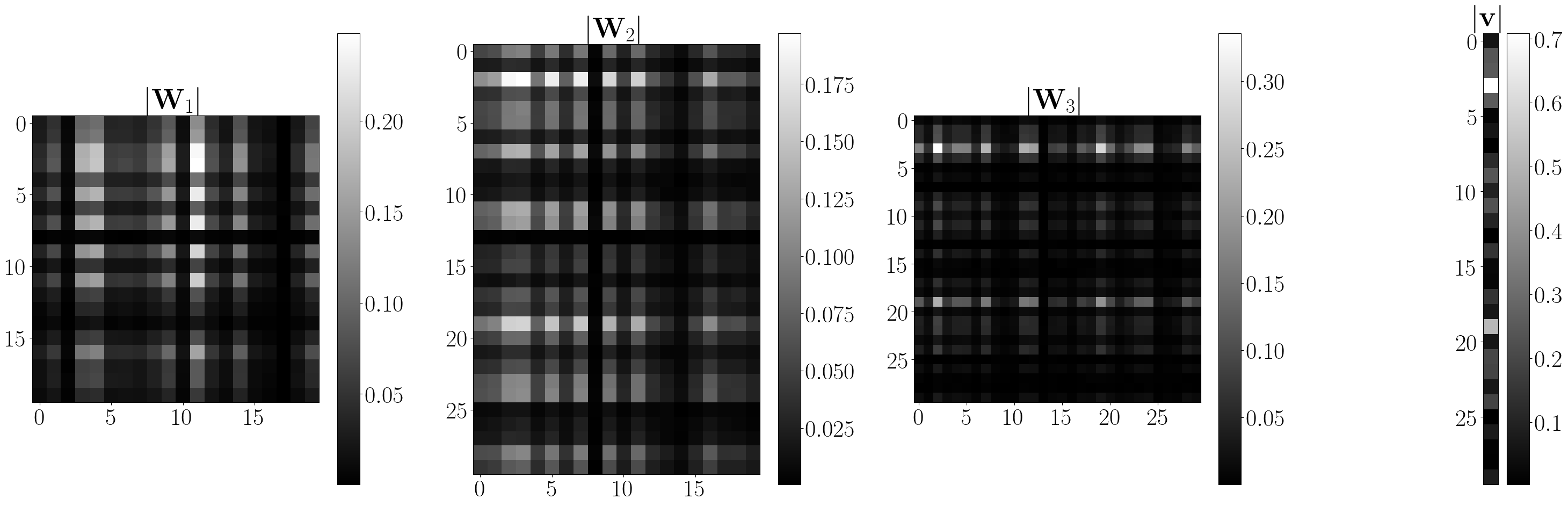}
		\\ (c) Weights at iteration $i_2$
	\end{minipage}
	\caption{We train a four-layer neural network whose output is $\rvv^\top\sigma(\rmW_3\sigma(\rmW_2\sigma(\rmW_1\rvx))) ,$ where $\sigma(x) = { \rm tanh}(x)$, and $\rvv\in \mathbb{R}^{30},\rmW_3 \in \mathbb{R}^{30 \times 30},\rmW_2 \in \mathbb{R}^{30 \times 20},\rmW_1  \in \mathbb{R}^{20 \times 20}$ are the trainable weights. The rows and columns of the weight matrices that become small near the origin remain small until gradient descent reaches the next saddle point, demonstrating preservation of the sparsity structure.}
	\label{fig:4l_tanh}
\end{figure}

\begin{figure}[htbp]
	\centering
	
	\begin{subfigure}{0.4\textwidth}
		\centering
		\includegraphics[width=\linewidth]{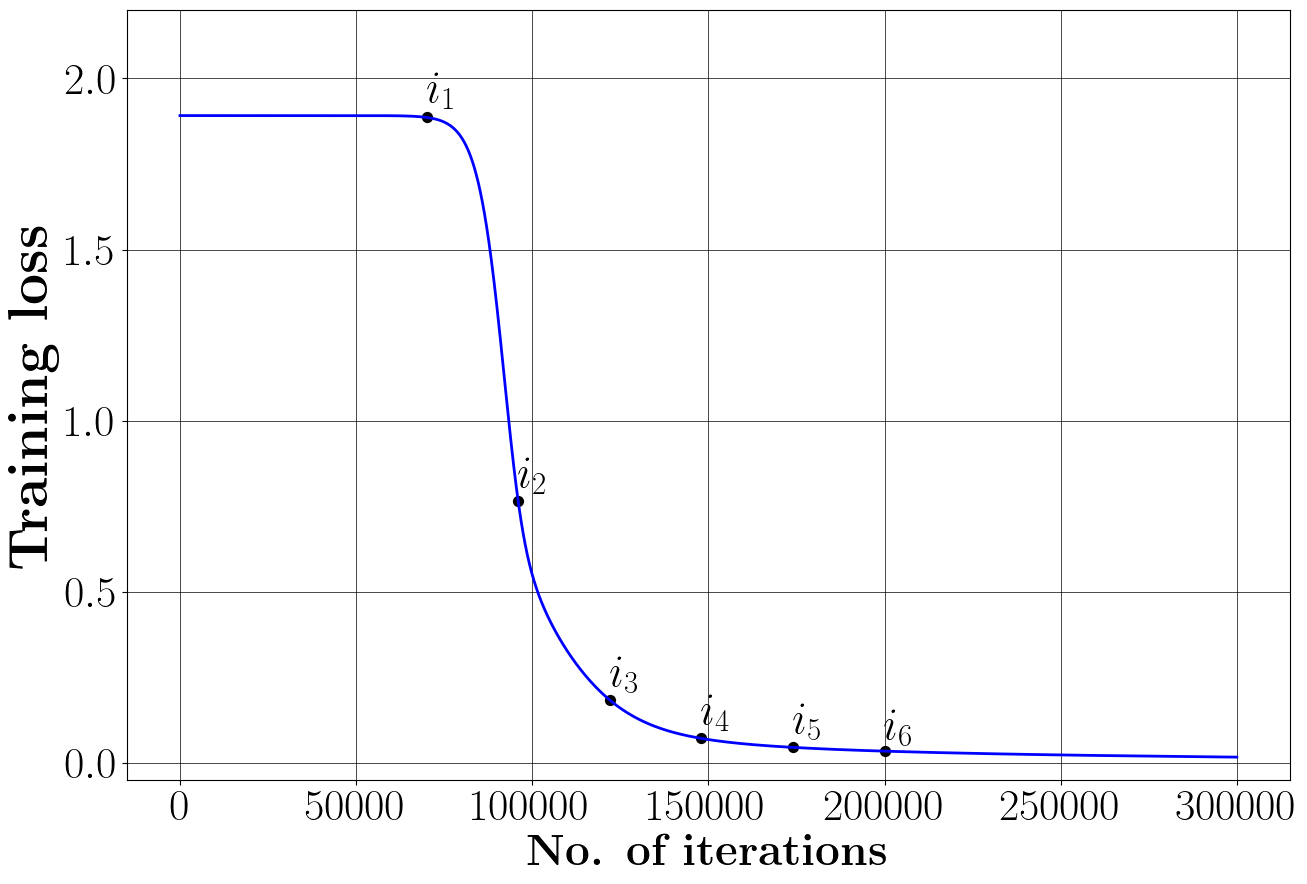}
		\caption{Evolution of training loss}
	\end{subfigure}

	\begin{subfigure}{0.35\textwidth}
		\centering
		\includegraphics[width=\linewidth]{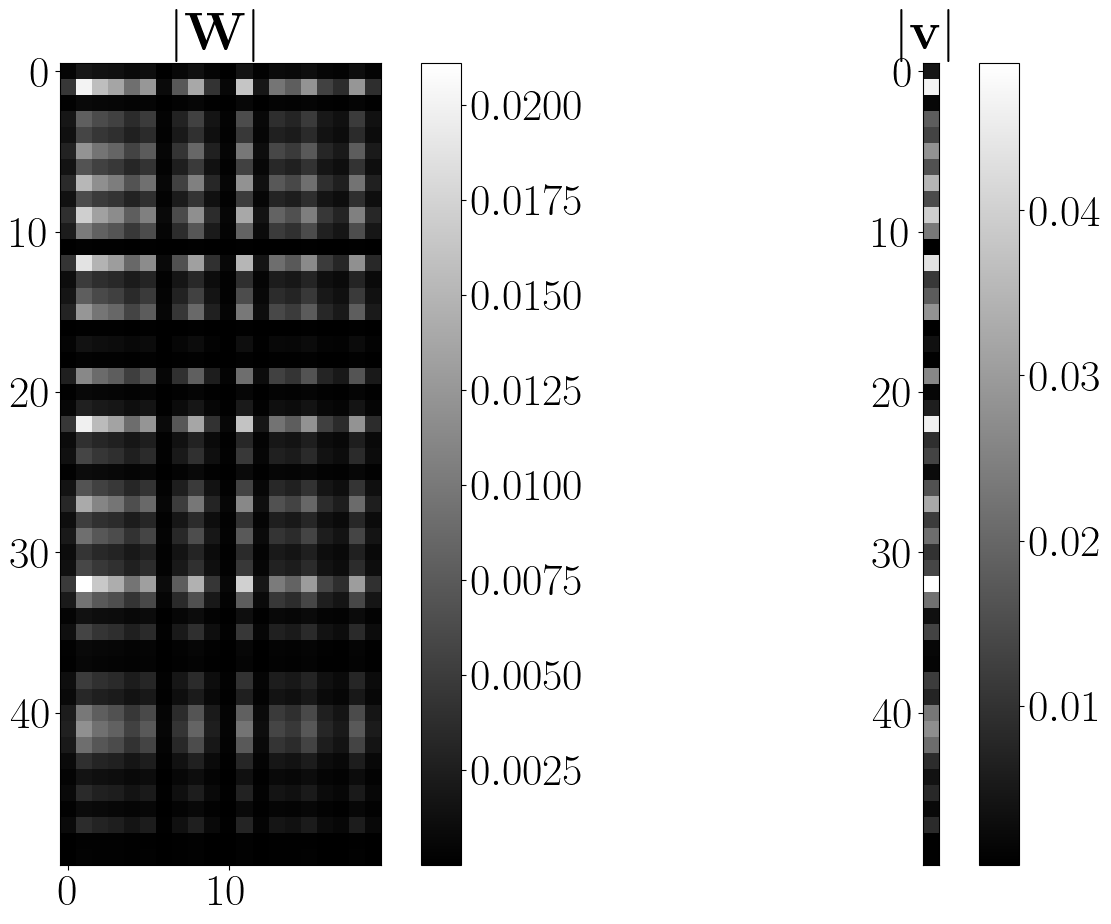}
		\caption{Weights at iteration $i_1$}
	\end{subfigure}\hfill
	\begin{subfigure}{0.35\textwidth}
		\centering
		\includegraphics[width=\linewidth]{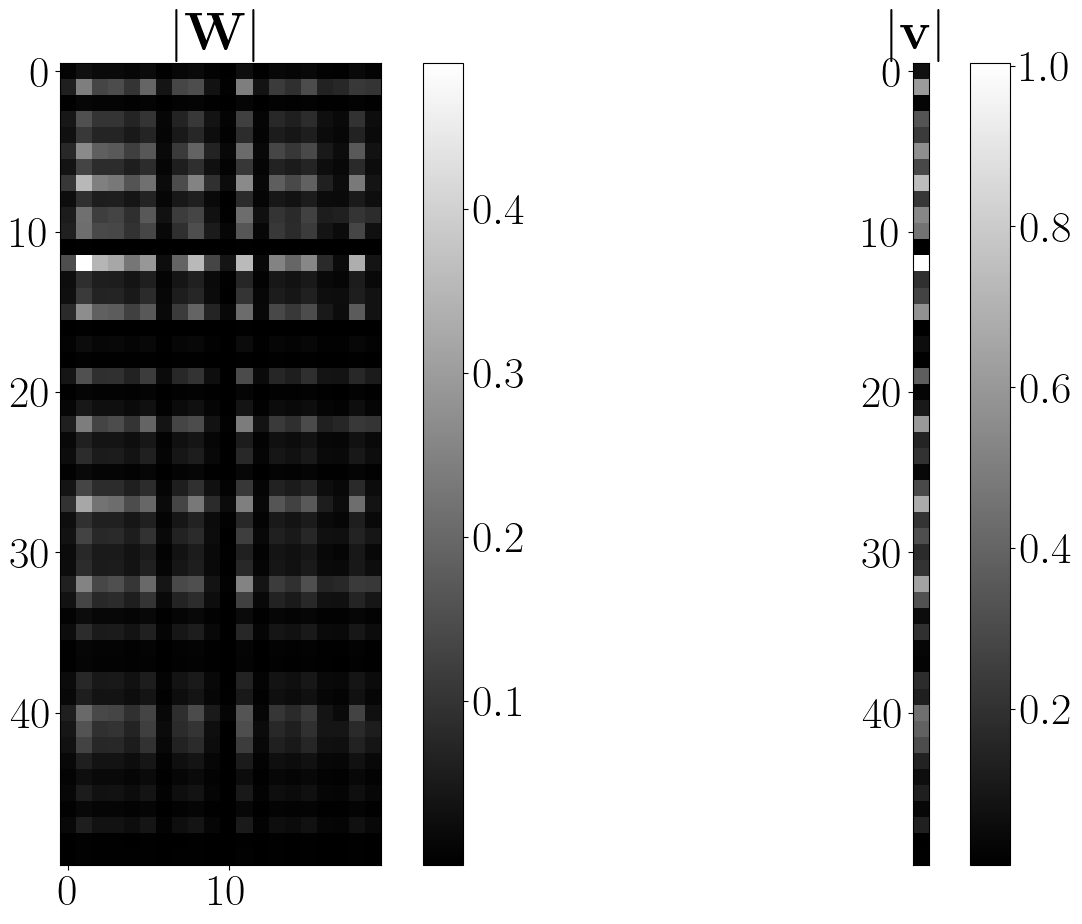}
		\caption{Weights at iteration $i_2$}
	\end{subfigure}

	\begin{subfigure}{0.35\textwidth}
		\centering
		\includegraphics[width=\linewidth]{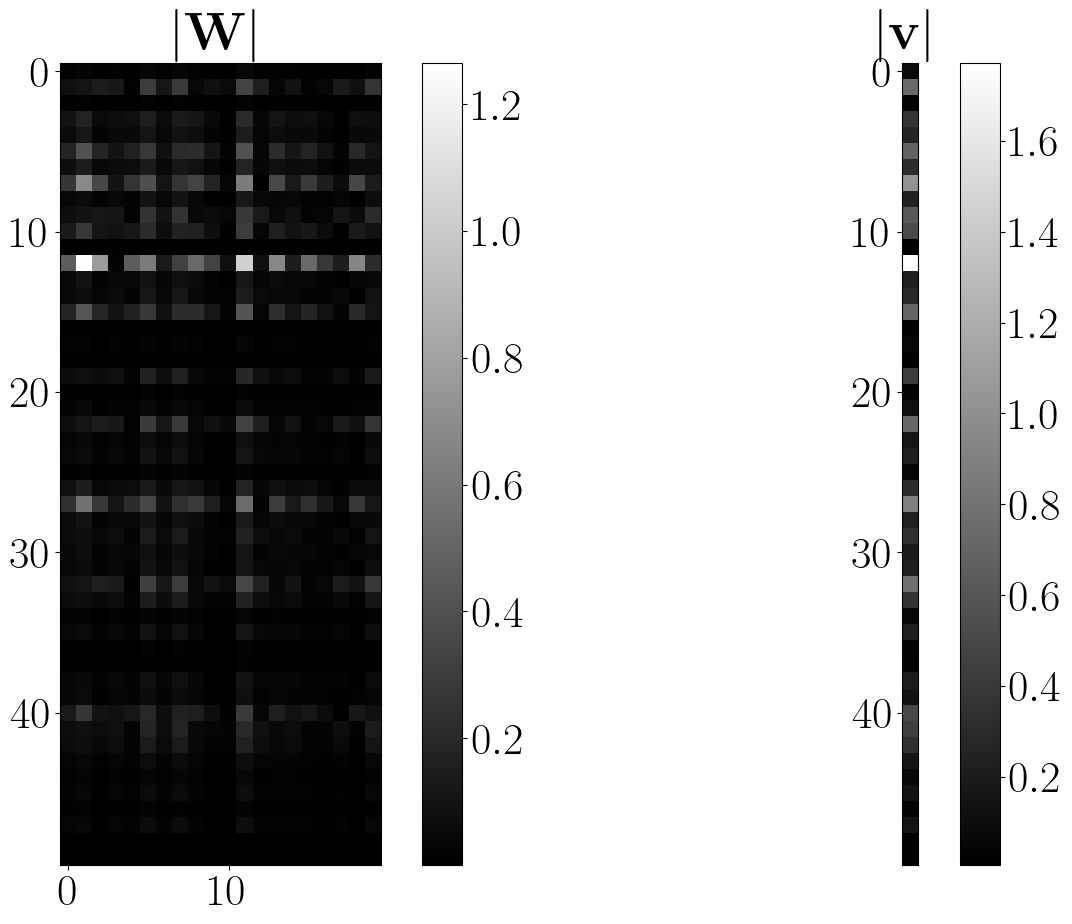}
		\caption{Weights at iteration $i_3$}
	\end{subfigure}\hfill
	\begin{subfigure}{0.35\textwidth}
		\centering
		\includegraphics[width=\linewidth]{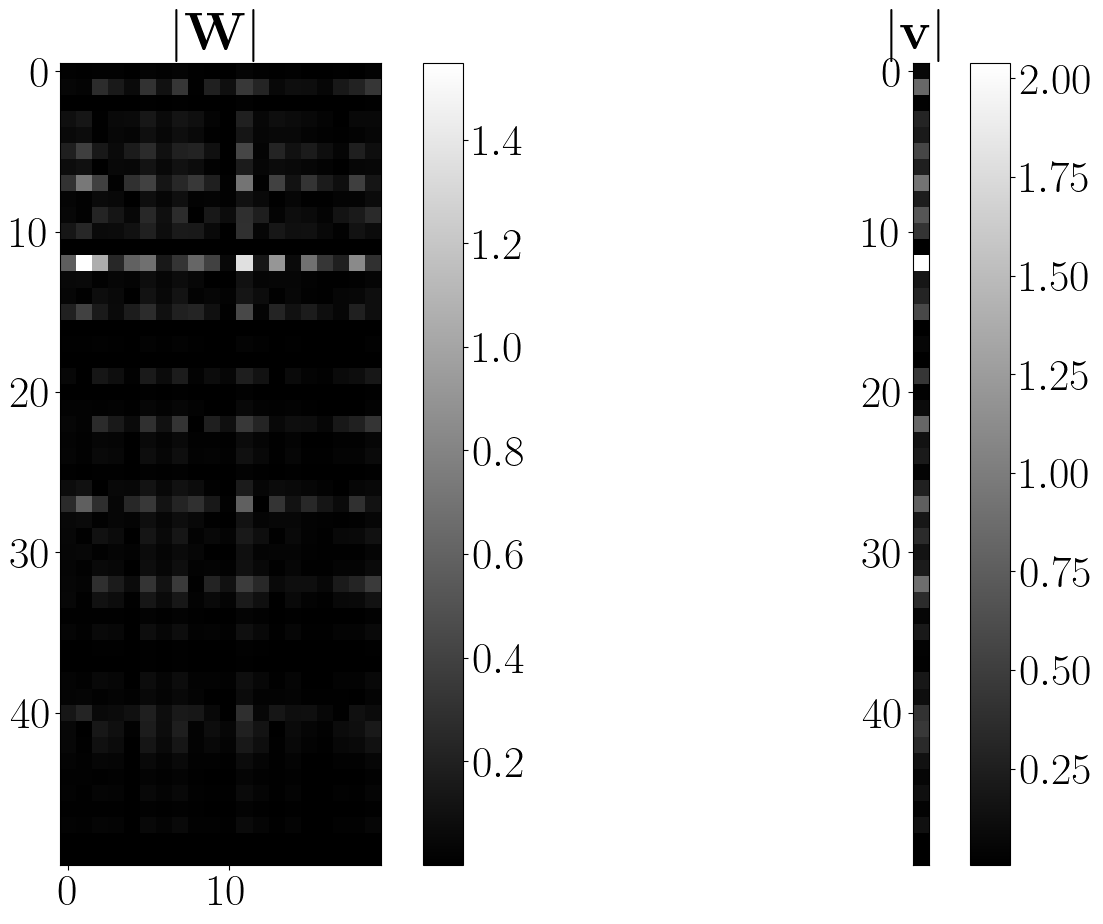}
		\caption{Weights at iteration $i_4$}
	\end{subfigure}

	\begin{subfigure}{0.35\textwidth}
		\centering
		\includegraphics[width=\linewidth]{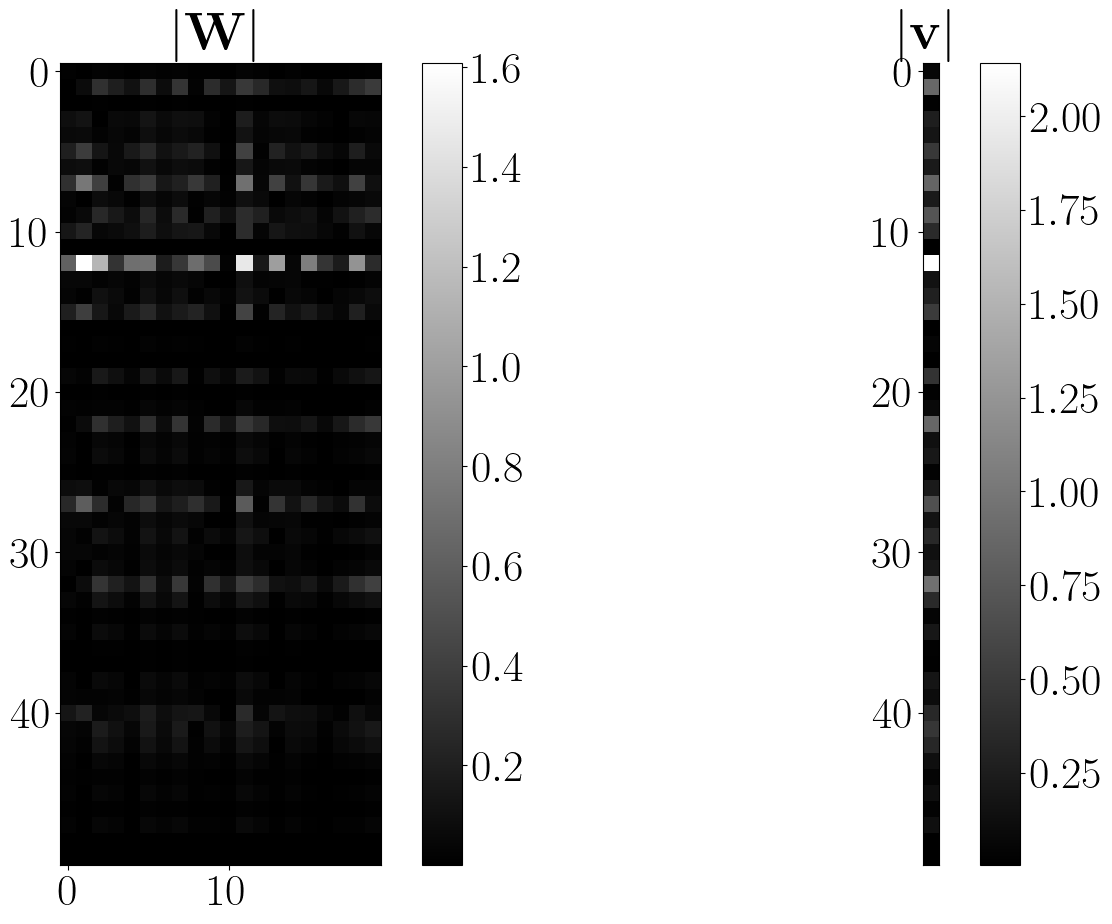}
		\caption{Weights at iteration $i_5$}
	\end{subfigure}\hfill
	\begin{subfigure}{0.35\textwidth}
		\centering
		\includegraphics[width=\linewidth]{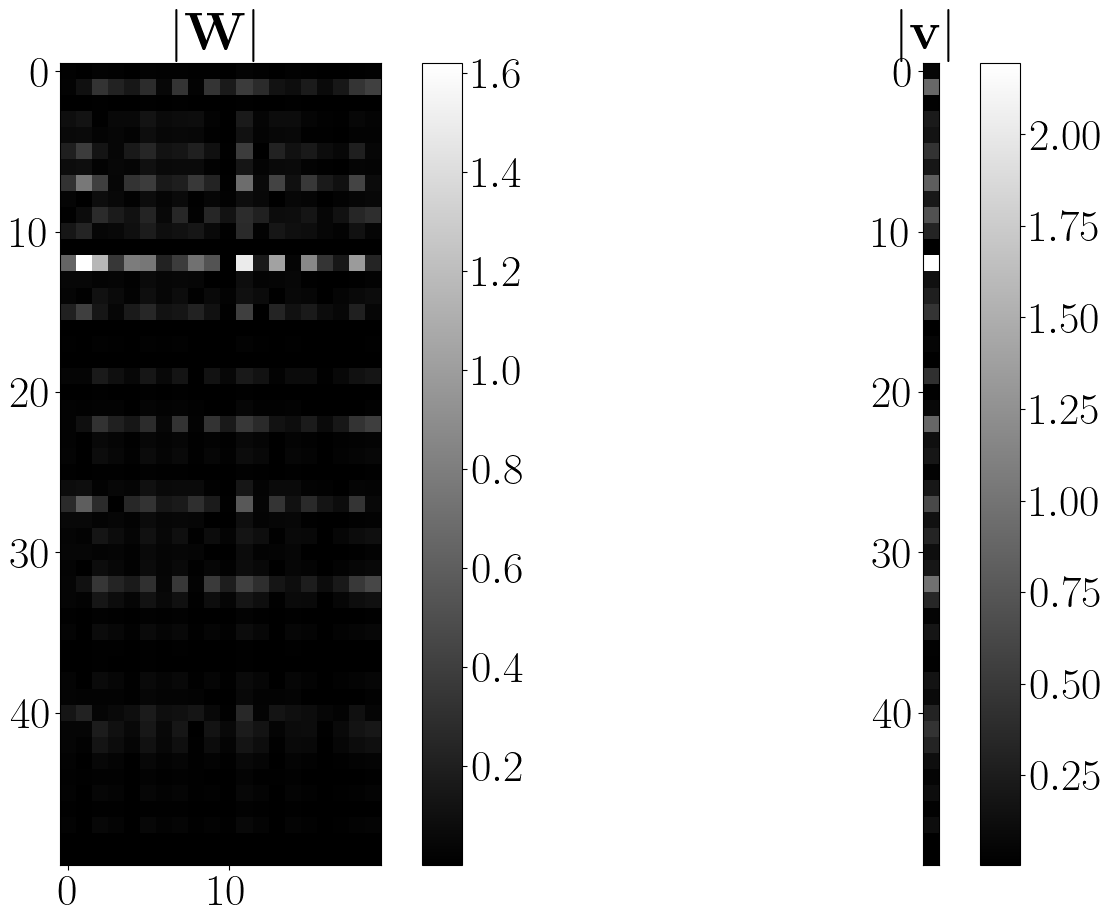}
		\caption{Weights at iteration $i_6$}
	\end{subfigure}
	
	\caption{We train a two-layer neural network whose output is $\rvv^\top\sigma(\rmW_1\rvx) ,$ where $\sigma(x) ={ \rm GELU}(x)$, and $\rvv\in \mathbb{R}^{50},\rmW_1  \in \mathbb{R}^{50 \times 20}$ are the trainable weights.}
	\label{fig:2l_gelu}
\end{figure}

\begin{figure}[htbp]
	\centering
	
	\begin{subfigure}{0.4\textwidth}
		\centering
		\includegraphics[width=\linewidth]{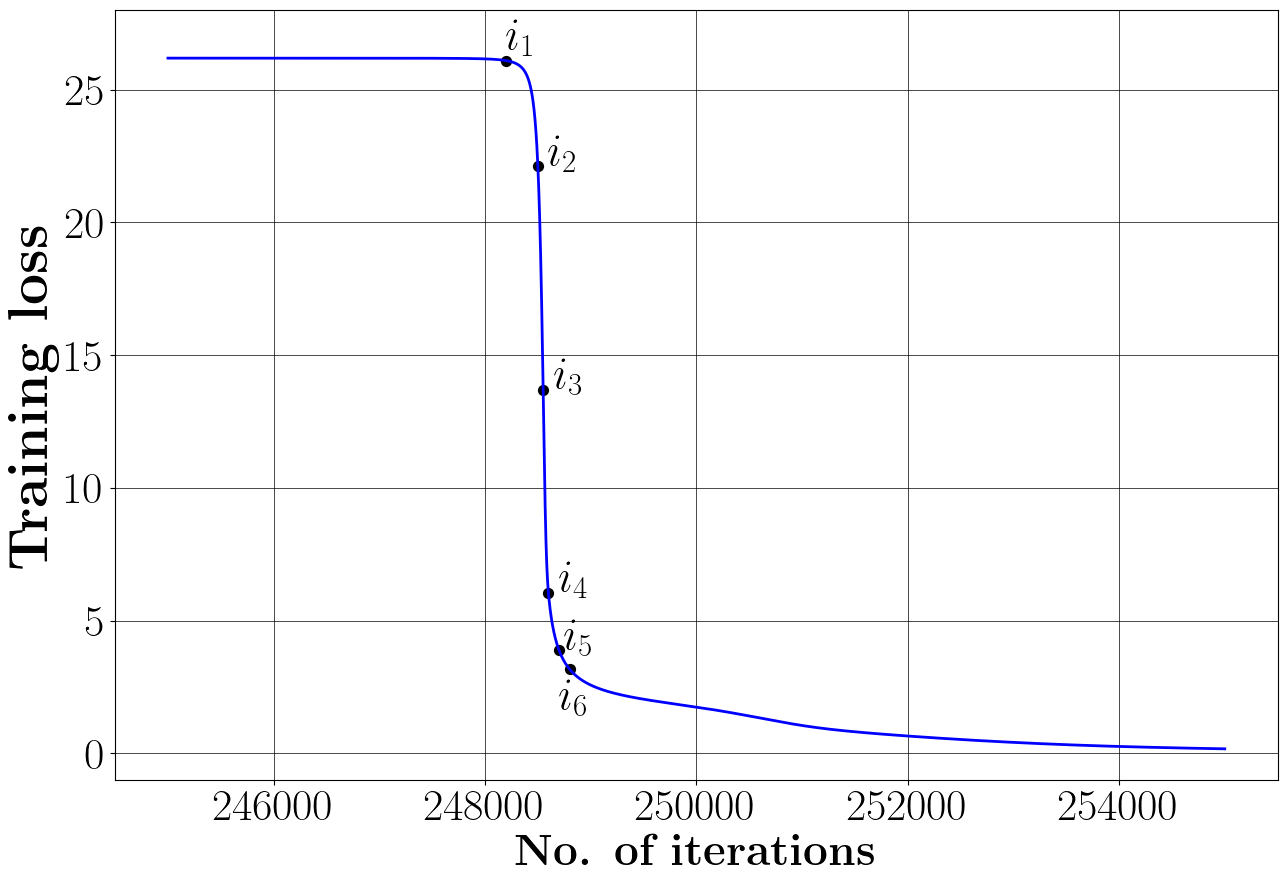}
		\caption{Evolution of training loss}
	\end{subfigure}

	\begin{subfigure}{0.45\textwidth}
		\centering
		\includegraphics[width=\linewidth]{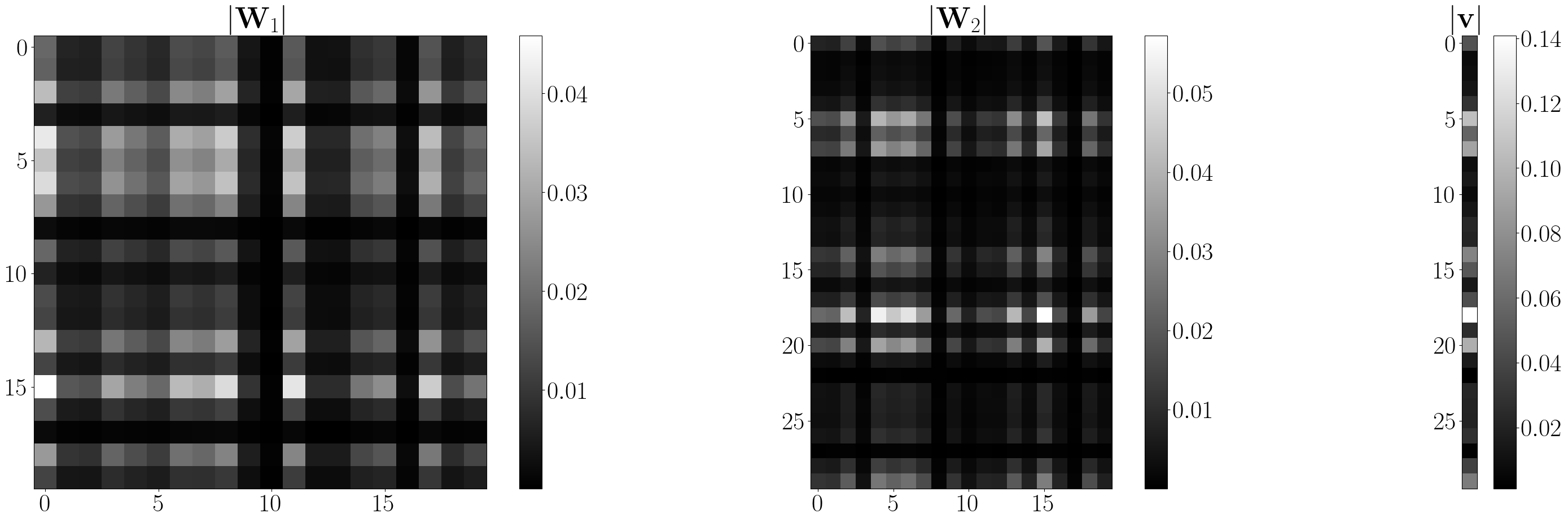}
		\caption{Weights at iteration $i_1$}
	\end{subfigure}\hfill
	\begin{subfigure}{0.45\textwidth}
		\centering
		\includegraphics[width=\linewidth]{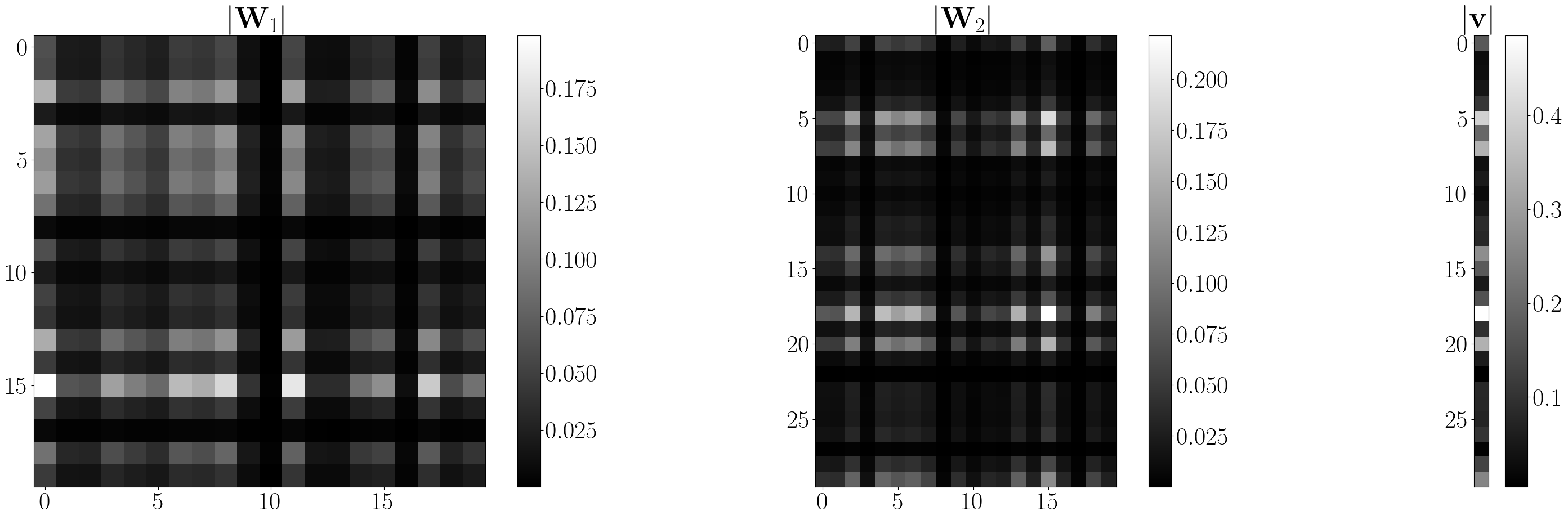}
		\caption{Weights at iteration $i_2$}
	\end{subfigure}

	\begin{subfigure}{0.45\textwidth}
		\centering
		\includegraphics[width=\linewidth]{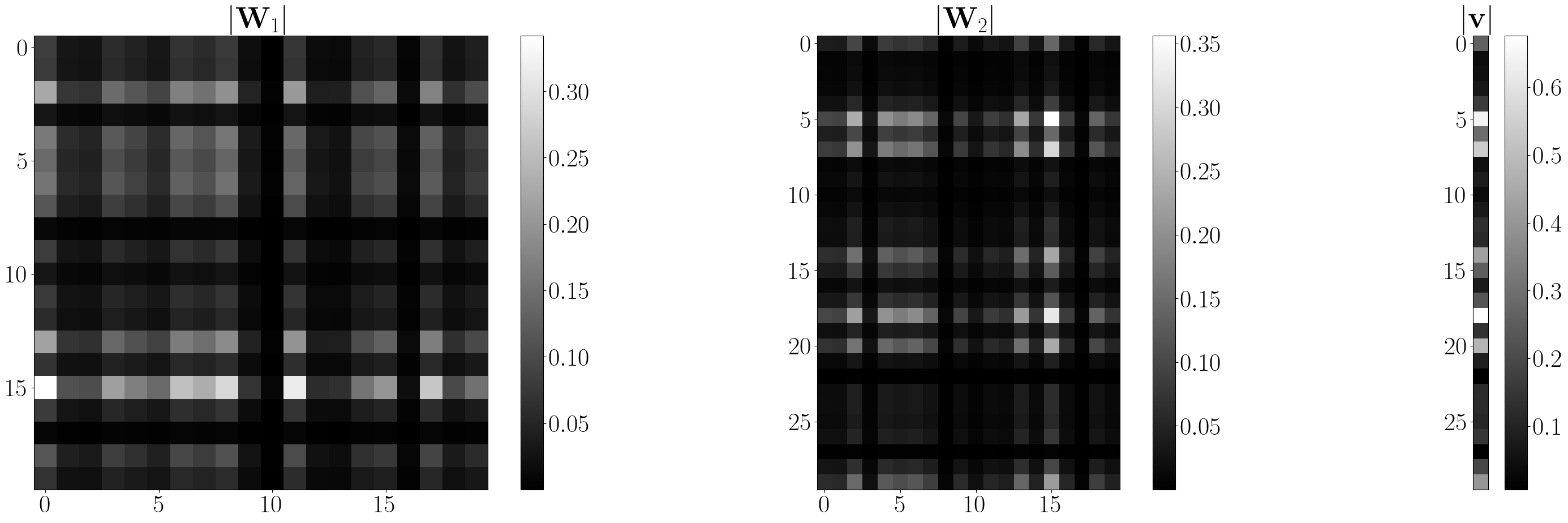}
		\caption{Weights at iteration $i_3$}
	\end{subfigure}\hfill
	\begin{subfigure}{0.45\textwidth}
		\centering
		\includegraphics[width=\linewidth]{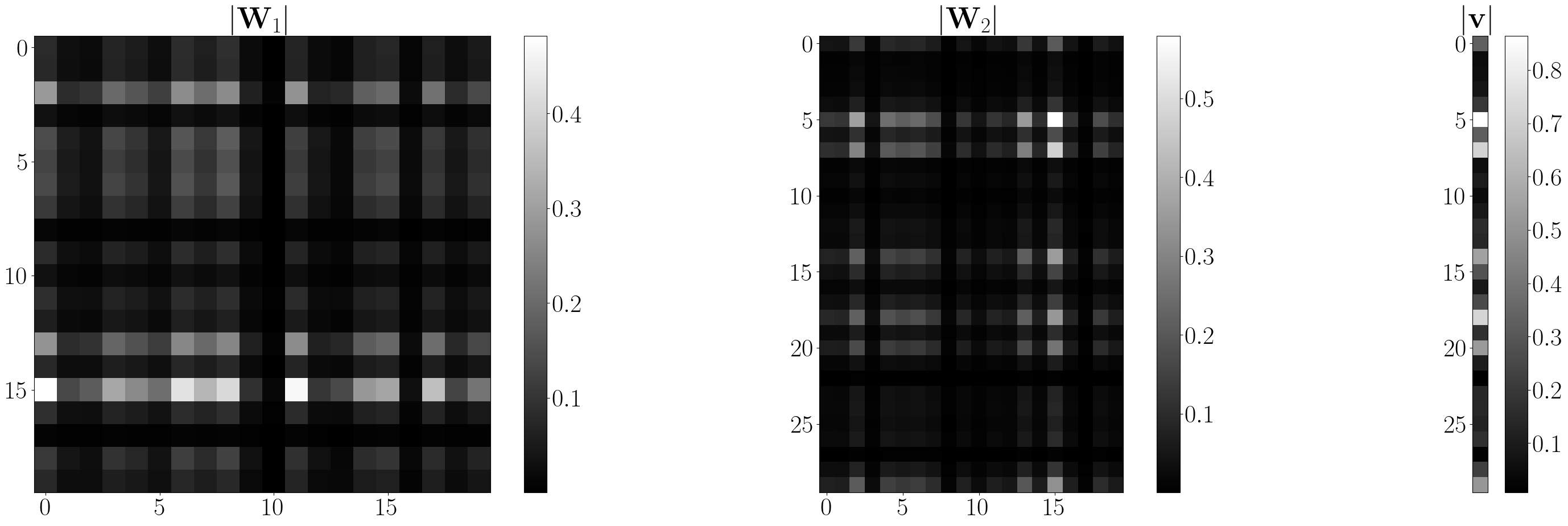}
		\caption{Weights at iteration $i_4$}
	\end{subfigure}

	\begin{subfigure}{0.45\textwidth}
		\centering
		\includegraphics[width=\linewidth]{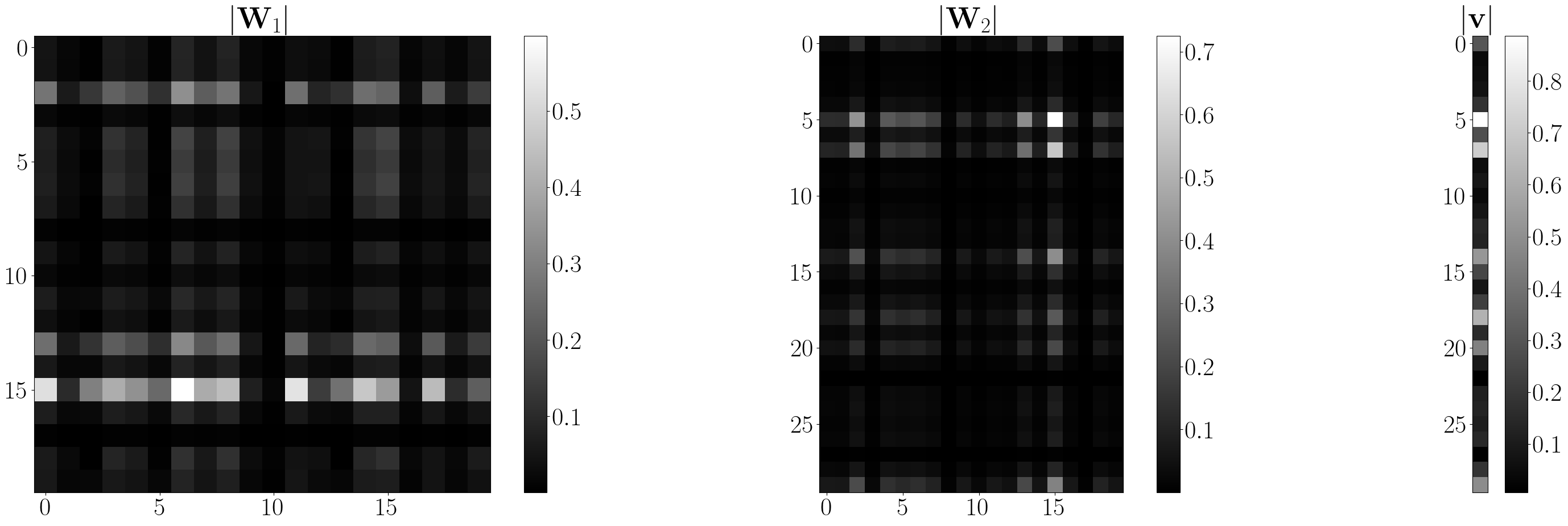}
		\caption{Weights at iteration $i_5$}
	\end{subfigure}\hfill
	\begin{subfigure}{0.45\textwidth}
		\centering
		\includegraphics[width=\linewidth]{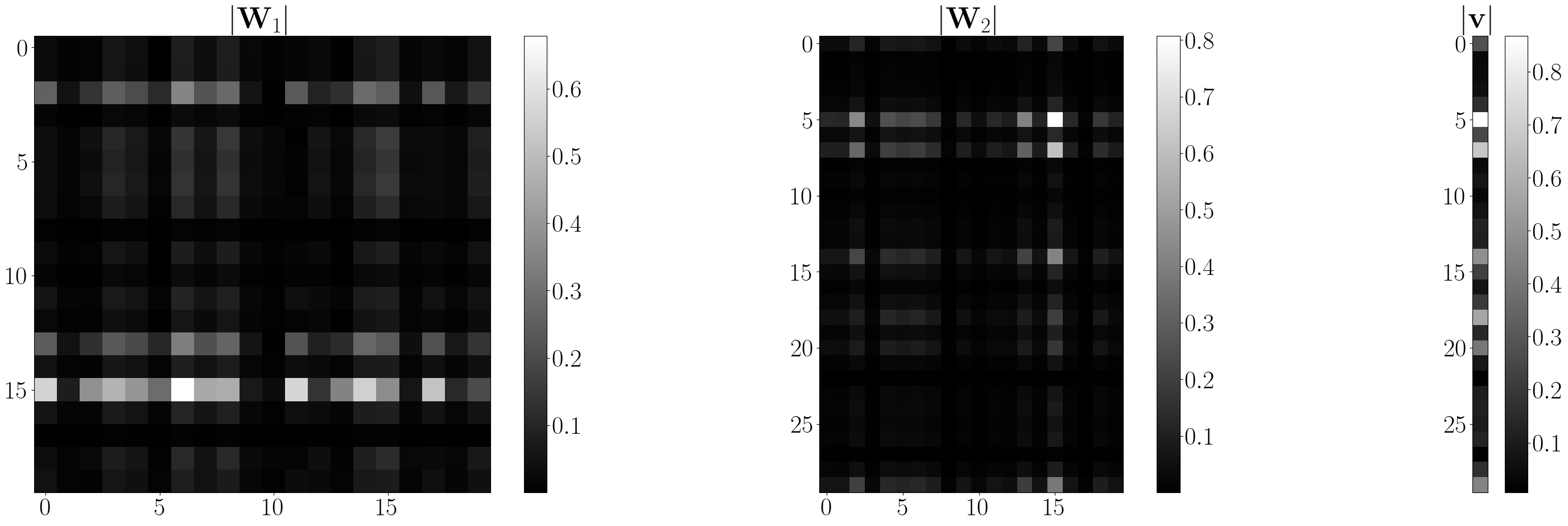}
		\caption{Weights at iteration $i_6$}
	\end{subfigure}
	
	\caption{We train a three-layer neural network whose output is $\rvv^\top\sigma(\rmW_2\sigma(\rmW_1\rvx)) ,$ where $\sigma(x) = { \rm GELU}(x)$, and $\rvv\in \mathbb{R}^{30},\rmW_2\in \mathbb{R}^{30 \times 20},\rmW_1  \in \mathbb{R}^{20 \times 20}$ are the trainable weights.}
	\label{fig:3l_gelu}
\end{figure}
\section{Training Dynamics Beyond the First Saddle Point}
\label{sec:bey_saddle}
Our theoretical results and experiments so far have focused on the regime after escaping the origin and until reaching the first saddle point. As discussed in \Cref{sec_related_works}, it is widely believed that the training dynamics follows a saddle-to-saddle trajectory: after escaping from one saddle point, the weights reach another, and this process continues until convergence. This naturally raises the question of whether the phenomenon of sparsity structure preservation also holds after escaping the saddle points.

To investigate this, we re-consider some of our earlier experiments, running them for additional iterations when needed to ensure that the weights have escaped from the first saddle point. In \Cref{fig:bey_2l_sq} and \Cref{fig:bey_3l_sq}, we revisit the experiment with two- and three-layer networks from \Cref{fig:two_layer_nn} and \Cref{fig:3_layer_nn}, respectively, while \Cref{fig:bey_4l_mnist} and \Cref{fig:bey_4l_mnist_wt1} considers the experiment in \Cref{fig:4_layer_rl_mnist} with four-layer network and MNIST data. In all cases, we plot the weights after they escape from the first saddle point. We observe that the sparsity structure is preserved among the weights after it escapes from the first saddle point and until reaching the next saddle point.

Overall, these experiments suggest that the sparsity structure is perhaps preserved even after escaping the saddle point and until reaching the next saddle point. Establishing a theoretical explanation for this phenomenon is an important direction for future research.

\begin{figure}[htbp]
	\centering
	
	\begin{subfigure}{0.3\textwidth}
		\centering
		\includegraphics[width=\linewidth]{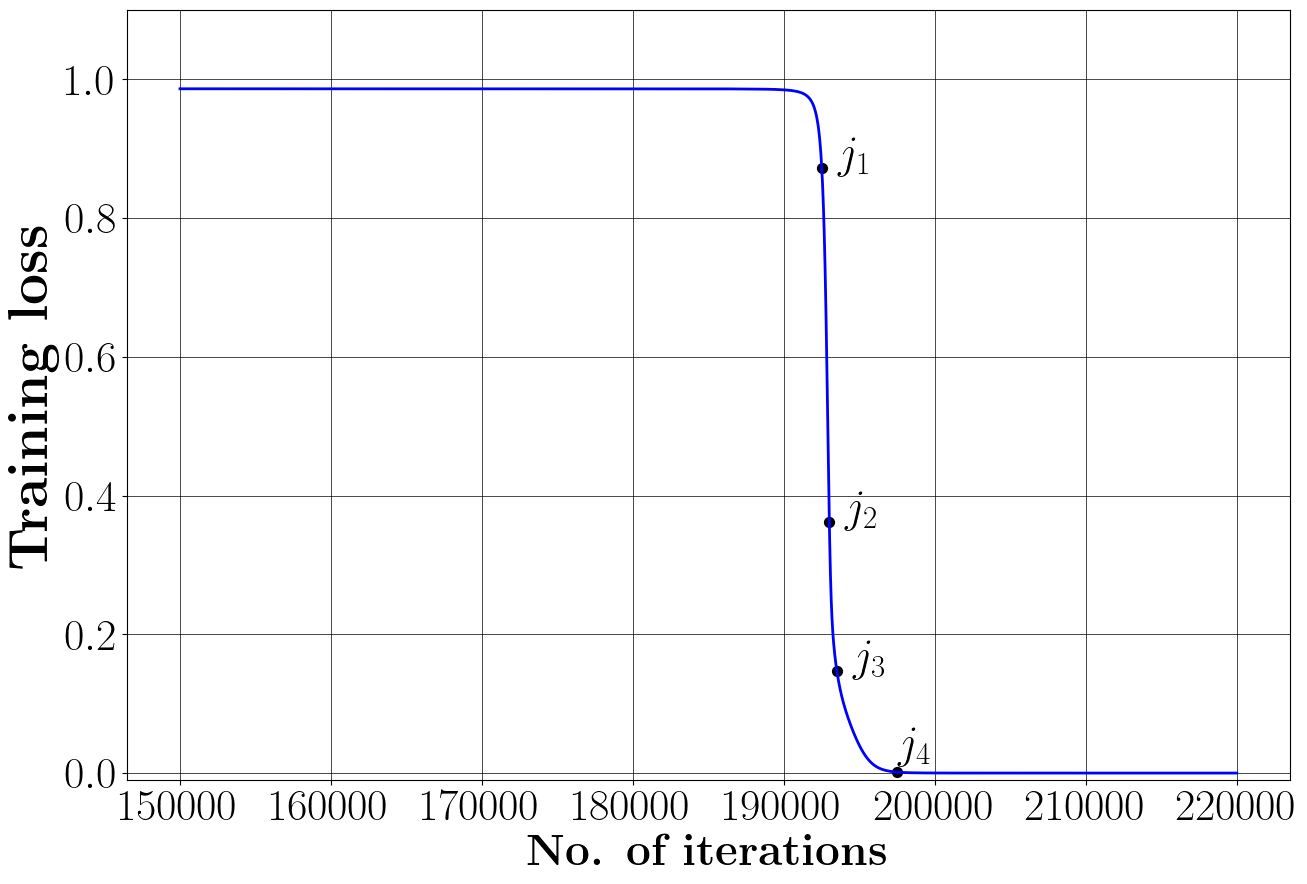}
		\caption{Evolution of training loss}
	\end{subfigure}

	\begin{subfigure}{0.4\textwidth}
		\centering
		\includegraphics[width=\linewidth]{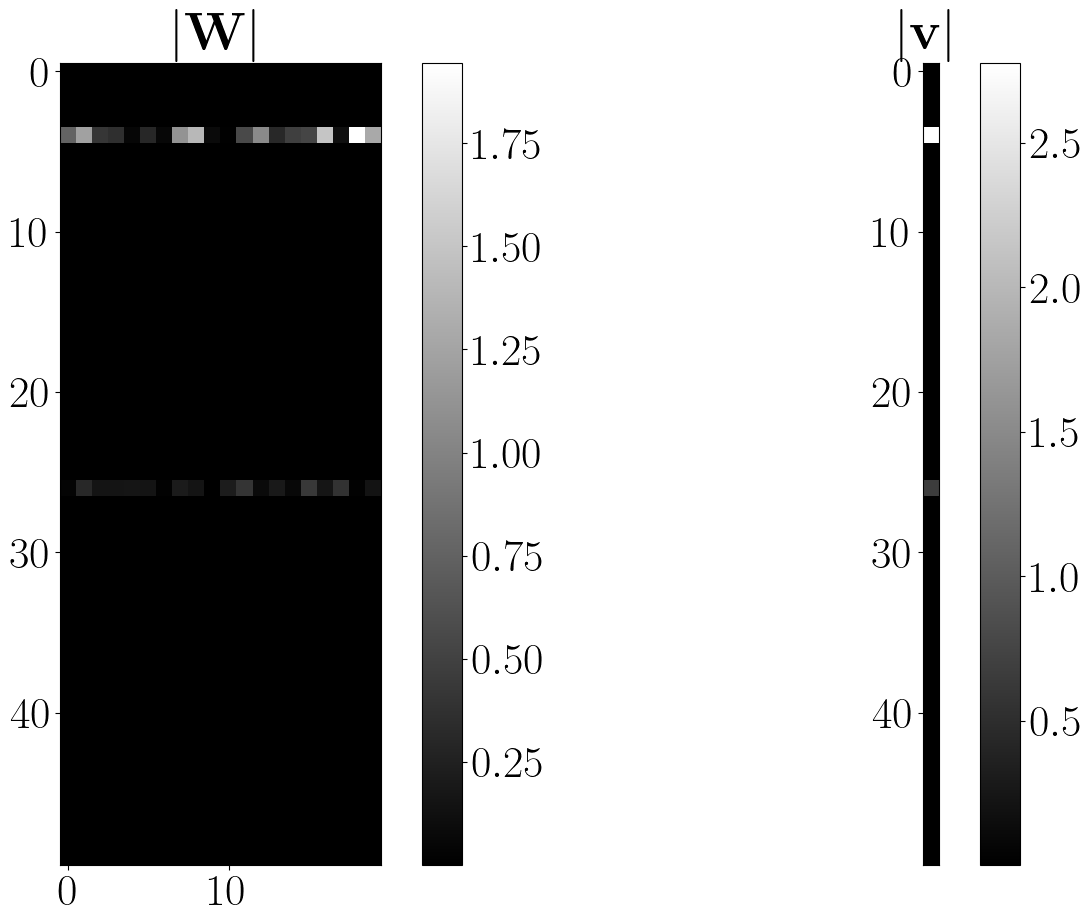}
		\caption{Weights at iteration $j_1$}
	\end{subfigure}\hfill
	\begin{subfigure}{0.4\textwidth}
		\centering
		\includegraphics[width=\linewidth]{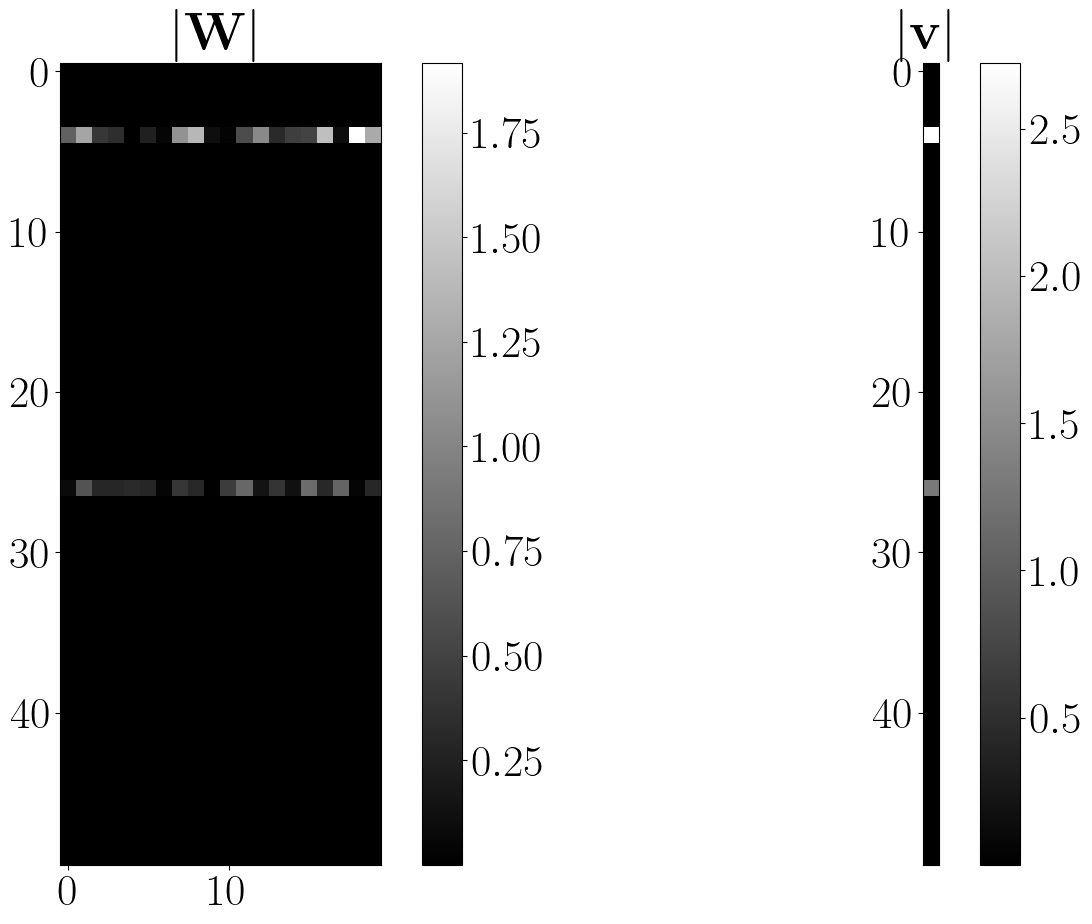}
		\caption{Weights at iteration $j_2$}
	\end{subfigure}

	\begin{subfigure}{0.4\textwidth}
		\centering
		\includegraphics[width=\linewidth]{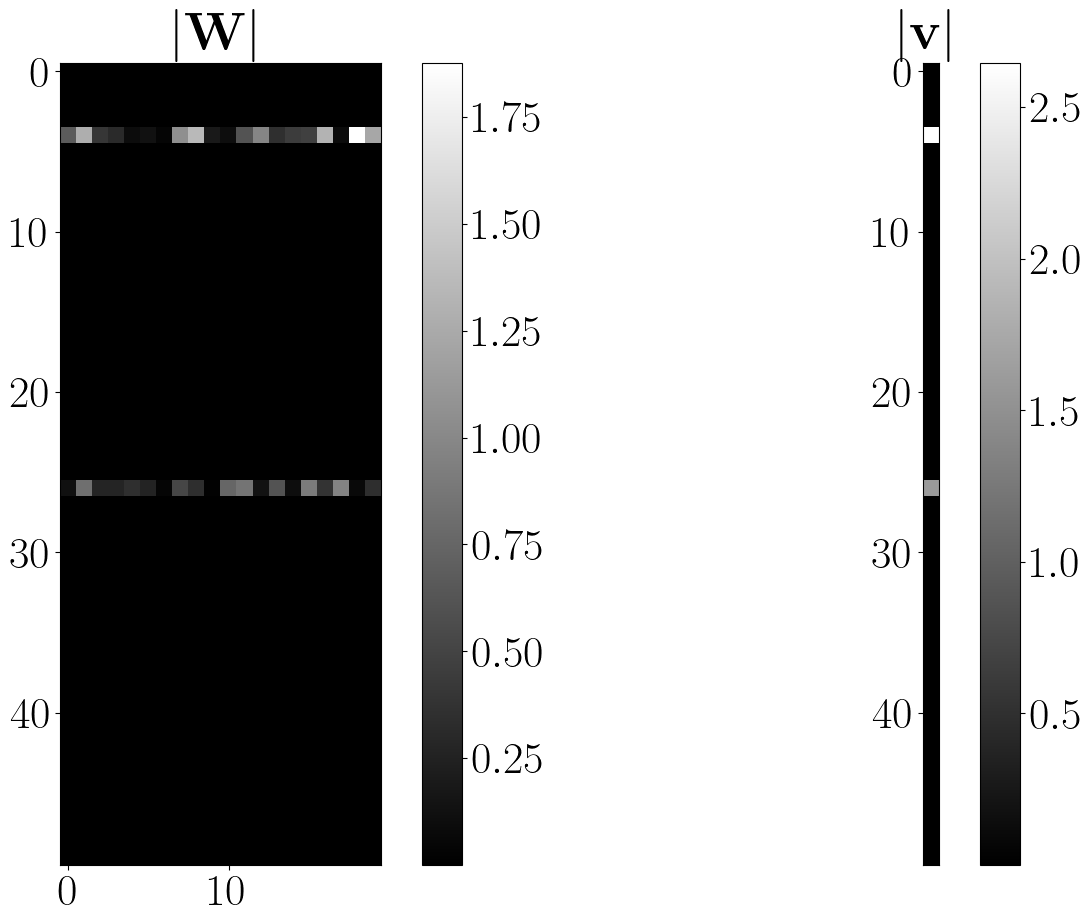}
		\caption{Weights at iteration $j_3$}
	\end{subfigure}\hfill
	\begin{subfigure}{0.4\textwidth}
		\centering
		\includegraphics[width=\linewidth]{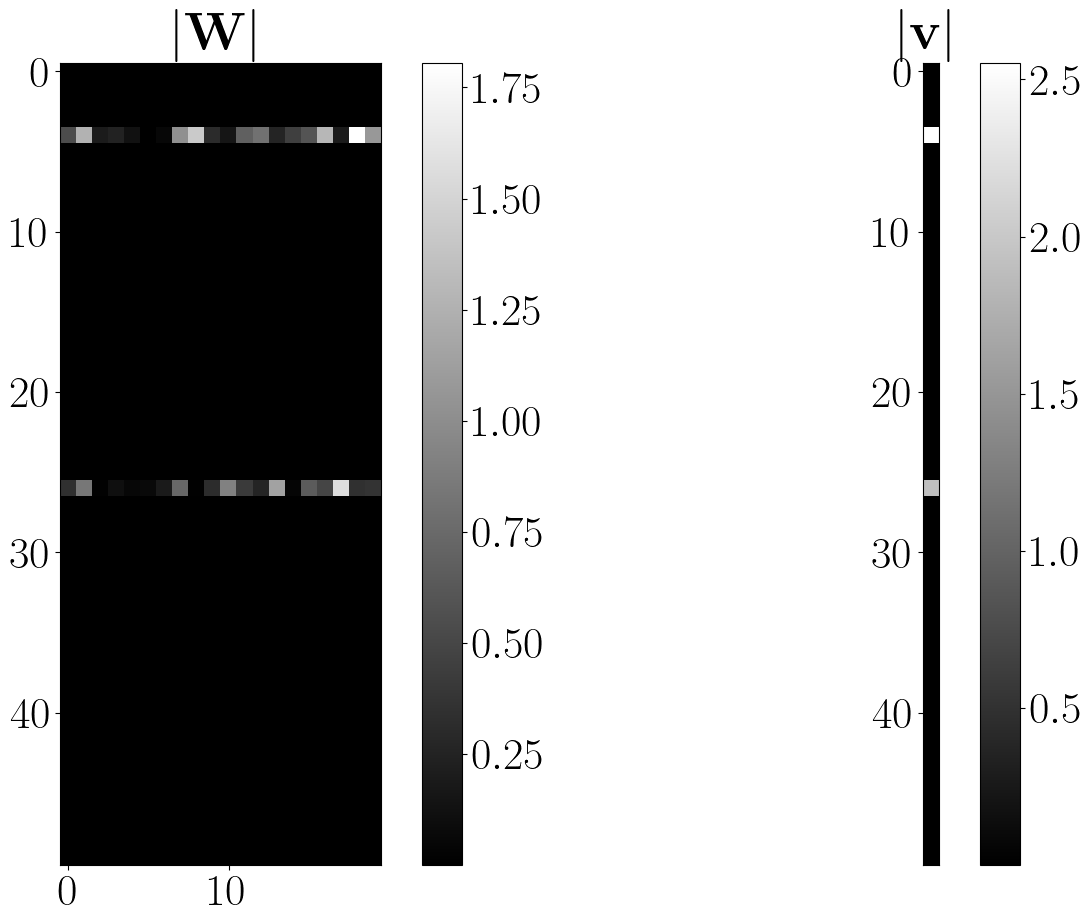}
		\caption{Weights at iteration $j_4$}
	\end{subfigure}
	
	\caption{The experiment in \Cref{fig:two_layer_nn} is considered again. Panel (a) depicts the evolution of training loss near and after escaping from the first saddle point. Panels (b)-(e) depict the weights at different stages of training (marked in panel (a)).}
	\label{fig:bey_2l_sq}
\end{figure}
\begin{figure}[htbp]
	\centering
	\begin{subfigure}{0.3\textwidth}
		\centering
		\includegraphics[width=\linewidth]{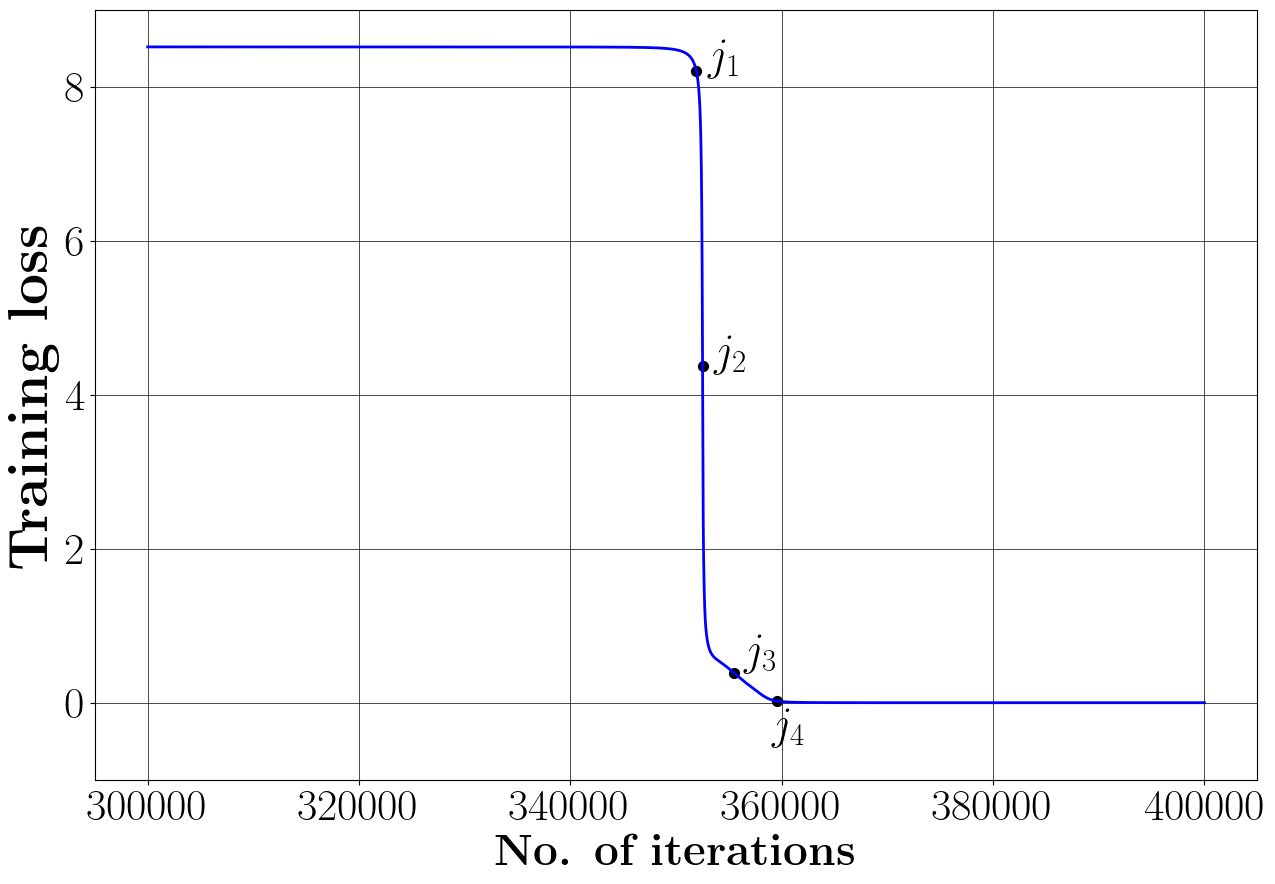}
		\caption{Evolution of training loss}
	\end{subfigure}

	\begin{subfigure}{0.45\textwidth}
		\centering
		\includegraphics[width=\linewidth]{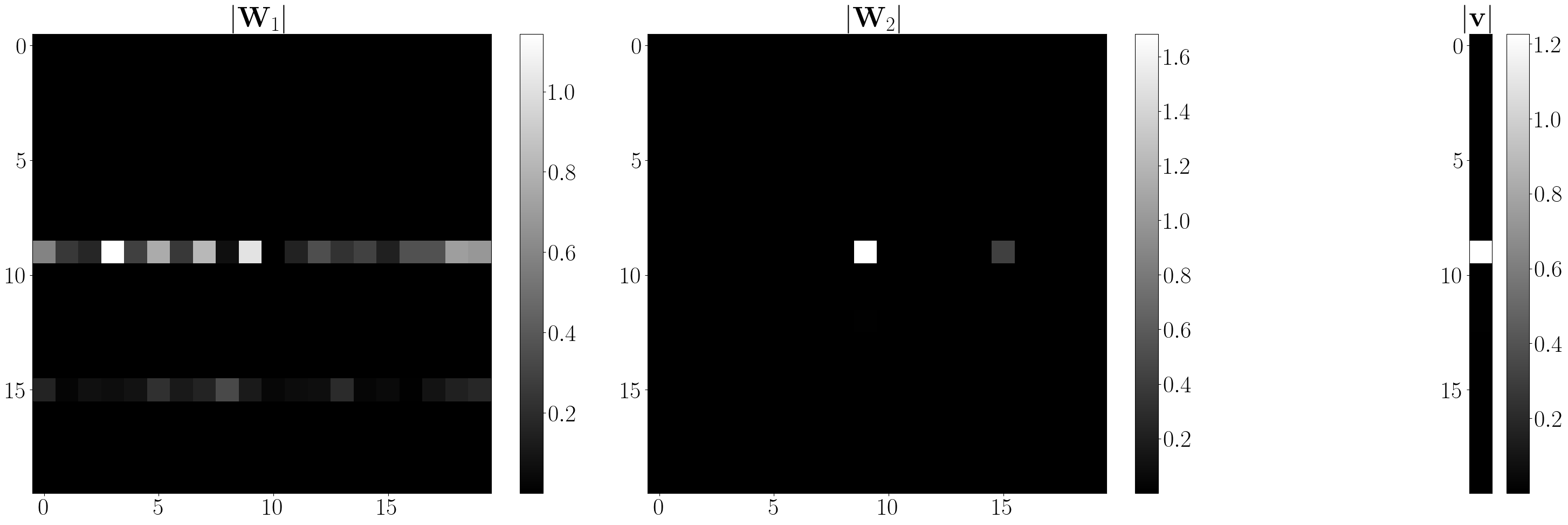}
		\caption{Weights at iteration $j_1$}
	\end{subfigure}\hfill
	\begin{subfigure}{0.45\textwidth}
		\centering
		\includegraphics[width=\linewidth]{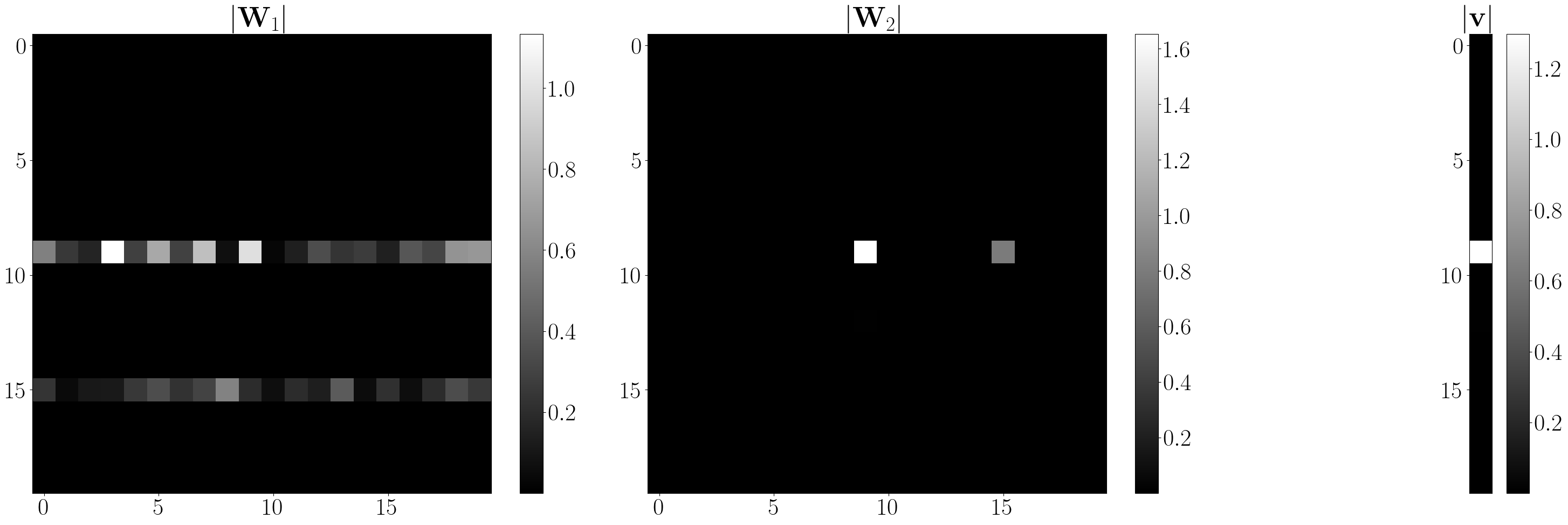}
		\caption{Weights at iteration $j_2$}
	\end{subfigure}

	\begin{subfigure}{0.45\textwidth}
		\centering
		\includegraphics[width=\linewidth]{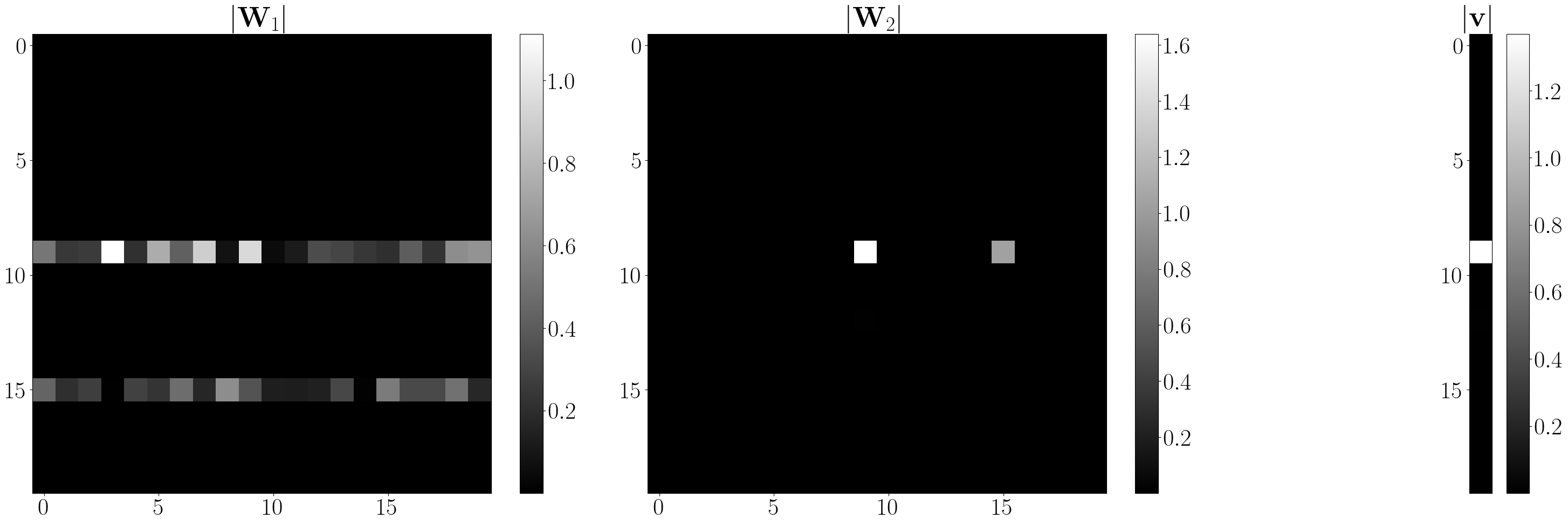}
		\caption{Weights at iteration $j_3$}
	\end{subfigure}\hfill
	\begin{subfigure}{0.45\textwidth}
		\centering
		\includegraphics[width=\linewidth]{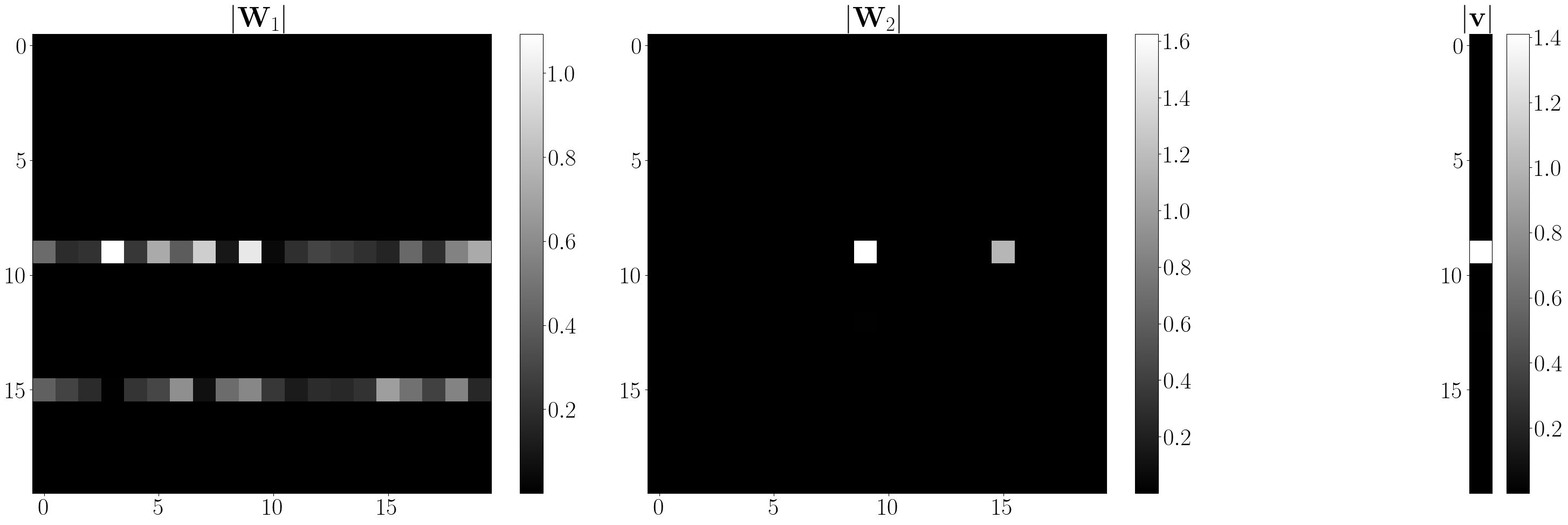}
		\caption{Weights at iteration $j_4$}
	\end{subfigure}
\caption{The experiment in \Cref{fig:3_layer_nn} is run for more iterations, specifically, until the weights escapes from the first saddle point. Panel (a) depicts the evolution of training loss near and after escaping from the first saddle point. Panels (b)-(e) depict the weights at different stages of training (marked in panel (a)).}
	\label{fig:bey_3l_sq}
\end{figure}
\begin{figure}[htbp]
	\centering
	\begin{subfigure}{0.3\textwidth}
		\centering
		\includegraphics[width=\linewidth]{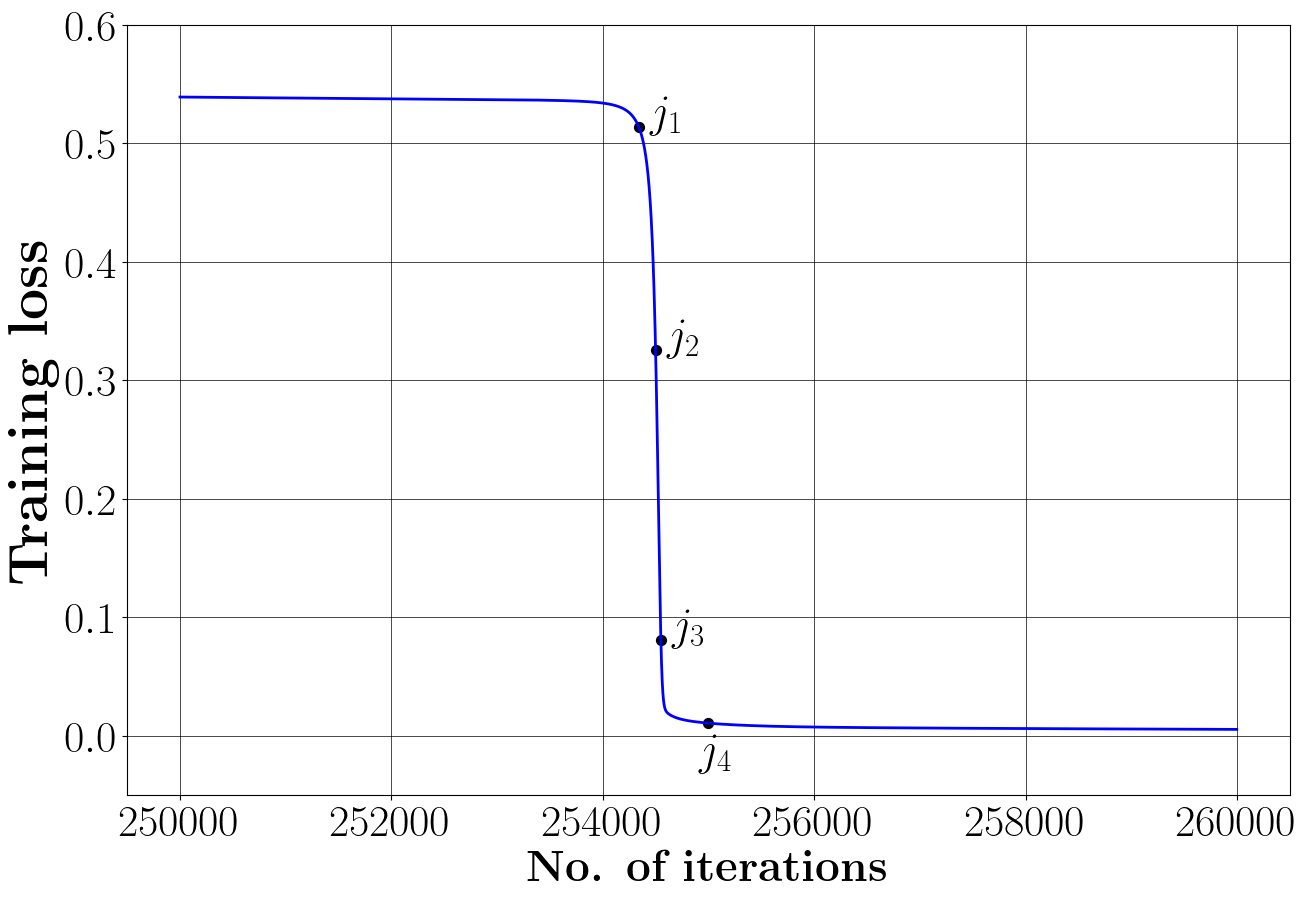}
		\caption{Evolution of training loss}
	\end{subfigure}

	\begin{subfigure}{0.45\textwidth}
		\centering
		\includegraphics[width=\linewidth]{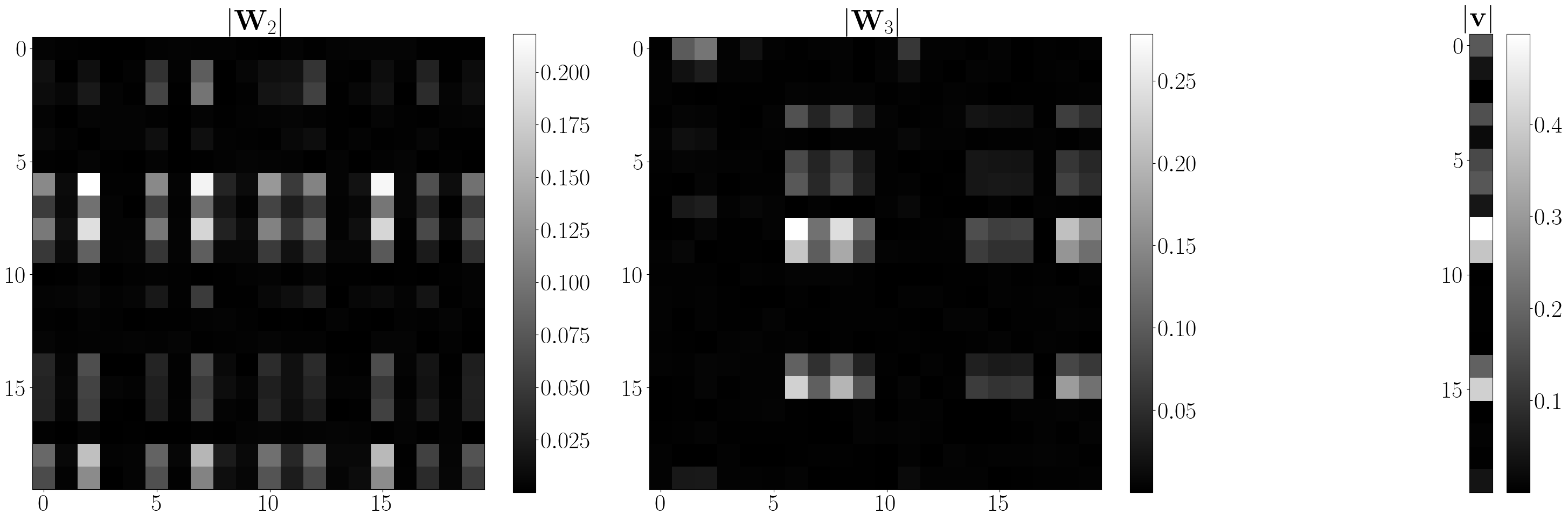}
		\caption{Weights of last three layers at iteration $j_1$}
	\end{subfigure}\hfill
	\begin{subfigure}{0.45\textwidth}
		\centering
		\includegraphics[width=\linewidth]{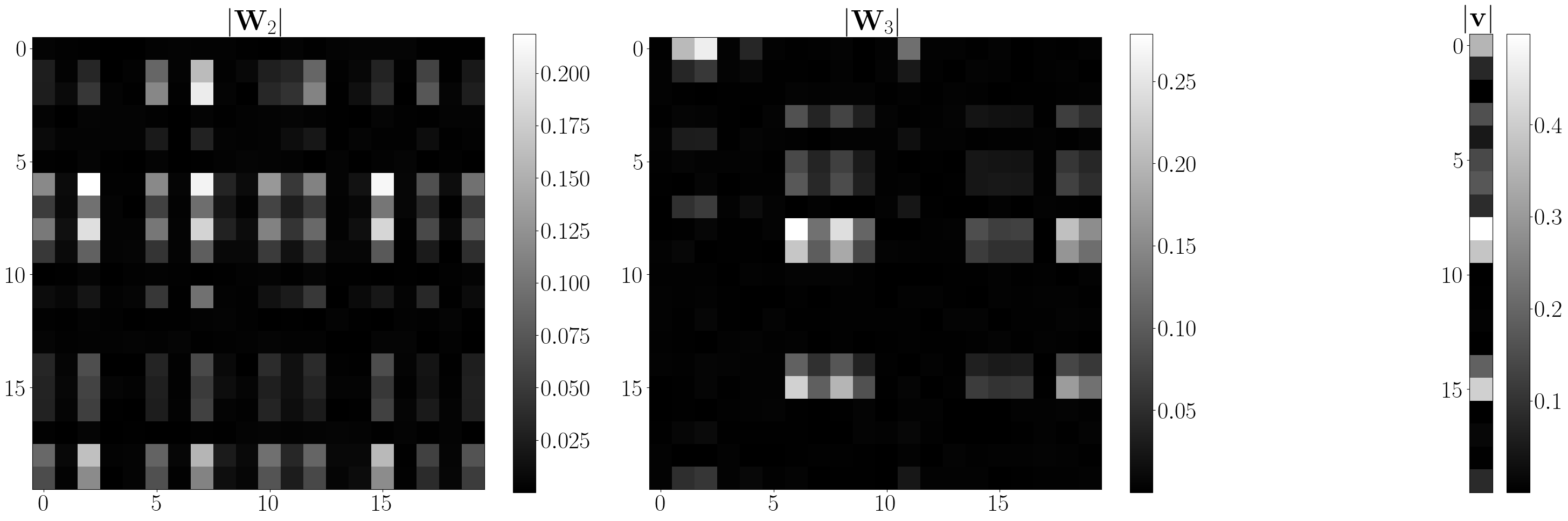}
		\caption{Weights of last three layers at iteration $j_2$}
	\end{subfigure}

	\begin{subfigure}{0.45\textwidth}
		\centering
		\includegraphics[width=\linewidth]{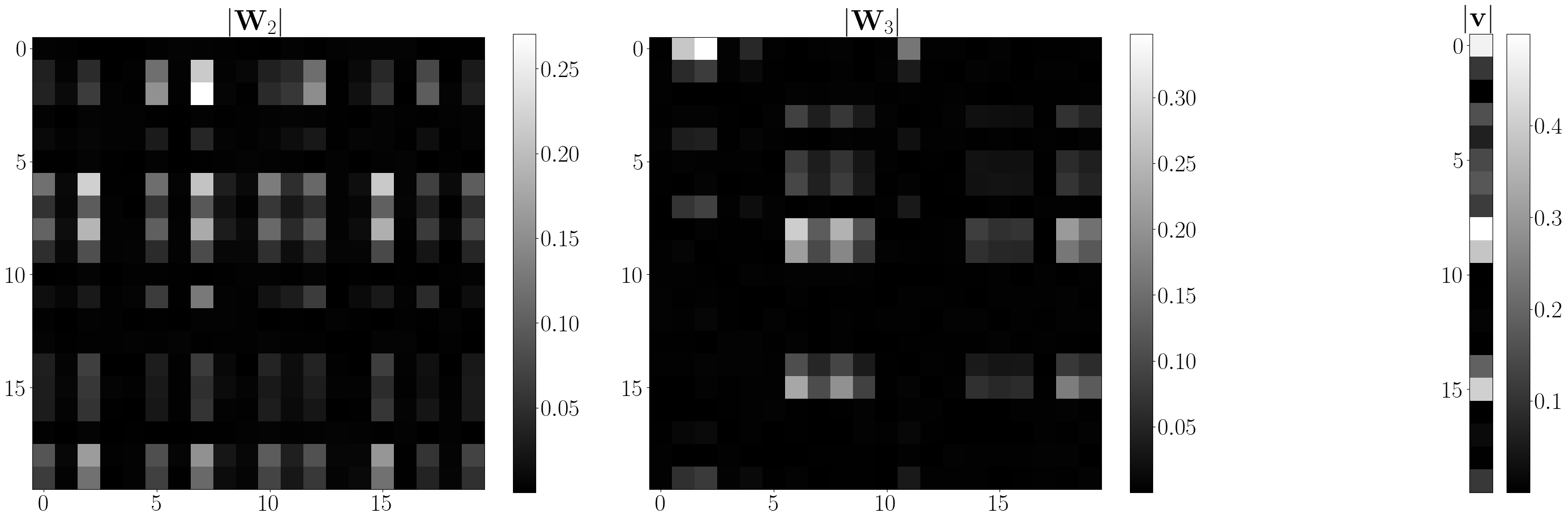}
		\caption{Weights of last three layers at iteration $j_3$}
	\end{subfigure}\hfill
	\begin{subfigure}{0.45\textwidth}
		\centering
		\includegraphics[width=\linewidth]{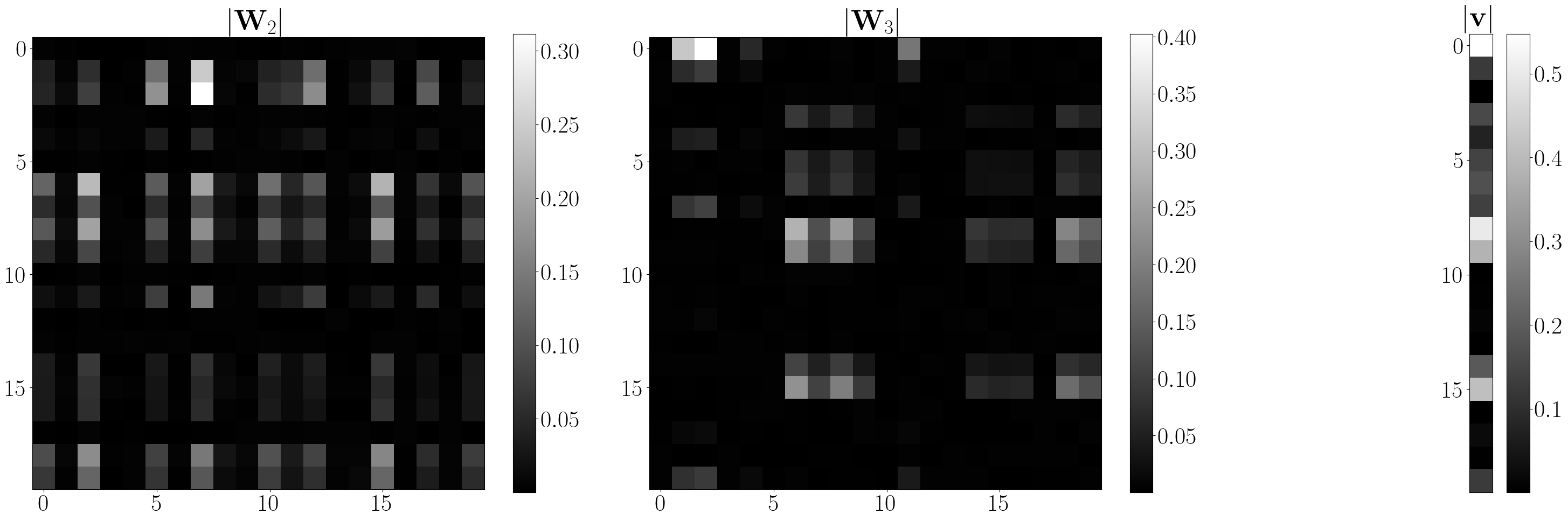}
		\caption{Weights of last three layers at iteration $j_4$}
	\end{subfigure}
	
\caption{The experiment in \Cref{fig:4_layer_rl_mnist} is considered again. Panel (a) depicts the evolution of training loss near and after escaping from the first saddle point. Panels (b)-(e) depict the weights of last three layers at different stages of training (marked in panel (a)). The weights of first layer are shown in \Cref{fig:bey_4l_mnist_wt1} due to limited space.}
	\label{fig:bey_4l_mnist}
\end{figure}

\begin{figure}[htbp]
	\centering
	\vspace{-5mm}
	\begin{subfigure}{0.9\textwidth}
		\centering
		\includegraphics[width=\linewidth]{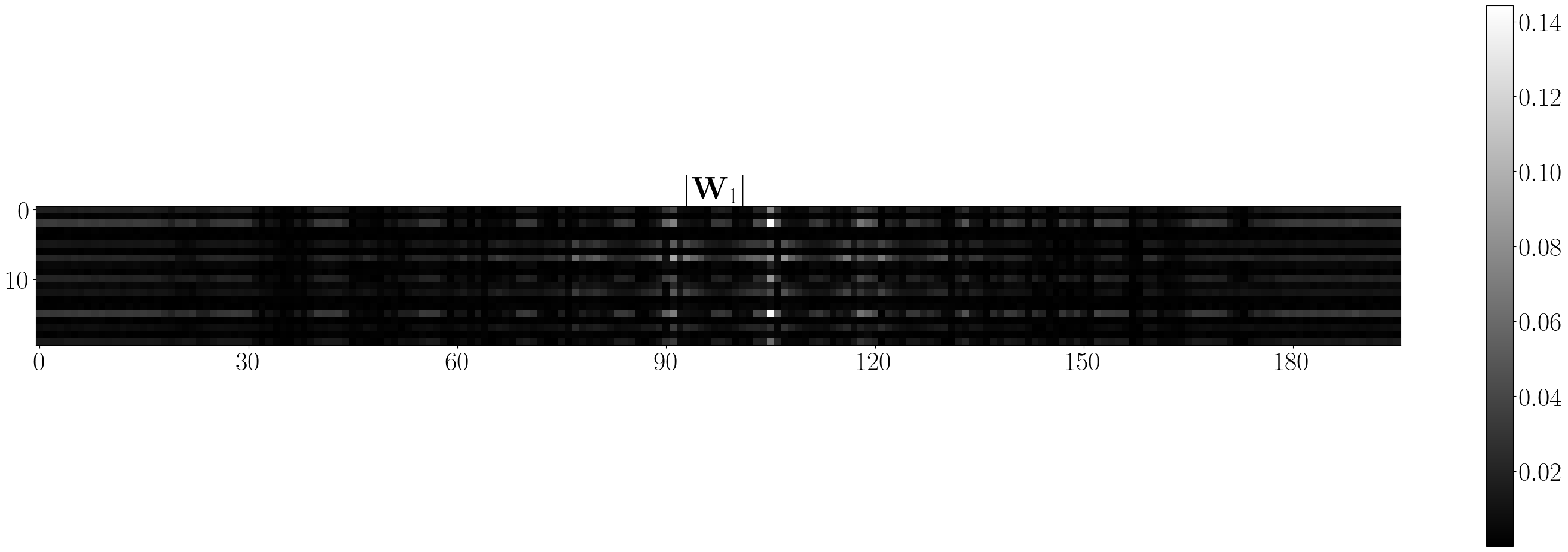}
		\caption{Weights of first layer at iteration $j_1$}
	\end{subfigure}

	\begin{subfigure}{0.9\textwidth}
		\centering
		\includegraphics[width=\linewidth]{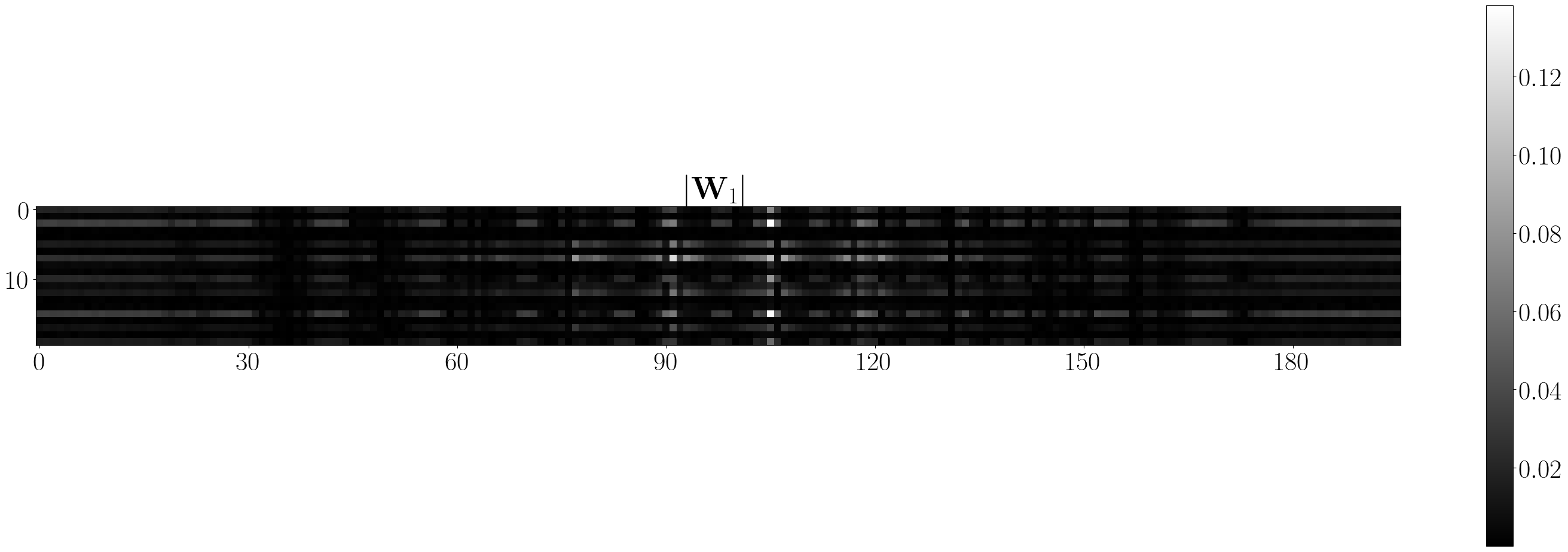}
		\caption{Weights of first layer at iteration $j_2$}
	\end{subfigure}

	\begin{subfigure}{0.9\textwidth}
		\centering
		\includegraphics[width=\linewidth]{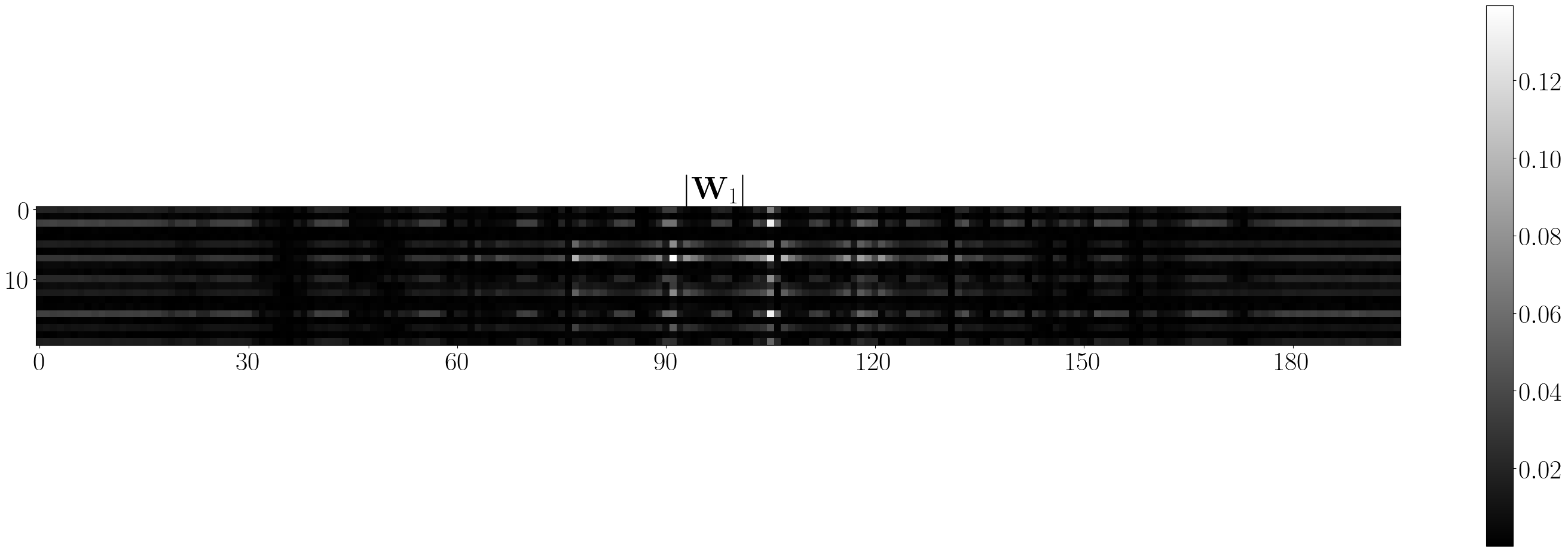}
		\caption{Weights of first layer at iteration $j_3$}
	\end{subfigure}

	\begin{subfigure}{0.9\textwidth}
		\centering
		\includegraphics[width=\linewidth]{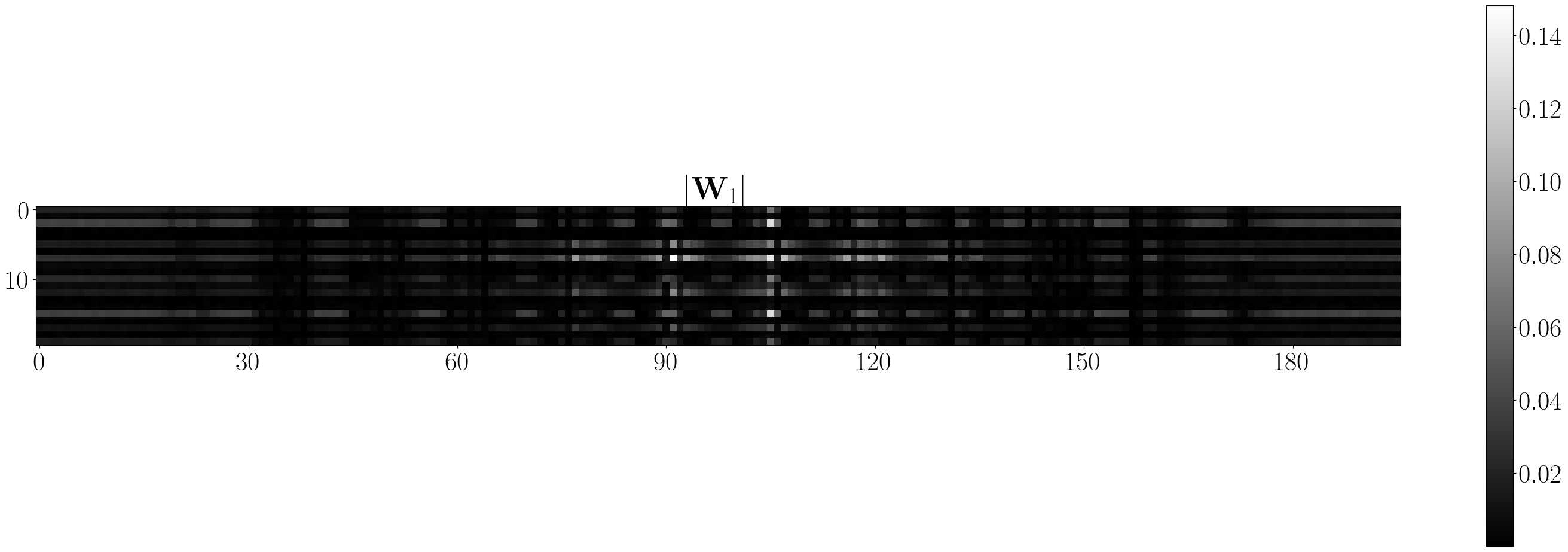}
		\caption{Weights of first layer at iteration $j_4$}
	\end{subfigure}
	
	\caption{The weights of first layer from the experiment of \Cref{fig:bey_4l_mnist}.}
	\label{fig:bey_4l_mnist_wt1}
\end{figure}

\vskip 0.2in
\bibliography{sample}

\end{document}